\documentclass[11pt]{article}
\usepackage[utf8]{inputenc}
\usepackage[round]{natbib}
\RequirePackage[OT1]{fontenc}
\RequirePackage{amsthm,amsmath}
\allowdisplaybreaks[4]
\usepackage{amssymb}
\usepackage{fullpage}	
\usepackage{amsthm}
\usepackage{diagbox}
\usepackage{enumerate}
\usepackage{bbm}
\usepackage{bm}
\usepackage{float}
\usepackage{graphicx}
\usepackage{xr-hyper}
\usepackage{hyperref}[]
\hypersetup{
	colorlinks=true,
	linkcolor=blue,
	filecolor=magenta,      
	urlcolor=cyan,
	citecolor=black,
}
\usepackage{cleveref}
\usepackage{subcaption}
\usepackage{soul}
\setstcolor{Mulberry}

\newcommand{\colorlinks}[2]{{\hypersetup{allcolors=#1}#2}}

\usepackage{algorithm,algpseudocode}
\algnewcommand{\Inputs}[1]{%
  \State \textbf{Input:}
  \Statex \hspace*{\algorithmicindent}\parbox[t]{.8\linewidth}{\raggedright #1}
}
\algnewcommand{\Initialize}[1]{%
  \State \textbf{Initialization:}
  \Statex \hspace*{\algorithmicindent}\parbox[t]{.8\linewidth}{\raggedright #1}
}
\algnewcommand{\Outputs}[1]{%
  \State \textbf{Output:} #1
}

\algdef{SE}[SUBALG]{Indent}{EndIndent}{}{\algorithmicend\ }%
\algtext*{Indent}
\algtext*{EndIndent}
\usepackage{xurl}
\usepackage[dvipsnames]{xcolor}
\usepackage{multirow}
\usepackage[hmargin={0.9in, 0.9in}, vmargin={1in, 1in}]{geometry}

\usepackage{mathtools}
\usepackage{booktabs}
\usepackage{array}
\usepackage{enumitem,epsfig,natbib}
\usepackage{multirow}
\usepackage{multicol}
\newcolumntype{Z}{>{\setbox0=\hbox\bgroup}c<{\egroup}@{\hspace*{-\tabcolsep}}}
\usepackage[title]{appendix}

\theoremstyle{definition}
\newtheorem{thm}{Theorem}[section]
\newtheorem{ass}{Assumption}[section]

\newtheorem{rem}{Remark}[section]

\newtheorem{defn}{Definition}[section]

\usepackage{setspace}
\usepackage{diagbox}
\usepackage{tabu}
\usepackage{authblk}
\DeclareMathOperator*{\argmin}{arg\,min}
\DeclareMathOperator*{\argmax}{arg\,max}
\DeclareMathOperator*{\loglog}{loglog}

\usepackage{titlesec}
\titlespacing*\section{0pt}{2pt plus 2pt minus 2pt}{0pt plus 2pt minus 2pt}
\titlespacing*\subsection{0pt}{2pt plus 2pt minus 2pt}{0pt plus 2pt minus 2pt}
\titlespacing*\subsubsection{0pt}{0pt plus 2pt minus 2pt}{0pt plus 2pt minus 2pt}

\newcommand{\bepsilon}{\boldsymbol{\varepsilon}}

\newcommand{\be}{\mathbb{E}}

\newcommand{\bp}{\mathbb{P}}

\newcommand{\tr}{\mbox{tr}}

\newcommand{\BE}{\begin{equation}}
\newcommand{\EE}{\end{equation}}
\newcommand{\BEqn}{\begin{eqnarray*}}
\newcommand{\EEqn}{\end{eqnarray*}}

\newcommand{\blind}{1}

\usepackage{xr}
\makeatletter
\newcommand*{\addFileDependency}[1]{
  \typeout{(#1)}
  \@addtofilelist{#1}
  \IfFileExists{#1}{}{\typeout{No file #1.}}
}
\makeatother

\newcommand*{\myexternaldocument}[1]{%
    \externaldocument{#1}%
    \addFileDependency{#1.tex}%
    \addFileDependency{#1.aux}%
}
\myexternaldocument{supplement}

\allowdisplaybreaks[3]
\begin{document}
\setlength{\abovedisplayskip}{2pt}
\setlength{\belowdisplayskip}{2pt}
\setlength{\abovedisplayshortskip}{1pt}
\setlength{\belowdisplayshortskip}{1pt}
\setlength\intextsep{4pt}

\if1\blind
{
\title{Contextual Dynamic Pricing: Algorithms, Optimality,\\[-1ex] and Local Differential Privacy Constraints}	
	\author{
 	Zifeng Zhao$^{1}$, Feiyu Jiang$^{2}$, Yi Yu$^{3}$\footnote{1. University of Notre Dame; 2. Fudan University; 3. University of Warwick}	
        }
	\date{}	
	\maketitle
} \fi

\if0\blind
{
\title{Contextual Dynamic Pricing: Algorithms, Optimality,\\[-1ex] and Local Differential Privacy Constraints}
	\author{}
	\date{}
	\maketitle
 \vspace{-1.5cm}
} \fi

\centerline{
\begin{minipage}{1.05\textwidth}
\begin{abstract}
We study contextual dynamic pricing problems where a firm sells products to $T$ sequentially-arriving consumers, behaving according to an unknown demand model. The firm aims to minimize its regret over a clairvoyant that knows the model in advance. The demand follows a generalized linear model (GLM), allowing for stochastic feature vectors in $\mathbb R^d$ encoding product and consumer information. We first show the optimal regret is of order~$\sqrt{dT}$, up to logarithmic factors, improving existing upper bounds by a~$\sqrt{d}$ factor. This optimal rate is materialized by two algorithms: a confidence bound-type algorithm and an explore-then-commit~(ETC) algorithm. A key insight is an intrinsic connection between dynamic pricing and contextual multi-armed bandit problems with many arms with a careful discretization.

We further study contextual dynamic pricing under local differential privacy~(LDP) constraints. We propose a stochastic gradient descent-based ETC algorithm achieving regret upper bounds of order $d\sqrt{T}/\epsilon$, up to logarithmic factors, where $\epsilon>0$ is the privacy parameter. The upper bounds with and without LDP constraints are matched by newly constructed minimax lower bounds, characterizing costs of privacy. \textcolor{black}{Moreover, we extend our study to dynamic pricing under mixed privacy constraints, improving the privacy-utility tradeoff by leveraging public data. This is the first time such setting is studied in the dynamic pricing literature and our theoretical results seamlessly bridge dynamic pricing with and without LDP.} Extensive numerical experiments and real data applications are conducted to illustrate the efficiency and practical value of our algorithms. 
\end{abstract}
\end{minipage}
}
\vspace{5mm}

{\small \noindent\textit{\bf Keywords}: Online learning, Confidence bound, Explore-then-commit, Differential privacy, Minimax optimality.}

\section{Introduction}\label{sec-introduction}

With the technological advances and prevalence of online marketplaces, many firms can now dynamically
make pricing decisions while having access to an abundance of contextual information such
as consumer characteristics, product features and the economic environment. On the other hand, in
practice, the demand model is unknown in advance and firms need to dynamically learn how the contextual
information impacts consumer demand. Thus, to maximize its revenue, the firm needs to implement
dynamic pricing, which aims to optimally balance the trade-off between learning the unknown
demand function and earning revenues by exploiting the estimated demand model. Due to its importance in revenue management, dynamic pricing has been extensively studied across
the fields of statistics, machine learning, and operations research under various settings. We refer the readers to Section \ref{sec-literature} for recent works and \cite{den2015survey} for a comprehensive review.

In this work, we consider the problem of contextual dynamic pricing, where the firm sells products to $T \in \mathbb{N}_+$ sequentially arriving consumers and the unknown demand follows a generalized linear model~(GLM). In particular, at each time period $t \in \{1, \ldots, T\}$, the seller is given a $d$-dimensional feature vector $z_t\in\mathbb{R}^d$, which encodes the contextual information, and needs to set a price $p_t$ in a compact price interval $[l,u].$ Given $(z_t,p_t)$, the demand~$y_t$ follows a GLM with an unknown parameter $\theta=(\alpha,\beta)$ such that $\mathbb E(y_t|z_t,p_t;\theta)= m\big(\alpha^\top z_t - (\beta^\top z_t)  p_t\big)$, where $m(\cdot)$ is the inverse of the GLM link function. For example, we have $m(x)=x$ for linear regression and $m(x)=1/\{1+\exp(-x)\}$ for logistic regression. Here, the $\alpha^\top z_t$ term models the intrinsic utility of the product and $\beta^\top z_t$ captures the price elasticity. The revenue of the firm is defined as $r_t=p_t y_t.$ Due to its interpretability and flexibility, this GLM demand model including the linear special case~(i.e.\ $m(x) = x$) has been widely studied in the dynamic pricing literature under various scenarios~\citep[e.g.][]{ban2021personalized, wang2021dynamic, bastani2022meta, chen2022privacy, lei2023privacy}.

In the vanilla version of this problem, the firm aims to design a policy that utilizes the raw context~$\{z_t\}$ and maximizes its expected revenue $\mathbb E\big(\sum_{t=1}^T r_t\big)$ over the $T$ periods~(i.e.\ minimizes its regret over a clairvoyant that knows $\theta$ in advance). In this paper, we first propose two algorithms and show that they achieve a near-optimal regret of order $\widetilde O(\sqrt{dT})$  under this setting, an improvement over the existing $\widetilde O(d\sqrt{T})$ regret upper bound in the literature~\citep[e.g.][]{ban2021personalized, wang2021dynamic}. In other words, the regret is \textit{sublinear} in the dimension of the context vector. 

To build intuition, consider the special (and simplified) case of a linear demand model with $y_t=\alpha^{\top} z_t-(\beta^{\top}z_t)p_t+\varepsilon_t$, where $\{\varepsilon_t\}_{t = 1}^T$ is an i.i.d.\ mean zero error process. Denote $e(i)$ as the $i$th standard basis of $\mathbb R^d$ for $i\in \{1,\ldots, d\}$. We further set the context vector $z_t=e(\lceil td/T \rceil)$, assuming $T$ is a multiplier of $d$.  Under this simplified setting, we are essentially dealing with $d$ \textit{independent} dynamic pricing problems. In particular, the $i$th sub-problem runs for time $t\in \{(i-1) T/d+1,\cdots, i T/d\}$ with a total of $T/d$ periods and assumes a linear demand model with \textit{no} context~(i.e.\ context-free).  That is $y_t=\alpha_i -\beta_i  p_t+\varepsilon_t$, where $\alpha_i$ and $\beta_i$ denote the $i$th entry of $\alpha$ and $\beta.$ The optimal regret of this context-free sub-problem, due to \cite{broder2012dynamic}, is of order $\widetilde{O}(\sqrt{T/d})$. Summing up over all $d$ sub-problems, intuitively, we can have a regret upper bound of order $\widetilde O(d\sqrt{T/d}) =\widetilde O( \sqrt{dT})$.

In recent years, there has been a growing concern over data privacy~\citep{lucas2021corporate}, as consumers become increasingly aware of the data collected about them and their use in personalized pricing and recommendation. Based on \textit{differential privacy} \citep[e.g.][]{dwork2006calibrating}, industry-led initiatives have emerged to address these concerns, such as the RAPPOR project~\citep{google2014} and the integration of privacy safeguards into Siri~\citep{apple2019}. In this paper, we further study contextual dynamic pricing under the local differential privacy~(LDP) constraint \citep[e.g.][]{evfimievski2003limiting, duchi2018minimax}.  LDP provides a decentralized version of  differential privacy in the sense that it ensures sensitive personal information remains protected even if the data security of the focal firm is compromised. For $\epsilon > 0$, an $\epsilon$-LDP algorithm masks the sensitive context $z_t$ at any time $t$ by ensuring that for any two possible outcomes of the raw data, the probabilities of their privatized data lying in any given measurable set differ by no more than a multiplicative constant $\exp(\epsilon)$. Therefore, a smaller $\epsilon$ implies higher privacy. Leveraging stochastic gradient descent (SGD), we devise the first dynamic pricing policy in the literature that is $\epsilon$-LDP and achieves a near-optimal regret of order $\widetilde{O}(d\sqrt{T}/\epsilon)$ under the GLM demand model, therefore characterizing the cost of privacy in parametric contextual dynamic pricing. %

\subsection{Our contributions}
First, we thoroughly study the contextual dynamic pricing problem with a parametric GLM demand function and characterize its near-optimal regret upper bound. In particular, we propose two  pricing policies in \Cref{sec-without-LDP} and \Cref{sec-supCB}, an ETC algorithm and a confidence bound-type (supCB) algorithm, and show that both achieve a regret of order $\widetilde{O}(\sqrt{dT})$, with $d$ and $T$ being the dimension of feature vectors and total number of consumers. Thus, we improve the existing upper bound $\widetilde{O}(d\sqrt{T})$ in the literature \citep[e.g.][]{ban2021personalized, wang2021dynamic, simchi2023pricing} by a factor of~$\sqrt{d}$, suggesting the impact of dimension~$d$ to regret can be controlled at a sublinear rate.

Under the condition that the optimal price is unique and an interior point of $[l,u]$, a common assumption in the dynamic pricing literature, we show in \Cref{sec-ETC} that a tuning-free and computationally efficient ETC algorithm can achieve the optimal regret up to a logarithmic term. \textcolor{black}{This is different from the contextual multi-armed bandit (MAB) literature where ETC is in general sub-optimal.} The key insight is that due to the structure of the reward function (i.e.\ the expected revenue), the regret incurred due to estimation error (or more precisely, prediction error) scales with the magnitude of $\|\widehat\theta-\theta\|^2$ instead of $\|\widehat\theta-\theta\|$. We refer to \Cref{lem_punique} and its discussion (presented following \Cref{thm:ETC}) for more details. We further examine the cause of sub-optimality of the existing upper bound based on the popular MLE-Cycle algorithm~\citep[e.g.][]{broder2012dynamic,ban2021personalized} \textcolor{black}{and Semi-Myopic algorithm~\colorlinks{black}{\citep[e.g.][]{keskin2014dynamic,den2014simultaneously,simchi2023pricing}} for contextual dynamic pricing, which is under-exploration with respect to dimension $d$, and suggest a simple fix to make it optimal.}


The supCB algorithm proposed in \Cref{sec-supCB} is built upon a confidence bound type algorithm. The key ingredient of supCB is a novel discovery of an intrinsic connection between contextual dynamic pricing and contextual MAB. We show that a carefully designed discretization of the continuous action space $[l, u]$ can transform the dynamic pricing problem to an MAB problem with negligible approximation error.  Importantly, this connection further helps supCB achieve the near-optimal $\widetilde O(\sqrt{dT})$ regret upper bound under weaker conditions than those imposed in the existing literature.

Second, we consider in \Cref{sec-ldp} the important and timely setting where the consumer data are subject to LDP constraints. Building upon \cite{duchi2018minimax} and \cite{chen2022differential}, we extend the notion of LDP to the setting of parametric contextual dynamic pricing. Based on SGD in conjunction with the $L_2$-ball mechanism in \cite{duchi2018minimax}, we design an efficient ETC-type algorithm that is provably $\epsilon$-LDP and achieves a near-optimal regret of order $\widetilde{O}(d\sqrt{T}/\epsilon)$. To our best knowledge, this is the first LDP algorithm for contextual dynamic pricing under the GLM demand model and provides the first characterization of the cost of LDP in the task of parametric dynamic pricing. As a byproduct, we study the SGD estimator under $\epsilon$-LDP and carefully analyze the bias-variance trade-off due to an additional truncation step in the $L_2$-ball mechanism for handling unbounded gradients. \textcolor{black}{We further show that our policy can be readily modified to cover dynamic pricing under approximate $(\epsilon,\Delta)$-LDP constraints, a less stringent notion than the $\epsilon$-LDP. Moreover, we study dynamic pricing under mixed privacy constraints in \Cref{subsec:mixed}, where consumers exhibit heterogeneous privacy preferences and the firm can access raw information from the non-private consumers. We design an ETC-LDP-Mixed algorithm that efficiently combines a private and a non-private SGD to further improve the privacy-utility tradeoff by leveraging the public data. We show its regret upper bound seamlessly bridges the $\widetilde O(\sqrt{dT})$ bound in non-private pricing and the $\widetilde{O}(d\sqrt{T}/\epsilon)$ bound in $\epsilon$-LDP dynamic pricing.}

Third, to demonstrate the optimality of our proposed methods with and without LDP constraints, we further establish the accompanying minimax lower bounds in \Cref{sec-optimality}. The proofs of lower bounds are based on a novel construction of $2^d$ indistinguishable demand models.  The results, i.e.\ the $\Omega(\sqrt{dT})$ lower bound under the non-private setting and the $\Omega(d\sqrt{T}/\epsilon)$ lower bound under the LDP setting, are first time seen in the literature. We note that compared to lower bound constructions for statistical estimation under the LDP setting~\citep[e.g.][]{duchi2018minimax,acharya2024unified}, where data are independent, our lower bound construction is much more involved as data are adaptively collected due to the pricing policy. Extensive numerical results including both simulations and real data applications are conducted in \Cref{sec-numerical} to support our theoretical findings and illustrate the efficiency of the proposed algorithms over existing ones in the contextual dynamic pricing literature.

{\color{black} \textbf{Literature Review}: Three streams of literature are closely related to our paper: dynamic pricing with demand learning,
contextual MAB, and differential privacy for online learning. We provide an in-depth literature review on related works in \Cref{sec-literature} of the supplement.
}

The rest of the paper is organized as follows. In \Cref{sec-without-LDP}, we formally introduce the problem setup and propose the supCB and ETC algorithms. In \Cref{sec-ldp}, we further consider contextual dynamic pricing under the additional LDP constraint. Minimax lower bounds are derived in \Cref{sec-optimality}. We conclude with extensive numerical studies in \Cref{sec-numerical}. 

\subsection{Notation} \label{sec-notation}
For a vector $a\in\mathbb{R}^d$, denote its Euclidean norm as $\|a\|$; for two square matrices $A$ and $B$, write $A\succeq B$ and $A \succ B$, if $A-B$ is semi-positive definite and positive definite.  Denote~$\otimes$ as the Kronecker product. Let $\|\cdot\|_{\mathrm{op}}$, $\lambda_{\min}(\cdot)$ and $\lambda_{\max}(\cdot)$ be the $L_2\to L_2$-operator norm, the smallest and largest eigenvalues of a matrix.  For $M \succ 0$ and any vector $v$ of compatible dimension, let $\|v\|_{M^{-1}} = \sqrt{v^{\top} M^{-1} v}$.  For any positive integer $K$, denote $[K]=\{1, \ldots, K\}$.  
For any vector $v$ and set~$S$, let $\Pi_S(v)$ be the $L_2$-projection of $v$ onto $S$.   We let $O(\cdot)$, $o(\cdot)$, $\Theta(\cdot)$ and $\Omega(\cdot)$ denote
Landau’s Big-O, little-o, Big-Theta and Big-Omega, respectively.  The notation $\widetilde{O}(\cdot)$ disregards poly-logarithmic factors for $O(\cdot)$.

\section{An improved upper bound on the regret}\label{sec-without-LDP}
In this section, we provide an improved upper bound on the regret for contextual dynamic pricing.  The detailed problem setup is introduced in \Cref{sec-problem-setup}, and two algorithms, namely supCB and ETC, are proposed and analyzed in Section \ref{sec-ETC}. \textcolor{black}{ \Cref{subsec-adver} further discusses the case of adversarial contexts.}

\subsection{Problem setup}\label{sec-problem-setup}
We consider a firm, hereafter referred to as the \textit{seller}, that sells a product to $T \in \mathbb{N}_+$ sequentially arriving consumers. In each time period $t \in [T]$, the seller observes contextual information featurized as a vector $z_t\in\mathbb{R}^{d}$, which may include consumer personal information and product characteristics. Upon observing $z_t$, the seller offers a price $p_t$ from a compact price interval $[l, u] \subseteq \mathbb{R}^+$ to the consumer and observes a demand $y_t \in \mathbb{R}$. Note that we do not require the knowledge of $T.$

\textbf{Model setup}: Denote the covariate as $x_t = (z_t^{\top}, -p_t z_t^\top)^{\top} \in \mathbb{R}^{2d}$ and denote the natural $\sigma$-field at time~$t$ as $\mathcal{F}_{t-1}=\sigma(x_t,x_{t-1},\ldots,x_1,y_{t-1},\ldots,y_1)$. We assume the demand $y_t$ follows a GLM with the probability distribution
\begin{equation}\label{eq:GLM}
    f(y_t=y|\mathcal{F}_{t-1};\theta^*)=f(y_t=y|x_t;\theta^*)=f(y_t=y|z_t,p_t;\theta^*)= \exp\left\{\frac{y  x_t^{\top}\theta^*-\psi(x_t^{\top}\theta^*)}{a(\phi)}+h(y)\right\},
\end{equation}
where $\theta^* = (\alpha^{*\top},\beta^{*\top})^{\top}\in\mathbb{R}^{2d}$ is the unknown model parameter, $\psi(\cdot):\mathbb {R}\mapsto \mathbb{R}$ is a known function, $a(\phi)$ is a fixed and known scale parameter and $h(\cdot)$ is a known normalizing function.  This setting covers important GLMs, including Gaussian, logistic and Poisson regression models.  

Conditioning on $x_t$, following \eqref{eq:GLM}, it holds that $\be(y_t|x_t;\theta^*)= \be(y_t|z_t,p_t;\theta^*)=\psi'(x_t^\top\theta^*)=\psi'(z_t^\top\alpha^*-z_t^\top \beta^*p_t)$, where the inverse function of $\psi'(\cdot)$ is known as the link function of a GLM. In other words, the consumer demand depends on (i) the intrinsic utility $z_t^\top\alpha^*$ and (ii) the pricing effect $(z_t^\top\beta^*) p_t$, where $z_t^\top \beta^*$ captures the price elasticity. Define $\varepsilon_t:=y_t- \psi'(x_t^\top\theta^*)$ as the error term.  We can thus rewrite \eqref{eq:GLM} as $y_t = \psi'(x_t^\top\theta^*) + \varepsilon_t$ and it holds that $\mathbb E(\varepsilon_t|x_t)=0$. We impose the following assumptions on the GLM defined in \eqref{eq:GLM}, the covariate process $\{z_t\}$ and the model parameter~$\theta^*$. 

\vspace{-2mm}
\begin{ass}\label{assum_feature}
(a) $\{z_t\}_{t \in [T]} \subset \mathcal{Z}\subseteq \{z\in \mathbb R^d: \, \|z\|\leq 1\}$ is a sequence of i.i.d.\ random vectors from an unknown distribution supported on $\mathcal{Z}$. (b) Letting $\Sigma_z = \mathbb E(z_1z_1^\top)$, we have that $\lambda_z:=\lambda_{\min}(\Sigma_z)>0.$ (c) There exist absolute constants $u_a, u_b > 0$ such that $|z^\top \alpha^*|\leq u_a$ and $|z^\top \beta^*|\leq u_b$ for any $z \in \mathcal{Z}$. (d) Define $\mathcal{X} := \big\{(z^\top,-pz^\top)^\top: \, z\in \mathcal{Z} \text{ and } p\in[l, u]\big\}$. There exists an absolute constant $\sigma>0$ such that for any $x \in \mathcal{X}$, we have $\be \{\exp(\lambda \varepsilon_t)|x_t=x\}\leq \exp(\lambda^2\sigma^2/2)$ for any $\lambda>0$.  (e) The function $\psi(\cdot)$ is three times continuously differentiable and its second order derivative $\psi''(a)>0$ for all $a\in\mathbb R$.
\end{ass}
\vspace{-2mm}

Assumption \ref{assum_feature} is standard in the dynamic pricing literature, see e.g.\ \cite{javanmard2019dynamic}, \cite{ban2021personalized} and \cite{chen2022statistical}. \Cref{assum_feature}(a), together with $p \in [l, u]$, implies that $\|x\|\leq u_x:=\sqrt{1+u^2}$ for all $x\in\mathcal{X}$. \Cref{assum_feature}(b) states that $\Sigma_z$ is of full rank. Note that \Cref{assum_feature}(a) implies that $\lambda_z \leq 1/d.$ \Cref{assum_feature}(c) imposes that the intrinsic utility $z^\top \alpha^*$ and the price sensitivity $z^\top\beta^*$ are bounded.  \Cref{assum_feature}(d) states that the error term $\varepsilon_t$ is sub-Gaussian.  \Cref{assum_feature}(e) holds for commonly used GLMs and ensures the log-likelihood function to be convex. 

Based on \Cref{assum_feature}, we define several absolute constants that will be used later in the paper:
\begin{align}\label{eq-constants-definitions}
    & M_{\psi 1} = \sup_{x\in\mathcal{X}} |\psi'(x^\top \theta^*)| \vee 1, \quad M_{\psi 2} = \sup_{\{x\in\mathcal X, ~\|\theta-\theta^*\|\leq 1\}} |\psi''(x^\top \theta)|\vee 1, \nonumber \\
    & M_{\psi 3} = \sup_{\{x\in\mathcal{X}, ~\|\theta-\theta^*\|\leq 1\}} |\psi'''(x^\top \theta)|\vee 1 \quad \text{and} \quad \kappa = \inf_{\{x\in\mathcal{X}, ~\|\theta-\theta^*\|\leq 1\}} \psi''(x^\top\theta)\wedge 1.
\end{align}
Note that it is easy to see that $\kappa>0$, $M_{\psi 1}, M_{\psi 2}, M_{\psi 3}<\infty$, and $(\kappa,M_{\psi 1}, M_{\psi 2}, M_{\psi 3})$ are absolute constants that only depend on $\psi(\cdot)$, $l$, $u$, $u_a$ and $u_b$ in \Cref{assum_feature}. 

\textbf{Pricing policies and performance metric}: In dynamic pricing, the key objective is to design an algorithm~(i.e.\ pricing policy) that maximizes the seller's expected revenue. Based on \eqref{eq:GLM}, the expected revenue function at any given price $p$ and context $z_t$ can be written as
\begin{align}\label{eq-revenue}
    r(p,z_t,\theta^*)= \be(p y_t|z_t,p;\theta^*)=p \psi'(z_t^\top\alpha^*-z_t^\top \beta^*p).
\end{align}
Denote $p_t^*$ as the optimal price that maximizes $r(p,z_t,\theta^*)$. 

Denote by $\Pi$ the family of all price processes $\{p_1, \ldots, p_T\}$ satisfying the condition that $p_t$ is $\mathcal{F}_{t-1}^\circ$-measurable for all $t\in [T]$, where $\mathcal{F}_{t-1}^\circ = \sigma (y_1,\cdots,y_{t-1},p_1,\cdots,p_{t-1},z_1,\cdots,z_{t})$ is the natural filtration generated by the demand, price and context history up to time $t-1$, together with the context $z_{t}$.  For $t = 1$, denote $\mathcal{F}^o_0 = \sigma\{z_1\}$.  In other words, we require the pricing policy to be non-anticipating. Given a pricing policy $\pi\in\Pi$, we evaluate its performance using the common notion of regret: the expected revenue loss compared with a clairvoyant that has the perfect knowledge of the model parameter $\theta^*$ (and thus sets the prices at $\{p_t^*\}_{t=1}^T$). In particular, with $r(\cdot, \cdot, \cdot)$ defined in~\eqref{eq-revenue}, let the regret be
\begin{align}\label{eq:regret}
    R_T^\pi= \sum_{t=1}^{T} \{r(p_t^*, z_t, \theta^*)-r(p_t^\pi, z_t, \theta^*)\}.
\end{align}
We aim to design pricing policies that can minimize $R_T^\pi$ both in high probability and in expectation.

\textbf{Non-asymptotic bounds for MLE}: For $\tau \in\mathbb{N}_+$, given a sample $\{(y_t,x_t)\}_{t\in[\tau]}$, the log-likelihood of the GLM takes the form $L(\theta)=\sum_{t=1}^\tau \{y_t x_t^\top \theta - \psi(x_t^\top \theta)\}$. Denote $V_\tau=\sum_{t=1}^\tau x_tx_t^\top$ as the design matrix. By \Cref{assum_feature}(e), $L(\theta)$ is strictly convex  in $\mathbb R^{2d}$ and has a unique maximizer~$\widehat\theta_\tau$~(i.e.\ MLE) given that $V_\tau\succ 0.$ \Cref{thm:theta_err} provides non-asymptotic performance guarantees for $\widehat\theta_\tau$.

\vspace{-2mm}
\begin{thm}\label{thm:theta_err}
Suppose that \Cref{assum_feature}(c), (d), (e) hold and $\{y_t\}_{t\in[\tau]}$ are \textit{conditionally independent} given $\{x_t\}_{t\in[\tau]}$, i.e.~$f(y_1,\cdots,y_\tau|x_1,\cdots,x_\tau)=\prod_{t=1}^\tau f(y_t|x_t;\theta^*)$, where $f(y_t|x_t;\theta^*)$ is the true GLM model in \eqref{eq:GLM}. We have that for any $\delta \in (0,1)$ and $\varsigma\in(0,1)$, with probability at least $1-\delta$, it holds that
(i)\ $\|\widehat\theta_\tau-\theta^*\| \leq \varsigma$, provided that  $\lambda_{\min}(V_\tau)\geq {3\sigma^2}/(\kappa^2\varsigma^2)  \{d+\log(1/\delta)\}$; (ii)\ $(\widehat\theta_\tau-\theta^*)^\top V_\tau (\widehat\theta_\tau-\theta^*) \leq {12\sigma^2}/{\kappa^2} \{d+\log(1/\delta)\}$ and  $\lambda_{\min}^{1/2}(V_{\tau})\|\widehat\theta_\tau-\theta^*\|\leq \sqrt{3}\sigma/\kappa \sqrt{d+\log(1/\delta)}$,
provided that $\lambda_{\min}(V_\tau)\geq {3\sigma^2}/\kappa^2 \{d+\log(1/\delta)\}$; and (iii)\ $\big|x^\top (\widehat{\theta}_\tau-\theta^*)\big|\leq {3\sigma}/{\kappa} \sqrt{\log(1/\delta)}\|x\|_{V_\tau^{-1}}$ for any $x\in \mathbb R^{2d}$, provided that $\lambda_{\min}(V_\tau)\geq {512 M_{\psi3}^2\sigma^2}/{\kappa^4} \left\{d^2+\log(1/\delta)\right\}$.
\end{thm}
\vspace{-2mm}

\Cref{thm:theta_err} states that once the minimum eigenvalue of $V_\tau$ surpasses some mild thresholds that scale with $d$ and $\log(1/\delta)$, the MLE $\widehat\theta_\tau$ will be consistent as stated in (i) and have controllable prediction and estimation error as stated in (ii) and (iii). \Cref{thm:theta_err}(i) and (ii) are derived in this work and serve as the theoretical foundation for the ETC algorithm.  \Cref{thm:theta_err}(iii) is adapted from \cite{li2017provably} and plays a key role in bounding the regret of the supCB algorithm. Importantly, \Cref{thm:theta_err} does not require a compact parameter space $\Theta$ and instead allows $\Theta=\mathbb R^{2d}$, due to a nice geometric result for the log-likelihood function of GLM in \cite{chen1999strong}~(Lemma A therein).

\subsection{The ETC algorithm}\label{sec-ETC}
In this section, we propose an explore-then-commit~(ETC) algorithm in \Cref{algorithm:ETC}, which is tuning-free and computationally efficient, and show that it achieves the optimal regret for dynamic pricing.

\vspace{2mm}
\begin{algorithm}[ht]
\begin{algorithmic}
    \State \textbf{Input}: Total rounds $T$, price interval $[l,u]$, exploration length $\tau$.
    \medskip
    
    \State a.\ (Exploration): For $t\in[\tau]$, uniformly choose $p_t \in [l,u]$, record $y_t$ and $x_t=(z_t^\top,-p_tz_t^\top)^\top$ into the price experiment set $\Psi$.
    \medskip

    \State b.\ (Estimation): Obtain MLE $\widehat\theta$ based on observations in $\Psi$.
    \medskip

    \State c.\ (Exploitation): For $t = \tau+1,\ldots, T$, offer the greedy price at $\widehat p^*_t=p^*(\widehat\theta,z_t)=\argmax\limits_{p\in[l,u]} r(p,z_t,\widehat\theta)$.
    
    \caption{The ETC algorithm for dynamic pricing}
    \label{algorithm:ETC}
\end{algorithmic}
\end{algorithm}

ETC employs a simple two-stage structure: in the exploration stage, it conducts $\tau$ rounds of price experiments where prices are randomly drawn from $[l,u]$; it then obtains an MLE $\widehat\theta$ based on the sample collected and switches to the exploitation stage, where it uses $\widehat\theta$ to set prices for the rest $T-\tau$ rounds. We remark that ETC is a standard algorithm in the bandit literature~\citep[e.g.][]{lattimore2020bandit} and is recently used in semi-parametric dynamic pricing in \cite{fan2024policy}. Importantly, ETC only requires a single computation of MLE, which is highly efficient.

Denote $\mu_p$ and $\sigma_p^2$ as the mean and variance of the uniform distribution on $[l,u]$ and denote $\Sigma_p=[1,-\mu_p; -\mu_p, \mu^2_p+\sigma^2_p] \in \mathbb{R}^{2 \times 2}$. It is easy to see that $\mu_p=(u+l)/2$, $\sigma_p^2 = (u-l)^2/12$ and 
\begin{equation} \label{eq-Lp-def}
    \lambda_{\min}(\Sigma_p)= L_p:=(u-l)^2/[4(u^2+l^2+ul+3)].
\end{equation}
Write $\Sigma_x =\Sigma_p\otimes \Sigma_z$, which is the covariance matrix of the feature $x_t=(z_t^\top, -z_t^\top p_t)^\top$ during the pure exploration phase. By the property of the Kronecker product and \Cref{assum_feature}(b), it holds that 
\begin{equation}\label{lambda_min}
    \lambda_{\min}(\Sigma_x)=\lambda_{\min}(\Sigma_p) \lambda_{\min}(\Sigma_z)\geq L_p\lambda_{\min}(\Sigma_z)= L_p \lambda_z,
\end{equation}
which quantifies the information gain from each consumer during price exploration. We proceed by introducing \Cref{assum:etc}, which is needed for theoretical guarantees of ETC.
\vspace{-2mm}
\begin{ass}\label{assum:etc}
    There exists an absolute constant $\varsigma_o>0$ such that for any $\theta \in \{\theta: \, \|\theta-\theta^*\|\leq \varsigma_o\}$ and any context vector $z\in \bar{\mathcal{Z}}$, with $\bar{\mathcal{Z}}$ being the closure of $\mathcal{Z}$, it holds that the optimal price $p^*=\argmax_{p\in[l,u]} r(p,z,\theta)$ is unique and is an interior point of the price range $[l,u]$.
\end{ass}
\vspace{-2mm}

We remark that \Cref{assum:etc} is a mild assumption commonly used in the dynamic pricing literature under various settings~\citep[e.g.][]{broder2012dynamic,javanmard2019dynamic,ban2021personalized,chen2021nonparametric,lei2023privacy,fan2024policy}. In particular, it requires that the optimal price $p^*$ is unique for all context vector and all $\theta$ within an arbitrarily small neighborhood of $\theta^*.$ To understand \Cref{assum:etc}, consider a linear demand model. We have that the optimal price is unique over $\mathbb R$ and is $p^*=z^\top \alpha/(2z^\top\beta)$ for any $z$ and $\theta=(\alpha^{\top}, \beta^{\top})^{\top}.$ Therefore, given that the true price elasticity $z^\top \beta^*$ is lower bounded by an absolute constant $c_b$ such that $z^\top \beta^*\geq c_b>0$ for all $z \in \bar{\mathcal{Z}}$~(i.e.\ all consumers are sensitive to prices), one can always find a sufficiently large interval $[l,u]$ and sufficiently small $\varsigma_o$ such that \Cref{assum:etc} holds. By simple algebra, the same holds under a logistic demand model. In the following, without loss of generality, we assume $\varsigma_o\in (0,1]$.

An important implication of \Cref{assum:etc} is that the optimal price $p^*$ is a Lipschitz continuous function of $\theta$ over $\|\theta-\theta^*\|\leq \varsigma_o$ uniformly for all $z\in\mathcal{Z}$, as stated in \Cref{lem_punique} of the supplement. In other words, when the parameter estimation error $\|\widehat\theta-\theta^*\|$ is sufficiently small, its impact on the estimation error of the optimal price is Lipschitz. Such a Lipschitz continuity property is explicitly assumed in the literature~\citep[e.g.][]{broder2012dynamic,ban2021personalized}, while we give a formal proof under mild conditions. As will be discussed later, this Lipschitz continuity property of the optimal price function plays a key role in the theoretical guarantees for ETC.

\vspace{-2mm}
\begin{thm}\label{thm:ETC}
Suppose Assumptions \ref{assum_feature} and \ref{assum:etc} hold. For any $\delta\in(0,1/3)$, set $\tau=\sqrt{Td \log(1/\delta)}$. For any $T$ satisfying
{\color{black}
\begin{align}\label{eq:ETC_strong_cond_short}
T \geq \left[ B_{E1}\cdot \frac{\log^2 (d) + \log^2(1/\delta)}{d\lambda_z^4}  \right] \vee \left[ B_{E2}\cdot\frac{d^2+\log^2(1/\delta)}{d\lambda_z^2} \right],
\end{align}}
we have that, with probability at least $1-3\delta$, the regret of ETC is upper bounded by $R_T \leq B_{E3}\cdot \sqrt{dT\log (1/\delta)}$,
where $B_{E1},B_{E2},B_{E3}>0$ are absolute constants that only depend on quantities in Assumptions \ref{assum_feature} and \ref{assum:etc}. In particular, setting $\delta=1/{T}$,
we have that $\mathbb E(R_T)\lesssim \sqrt{dT\log (T)}$. 
\end{thm}
\vspace{-2mm}

The details of $B_{E1}$, $B_{E2}$ and $B_{E3}$ can be found in the proof of \Cref{thm:ETC} in the supplement. Remarkably, despite its simplicity, ETC achieves a near-optimal regret $\widetilde O(\sqrt{dT})$ for sufficiently large~$T$. The key intuition lies in a second-order structure of the revenue function due to \Cref{assum:etc} and \Cref{lem_punique}. In particular, for the regret incurred in the exploitation stage, we have that
\begin{align}\label{eq:second_order}
    \sum_{t=\tau+1}^T|r(p^*_t,z_t,\theta^*)-r(\widehat p^*_t, z_t,\theta^*)|&=\sum_{t=\tau+1}^T \left|\frac{\partial^2 r(p^+_t, z_t,\theta^*)}{\partial p^2} \cdot (p_t^*-\widehat p^*_t)^2\right|\nonumber\\
    &\leq C_r \cdot C_\varphi^2 \cdot\sum_{t=\tau+1}^T\big\{|z_t^\top (\alpha^*-\widehat\alpha)|^2 + |z_t^\top (\beta^*-\widehat\beta)|^2\big\},
\end{align}
where the equality follows from a second-order Taylor expansion of the revenue function $r(\widehat p^*_t, z_t,\theta^*)$ around $p_t^*$~(with $p_t^+$ between $p_t^*$ and $\widehat p_t^*$) and the fact that $p_t^*\in (l,u)$ is the maximizer of the revenue function, and the inequality follows from the fact that $|{\partial^2} r(p, z_t,\theta^*)/{\partial p^2}|<C_r$ for an absolute constant~$C_r$~(\Cref{lem:Lips}) and the Lipschitz continuity of the optimal price~(\Cref{lem_punique}). In the proof of \Cref{thm:ETC}, we show that the last term in \eqref{eq:second_order} is related to the \textit{squared} prediction error of $\widehat \theta$, which can be well controlled using results in \Cref{thm:theta_err}(ii) and a matrix concentration inequality. Importantly, this helps make the regret bound in \Cref{thm:ETC} independent of $\lambda_z$.

This indicates that the regret in \eqref{eq:second_order} is of order $\widetilde O(dT/\tau)$ when the exploration set is of length~$\tau$.  This is the key for the optimality of ETC. Without \Cref{assum:etc}, the arguments in \eqref{eq:second_order} will not hold and we can only control the regret at the order of $\widetilde O(dT/\sqrt{\tau})$ using the Lipschitz property of the revenue function, leading to a sub-optimal regret and making supCB in \Cref{sec-supCB} necessary.


\textbf{Unknown $T$}. When $T$ is unknown, we couple ETC with the standard doubling trick~\citep[e.g.][]{auer1995gambling}. We summarize the algorithm, named ETC-Doubling, in \Cref{algorithm:ETC-Doubling}. For a given dimension~$d$, denote~$T_{d}^*$ as the minimum $T$ such that condition \eqref{eq:ETC_strong_cond_short} holds with $\delta=1/T$. 

\vspace{2mm}
\begin{algorithm}[ht]
\begin{algorithmic}
    \State \textbf{Input}: Price interval $[l,u]$, $\Psi=\varnothing.$
    \medskip

    \For{$k=1,2,\cdots$}
    \State a.\ (New episode): Set the episode length as $E_k=2^{k}$. 
    
    \State b.\ (Exploration): For the first $\tau_k=\min\left\{\sqrt{dE_k\log (E_k)}, E_k\right\}$ rounds, randomly choose \\\hspace{10mm} $p_t \in [l,u]$, and record $y_t$ and $x_t=(z_t^\top,-p_tz_t^\top)^\top$ into the price experiment set $\Psi$.
    \medskip

    \State c.\ (Estimation): Obtain MLE $\widehat\theta$ based on observations in $\Psi$.
    \medskip

    \State d.\ (Exploitation): For the rest $E_k-\tau_k$ rounds, offer the greedy price at $p^*(\widehat\theta,z_t)$ based on $\widehat\theta$.
    \EndFor
    
    \caption{The ETC-Doubling algorithm for dynamic pricing}
    \label{algorithm:ETC-Doubling}
\end{algorithmic}
\end{algorithm}

\vspace{-2mm}
\begin{thm}\label{thm:etc_regret_unknownT}
Suppose that Assumptions \ref{assum_feature} and \ref{assum:etc} hold. For any $T$ satisfying $T\geq T_{d}^{*2}/d$, with probability at least $1-6/\sqrt{dT}$, the regret of ETC-Doubling is upper bounded by $R_T \leq B_{E4}\sqrt{dT\log(T)}$, where $B_{E4}>0$ is an absolute constant only depending on quantities in Assumptions \ref{assum_feature} and \ref{assum:etc}. 
\end{thm}
\vspace{-2mm}

\vspace{-2mm}
\begin{rem}[\textcolor{black}{Comparison with MLE-Cycle and Semi-Myopic}]
    Two popular algorithms in the dynamic pricing literature are the MLE-Cycle algorithm~\citep[e.g.][]{broder2012dynamic,ban2021personalized} \textcolor{black}{and the Semi-Myopic type algorithm~\colorlinks{black}{\citep[e.g.][]{keskin2014dynamic,simchi2023pricing}},} both achieving a regret of order $\widetilde O(d\sqrt{T})$. The key idea of MLE-Cycle is to divide the horizon $\mathbb{N}_+$ into cycles. The $c$th cycle, $c \in \mathbb{N}_+$, is of length $c + k$, where $k \in \mathbb{N}_+$ is a \textit{constant}. Within each cycle, the first $k$ rounds are used for price experiments and the next $c$ rounds for price exploitation. \textcolor{black}{The key idea of Semi-Myopic is to run a greedy dynamic pricing policy but further inject a controlled deviation (in particular, of size $\kappa t^{-1/4}$ for some \textit{constant} $\kappa>0$) into the greedily estimated optimal price. In Sections \ref{sec:mle_cycle} and \ref{sec:semi-myopic} of the supplement, we show the sub-optimality of MLE-Cycle and Semi-Myopic is due to under-exploration, i.e.\ the constant $k$ and $\kappa$ do \textit{not} scale with the dimension~$d$ optimally. Moreover, we show that via a simple modification that boosts exploration, both algorithms can achieve the optimal regret of order $\widetilde O(\sqrt{dT})$. On the other hand, due to their design, MLE-Cycle and Semi-Myopic require the computation of $O(\sqrt{T})$ and $O(T)$ many MLEs, respectively, while ETC-Doubling only requires $O(\log_2 (T))$.} In \Cref{sec-numerical}, we show via extensive numerical experiments that ETC-Doubling achieves better statistical and computational efficiency in finite-sample.
\end{rem}
\vspace{-2mm}

\subsubsection{The supCB Algorithm}\label{subsec-supcb-brief}
\textcolor{black}{In \Cref{sec-supCB} of the supplement, we further propose supCB, which achieves a near-optimal $\widetilde O(\sqrt{dT})$ regret under assumptions weaker than the ones required in the existing dynamic pricing literature. In particular, it does not require \Cref{assum:etc}, which imposes uniqueness of the true optimal price.}

The supCB algorithm is a more complex version of the widely used upper confidence bound~(UCB) algorithm and is inspired by algorithms designed for the classical contextual MAB problem~\citep[e.g.][]{auer2002using,chu2011contextual,li2017provably}. The  basis of supCB is a novel discovery of an intrinsic connection between contextual dynamic pricing and contextual MAB. By discretizing the \textit{one-dimensional} continuous action space $[l,u]$ with $K$ equally-spaced price points, we can \textit{approximately} treat the contextual dynamic pricing problem as a contextual MAB with $K$ arms. Importantly, we show that regret due to the approximation error of such discretization can be controlled at $O(T/K).$ It is well-known that the optimal regret of contextual MAB with $K$ arms is $\widetilde O(\sqrt{dT\log (K)})$. Therefore, setting $K=O(\sqrt{T/d})$ will intuitively provide a desired $\widetilde O(\sqrt{dT})$ upper bound for contextual dynamic pricing. We refer to \Cref{sec-supCB} of the supplement for more details of supCB.

Both supCB and ETC achieve the near-optimal regret $\widetilde O(\sqrt{dT})$, with ETC offering a slightly better rate than supCB by a $\log(T)$ factor under an additional mild \Cref{assum:etc}. However, ETC is much simpler in its design with better computational efficiency. In particular, ETC only requires a single computation of MLE, while supCB requires the computation of at least $O(T)$ many MLE. Indeed, extensive numerical experiments in \Cref{subsubsec:SupCB_ETC} of the supplement indicate that ETC is superior than supCB in both statistical and computational efficiency, where ETC gives an expected regret 3-4 times lower than supCB while using only 1/50 of its computational time. We refer to \Cref{subsubsec:SupCB_ETC} for details. Therefore, in practice, we recommend ETC over supCB.


{\color{black}
\subsubsection{Dynamic pricing under adversarial contexts}\label{subsec-adver}

In this paper, we focus on dynamic pricing under stochastic contexts, where the feature vector $\{z_t\}$ is a sequence of i.i.d. random vectors with full-rank covariance satisfying \Cref{assum_feature}(a)-(b). To our knowledge, all existing works in dynamic pricing under parametric GLM demand models assume stochastic contexts, with the exception of \colorlinks{black}{\cite{wang2021dynamic}}, where the only assumption on $\{z_t\}$ is that it is a sequence of bounded vectors (i.e.\ $\|z_t\|\leq 1$), namely the adversarial contexts.

In general, for an online learning problem, the upper bound algorithms and minimax rates can be substantially different under stochastic and adversarial contexts. Nevertheless, for a more comprehensive view, we provide a brief study of dynamic pricing under adversarial contexts in Sections~\ref{sec-supp-supCB-adversarial} and \ref{sec:adversarial_supplement} of the supplement. We give an overview of the results below.

Note that due to the adversarial contexts, the MLE $\widehat\theta_\tau$ may not be well-defined as the design matrix $V_\tau$ can be rank-deficient. Thus, we instead rely on the regularized MLE $\widetilde\theta_\tau$, which is the unique minimizer of the regularized (negative) log-likelihood $-L(\theta)+\|\theta\|^2.$ In particular, based on $\widetilde\theta_\tau$, we propose a modified supCB algorithm in \Cref{sec-supp-supCB-adversarial} and show that, under an additional assumption on the link function $\psi(\cdot)$ of the GLM (which is satisfied by the linear demand model), it achieves the near-optimal $\widetilde{O} (\sqrt{dT})$ regret upper bound under adversarial contexts.

Furthermore, in \Cref{sec:adversarial_supplement}, we propose a novel UCB algorithm that works for a general GLM demand model under adversarial contexts and matches the $\widetilde O(d\sqrt{T})$ regret upper bound in \colorlinks{black}{\cite{wang2021dynamic}}. Notably, our algorithm requires weaker assumptions and is computationally more efficient than the one in \colorlinks{black}{\cite{wang2021dynamic}}. The key novelty of our proposed UCB algorithm lies in its design of the upper confidence bound for the true revenue function. In particular, our upper confidence bound takes a \textit{closed-form} by leveraging the structure of the revenue function $r(p,z_t,\theta^*)$ and the convexity of the link function $\psi(\cdot)$ of GLM~(i.e.\ \Cref{assum_feature}(e), which makes $\psi'(\cdot)$ a strictly increasing function).

We refer to the supplement for more detailed results of our proposed algorithms, theoretical guarantees, and numerical experiments that showcase the efficiency (in terms of both regret and computational time) of our algorithm compared to the algorithms in \colorlinks{black}{\cite{wang2021dynamic}}.
}

{\color{black}
}

\section{Contextual dynamic pricing under LDP constraints}\label{sec-ldp}

As discussed in the introduction, there is a growing trend for implementing privacy preserving practice in business operations where personal information is used. This motivates us to further study contextual dynamic pricing within the framework of LDP. We first introduce the notion of LDP for contextual dynamic pricing, then propose the ETC-LDP algorithm in \Cref{subsec:etc-ldp} and establish its regret upper bound in \Cref{subsec:etc-ldp-theory}. \textcolor{black}{In \Cref{subsec:mixed}, we further extend our study to contextual dynamic pricing under mixed privacy constraints, which bridges non-private and private dynamic pricing.} 

For $t \in [T]$, let $s_t=(z_t,y_t,p_t)$ be the raw data observed at time~$t$. We follow the formulation of LDP in dynamic pricing put forward by the seminal work  \cite{chen2022differential}, where the privacy-preserving policy consists of a two-part structure that
\begin{flalign}
    p_t\sim& A_t(\cdot|z_t,w_1,\cdots, w_{t-1}) = A_t(z_t,w_{<t}), \label{eq:pricing} \\
    w_t\sim& Q_t(\cdot|s_t,w_1,\cdots,w_{t-1}) = Q_t(s_t,w_{<t}), \label{eq:privacy} 
\end{flalign}
with $A_t$ being the pricing strategy and $Q_t$ being the privacy mechanism. Here, $w_t$ denotes the privatized statistic based on the raw data $s_t$ and we denote $w_{<t}=(w_1,\cdots,w_{t-1})$ with $w_{<0}:=\varnothing$. 

The pricing strategy $A_t(z_t,w_{<t})$ utilizes the true context $z_t$ of the $t$th customer and the privatized information $w_{<t}$ from previous time periods. Note that while the raw context $z_t$ is utilized for pricing purposes, it is never retained or stored by the seller. In particular, the personalized price $p_t$ is computed based on $z_t$ \textit{locally} on the customer’s device and the only data passed to the seller is the privatized~$w_t.$ The privatization mechanism $Q_t(s_t,w_{<t})$  is used to transform the raw data of the $t$th customer,~$s_t$, into a privatized statistic $w_t$. Note that $Q_t$ only utilizes the privatized statistics $w_{<t}$ and thus maintains privacy for previously encountered consumers. We remark that the seller does not store raw customer data $s_t$, but only the privatized statistic $w_t$ is retained for future pricing.  Following \cite{duchi2018minimax} and \cite{chen2022differential}, we have the below definition.

\vspace{-2mm}
\begin{defn}[$\epsilon$-LDP for dynamic pricing]\label{def:LDP}
For any $\epsilon> 0$, a policy $\pi=\{Q_t,A_t\}_{t=1}^T$ satisfies $\epsilon$-LDP if for any $t \in [T]$, $w_{<t}$ and $s_t,s_t'$, we have $Q_t(W|s_t,w_{<t})\leq e^{\epsilon}Q_t(W|s_t',w_{<t})$, for any measurable set $W$.
\end{defn}
\vspace{-2mm}

We remark that our theoretical results hold for any $\epsilon\in (0,b]$, where $b$ is an absolute constant. For simplicity, we assume in the following that $\epsilon\in (0,1]$ in all of our theoretical results regarding LDP.

\subsection{The ETC-LDP algorithm}\label{subsec:etc-ldp}

\Cref{algorithm:dp} presents the proposed ETC-LDP algorithm for privacy-preserving contextual dynamic pricing, with an $L_2$-ball mechanism subroutine~\citep{duchi2018minimax} in \Cref{algorithm:l2ball}. The high-level structure of ETC-LDP is the same as the ETC algorithm in \Cref{sec-ETC} for non-private dynamic pricing. In particular, it partitions the entire horizon into an exploration phase and an exploitation phase. The exploration phase is used for \textit{private} parameter estimation while in the exploitation phase, the seller uses a greedy approach to set optimal prices based on the estimated model parameter.

Unlike ETC where the seller can simply conduct a {single} maximum likelihood estimation \textit{after} collecting all raw data $\{s_t\}_{t\in[\tau]}$ in the exploration phase, an alternative estimator is needed for ETC-LDP to meet the LDP constraint. In particular, leveraging the \textit{sequential} nature of the LDP notion and dynamic pricing, we employ a stochastic gradient descent~(SGD) based privacy-preserving estimation procedure and integrate it into the exploration phase of ETC-LDP~(i.e.\ Steps a1-a3 in \Cref{algorithm:dp}).

\vspace{2mm}
\begin{algorithm}[ht]
\begin{algorithmic}
    \Indent \textbf{Input}: {Total rounds $T$, price interval $[l,u]$, parameter space $\Theta,$ privacy parameter $\epsilon$,
    \\ exploration length $\tau$, learning rate of SGD $\zeta_l$, gradient truncation parameter $C_g$ }
    \EndIndent
    
    \medskip
    \Indent a.~(Exploration) 
    \State Randomly sample an initial estimator $\widehat{\theta}_0$ from $\Theta$.
    \For {$t\in[\tau]$}
        \State a1.~(Experiment): uniformly choose $p_t$ from $[l,u]$, observe raw data $s_t=(z_t,y_t,p_t).$
        
        \State a2.~(Privatization): compute the truncated gradient
        \[
        g_t^{[C]} = g_t^{[C]}(\widehat{\theta}_{t-1};s_t) = \Pi_{\mathbb B_{C_g}}[g(\widehat\theta_{t-1};s_t)], \text{ with } g(\widehat\theta_{t-1};s_t) \mbox{ defined in \eqref{eq-truncated-gradient};}
        \]
        \State  run the $L_2$-ball subroutine on $g_t^{[C]}$ with parameter $C_g$ and $\epsilon$, and obtain $w_t=W(g_t^{[C]})$.
        
        \State a3.~(SGD): update estimation via $\widehat{\theta}_{t}=\Pi_{\Theta}[\widehat{\theta}_{t-1}+\eta_{t}w_t]$ with step size $\eta_{t}= (\zeta_l t)^{-1}$.
    \EndFor
    \EndIndent

    \medskip
    \Indent b.~(Exploitation)  
    \For {$t=\tau+1,\cdots,T$}        \State  offer the greedy price at  $p_t=p^*(\widehat{\theta}_{\tau}, z_t)=\argmax\limits_{p\in[l,u]} r(p,z_t,\widehat\theta_\tau)$ based on $\widehat\theta_\tau.$
    \EndFor
    \EndIndent
    \caption{The ETC-LDP algorithm for privacy-preserving dynamic pricing}
    \label{algorithm:dp}
\end{algorithmic}
\end{algorithm} 

\vspace{2mm}
\begin{algorithm}[ht]
\begin{algorithmic}
    \State \textbf{Input}: Raw data $g\in\mathbb{R}^{2d}$, upper bound $C_g$ such that $\|g\|\leq C_g$, privacy parameter $\epsilon$.
    \medskip
    \State 1. Generate $\widetilde{X}=(2 b-1) g$ where $b \sim \mathrm{Bernoulli}(1/2 + \|g\|/(2 C_g))$. 
    \State 2. Generate random vector $ \mathrm{W}(g)$ via
\begin{align*}
    \mathrm{W}(g) \sim \begin{cases}
\text{Unif}\big\{w \in \mathbb{R}^d: \, w^\top \widetilde{X}>0, \, \|w\|=C_gr_{\epsilon, d}\big\}, & \text {with probability } e^{\epsilon} /\left(1+e^{\epsilon}\right), \\
\text{Unif}\big\{w\in \mathbb{R}^d: \, w^\top\widetilde{X} \leq 0, \, \|w\|=C_gr_{\epsilon, d}\big\}, & \text {with probability } 1 /\left(1+e^{\epsilon}\right),
\end{cases}
\end{align*}
where $r_{\epsilon, d}=  \sqrt{\pi}\dfrac{e^{\epsilon}+1}{e^{\epsilon}-1} \dfrac{d \Gamma\left(d+\frac{1}{2}\right)}{\Gamma\left(d+1\right)}$ and $\Gamma(\cdot)$ is the Gamma function.

    \State \textbf{Output}: $w=\mathrm{W}(g)$.
    \caption{The $L_2$-ball mechanism \citep{duchi2018minimax}. $L_2$-ball ($g,C_g,\epsilon$)}
    \label{algorithm:l2ball}
\end{algorithmic}
\end{algorithm}

We now give a detailed explanation of Steps a1-a3. Step a1 creates a new data point $s_t=(z_t,p_t,y_t)$ based on price experiments. Define the log-likelihood function and its gradient based on $\{s_t\}$ as
\begin{align}\label{eq-truncated-gradient}
    l(\theta;s_t)=y_tx_t^{\top}\theta-\psi(x_t^{\top}\theta) \quad \mbox{and} \quad g(\theta;s_t)=[y_t-\psi'(x_t^{\top}\theta)]x_t,  
\end{align}
respectively. The update direction of the standard SGD will be $g(\widehat\theta_{t-1};s_t)$ given the current estimator~$\widehat\theta_{t-1}.$ However, to preserve privacy, we need to privatize $g(\widehat\theta_{t-1};s_t)$, done in Step a2 via the $L_2$-ball mechanism~\citep{duchi2018minimax}. In particular, Step a2 first computes a truncated gradient $g_t^{[C]}$ by projecting $g(\widehat\theta_{t-1};s_t)$ to the ball $\mathbb B_{C_g} \subset \mathbb R^{2d}$, as the $L_2$-ball mechanism requires a bounded input. It then runs the $L_2$-ball subroutine and obtains the privatized gradient $w_t$. Step a3 updates the estimator via $\widehat{\theta}_{t}=\Pi_{\Theta}[\widehat{\theta}_{t-1}+\eta_{t}w_t]$ with a step size $\eta_t=1/(\zeta_l t)$ that scales with $1/t.$ 

After exploration, we obtain the privacy-preserving estimator $\widehat\theta_\tau$, which is then used in the exploitation phase. In \Cref{prop-eps-LDP} of the supplement, we show that ETC-LDP is indeed $\epsilon$-LDP.

\subsection{Regret upper bound of ETC-LDP}\label{subsec:etc-ldp-theory}
We start with a mild modeling assumption.

\vspace{-2mm}
\begin{ass}\label{assum_ldp}
(a) The parameter space $\Theta$ is compact and satisfies that $\theta^* \in \Theta \subseteq \{\theta: \|\theta-\theta^*\|\leq c_\theta\}$, for some absolute constant $c_{\theta} > 0$. (b) There exist absolute constants $c_{\lambda 1}, c_{\lambda 2}>0$ such that $c_{\lambda 1}/d \leq \lambda_z \leq \bar\lambda_z \leq c_{\lambda 2} /d$, where recall $\lambda_z=\lambda_{\min}(\Sigma_z)$ and $\bar\lambda_z=\lambda_{\max}(\Sigma_z)$, with $\{z_t\}$ satisfying \Cref{assum_feature} and $\Sigma_z = \mathbb{E}(z_1 z_1^{\top})$.
\end{ass}
\vspace{-2mm}

Assumptions \ref{assum_feature}(c) and \ref{assum_ldp}(a) together ensure that $\{x^{\top}\theta: x\in \mathcal{X} \text{ and } \theta \in \Theta\}$ is bounded. Define $M_{\psi1}^*=\sup_{(x,\theta)\in \mathcal{X}\otimes \Theta}|\psi'(x^{\top}\theta)|\vee 1$ and $\kappa^* =  \inf_{(x,\theta)\in \mathcal{X}\otimes \Theta}\psi''(x^{\top}\theta)\wedge 1$.
By simple algebra and \eqref{lambda_min}, we have that the population log-likelihood function $l(\theta):=\mathbb E[l(\theta;s_t)]$ in the exploration phase of ETC-LDP is $\phi_z$-strongly concave over $\Theta$ with $\phi_z:=L_p\kappa^*\lambda_z.$ This serves as the foundation for the theoretical analysis of the SGD estimator $\widehat\theta_\tau$. Note that if we assume $\|\theta^*\|\leq c$ for some absolute constant $c$, \Cref{assum_ldp}(a) then holds by setting $\Theta=\{\theta: \|\theta\|\leq c\}$ and $c_\theta=2c$. 

\Cref{assum_ldp}(b) imposes the mild condition that $\lambda_z \asymp\bar\lambda_z\asymp 1/d$. This essentially requires $z_t$ to have a balanced covariance matrix, which is a standard assumption in the statistical regression literature. Note that the $1/d$ factor is due to the fact that $\|z_t\|\leq 1$ as in \Cref{assum_feature}(a). \Cref{lem:sgd} gives a high-probability upper bound on the estimation error for the private SGD estimator $\widehat\theta_\tau$ in the exploration phase of ETC-LDP, serving as the basis for bounding the regret of ETC-LDP.

\vspace{-2mm}
\begin{thm} \label{lem:sgd}
Suppose that Assumptions \ref{assum_feature} and \ref{assum_ldp}(a) hold. For any $\tau\geq  4$, set $C_g= 2\sqrt{1+u^2}(M_{\psi1}^*+\sigma\sqrt{\log(\tau)})$ and  $\zeta_l = c_l \phi_z$ with some $c_l\in (0,1]$, where recall $\phi_z=L_p\kappa^*\lambda_z$. We have that, for any $\epsilon>0$ and $0<\delta < \{e\vee \log (\tau)\}^{-1}$, with probability at least $1-\delta\log(\tau)$, it holds that $\|\widehat{\theta}_\tau-\theta^*\|^2 \leq   B_{L1} d\log(1/\delta)\log(\tau)/(c_l^{2} \lambda_z^{2}\epsilon^{2}\tau)$,
where $B_{L1}>0$ is an absolute constant that only depends on quantities in Assumptions \ref{assum_feature} and \ref{assum_ldp}(a).
\end{thm}
\vspace{-2mm}

We remark that our private SGD estimator for the GLM is adapted from a similar procedure in \cite{duchi2018minimax} without the truncation mechanism. Assuming the boundedness of the gradient, \cite{duchi2018minimax} provide an asymptotic estimation error bound (in our notation) such that $\mathbb E\{\|\widehat{\theta}_\tau-\theta^*\|^2\}\lesssim d/(\lambda_z^2\epsilon^2\tau)$,
and show its asymptotic minimax optimality. In contrast, we provide a non-asymptotic high-probability bound for the private SGD estimator via substantially different technical arguments adapted from the standard SGD literature~\citep{rakhlin2011making}. As we allow an unbounded gradient, our result requires a careful analysis of the additional bias-variance trade-off due to truncation. Our result matches the one in \cite{duchi2018minimax}, up to logarithmic terms, with $\log(1/\delta)$ stemming from the high probability nature of our bound and $\log(\tau)$ from the truncation bias.

For the case where the minimum eigenvalue $\lambda_z \asymp d^{-1}$, the MLE in \Cref{thm:theta_err}(ii) and the private SGD in \Cref{lem:sgd} achieve high-probability upper bounds, up to logarithmic factors, $\|\widehat{\theta}^{\mathrm{MLE}}_{\tau} - \theta^*\|^2 \lesssim d^2/\tau$ and $\|\widehat{\theta}^{\mathrm{SGD}}_{\tau} - \theta^*\|^2 \lesssim d^3/(\epsilon^2 \tau)$.  This suggests that the LDP constraint shrinks the effective sample size from $\tau$ to $\tau \epsilon^2/d$, a phenomenon repeatedly observed in the LDP literature \citep[e.g.][]{li2023robustness}.

\vspace{-2mm}
\begin{rem}[Learning rate $\zeta_l$]\label{rem:learning_rate}
    A key tuning parameter of the private SGD estimator is the learning rate $\zeta_l$, which regulates the step size $\eta_t=1/(\zeta_l t).$ Inline with the standard SGD literature~\citep[e.g.][]{rakhlin2011making}, we require $\zeta_l=c_l \phi_z$ for some $c_l\in(0,1]$ to achieve the optimal error bound. The term $\phi_z=L_p\kappa^*\lambda_z$ is the strong-concavity parameter of the population log-likelihood function and in general unknown. However, if the order of $\lambda_z$ is known, e.g.\ $\lambda_z\asymp 1/d$ as in \Cref{assum_ldp}(b), one can set $c_l=1/\loglog(\tau)$ and replace $\phi_z$ by its order $1/d$ in $\zeta_l.$ It is easy to see that this $\zeta_l$ will be valid for all sufficiently large $\tau$ and the estimation error of SGD will only be inflated by a poly-logarithmic factor. Therefore, for simplicity, we set $\zeta_l=\phi_z$ in the regret analysis of ETC-LDP in \Cref{thm:LDP-ETC}.
\end{rem}
\vspace{-2mm}

\vspace{-2mm}
\begin{thm}\label{thm:LDP-ETC}
Suppose that Assumptions \ref{assum_feature}, \ref{assum:etc} and \ref{assum_ldp} hold. For any $\epsilon>0$ and $0<\delta<\{e\vee \log (T)\}^{-1}$, set the learning rate $\zeta_l=\phi_z$, the exploration length $\tau=d\sqrt{T\log(T) \log (1/\delta)}/\epsilon$ and the truncation parameter $C_g=2 \sqrt{1+u^2}(M_{\psi1}^*+\sigma\sqrt{\log(\tau)})$. For any $T$ satisfying
\begin{equation}\label{eq:minTSGD}
   T\geq B_{L2} \cdot \{\lambda_z^{-4} \epsilon^{-2}\log (T)\log(1/\delta)\},
\end{equation}
we have that, with probability at least $1-\delta\log(T)$, the regret of ETC-LDP is upper bounded by $ R_T\leq B_{L3} d\sqrt{\log (T)\log (1/\delta)T}/\epsilon$, where $B_{L2}, B_{L3}>0$ are absolute constants only depending on quantities in  Assumptions \ref{assum_feature}, \ref{assum:etc} and \ref{assum_ldp}. In particular, setting $\delta=1/T$, we have that $\mathbb{E}(R_T)\lesssim {\epsilon}^{-1}{d\sqrt{T} \log (T)}.$
\end{thm}
\vspace{-2mm}

\Cref{thm:LDP-ETC} suggests that ETC-LDP achieves a regret of order $\widetilde{O}(d\sqrt{T}/\epsilon)$. Compared to ETC, which achieves a regret of order $\widetilde{O}(\sqrt{dT})$, we see that the LDP constraint inflates the regret by $\sqrt{d}/\epsilon$, matching the phenomenon for the private SGD estimator when compared with MLE. 

In Theorems \ref{lem:sgd} and \ref{thm:LDP-ETC}, we require a diverging $C_g$ to control the bias due to truncation of the gradient. For GLMs where $y_t$ is bounded, $C_g$ can be made tighter as a constant. For example, we can set $C_g=\sqrt{1+u^2}$ for the logistic regression. This would improve the estimation error (and hence regret) by a poly-logarithmic factor. Same with supCB and ETC, we can further extend ETC-LDP to the case where $T$ is unknown with the doubling trick. We omit the details to conserve space.

\vspace{-2mm}
{\color{black}
\begin{rem}[Dynamic pricing under $(\epsilon,\Delta)$-LDP]\label{rem:approx_LDP}
A relaxation of $\epsilon$-LDP is $(\epsilon,\Delta)$-LDP, where one further allows a $\Delta\in (0,1]$ probability of privacy leakage (see \Cref{def:LDP_approx} for more details). In \Cref{subsec:approx_ldp_supplement} of the supplement, we propose an ETC-LDP-Approx algorithm for dynamic pricing under $(\epsilon,\Delta)$-LDP and show that it achieves the same order of regret upper bound as ETC-LDP under $\epsilon$-LDP (up to a $O(\sqrt{\log(1.25/\Delta)})$ term), a common phenomenon in the LDP literature~\colorlinks{black}{\citep[][]{duchi2018minimax}}. The key modification is to replace the $L_2$-ball privacy mechanism of ETC-LDP with a Gaussian mechanism of a carefully calibrated variance. Technically, we extend the uniform concentration inequality for martingale difference sequences in \colorlinks{black}{\cite{rakhlin2011making}} from bounded random variables to random variables with certain moment conditions (e.g.\ sub-Gaussian).  This inequality is of independent interest (see \Cref{lem_freedman}). Numerical experiments are further conducted to show the similar performance by ETC-LDP and ETC-LDP-Approx in finite-sample. Thus, in practice, we recommend ETC-LDP as it offers stronger privacy guarantees. We refer to \Cref{subsec:approx_ldp_supplement} for more details.
\end{rem}
}
\vspace{-2mm}

{\color{black}
\subsection{Dynamic pricing under mixed privacy constraints}\label{subsec:mixed}

\textcolor{black}{The ETC-LDP algorithm and its theoretical guarantees, as studied in Sections~\ref{subsec:etc-ldp} and \ref{subsec:etc-ldp-theory}, are designed for the scenario where all users have the same level of privacy guarantees. However, in reality, consumers may exhibit heterogeneous preferences towards the trade-off between privacy and personalization, or simply different attitudes towards privacy. To accommodate this, firms often offer consumers some flexibility regarding their privacy levels, such as cookies consent banners, mobile apps permissions, social media privacy settings, search engine privacy modes and e-commerce personalization preferences, where users typically are given the choices to opt in or out of certain privacy protection mechanism. This flexibility and heterogeneity present both opportunities and challenges in statistical privacy. For example, recent research has explored leveraging public data in privacy analysis to improve utility and developing personalized learning algorithms that adapt to individual privacy preferences; see \colorlinks{black}{\cite{cummings2024advancing}} and references therein.} 



This motivate us to further extend our study to dynamic pricing under \textit{mixed} privacy constraints. In particular, we consider the scenario where arriving consumers exhibit two levels of privacy preferences: non-private, where they are willing to share raw information $s_t=(z_t,y_t,p_t)$ with the firm, and private, where they require an $\epsilon$-LDP protection of $s_t$. More specifically, we assume that the consumer's privacy preference is encoded by a random variable $m_t\in \{0,1\}$ where $m_t=0$ indicates a private consumer and $m_t=1$ indicates a non-private consumer, and there is an \textit{unknown} mixture probability $p^*\in [0,1]$ such that each arriving consumer has $p^*$ probability of being non-private and $1-p^*$ probability of being private with $\epsilon$-LDP. In addition, $\{m_t\}$ is an i.i.d.\ process across consumers. Same as before, we assume $\epsilon\in (0,1]$ for notational simplicity and remark that all results hold for any $\epsilon\in (0,b]$, where $b$ is an absolute constant.  


Intuitively, the presence of non-private consumers (regulated by $p^*$) can help improve model estimation for $\theta^*$ and thus attain a lower regret compared to dynamic pricing under \textit{pure} $\epsilon$-LDP constraints (i.e.\ $p^*\equiv 0$) in \Cref{subsec:etc-ldp}. To achieve this, we propose the ETC-LDP-Mixed algorithm in \Cref{algorithm:dp_mixed}, which is adapted from ETC-LDP in \Cref{algorithm:dp} with several important modifications to effectively utilize the information from both private and non-private consumers.

First, thanks to the information from the non-private consumers, we are able to explore less than ETC-LDP. In particular, given $p^*\in(0,1)$, the optimal exploration length now takes the form
\begin{align}\label{eq:tau_p_mixed}
    \tau(p^*):=\frac{\log (T)\sqrt{Td}}{\sqrt{p^*+(1-p^*)\epsilon^2/d}}.
\end{align}
Recall as discussed following \Cref{lem:sgd}, $\epsilon^2/d$ can be viewed as the effective sample size of an $\epsilon$-LDP private consumer. Since $\epsilon^2/d\leq 1$, $\tau(p^*)$ is a decreasing function of $p^*$. That is, having more non-private consumers (thus higher quality data) leads to a shorter exploration length. Note that we have $\tau(1)=\sqrt{dT}\log (T)$ for $p^*=1$ and $\tau(0)=d\sqrt{T}\log (T)/\epsilon$ for $p^*=0$, which recover the optimal order of exploration length for ETC in \Cref{thm:ETC} under the non-private case and for ETC-LDP in \Cref{thm:LDP-ETC} under the $\epsilon$-LDP private case, respectively. Therefore, $\tau(p^*)$ is intuitive and bridges private and non-private dynamic pricing seamlessly.

In practice, $p^*$ may be unknown, which further motivates the two-stage exploration design for ETC-LDP-Mixed. In Stage I exploration, which lasts for $\tau_1=\sqrt{dT}$ rounds, we estimate $\widehat p^*$ from the consumer privacy preference data $\{m_t\}_{t=1}^{\tau_1}$ while conducting parameter estimation of $\theta^*$ via private SGD. The Stage II exploration is then run for $\tau(\widehat p^*)-\tau_1$ rounds, where we estimate $\theta^*$ via both a private SGD and a non-private SGD to effectively utilize the information from both private and non-private consumers.

In particular, due to the $\epsilon$-LDP constraint, the private SGD needs to be updated right after each private consumer arrives, while we can conduct the non-private SGD at the end of Stage II exploration, as the non-private data can be shared with the firm and stored in its database $\mathcal{S}$. Same as the logic of $\tau(p^*)$ in \eqref{eq:tau_p_mixed}, we adjust the $k$th step size of the non-private SGD (step B3 in \Cref{algorithm:dp_mixed}) as
\begin{align}\label{eq:step_size_nonprivate_SGD}
    \eta_k=\zeta_l^{-1}\{(\tau(\widehat p^*)-|\mathcal{S}|)\epsilon^2/d+k\}^{-1},
\end{align}
where $\zeta_l$ is the learning rate and $(\tau(\widehat p^*)-|\mathcal{S}|)\epsilon^2/d$ accounts for the effective sample size of the private SGD estimate based on the $\tau(\widehat p^*)-|\mathcal{S}|$ private consumers in the exploration phase.

\vspace{1mm}
\begin{algorithm}[ht]
\color{black}
\begin{algorithmic}
    \Indent \textbf{Input}: {Total rounds $T$, price interval $[l,u]$, parameter space $\Theta,$ privacy parameter $\epsilon$,
    \\ learning rate of SGD $\zeta_l$, gradient truncation parameter $C_g$
    \EndIndent

    \Indent \textbf{Initialization}: Randomly sample an initial estimator $\widehat{\theta}_0$ from $\Theta$. Set $\mathcal{S}=\varnothing$}. Set $\tau_1=\sqrt{dT}$.
    \EndIndent
    
    \medskip
    \Indent A.~(Stage I Exploration) 
    \For {$t\in \{1,2,\cdots,\tau_1\}$}
        \State A1.\ uniformly choose $p_t$ from $[l,u]$, observe privacy preference $m_t$ and raw data $s_t=(z_t,y_t,p_t)$.
        \If {$m_t=0$ (i.e.\ private consumer)} 
            \State A2.\ run a2-a3 in Alg.\ \ref{algorithm:dp} with step size $\eta_t=1/(\zeta_l(t-|\mathcal{S}|)) $ to update $\widehat\theta_t$. \Comment{(Private SGD)}
        \Else{ (i.e.\ $m_t=1$, non-private consumer)} 
            \State A3.\ set $\mathcal{S}=\mathcal{S}\cup \{s_t\}$ and set $\widehat\theta_t=\widehat\theta_{t-1}$. \Comment{(Store non-private data)}
        \EndIf
    \EndFor
    \EndIndent

    \medskip
    \Indent B.~(Stage II Exploration)
    \State B1.\ set $\widehat p^*=|\mathcal{S}|/\tau_1$ and $\tau_2=\tau(\widehat p^*)$ with $\tau(p^*)$ defined in \eqref{eq:tau_p_mixed}. \Comment{(Estimated explr.\ length)}
    \For {$t\in \{\tau_1+1,\cdots, \tau_2\}$}
        \State B2.\ run step A1-A3 for the $t$th consumer.
    \EndFor
    \For {$k\in \{1,2,\cdots,|\mathcal{S}|\}$} \Comment{(Non-private SGD on $\mathcal{S}$)}
    \State B3.\ set step size $\eta_k=\zeta_l^{-1}((\tau_2-|\mathcal{S}|)\epsilon^2/d+k)^{-1}$ and update\\
    \hspace{1.8cm}
        $\widehat{\theta}_{\tau_2}=\Pi_{\Theta}[\widehat{\theta}_{\tau_2}+\eta_{k}g(\widehat\theta_{\tau_2}; s_k)]$ with $s_k$ being the $k$th data in $\mathcal{S}$ and $g(\theta;s)$ defined in \eqref{eq-truncated-gradient}.
    \EndFor
    \EndIndent
    
    \medskip
    \Indent C.~(Exploitation)  
    \For {$t\in \{\tau_2+1,\cdots,T\}$}
        \State  offer the greedy price at  $p_t=p^*(\widehat{\theta}_{\tau_2}, z_t)=\argmax\limits_{p\in[l,u]} r(p,z_t,\widehat\theta_{\tau_2})$ based on $\widehat\theta_{\tau_2}.$
    \EndFor
    \EndIndent
    \caption{\color{black} The ETC-LDP-Mixed algorithm for dynamic pricing under mixed privacy constraints}
    \label{algorithm:dp_mixed}
\end{algorithmic}
\end{algorithm} 


\vspace{-2mm}
\begin{thm}\label{thm:LDP-ETC-Mixed}
Suppose the conditions in \Cref{thm:LDP-ETC} hold. For any $\epsilon>0$ and $0<\delta<\{e\vee \log(T)\}^{-1}$, set the learning rate $\zeta_l$ and truncation parameter $C_g$ same as in \Cref{thm:LDP-ETC}. For any $T$ satisfying 
$T\geq B_{M2} \cdot \{(\lambda_z^{-4} \epsilon^{-2}\log (T)\log(1/\delta))\vee (d\epsilon^{-4}\log^2(T))\},$ with probability at least $1-\delta\log(T)$, the regret of ETC-LDP-Mixed is upper bounded by $ R_T\leq B_{M3} \sqrt{\log (T)\log (1/\delta)dT}/\sqrt{p^*+(1-p^*)\epsilon^2/d}$, where $B_{M2}, B_{M3}>0$ are absolute constants only depending on quantities in  Assumptions \ref{assum_feature}, \ref{assum:etc} and \ref{assum_ldp}. 
\end{thm}
\vspace{-2mm}

Note that \Cref{thm:LDP-ETC-Mixed} recovers and thus bridges the rate of regret upper bounds for non-private dynamic pricing in \Cref{thm:ETC} and $\epsilon$-LDP private dynamic pricing in \Cref{thm:LDP-ETC}, by setting $p^*=1$ and $p^*=0$, respectively. To our knowledge, this is the first time such result is seen in the literature. In particular, when $p^*\in (0,1]$, ETC-LDP-Mixed can effectively leverage the non-private consumers and therefore achieve an improved privacy-utility trade-off. One key ingredient of \Cref{thm:LDP-ETC-Mixed} is a novel non-asymptotic bound on the estimation error of $\theta^*$ that combines a private SGD and a non-private SGD, which is detailed in \Cref{lem:sgdmix} of the supplement and can be of independent interest. Extensive numerical experiments in \Cref{sec-numerical} suggest that notable improvement in regret can be observed even with a small $p^*$ such as $0.05$ or $0.1$.
}

\section{Minimax lower bounds on the regret}\label{sec-optimality}

In this section, we verify the optimality of our proposed supCB, ETC and ETC-LDP algorithms.

\vspace{-2mm}
\begin{thm}[Non-private lower bound]\label{thm_lb}
    Let $\mathcal{P}$ be the collection of distributions satisfying the GLM model in \eqref{eq:GLM}, Assumptions \ref{assum_feature} and \ref{assum:etc}, and $\Pi$ be the set of all non-anticipating policies. There exists an absolute constant  $\widetilde{c} > 0$ such that $\inf_{\pi\in\Pi} \sup_{\bm{v}\in\mathcal{P}} \mathbb{E}_{\bm{v}}^{\pi}R_T\geq \widetilde{c }\sqrt{dT}$, for all $T\geq 4d^2\log (T)$.
\end{thm}
\vspace{-2mm}

\Cref{thm_lb} shows that both supCB and ETC achieve the minimax optimal regret, up to logarithmic factors. The sub-linear dependence on the dimensionality $d$ is the first time seen in the contextual dynamic pricing literature and extends the $\Omega(\sqrt{T})$ lower bound for contextual-free dynamic pricing in \cite{broder2012dynamic}. Unlike \cite{broder2012dynamic}, which connects the regret with the error of a hypothesis testing problem between two highly indistinguishable demand models and applies Le Cam's lemma, our technical arguments are substantially different and more involved. In particular, to capture the effect of the dimension $d$, we construct $2^d$ demand models that can be grouped into pairs, where each pair is highly indistinguishable along one dimension of the context vector. We connect the regret with the error of a multiple-classification problem and establish the lower bound based on Assouad's method. We remark that due to adaptivity of the data~(i.e.\ $\{y_t,z_t,p_t\}$ generated by the pricing policy are not i.i.d.), our arguments are more involved than standard applications of Assouad's method in lower bounds for statistical estimation problems.

\vspace{-2mm}
\begin{thm}[Private lower bound]\label{thm:LDP_lb}
    Let $\mathcal{P}$ be the collection of distributions satisfying the GLM model in \eqref{eq:GLM}, Assumptions \ref{assum_feature}, \ref{assum:etc} and \ref{assum_ldp}, and $\Pi(\epsilon)$ be the set of all non-anticipating $\epsilon$-LDP policies as defined in \Cref{def:LDP}. There exists an absolute constant $\widetilde{c} > 0$ such that $\inf_{\pi\in\Pi(\epsilon)} \sup_{\bm{v}\in\mathcal{P}}\mathbb{E}_{\bm{v}}^{\pi}R_T\geq \widetilde{c }d{\sqrt{T}}/{\epsilon}$, for 
    all $T\geq  \widetilde c' \cdot d^4\epsilon^{-2}\log(T)$, where $\widetilde c'>0$  is an absolute constant. 
\end{thm}
\vspace{-2mm}

\Cref{thm:LDP_lb} shows that ETC-LDP achieves the minimax optimal regret, up to logarithmic factors. Compared with \Cref{thm_lb}, the regret lower bound under LDP inflates by a factor of $\sqrt{d}/\epsilon$, and the condition on $T$ is more stringent, echoing observations in the privacy literature \citep[e.g.][]{cai2021cost}.   The proof of \Cref{thm:LDP_lb} shares some similarity with that of \Cref{thm_lb}, with the Assouad's method replaced by a private version tailored for LDP. While this strategy has been used in the LDP literature for establishing lower bounds for private estimation with independent data~\citep[e.g.][]{duchi2018minimax,acharya2024unified}, our setting is much more involved and thus requires more delicate analysis. {\color{black} In particular, in our lower bound setting, we allow the data (especially price) to be adaptively generated by \textit{any} pricing policy $\pi$, making the data inherently dependent. We provide a sketch of proof below and refer to \Cref{subsec:roadmap_lb} of the supplement for a detailed roadmap of proof and technical contributions.

\noindent \textbf{Sketch of proof for \Cref{thm:LDP_lb}.}  The proof is based on the construction of $2^d$ highly indistinguishable revenue functions, where the details of the construction can be found in \Cref{sec-lb-construction}.  Given this setup, the remainder of the proof consists of two key components.
\begin{itemize}
    \item [S1.] (Lower bounding the regret by classification errors).  We first apply a general reduction to relate the regret bound to the classification error in certain classification problems. In particular, inspired by \colorlinks{black}{\cite{chen2022differential}}, we show that for any non-anticipating $\epsilon$-LDP policy $\pi$, its regret can be lower bounded by the classification errors of $d$ classification problems subject to $2\epsilon$-LDP constraints. This reduction is established in \Cref{lem_class}.
    
    \item [S2.] (A new sharp lower bound on classification errors) At a high level, we further lower bound the classification errors in terms of the Hellinger distance between pairs of distributions, with the Hellinger distance itself being upper bounded using the information contraction bound techniques developed in \colorlinks{black}{\cite{acharya2024unified}}. Importantly, we make non-trivial modifications to adapt the technical tools in \colorlinks{black}{\cite{acharya2024unified}}, which is designed for statistical estimation problems under LDP with i.i.d.\ data, to our setting of dynamic pricing under LDP, where the data can be adaptively collected by \textit{any} non-anticipating pricing policy $\pi$ and thus exhibit inherent dependencies. To address this, we develop a novel conditioning argument via martingale decomposition when bounding the Hellinger distance between joint distributions, making the intrinsic dependence of the data tractable for analysis. This step is established in Lemma  \ref{lem_assouad2}. 
\end{itemize}

We remark that our lower bound analysis in \Cref{thm:LDP_lb} is new and first time seen in the dynamic pricing literature~(in particular Lemmas \ref{lem_assouad2} and \ref{lem_varphi}). In fact, existing lower bound construction in contextual bandit~\colorlinks{black}{\citep{han2021generalized}} and contextual (non-parametric) dynamic pricing~\colorlinks{black}{\citep{chen2022differential}} under LDP constraints mainly utilize the well-known private Le Cam Lemma and private Assouad's Lemma developed in \colorlinks{black}{\cite{duchi2018minimax}}, combined with KL divergence bounds. However, we find this strategy can only provide a $\Omega(\sqrt{dT}/\epsilon)$ lower bound for our setting, which does not fully characterize the cost of privacy in dimension $d.$ Instead, our proof strategy is to leverage the Hellinger distance, which is sharper than the KL divergence, and adapt the information contraction bound in \colorlinks{black}{\cite{acharya2024unified}} to derive the optimal $\Omega(d\sqrt{T}/\epsilon)$ bound. This proof strategy is the first time seen in the literature for online learning under LDP constraints and thus is of independent interest.


}

\section{Numerical analysis}\label{sec-numerical}

In this section, we conduct extensive numerical experiments to investigate the performance of the proposed algorithms and further validate our theoretical findings. \Cref{subsec:simu_setting} discusses the general simulation settings, followed by synthetic and real-world data analysis in Sections~\ref{subsec:num_synthetic} and \ref{subsec:autoloan}. 

\subsection{General simulation settings}\label{subsec:simu_setting}

\textbf{Data generating process in \Cref{subsec:num_synthetic}.} We consider a logistic regression setting, where given $(z_t,p_t)$, the consumer demand $y_t$ is a Bernoulli random variable such that
\begin{align}\label{eq:logistic_pricing}
    \mathbb E(y_t|z_t,p_t)=\psi'\left(z_t^\top \alpha^*-(z_t^\top\beta^*)p_t\right),
\end{align}
with $\psi'(\cdot)$ being the logistic function. The feature vector $\{z_t=(z_{t,1},\cdots,z_{t,d})^{\top}\}_t$ are i.i.d.\ random vectors across time. We consider two simulation scenarios. 

\textbf{(S1)}.\ We set $\alpha^*=1.6 \times \mathbf{1}_d/\sqrt{d}$ and $\beta^*=\mathbf{1}_d/\sqrt{d}$. For each feature vector $z_t$, we set $\{z_{t,i}\}_{i=1}^d$ as i.i.d.\ uniform$(1/\sqrt{d}, 2/\sqrt{d})$ random variables. \textbf{(S2)}.\ We set $\alpha^*=\mathbf{1}_d$ and $\beta^*=\mathbf{1}_d$. Each feature vector $z_t$ is uniformly drawn from $\{e_1,\ldots, e_d\}$, which consists of the $d$ standard orthonormal basis in $\mathbb R^d.$ For (S1), the optimal price is always in $[0.91,2.68]$ for all $d$, while for (S2), the optimal price is always~$1.57$. Therefore, we set the price range $[l,u]$ to be $[0,3]$ for all algorithms in \Cref{subsec:num_synthetic}.

\medskip
\noindent\textbf{Implementation details}. For a known horizon $T$, the algorithms are implemented as follows.
\begin{itemize}
    \vspace{-2mm}
    \item For ETC, following Theorem \ref{thm:ETC}, we set the length of price experiments as $\tau=\bigl\lceil\sqrt{dT\log(T)}\,\bigr\rceil$.
    \vspace{-2mm}
    \item For ETC-LDP, following \Cref{thm:LDP-ETC}, we set the length of price experiments as $\tau=\bigl\lceil2d\sqrt{T}\log(T)/\epsilon \bigr\rceil$ and the learning rate of SGD as $\zeta_l=L_p/d$.  For the $L_2$-ball mechanism, it is easy to derive that $C_g=\max_t\|z_t\|\cdot \sqrt{1+u^2}$ under the logistic regression model in \eqref{eq:logistic_pricing}. \textcolor{black}{For ETC-LDP-Mixed, following \Cref{thm:LDP-ETC-Mixed}, we further set $\tau(\widehat p^*)=\bigl\lceil2\sqrt{dT}\log(T)/\sqrt{\widehat p^*+(1-\widehat p^*)(\epsilon^2/d)} \bigr\rceil$}.
\end{itemize}
For the case where $T$ is unknown, all algorithms are implemented with the standard doubling trick.

{\color{black} 
\noindent \textbf{Additional simulations in the supplement.} To conserve space, numerical experiments on the performance of the proposed UCB algorithm for dynamic pricing under adversarial contexts and the performance of the proposed ETD-LDP-Approx algorithm for dynamic pricing under $(\epsilon,\Delta)$-LDP are presented in Sections~\ref{subsec:approx_ldp_supplement} and \ref{sec:adversarial_supplement} of the supplement, respectively.
}

\subsection{Experiments on synthetic data} \label{subsec:num_synthetic}

\subsubsection[]{ETC and supCB with known $T$}\label{subsubsec:ETC}
In this section, we examine the performance of ETC and supCB when the horizon $T$ is known. We set the horizon $T\in \{(1,2,3,4,5,6,7)^2\times 10000\}$ and the dimension $d\in \{(1,2,3,4,5)^2\}= \{1,4,9,16,25\}$. Given $(T,d)$, we conduct 500 independent experiments and record the realized regrets $\{R_{T,d}^{(i)}\}_{i=1}^{500}$. As discussed in \Cref{subsec-supcb-brief}, ETC achieves much better statistical and computational efficiency in finite-sample than supCB. We thus refer to \Cref{subsubsec:SupCB_ETC} of the supplement for results of supCB. 

\textbf{Reported metrics}: We report the sample mean regret $\overline{R}_{T,d}=\sum_{i=1}^{500}R_{T,d}^{(i)}/500$. Denote $S_{T,d}$ as the sample standard deviation of $\{R_{T,d}^{(i)}\}_{i=1}^{500}$, we further construct $I_{T,d}=[\overline{R}_{T,d}-3S_{T,d}/\sqrt{500}, ~\overline{R}_{T,d}+3S_{T,d}/\sqrt{500}]$, which is the 99\% confidence interval of the expected regret $\mathbb E(R_{T,d})$ based on normal approximation. Based on $\bigl\{\log (\overline{R}_{T,d}) \bigr\}$ across all $(d,T)$, we further fit a linear regression
\begin{align}\label{eq:linear_reg}
    \log (\overline{R}_{T,d}) = \beta_0+\beta_d \log (d) + \beta_T \log (T) + \text{offset},
\end{align}
and examine the coefficients $(\beta_d,\beta_T).$ Here, the offset term is used to remove the impact of the poly-logarithmic factor that appears in the expected regret $\mathbb E(R_{T,d})$.

\Cref{fig:ETC1} reports the performance of ETC under (S1). Specifically, \Cref{fig:ETC1}(left) reports the mean regret $\overline{R}_{T,d}$ (and the confidence interval $I_{T,d}$) of ETC at different $(T,d)$ under (S1). As can be seen, given a fixed dimension $d$, the regret scales sublinearly with respect to $T$, and given a fixed horizon $T$, the regret increases with the dimension $d.$ \Cref{fig:ETC1}(middle) plots $\log (\overline{R}_{T,d})$ vs.\ $\log(T)$ for each fixed~$d$, and further plots the estimated regression lines based on \eqref{eq:linear_reg} with an offset term $0.5\loglog(T)$. In particular, the regression coefficients are estimated as $\beta_d=0.48$ and $\beta_T=0.49$, providing some numerical evidence for Theorem \ref{thm:ETC}. For illustration, \Cref{fig:ETC1}(right) further gives the boxplot of $\{R_{T,d}^{(i)}\}_{i=1}^{500}$ for each $T\in \{(1,2,3,4,5,6,7)^2\times 10000\}$ with a fixed $d=9.$ The performance of ETC under (S2) is reported in \Cref{fig:ETC2} of the supplement, where similar phenomenon is observed.

\begin{figure}[ht]
\vspace{-5mm}
\hspace*{-6mm}
    \begin{subfigure}{0.32\textwidth}
	\includegraphics[angle=270, width=1.1\textwidth]{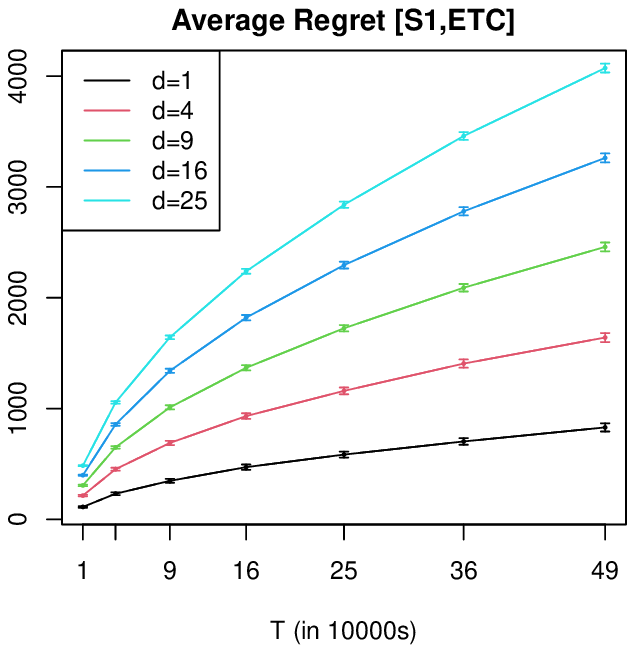}
	\vspace{-0.5cm}
    \end{subfigure}
    ~
    \begin{subfigure}{0.32\textwidth}
	\includegraphics[angle=270, width=1.1\textwidth]{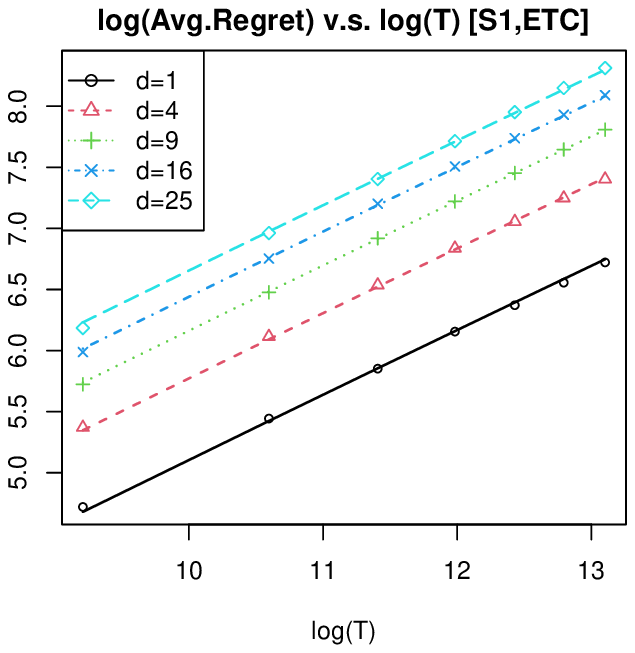}
	\vspace{-0.5cm}
    \end{subfigure}
    ~
    \begin{subfigure}{0.32\textwidth}
	\includegraphics[angle=270, width=1.1\textwidth]{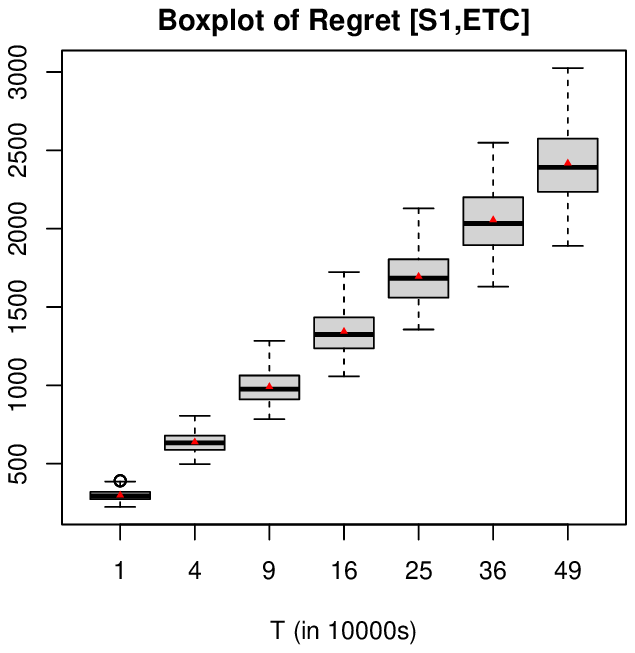}
	\vspace{-0.5cm}
    \end{subfigure}
    \caption{ETC under (S1). [Left]:  Mean regret (with C.I.) under different $(d,T)$. [Middle]: Mean regret (in log scale) with fitted regression lines. [Right]: Boxplot of regrets at different $T$ ($d=9$).}
    \label{fig:ETC1}
\end{figure}

\subsubsection[]{ETC-Doubling, MLE-Cycle \textcolor{black}{and Semi-Myopic} with unknown $T$}\label{subsubsec:ETC_MLECycle}
In this section, we examine ETC-Doubling and \textcolor{black}{compare with both the original and modified MLE-Cycle and Semi-Myopic algorithms detailed in Sections \ref{sec:mle_cycle} and \ref{sec:semi-myopic} of the supplement.}

We set $T=49\times 10000$ and $d\in \{1,4,9,16,25\}$. Note that $T$ is unknown to the pricing algorithm and it only means that we stop at $T$. Given $(T,d)$, for each algorithm, we conduct 500 independent experiments and record the realized sample paths of regrets $\{R_{t,d}^{(i)}, t\in [T]\}_{i=1}^{500}$. For each $t\in[T]$, we compute its sample mean regret $\overline{R}_{t,d}$ and confidence interval $I_{t,d}$ as described in \Cref{subsubsec:ETC}.


Here, for the $q$th episode in ETC-Doubling of length $T_q=2^q$, we set the number of price experiments as $\tau_q=\min\big\{T_q, (\sqrt{2}-1) \sqrt{dT_q\log (T_q)} \big\}$. The $\sqrt{2}-1$ factor is introduced such that the total price experiments conducted in ETD-Doubling is approximately $\sqrt{dT\log (T)}$, which matches the number of price experiments conducted by ETC when $T$ is known. We refer to Sections \ref{sec:mle_cycle} and \ref{sec:semi-myopic} of the supplement for implementation details of MLE-Cycle and Semi-Myopic.

\Cref{fig:unknownT1}(left) plots the mean regret path $\{\overline{R}_{t,d}, t\in[T]\}$ and confidence bound $\{I_{t,d}, t\in [T]\}$ of ETC-Doubling across $d\in \{1,4,9,16,25\}$ under (S1). The episodic nature of ETC-Doubling can be clearly seen, as the regret accumulates sharply during the exploration stage of each episode and then switches to a mild increase during exploitation. Comparing \Cref{fig:unknownT1}(left) to \Cref{fig:ETC1}(left), we can see ETC-Doubling incurs larger regret without the knowledge of the horizon $T$, which is intuitive. 

\Cref{fig:unknownT1}(middle and right) report the performance of \textit{modified} MLE-Cycle and \textcolor{black}{\textit{modified} Semi-Myopic}, which give a smooth regret path due to the design. 
\textcolor{black}{\Cref{fig:unknownT1_box}(left) further gives the boxplot of the final regret $\{R_{T,d}^{(i)}\}_{i=1}^{500}$ incurred by ETC-Doubling and modified MLE-Cycle and Semi-Myopic. Overall, the three achieve similar performances with ETC-Doubling having the smallest regret, especially for large~$d$. Moreover, in Figure \ref{fig:unknownT1_box}(middle and right), we compare the boxplot of the final regret $\{R_{T,d}^{(i)}\}_{i=1}^{500}$ of the \textit{modified} MLE-Cycle and Semi-Myopic, and their \textit{original} versions, which under-explore and thus give much higher regret for large $d$. This further illustrates the importance of our findings.} The result under (S2) are reported in Figures \ref{fig:unknownT2} and \ref{fig:unknownT2_box} of the supplement, where similar findings are observed.

\begin{figure}[ht]
\hspace*{-6mm}
    \begin{subfigure}{0.32\textwidth}
	\includegraphics[angle=0, width=1.1\textwidth]{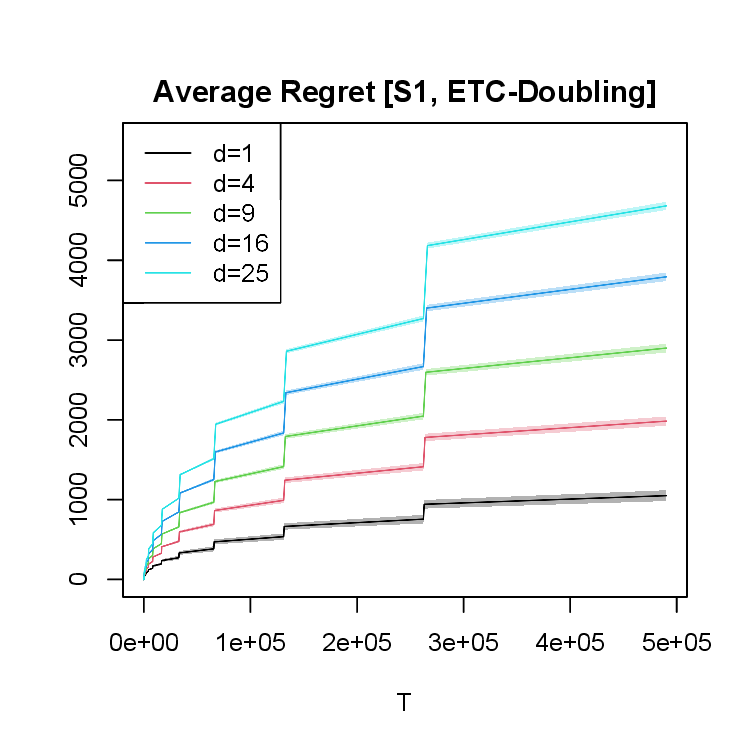}
	\vspace{-0.5cm}
    \end{subfigure}
    ~
    \begin{subfigure}{0.32\textwidth}
	\includegraphics[angle=0, width=1.1\textwidth]{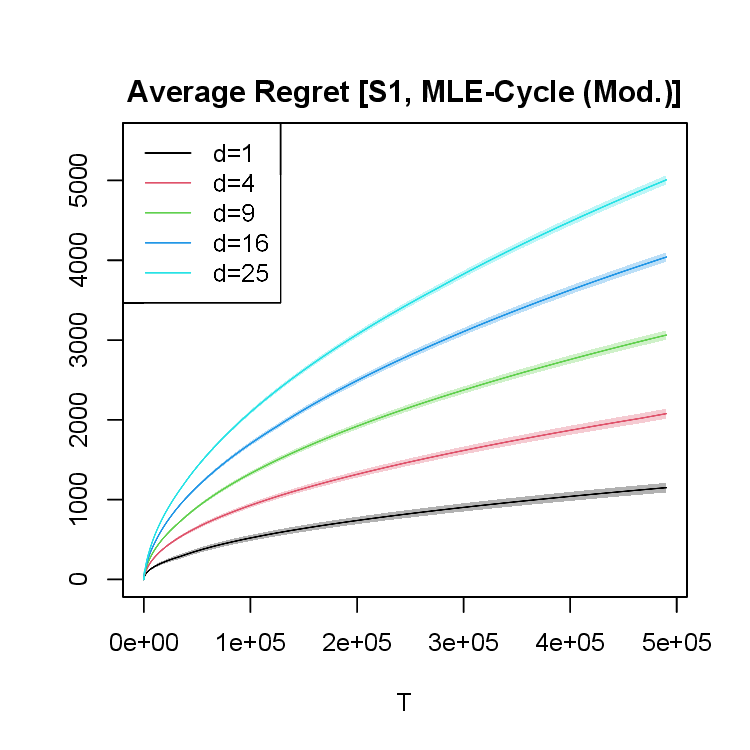}
	\vspace{-0.5cm}
    \end{subfigure}
    ~
    \begin{subfigure}{0.32\textwidth}
	\includegraphics[angle=0, width=1.1\textwidth]{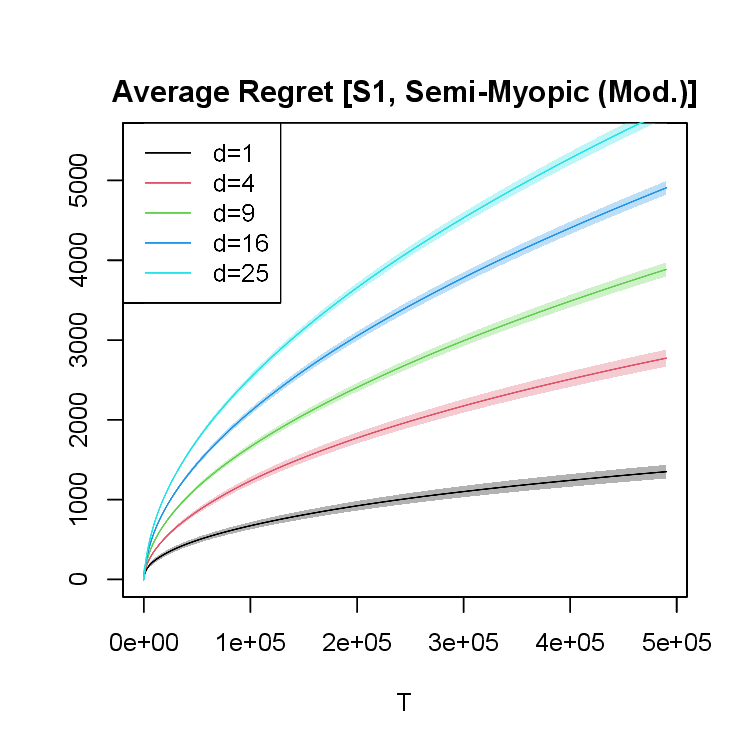}
	\vspace{-0.5cm}
    \end{subfigure}
    \vspace{-0.5cm}
    \caption{\textcolor{black}{(S1) Mean regret (with C.I.) of ETC-Doubling, \textit{modified} MLE-Cycle and \textit{modified} Semi-Myopic with unknown $T=49000$.}}
    \label{fig:unknownT1}
\end{figure}

\begin{figure}[ht]
\vspace{-3mm}
\hspace*{-6mm}
    \begin{subfigure}{0.32\textwidth}
	\includegraphics[angle=0, width=1.1\textwidth]{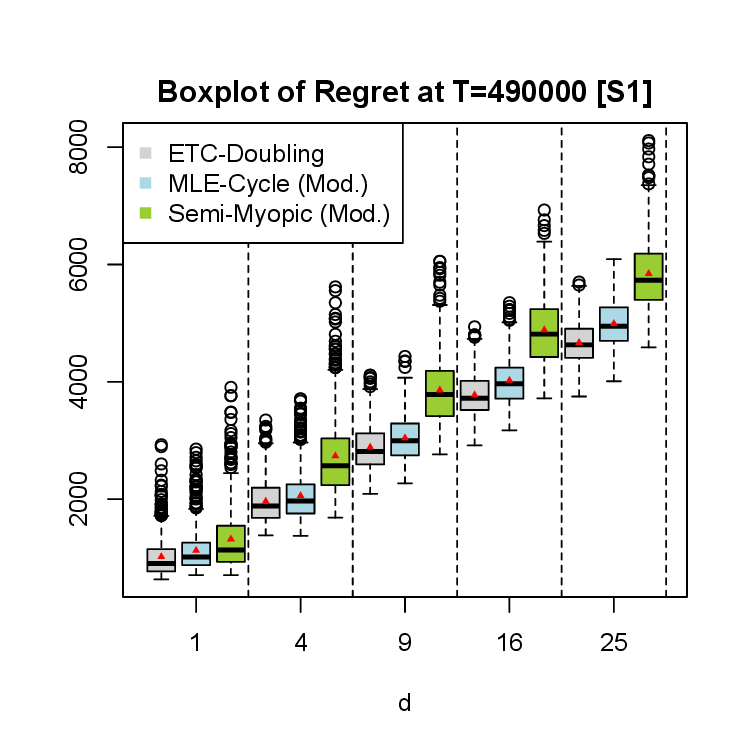}
	\vspace{-0.5cm}
    \end{subfigure}
    ~
    \begin{subfigure}{0.32\textwidth}
	\includegraphics[angle=0, width=1.1\textwidth]{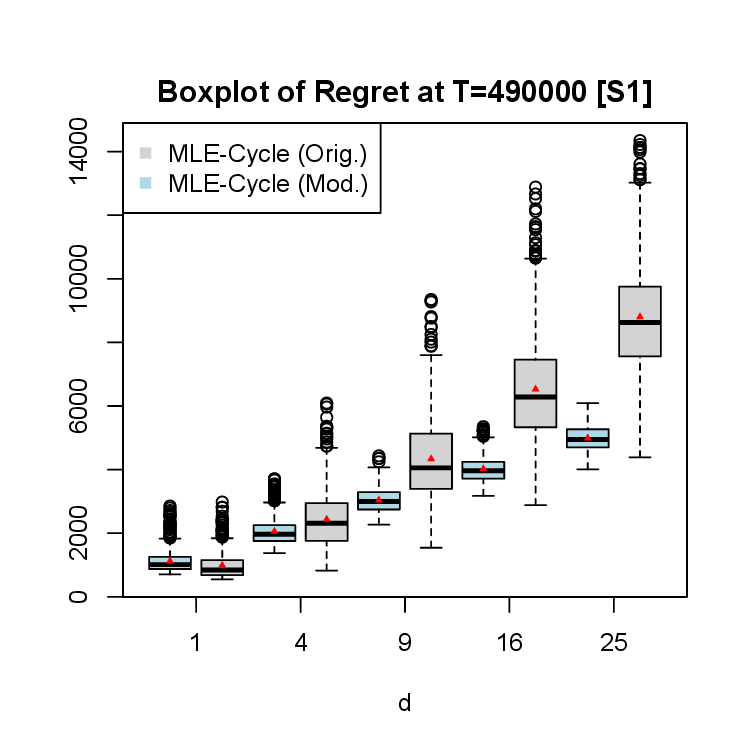}
	\vspace{-0.5cm}
    \end{subfigure}
    ~
    \begin{subfigure}{0.32\textwidth}
	\includegraphics[angle=0, width=1.1\textwidth]{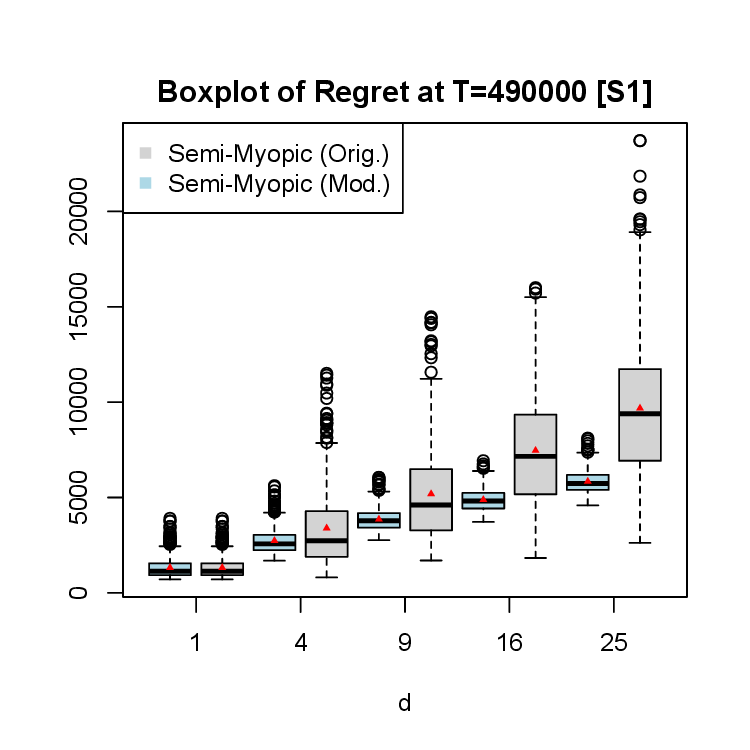}
	\vspace{-0.5cm}
    \end{subfigure}
    \vspace{-0.5cm}
    \caption{\textcolor{black}{(S1) Boxplot of regrets at $T=490000$ for \textit{modified} MLE-Cycle v.s.\ \textit{modified} Semi-Myopic v.s.\ ETC-Doubling, MLE-Cycle (Original v.s.\ Modified), and Semi-Myopic (Original v.s.\ Modified).}}
    \label{fig:unknownT1_box}
\end{figure}

\textcolor{black}{We further report the average computation time of ETC-Doubling, \textit{modified} MLE-Cycle and Semi-Myopic in \Cref{tab:unknownT} of the supplement, where it is seen that ETC-Doubling is much more efficient than MLE-Cycle (around $1.5\sim 6$ times faster) and Semi-Myopic (around $1.5\sim 14$ times faster), especially when the dimension $d$ is large, while at the same time ETC-Doubling improves upon MLE-Cycle for around 6-9\% and Semi-Myopic for around 20-25\% in terms of final regret as suggested in \Cref{fig:unknownT1_box}(left).}

\subsubsection[]{ETC-LDP and \textcolor{black}{ETC-LDP-Mixed} with known $T$}\label{subsubsec:ETC_LDP}
In this section, we examine ETC-LDP under $\epsilon$-LDP and \textcolor{black}{ETC-LDP-Mixed under mixed privacy constraints}. We set $T\in \{(1,3,5,7,9)\times 10^5\}$ and set $d\in \{1,2,4,6\}$. We focus on simulation settings with small dimension $d$, as estimation and thus pricing are intrinsically much more challenging under the LDP setting. Following the LDP literature~\citep[e.g.][]{duchi2018minimax}, we set $\epsilon \in \{1,2,4\}$, which are the most commonly used LDP parameters in the industry~\citep[e.g.][]{han2021generalized}. \textcolor{black}{For the case of mixed privacy constraints, we further consider the proportion of non-private consumers to be $p^*\in\{0.05,0.1\}.$}

Recall under LDP, by \Cref{assum_ldp}(a), the parameter space $\Theta$ is required to be compact. Here, we set $\Theta=\{\theta: \|\theta-\theta^*\|\leq \sqrt{d}\}$, i.e.~we allow $\Theta$ to grow with the dimension $d$ in the numerical experiments, which is more realistic. For each $(T,d,\epsilon)$, we conduct 500 independent experiments and record the realized regrets $\{R_{T,d,\epsilon}^{(i)}\}_{i=1}^{500}$ of ETC-LDP. We then compute the sample mean regret $\overline{R}_{T,d,\epsilon}$ and the confidence interval $I_{T,d,\epsilon}$ based on $\{R_{T,d,\epsilon}^{(i)}\}_{i=1}^{500}$ as described in \Cref{subsubsec:ETC}.

\Cref{fig:ETCLDP_S1}(left) reports the mean regret $\overline{R}_{T,d,\epsilon}$ of ETC-LDP with $\epsilon=1$ at different $(T,d)$ under (S1). As can be seen, given a fixed $d$, the regret scales sublinearly with respect to $T$, and given a fixed horizon $T$, the regret grows with the dimension $d.$ 
Compared to ETC in \Cref{subsubsec:ETC}, it is clear that to accommodate privacy, ETC-LDP incurs a substantially higher regret~(and larger variance). In particular, a careful comparison between \Cref{fig:ETC1}(left) and \Cref{fig:ETCLDP_S1}(left) shows that the mean regret of ETC-LDP is around 7-8 times higher than the mean regret of ETC for comparable $(T,d)$ under (S1). This is indeed intuitive as the private SGD estimator is much less efficient than MLE. \Cref{fig:ETCLDP_S1}(middle) reports the performance of ETC-LDP with $\epsilon=4$. 
\Cref{fig:ETCLDP_S1}(right) further gives the boxplot of $\{R_{T,d,\epsilon}^{(i)}\}_{i=1}^{500}$ for each $T\in \{(1,3,5,7,9)\times 10^5\}$ and $\epsilon\in\{1,4\}$ with a fixed $d=6.$ Clearly, as the privacy constraint eases, ETC-LDP achieves lower mean regret with smaller variance.

\begin{figure}[ht]
\vspace{-4mm}
\hspace*{-6mm}
    \begin{subfigure}{0.32\textwidth}
	\includegraphics[angle=270, width=1.1\textwidth]{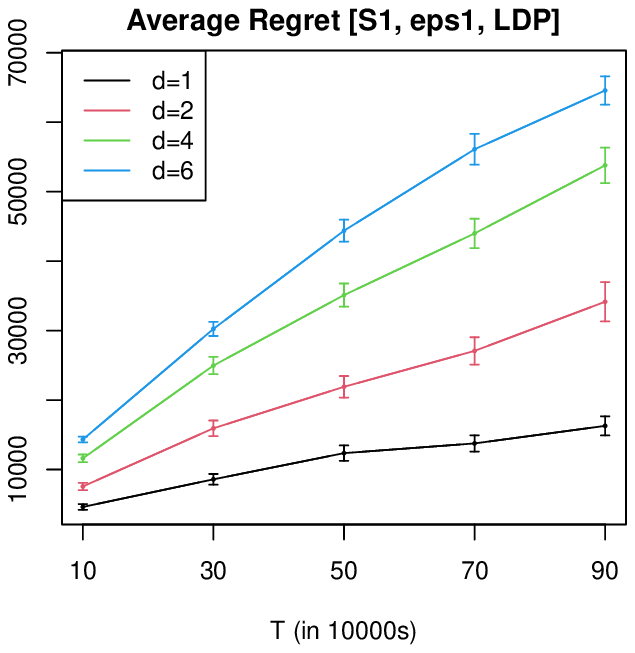}
	\vspace{-0.5cm}
    \end{subfigure}
    ~
    \begin{subfigure}{0.32\textwidth}
	\includegraphics[angle=270, width=1.1\textwidth]{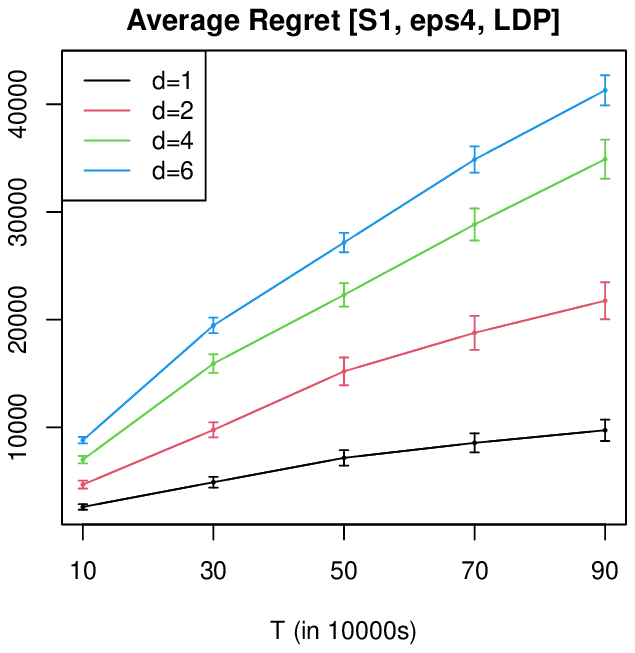}
	\vspace{-0.5cm}
    \end{subfigure}
    ~
    \begin{subfigure}{0.32\textwidth}
	\includegraphics[angle=270, width=1.1\textwidth]{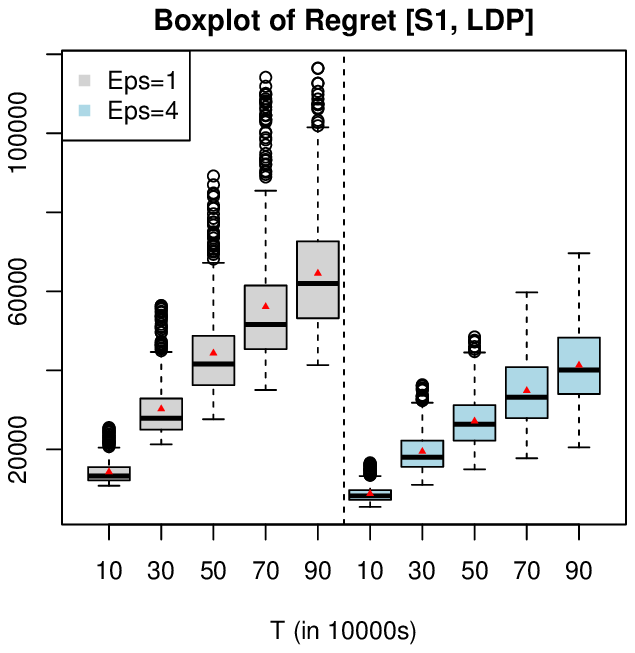}
	\vspace{-0.5cm}
    \end{subfigure}
    \caption{Performance of ETC-LDP under (S1). [Left]:  Mean regret (with C.I.) under different $(d,T)$ and $\epsilon=1$. [Middle]: Mean regret (with C.I.) under different $(d,T)$ and $\epsilon=4$. [Right]: Boxplot of regrets (based on 500 experiments) at different $T$ (with $d=6$) with $\epsilon\in\{1,4\}$.}
    \label{fig:ETCLDP_S1}
\end{figure}

{\color{black}
\Cref{fig:ETCLDP_Mixed_S1}(left and middle) report the mean regret $\overline{R}_{T,d,\epsilon,p^*}$ of ETC-LDP-Mixed with $\epsilon=1$ and $4$ at different $(T,d,p^*)$. Similar to ETC-LDP, the regret scales sublinearly with $T$ and grows with $d.$ Moreover, a larger proportion of non-private consumers $(p^*=0.1)$ clearly leads to lower regret, especially under small $\epsilon$ and large $d$. \Cref{fig:ETCLDP_Mixed_S1}(right) further compares the boxplot of regret $\{R_{T,d,\epsilon}^{(i)}\}_{i=1}^{500}$ achieved by ETC-LDP and ETC-LDP-Mixed (under $p^*=0.1$) for each $T\in \{(1,3,5,7,9)\times 10^5\}$ with a fixed $d=6$ and $\epsilon=1$. As can be seen, compared to pure $\epsilon$-LDP, the presence of non-private consumers can substantially lower regret, confirming the theoretical results in \Cref{thm:LDP-ETC-Mixed}. The performance of ETC-LDP and ETC-LDP-Mixed under (S2) is reported in Figures \ref{fig:ETCLDP_S2}-\ref{fig:ETCLDP_Mixed_S2} of the supplement, where similar phenomenon as the one in Figures \ref{fig:ETCLDP_S1}-\ref{fig:ETCLDP_Mixed_S1} is observed.
}

\begin{figure}[ht]
\vspace{-2mm}
\hspace*{-6mm}
    \begin{subfigure}{0.32\textwidth}
	\includegraphics[angle=270, width=1.1\textwidth]{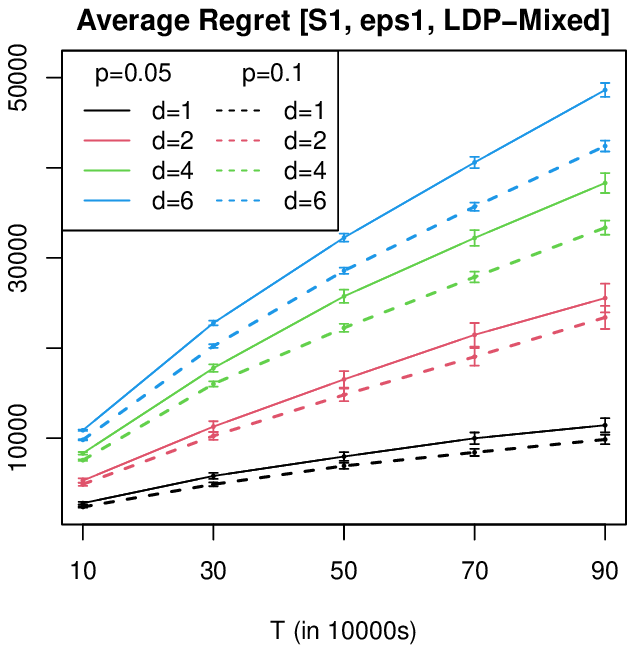}
	\vspace{-0.5cm}
    \end{subfigure}
    ~
    \begin{subfigure}{0.32\textwidth}
	\includegraphics[angle=270, width=1.1\textwidth]{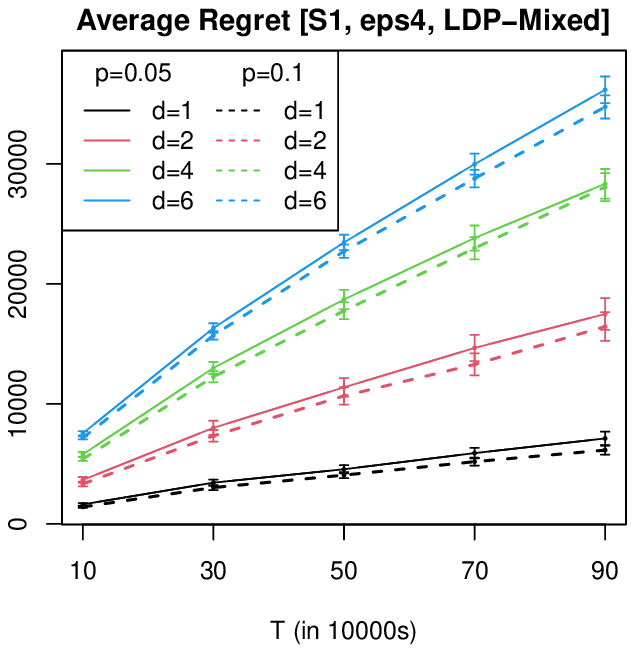}
	\vspace{-0.5cm}
    \end{subfigure}
    ~
    \begin{subfigure}{0.32\textwidth}
	\includegraphics[angle=270, width=1.1\textwidth]{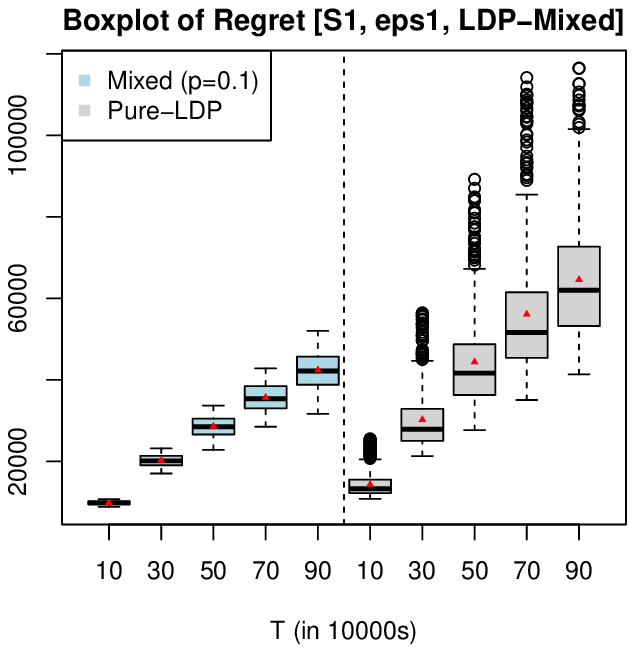}
	\vspace{-0.5cm}
    \end{subfigure}
    \caption{{\color{black} Performance of ETC-LDP-Mixed under (S1). [Left]:  Mean regret (with C.I.) under different $(d,T)$ and $\epsilon=1$ for $p^*=(0.05,0.1)$. [Middle]: Mean regret (with C.I.) under different $(d,T)$ and $\epsilon=4$ for $p^*=(0.05,0.1)$. [Right]: Boxplot of regrets (based on 500 experiments) at different $T$ (with $d=6$ and $\epsilon=1$) for pure LDP (i.e. $p^*=0$) and mixed LDP $(p^*=0.1)$.}}
    \label{fig:ETCLDP_Mixed_S1}
\end{figure}

\subsection{Real data analysis}\label{subsec:autoloan}
In this section, we explore practical utility of the proposed algorithms on a real-world auto loan dataset provided by the Center for Pricing and Revenue Management at Columbia University. 

The dataset records all auto loan applications received by a major online lender in the U.S.\ from July 2002 to November 2004 with a total of 208,085 applications. For each application, we observe information about the loan~(e.g.\ term and amount), the prospective consumer~(e.g.\ FICO score), and the economic environment~(e.g.\ prime rate). We also observe the monthly payment required for the approved loan, which can be viewed as the pricing decision by the company, and whether the offer was accepted -- a binary purchasing decision by a consumer.

We follow the feature selection result in \cite{luo2021distribution} and \cite{bastani2022meta} and use the loan amount approved, term, prime rate, the competitor's rate and FICO score as covariates, i.e.~$d=5.$ We refer to \Cref{tab:loan_covariates} in the supplement for detailed description of the covariates. We further standardize each covariate by its sample mean. The price $p$ of a loan is computed as the net present value of future payment minus the loan amount, i.e.\ $p=\text{Monthly Payment} \times \sum_{i=1}^{\text{Term}}(1+\text{Rate})^{-i} - \text{Loan Amount}.$ We set Rate as $0.12\%$, an average of the monthly London interbank offered rate for the studied time period. For convenience, we use one thousand dollars~$(\$1000)$ as the basic unit for loan prices. 

Note that it is impossible to obtain consumers' real online responses to any pricing strategy unless they were used in the system when the data were collected. Therefore, following the literature, we first estimate the demand model based on the entire observations and use it as the ground truth to generate consumer responses. In particular, denote the $t$th consumer decision as $y_t$, the covariates as $z_t \in \mathbb R^5$, and loan price as $p_t$, we fit the logistic regression model in \eqref{eq:logistic_pricing} to $\{y_t,z_t,p_t\}_{t=1}^{208085}$. The estimated $\theta^*=(\alpha^*,\beta^*)$ is reported in \Cref{tab:loan} of the supplement. \Cref{fig:realdata_description} further gives the histograms of covariates norm $\{\|z_t\|_2\}$, price sensitivity $\{z_t^\top \beta^*\}$ and optimal prices $\{p_t^*\}$ based on the fitted ground truth model. To ensure the optimal price is bounded in $[l,u]=[0,10]$, we remove covariates $z_t$ with highest 1\% norm or lowest 1\% price sensitivity. We denote the remaining $\{z_t\}$ as $\mathcal{Z}$, which consists of 204782 covariates. We refer to \Cref{subsec:add_real} for more details of this preprocessing.   



We run all algorithms (i.e.\ ETC-Doubling, modified MLE-Cycle, \textcolor{black}{modified Semi-Myopic}, ETC-LDP, \textcolor{black}{ETC-LDP-Mixed (assuming 10\% consumers are non-private, i.e.\ $p^*=0.1$) and ETC-LDP-Approx (with $\Delta=10^{-3}$)} based on the ground truth model for $T=208085$, where the covariate $z_t$'s are i.i.d.\ random vectors sampled from $\mathcal{Z}$ and $y_t$ follows the ground truth model given $(z_t,p_t)$. For each algorithm, we conduct 500 experiments and record regret $\{R_T^{(i)}\}_{i=1}^{500}$. The horizon $T$ is set as unknown for all algorithms. For ETC-Doubling, MLE-Cycle and Semi-Myopic, the implementation follows that in \Cref{subsubsec:ETC_MLECycle}. For ETC-LDP and ETC-LDP-Mixed, the implementation follows that in \Cref{subsubsec:ETC_LDP}, while to handle the unknown horizon $T$, we use doubling trick. We refer to \Cref{subsec:add_real} of the supplement for more implementation details.

\Cref{fig:realdata_regret}(left) gives the mean regret path $\{\overline{R}_{t}, t\in[1,T]\}$ and the pointwise confidence bound $\{I_{t}, t\in [1,T]\}$ of ETC-Doubling, MLE-Cycle \textcolor{black}{and Semi-Myopic}, where similar phenomenon as that in \Cref{fig:unknownT1} is seen. In particular, the three methods provide similar performance with ETC-Doubling having a better regret than MLE-Cycle (around 6\%) \textcolor{black}{and Semi-Myopic (around 25\%)}, where the final mean regrets of the three algorithms are 2336.43, 2478.64 and 3011.84, respectively. Based on $\bigl\{\log (\overline{R}_{t}) \bigr\}$ across $t\in [1,T]$, the estimated linear regression in \eqref{eq:linear_reg} with an offset term of $0.5\loglog T$ gives $\beta_T=0.49$ for ETC-Doubling, $\beta_T=0.48$ for MLE-Cycle, and $\beta_T=0.46$ for Semi-Myopic.

\Cref{fig:realdata_regret}(middle) gives the mean regret path $\{\overline{R}_{t}, t\in[1,T]\}$ and the pointwise confidence bound of ETC-LDP for $\epsilon\in\{1,2,4\}$. The final mean regrets are $15591.76, 10057.41$ and $7021.67$ for $\epsilon=1,2,4$, respectively. \textcolor{black}{It further gives the performance of ETC-LDP-Mixed and ETC-LDP-Approx at $\epsilon=1$, where the final mean regrets are 12666.85 and 15930.78, respectively.} The estimated linear regression in \eqref{eq:linear_reg} with an offset term of $\loglog (T)$ gives $\beta_T=0.57,0.53,0.57$ for ETC-LDP at $\epsilon=1,2,4$, \textcolor{black}{and gives $\beta_T=0.53,0.58$ for ETC-LDP-Mixed and ETC-LDP-Approx at $\epsilon=1$, respectively}. For illustration, \Cref{fig:realdata_regret}(right) further gives the boxplot of $\{R_T^{(i)}\}_{i=1}^{500}$ for all algorithms. \textcolor{black}{Note that compared to non-private dynamic pricing~(i.e.\ ETC-Doubling, MLE-Cycle and Semi-Myopic), ETC-LDP incurs substantially higher regret.} With $\epsilon=1$, ETC-LDP incurs around 6 times higher regret than ETC-Doubling. By simple algebra, this translates to 13.26 million dollars loss in revenue for $T=208085$ customers, which highlights the trade-off between privacy preservation and profitability of the firm. \textcolor{black}{Moreover, ETC-LDP-Mixed can effectively leverage the non-private consumers (recall $p^*=0.1$) and helps reduce revenue loss of ETC-LDP by around 2.92 million dollars (i.e.\ 23\% improvement).}


\begin{figure}[ht]
\hspace*{-6mm}
    \begin{subfigure}{0.32\textwidth}
	\includegraphics[angle=0, width=1.1\textwidth]{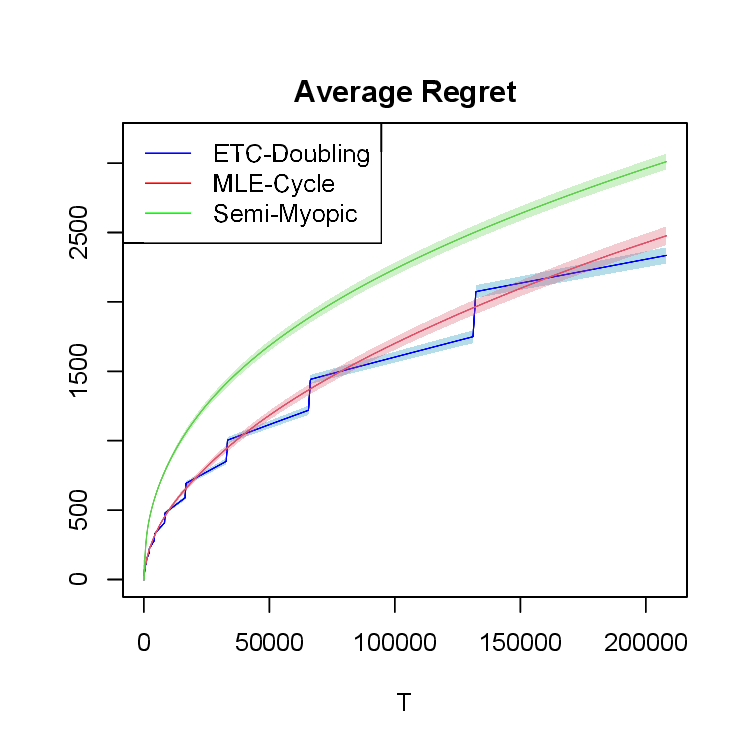}
	\vspace{-0.5cm}
    \end{subfigure}
    ~
    \begin{subfigure}{0.32\textwidth}
	\includegraphics[angle=0, width=1.1\textwidth]{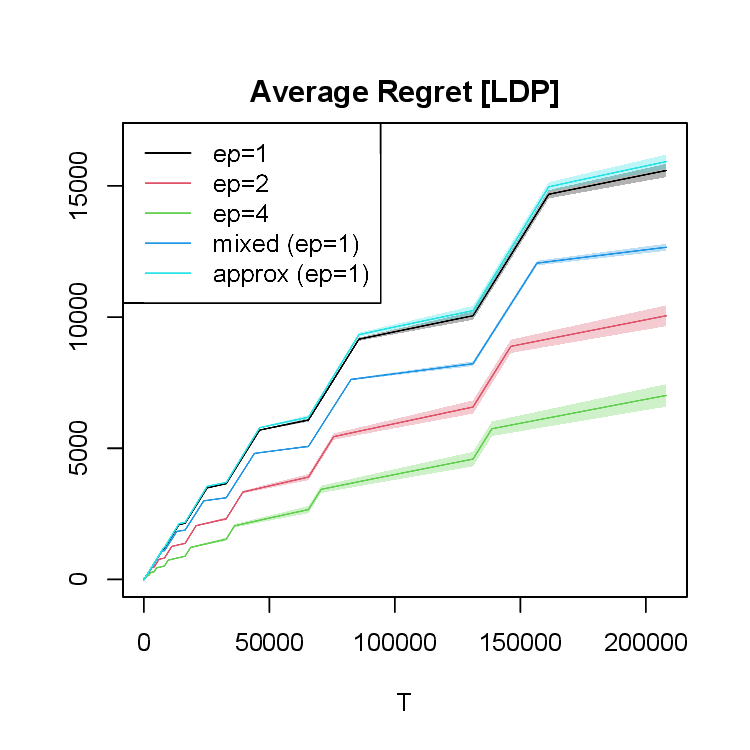}
	\vspace{-0.5cm}
    \end{subfigure}
    ~
    \begin{subfigure}{0.32\textwidth}
	\includegraphics[angle=0, width=1.1\textwidth]{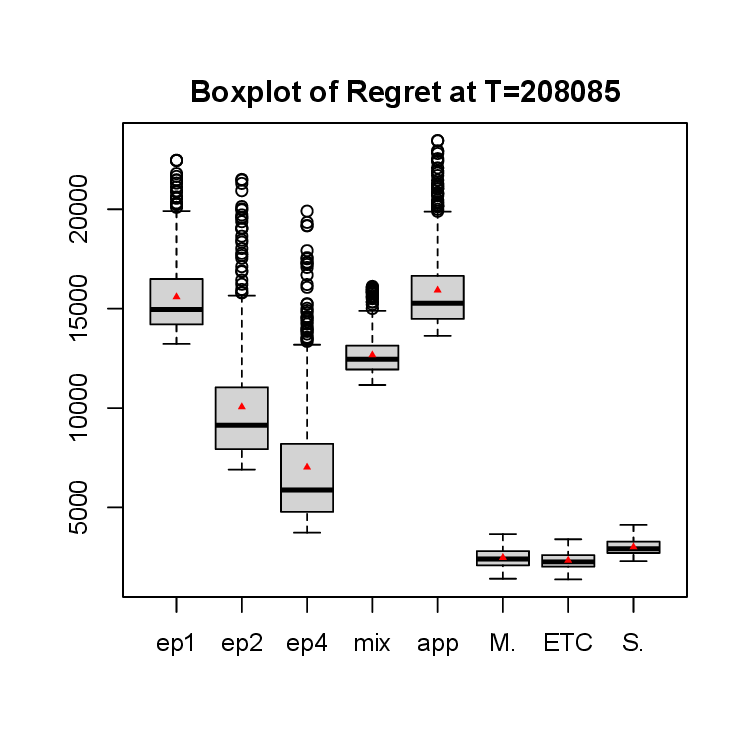}
	\vspace{-0.5cm}
    \end{subfigure}
    \vspace{-2mm}
    \caption{Performance of different algorithms for real data [Left]:  Mean regret (with C.I.) of ETC-Doubling and modified MLE-Cycle \textcolor{black}{and modified Semi-Myopic}. [Middle]: Mean regret (with C.I.) of ETC-LDP with $\epsilon\in\{1,2,4\}$ \textcolor{black}{and ETC-LDP-Mixed and ETC-LDP-Approx with $\epsilon=1$}. [Right]: Boxplot of regrets of all algorithms (based on 500 experiments) at $T=208085$. \textcolor{black}{M.\ stands for MLE-Cycle, ETC stands for ETC-Doubling and S.\ stands for Semi-Myopic.}}
    \label{fig:realdata_regret}
\end{figure}

\section{Conclusion}\label{sec-conclusion}
In this paper, we studied the contextual dynamic pricing problem with a GLM demand model and a stochastic contextual vector of dimension $d$. For the non-private version, we propose two algorithms, supCB and ETC, that can achieve the optimal regret of order $\widetilde O(\sqrt{dT})$. Importantly, this improves upon the state-of-the-art in the literature by a factor of $\sqrt{d}$. In addition, we consider the timely problem of dynamic pricing under LDP and propose a private SGD based ETC-LDP algorithm, which gives a regret of order $\widetilde O(d\sqrt{T}/\epsilon)$, highlighting the cost of privacy. We further establish lower bounds and illustrate the utility of the proposed methods via extensive numerical experiments. One interesting direction for future research is to consider scenarios where there are unknown changes in the context distribution or the parameters of the GLM demand model, which can be more realistic in practice.


\vspace{2mm}
\setlength{\bibsep}{1pt plus 1ex}
\begin{spacing}{1.1}
	\bibliographystyle{apalike}
	\bibliography{reference}
\end{spacing}

\end{document}


\setlength{\abovedisplayskip}{2pt}
\setlength{\belowdisplayskip}{2pt}
\setlength{\abovedisplayshortskip}{1pt}
\setlength{\belowdisplayshortskip}{1pt}
\setlength\intextsep{4pt}

\if1\blind
{
\title{Supplementary Material -- Contextual Dynamic Pricing: Algorithms,\\[-1ex] Optimality, and Local Differential Privacy Constraints}	
	\author{
 	Zifeng Zhao$^{1}$, Feiyu Jiang$^{2}$, Yi Yu$^{3}$\footnote{1. University of Notre Dame; 2. Fudan University; 3. University of Warwick}	
        }
	\date{}	
	\maketitle
} \fi

\if0\blind
{
\title{Supplementary Material -- Contextual Dynamic Pricing: Algorithms,\\[-1ex] Optimality, and Local Differential Privacy Constraints}	
	\author{}
	\date{}
	\maketitle
} \fi

The supplementary material is organized as follows. \Cref{sec-literature} provides an in-depth literature review. Section \ref{sec:add_num_supplement} provides additional results for simulation and real data application. \Cref{sec-supCB} provides details of the supCB algorithm for non-private dynamic pricing. \textcolor{black}{\Cref{sec:adversarial_supplement} provides details of the UCB algorithm for dynamic pricing under adversarial contexts.} \Cref{sec:supplement_proof_nonprivate} provides technical details for theoretical guarantees of the supCB and ETC algorithms in \Cref{sec-without-LDP} of the main text. \Cref{sec:supplement_proof_private} provides technical details for theoretical guarantees of the ETC-LDP and \textcolor{black}{ETC-LDP-Mixed algorithm} in \Cref{sec-ldp} of the main text. \textcolor{black}{An ETC-LDP-Approx algorithm is further proposed and analyzed for dyanmic pricing under $(\epsilon,\Delta)$-LDP.} \Cref{sec:supplement_proof_lb} provides technical details for the minimax lower bound results in \Cref{sec-optimality} of the main text.

\section{Literature review}\label{sec-literature}
Three streams of literature are closely related to our paper: dynamic pricing with demand learning, contextual multi-armed bandit, and differential privacy for online learning.

\subsection{Dynamic pricing with demand learning} 
Due to the increasing popularity of online retailing, dynamic pricing with demand learning has been extensively studied across statistics, machine learning, and operations research. Early works in dynamic pricing mainly focus on the context-free setting, see e.g.\ \cite{besbes2009dynamic} and \cite{broder2012dynamic}. In particular, \cite{broder2012dynamic} show that there exist two regimes in context-free dynamic pricing, the well-separated case where demand curves do not overlap at any price and the general case where there exists an {uninformative price}. \cite{broder2012dynamic} show that the optimal regret for the well-separated case is $O(\log(T))$ and for the general case is $O(\sqrt{T})$.  

For parametric contextual dynamic pricing, \cite{qiang2016dynamic} consider the well-separated case with a linear demand function and derive a regret upper bound of order $O(d\log (T))$, which is further corroborated by \cite{javanmard2019dynamic}, with additional results in high-dimensional covariates cases. For the general case, \cite{ban2021personalized} adapt the MLE-cycle algorithm proposed in \cite{broder2012dynamic} and achieve a regret of order $\widetilde{O}(d\sqrt{T})$ assuming GLM demand functions.  Allowing for adversarial contexts, \cite{wang2021dynamic} propose a UCB algorithm and derive the same $\widetilde{O}(d\sqrt{T})$ rate for the regret under GLM demand functions assumptions. 

Beyond parametric models, \cite{chen2021nonparametric} study contextual dynamic pricing under nonparametric demand models, while \cite{wang2023online}, \cite{luo2024distribution} and \cite{fan2024policy} study semiparametric cases. In particular, \cite{fan2024policy} show that under a linear valuation model and an infinitely differentiable market noise function, the optimal regret is of order $\widetilde O(\sqrt{dT})$, which in some sense complements our result under the parametric GLM model. Contextual dynamic pricing has also been studied from a wide range of aspects, such as consumer strategic behavior~\citep{liu2023contextual}, non-stationarity~\citep{zhao2023high}, fairness constraints~\citep{chen2023utility}, and differential privacy constraints~(see later for more detailed discussion).

\subsection{Contextual multi-armed bandit}\label{subsec-lit_MAB}
Most contextual MAB literature focuses on the linear contextual bandit setting where the reward assumes a linear model with $r_{t,a}=x_{t,a}^{\top}\theta+\epsilon_t$, where $x_{t,a}\in \mathcal{X}$ is the context associated with arm $a$, $\theta$ is the unknown parameter and $\epsilon_t$ is the error process. It is well-known that the optimal regret is of order $\widetilde O(d\sqrt{T})$ when the arm space $\mathcal{X}$ is a compact (but otherwise arbitrary) subset of $\mathbb R^d$ and can be achieved via UCB~\citep{dani2008stochastic,abbasi2011improved}. In contrast, for the setting where the number of available arms $K$ in each round is finite, the regret can be further improved to $\widetilde O(\sqrt{dT\log(K)})$~\citep{auer2002using,chu2011contextual}. However, such a regret bound can only be achieved via a more involved confidence bound-type algorithm. This type of algorithm is first proposed by \cite{auer2002using} under the name sup-LinRel and is extended to the generalized linear bandit in \cite{li2017provably}, which inspires the design of the proposed supCB algorithm in our paper. Note that contextual dynamic pricing can in some sense be viewed as a contextual MAB with infinite number of arms, as the price can take any value in $[l,u].$ However, the key insight of our work is that, despite being uncountable, this action space is of dimension \textit{one} regardless of the contextual dimensionality $d$, and thus can be well-approximated by a contextual MAB with \textit{reasonably many} arms via discretization and achieves an optimal regret of order $\widetilde O(\sqrt{dT})$.

\subsection{Differential privacy for online learning}
The most prevalent privacy concept used in both academia and industry is the notion of differential privacy~(DP) proposed by \cite{dwork2006calibrating}. In the DP literature, there are two widely used privacy schemes known as central differential privacy~\citep[CDP,][]{dwork2006calibrating} and local differential privacy~\citep[LDP,][]{evfimievski2003limiting}.

CDP assumes the existence of a trusted central handler~(e.g.\ the firm) that can securely store and process raw consumer information. This may raise concerns about data vulnerability, particularly the risk of attacks on the central server, leading to unauthorized access to sensitive customer information. In contrast, LDP offers a decentralized approach that prioritizes individual privacy. Instead of relying on a central authority to store and process data, LDP allows each customer to locally randomize their data before sharing them with the firm or any other party.

Due to the practical importance of privacy, there is a recent surge in designing online learning algorithms with DP guarantees for contextual MAB and contextual dynamic pricing. \cite{shariff2018differentially} study linear contextual bandits under CDP, while \cite{zheng2020locally} explore them under LDP. Based on a covariate-diversity assumption, \cite{han2021generalized} propose an exploration-free algorithm for generalized linear bandits under LDP. However, this assumption in general does not hold under dynamic pricing as the adaptively chosen price is part of the covariate.

\cite{chen2022privacy} examine dynamic pricing under a GLM demand model with CDP constraints and \cite{chen2022differential}  study nonparametric dynamic pricing under both LDP and CDP constraints. Under the offline setting, where all historical data are accessible for learning, \cite{lei2023privacy} study contextual dynamic pricing for a linear demand model under LDP, focusing primarily on preserving the privacy of contexts. Leveraging the structure of the linear model, they propose a private ordinary least squares estimator, which cannot be used for GLM. In this paper, we consider contextual dynamic pricing with a GLM demand function under the online setting and propose a stochastic gradient descent based algorithm that offers LDP guarantees for both consumer responses and contexts.

\section{Additional numerical results}\label{sec:add_num_supplement}

\subsection{Additional results for ETC and supCB with known $T$}\label{subsubsec:SupCB_ETC}
This section provides additional numerical results for the regret and computational time of ETC and supCB with known $T$ as in \Cref{subsubsec:ETC} of the main text.

\noindent\textbf{Implementation details of supCB}. Following \Cref{thm:ucb_regret}, we set the number of stages as $S=\lfloor \log_2 (T)\rfloor$, the length of price experiments as $\tau=\bigl\lceil\sqrt{dT}\,\bigr\rceil$ and the discretization rate as $K=\bigl\lceil\sqrt{T/d}/\log (T)\,\bigr\rceil$. We set $\alpha=\loglog (T)\cdot \sqrt{\log (3T^{1.5}KS)}$. Here, a $\loglog (T)$ factor is included in $\alpha$ as in practice $(\sigma, M_{\psi2}, \kappa)$ is unknown. This at most inflates the regret of supCB by $\loglog (T).$ 

\noindent\textbf{Results}. \Cref{fig:SupCB1} reports the performance of supCB under (S1). The estimated linear regression based on \eqref{eq:linear_reg} with an offset term of $1.5\loglog (T)$ gives $(\beta_d,\beta_T)=(0.50,0.49)$, which provides numerical evidence for Theorem \ref{thm:ucb_regret}. However, though its regret scales sublinearly with $(d,T)$, supCB records an expected regret that is overall 3 to 4 times higher than ETC (see \Cref{fig:ETC1} of the main text). The performance of ETC and supCB under (S2) is reported in \Cref{fig:ETC2} and \Cref{fig:SupCB2}, where similar phenomenon is observed. Therefore, in terms of statistical efficiency, ETC is preferred over supCB. 

\begin{figure}[ht]
\vspace{-4mm}
\hspace*{-6mm}
    \begin{subfigure}{0.32\textwidth}
	\includegraphics[angle=270, width=1.1\textwidth]{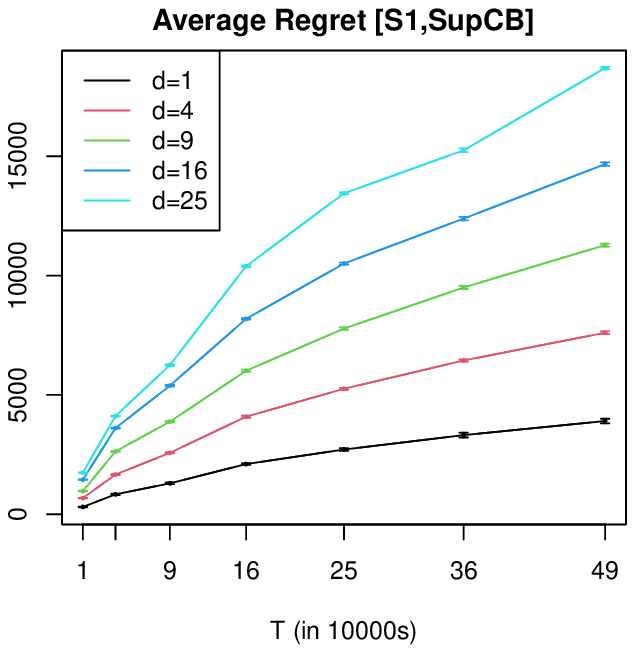}
	\vspace{-0.5cm}
    \end{subfigure}
    ~
    \begin{subfigure}{0.32\textwidth}
	\includegraphics[angle=270, width=1.1\textwidth]{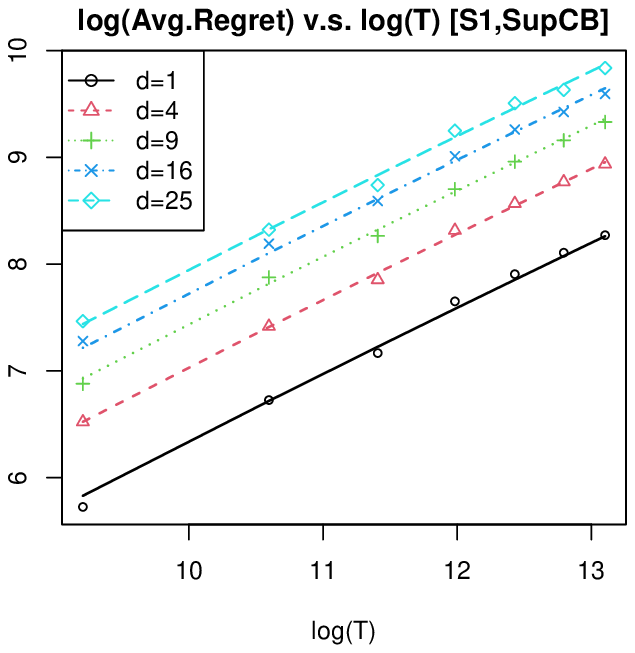}
	\vspace{-0.5cm}
    \end{subfigure}
    ~
    \begin{subfigure}{0.32\textwidth}
	\includegraphics[angle=270, width=1.1\textwidth]{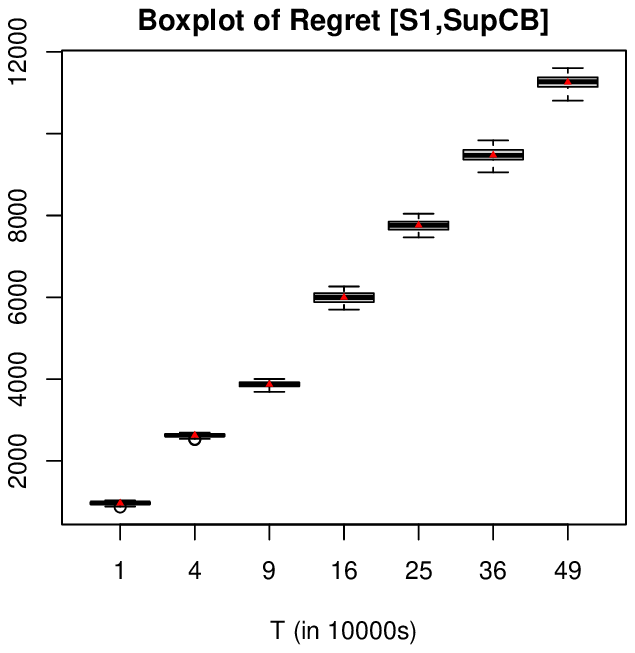}
	\vspace{-0.5cm}
    \end{subfigure}
    \caption{supCB under (S1). [Left]:  Mean regret (with C.I.) under different $(d,T)$. [Middle]: Mean regret (in log scale) with fitted regression lines. [Right]: Boxplot of regrets at different $T$ ($d=9$).}
    \label{fig:SupCB1}
\end{figure}

\begin{figure}[H]
\vspace{-4mm}
\hspace*{-6mm}
    \begin{subfigure}{0.32\textwidth}
	\includegraphics[angle=270, width=1.1\textwidth]{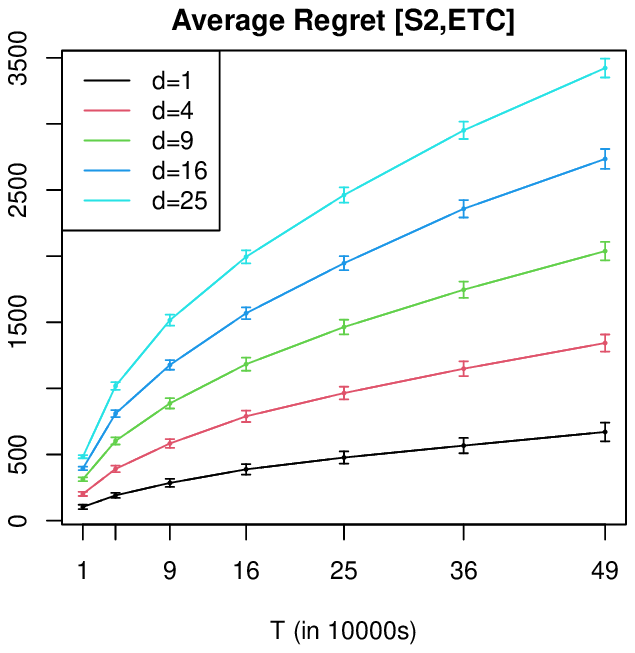}
	\vspace{-0.5cm}
    \end{subfigure}
    ~
    \begin{subfigure}{0.32\textwidth}
	\includegraphics[angle=270, width=1.1\textwidth]{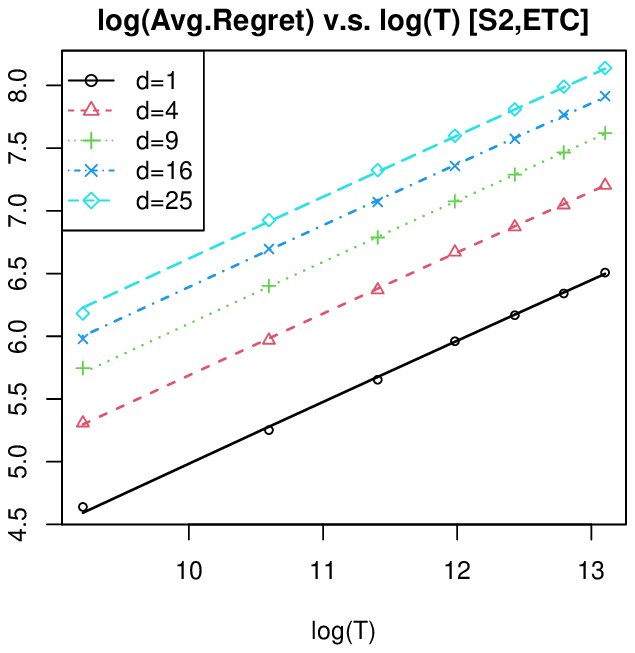}
	\vspace{-0.5cm}
    \end{subfigure}
    ~
    \begin{subfigure}{0.32\textwidth}
	\includegraphics[angle=270, width=1.1\textwidth]{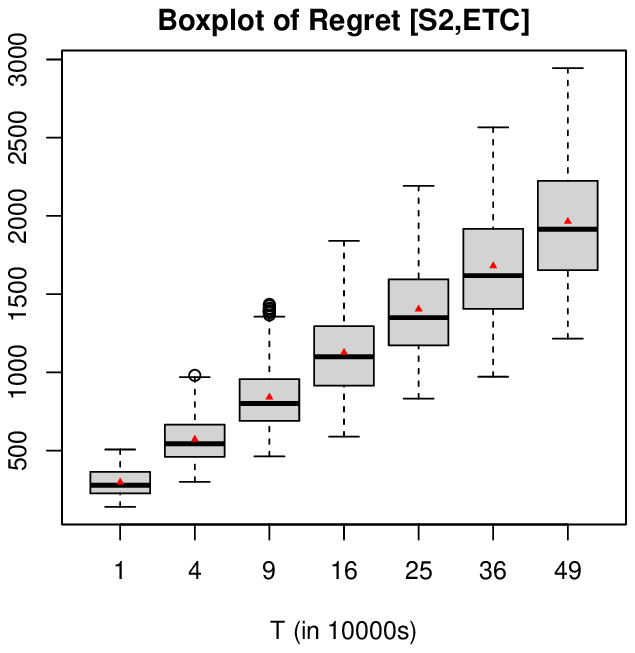}
	\vspace{-0.5cm}
    \end{subfigure}
    \caption{Performance of ETC under (S2) [Left]:  Mean regret (with C.I.) under different $(d,T)$. [Middle]: Mean regret (in log scale) with fitted linear regression lines ($(\beta_d,\beta_T)=(0.51,0.45)$). [Right]: Boxplot of regrets (based on 500 experiments) at different $T$ (with $d=9$).}
    \label{fig:ETC2}
\end{figure}

\begin{figure}[H]
\vspace{-5mm}
\hspace*{-6mm}
    \begin{subfigure}{0.32\textwidth}
	\includegraphics[angle=270, width=1.1\textwidth]{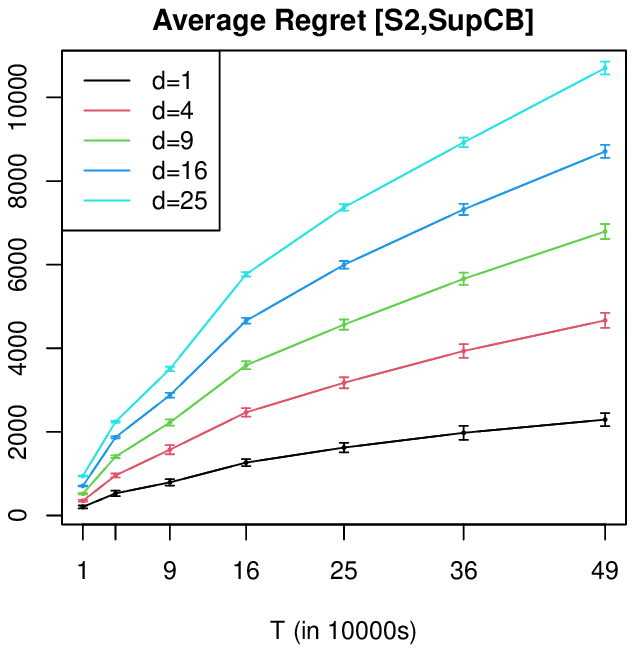}
	\vspace{-0.5cm}
    \end{subfigure}
    ~
    \begin{subfigure}{0.32\textwidth}
	\includegraphics[angle=270, width=1.1\textwidth]{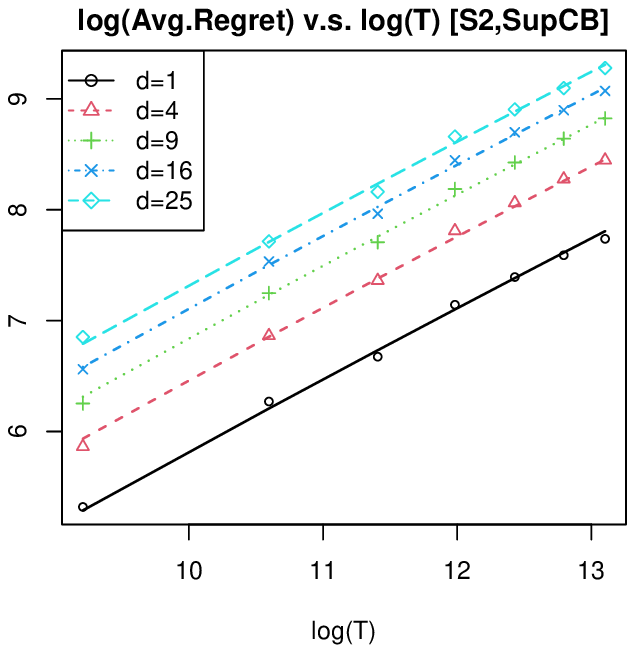}
	\vspace{-0.5cm}
    \end{subfigure}
    ~
    \begin{subfigure}{0.32\textwidth}
	\includegraphics[angle=270, width=1.1\textwidth]{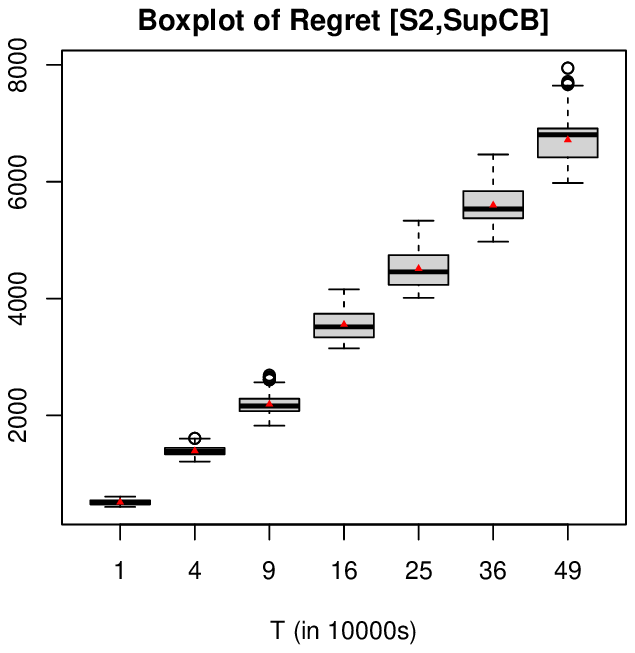}
	\vspace{-0.5cm}
    \end{subfigure}
    \caption{Performance of supCB under (S2) [Left]:  Mean regret (with C.I.) under different $(d,T)$. [Middle]: Mean regret (in log scale) with fitted linear regression lines ($(\beta_d,\beta_T)=(0.47,0.51)$). [Right]: Boxplot of regrets (based on 500 experiments) at different $T$ (with $d=9$).}
    \label{fig:SupCB2}
\end{figure}

\noindent\textbf{Computation time}. \Cref{fig:timeS1} reports the mean computation time of ETC and supCB at different $(T,d)$ computed over 500 experiments under (S1). Overall, conducting ETC is faster than supCB by around 10$\sim$100 times~(with an average of 54 times), where the improvement is larger under larger dimension $d.$ Therefore, to summarize, in practice, we recommend to use ETC as the default algorithm for dynamic pricing, as it is simple and achieves better statistical and computational efficiency.

\begin{figure}[H]
\vspace{-4mm}
\hspace*{-18mm}
\centering
    \begin{subfigure}{0.33\textwidth}
	\includegraphics[angle=270, width=1.1\textwidth]{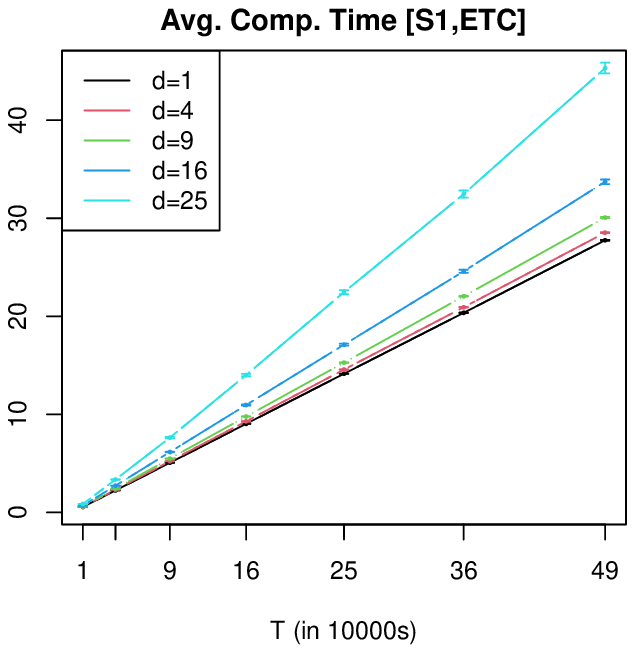}
	\vspace{-0.5cm}
    \end{subfigure}
    ~
    \begin{subfigure}{0.33\textwidth}
	\includegraphics[angle=270, width=1.1\textwidth]{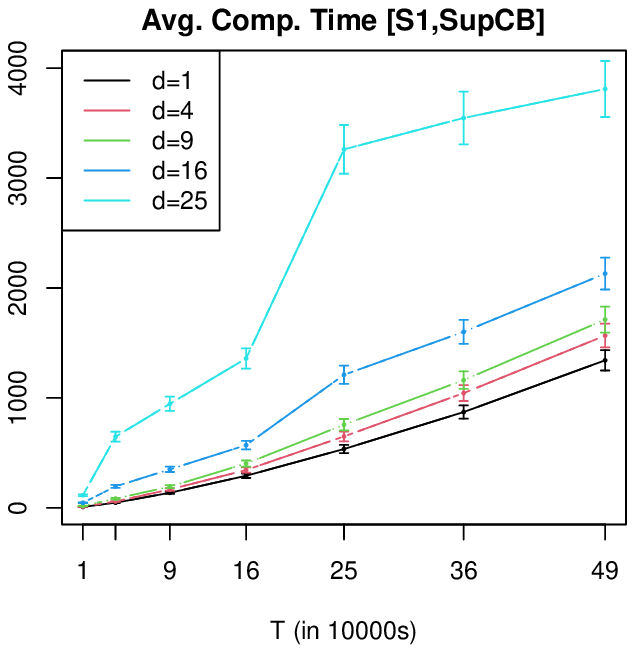}
	\vspace{-0.5cm}
    \end{subfigure}
    \caption{Average computation time (with C.I.) of ETC and supCB under (S1).}
    \label{fig:timeS1}
\end{figure}

\subsection{Additional results for experiments on synthetic data}\label{subsec:add_num}

\begin{table}[H]
\caption{Mean and sample SD (in parentheses) of regrets and computational time (in seconds) over 500 experiments by ETC-Doubling and \textcolor{black}{\textit{modified} MLE-Cycle and \textit{modified} Semi-Myopic} under (S1) with unknown horizon $T=490000.$}
\centering
\small
\begin{tabu}{lccccc|lccccc}
  \hline\hline
 & \multicolumn{10}{c}{ETC-Doubling}  \\\hline
 $d$ &  1 & 4 & 9 & 16 & 25 & $d$ & 1 & 4 & 9 & 16 & 25\\ \hline
 Regret & 1048.8 & 1982.4 & 2897.6 & 3790.2 & 4679.2 & Time & 58.2 & 59.1 & 61.9 & 66.5 & 85.5 \\ 
 & (478.9) & (436.4) & (411.0) & (399.9) & (381.5) & & (0.6) & (1.0) & (0.5) & (1.6) & (5.5)\\ \hline
 & \multicolumn{10}{c}{Modified MLE-Cycle}  \\\hline
 $d$ &  1 & 4 & 9 & 16 & 25 & $d$ & 1 & 4 & 9 & 16 & 25\\ \hline
 Regret & 1148.8 & 2076.2 & 3060.3 & 4038.7 & 5004.4 & Time & 59.3 & 84.3 & 144.9 & 282.3 & 545.9 \\ 
 & (451.6) & (471.9) & (436.1) & (444.2) & (425.8) & & (2.5) & (2.6) & (6.9) & (21.5) & (71.7) \\ \hline
 \rowfont{\color{black}}
 & \multicolumn{10}{c}{Modified Semi-Myopic}  \\\hline
 \rowfont{\color{black}}
 $d$ &  1 & 4 & 9 & 16 & 25 & $d$ & 1 & 4 & 9 & 16 & 25\\ \hline
 \rowfont{\color{black}}
 Regret & 1350.0 & 2772.6 & 3884.2 & 4906.2 & 5872.6 & Time & 59.6 & 135.0 & 296.7 & 621.3 & 1218.1 \\ 
 \rowfont{\color{black}}
 & (639.6) & (797.3) & (657.1) & (644.8) & (683.5) & & (8.6) & (37.6) & (93.7) & (180.5) & (381.9) \\ 
 \hline\hline
\end{tabu}\label{tab:unknownT}
\end{table}

\begin{figure}[H]
\vspace{-3mm}
\hspace*{-6mm}
    \begin{subfigure}{0.32\textwidth}
	\includegraphics[angle=0, width=1.1\textwidth]{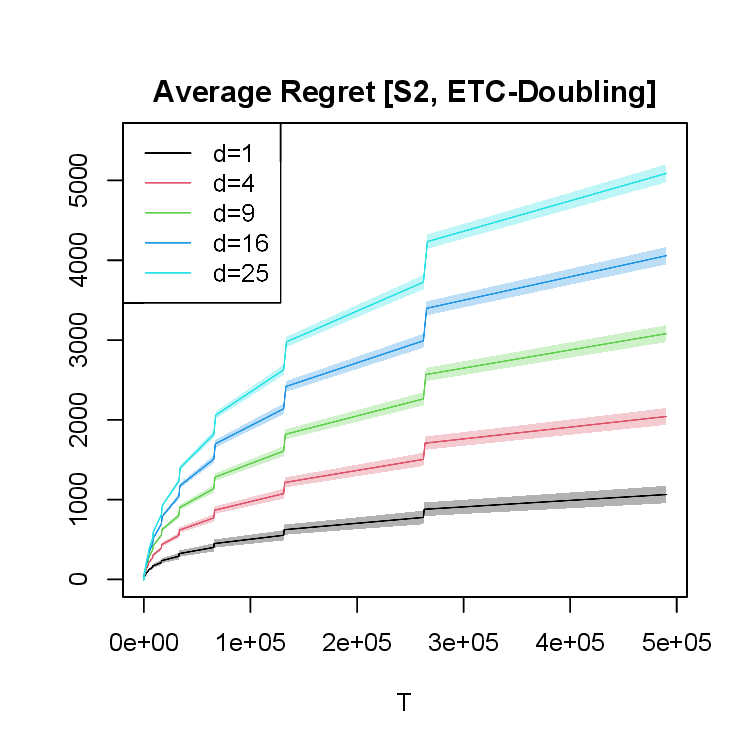}
	\vspace{-0.5cm}
    \end{subfigure}
    ~
    \begin{subfigure}{0.32\textwidth}
	\includegraphics[angle=0, width=1.1\textwidth]{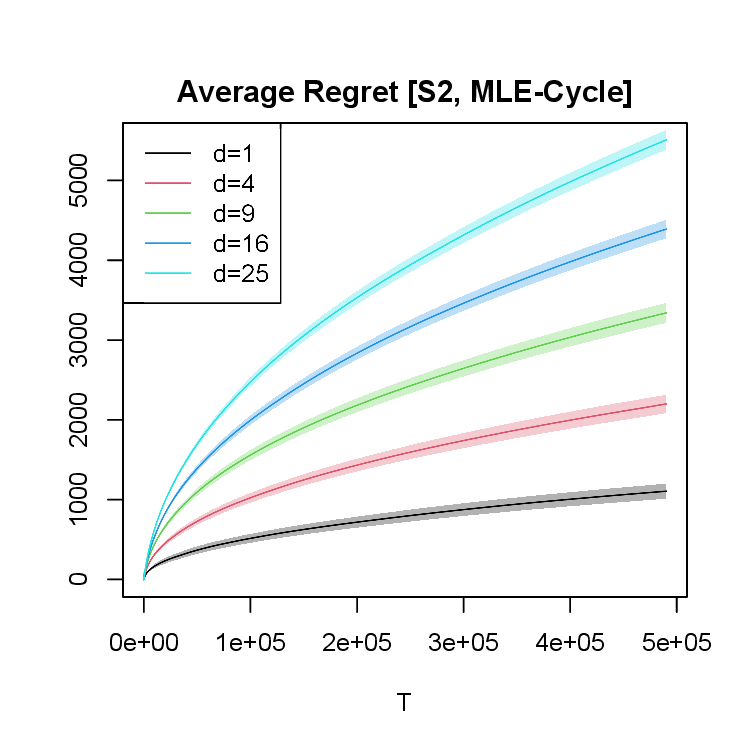}
	\vspace{-0.5cm}
    \end{subfigure}
    ~
    \begin{subfigure}{0.32\textwidth}
	\includegraphics[angle=0, width=1.1\textwidth]{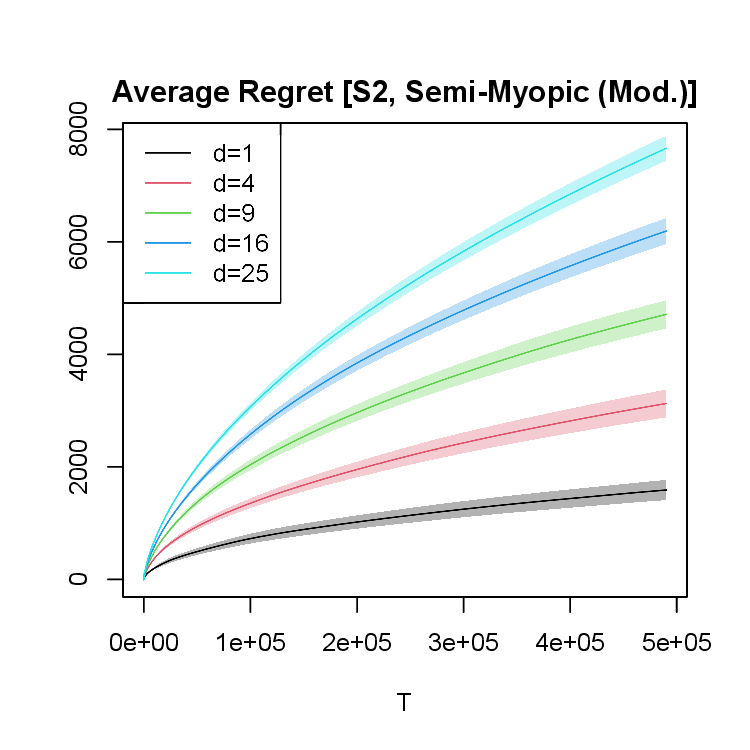}
	\vspace{-0.5cm}
    \end{subfigure}
    \caption{\textcolor{black}{(S2) Mean regret (with C.I.) of ETC-Doubling and \textit{modified} MLE-Cycle and \textit{modified} Semi-Myopic with unknown $T$.}}
    \label{fig:unknownT2}
\end{figure}

\begin{figure}[ht]
\hspace*{-6mm}
    \begin{subfigure}{0.32\textwidth}
	\includegraphics[angle=0, width=1.1\textwidth]{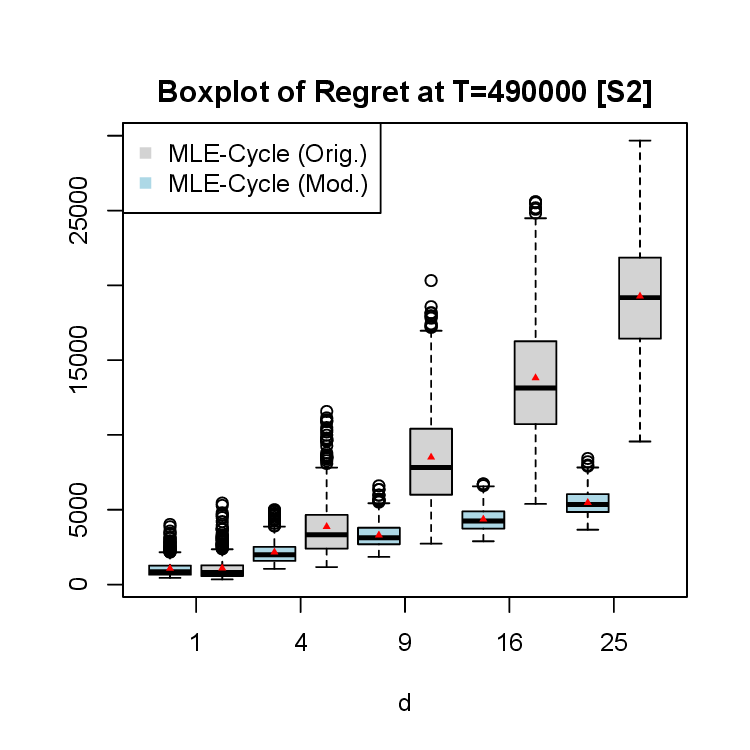}
	\vspace{-0.5cm}
    \end{subfigure}
    ~
    \begin{subfigure}{0.32\textwidth}
	\includegraphics[angle=0, width=1.1\textwidth]{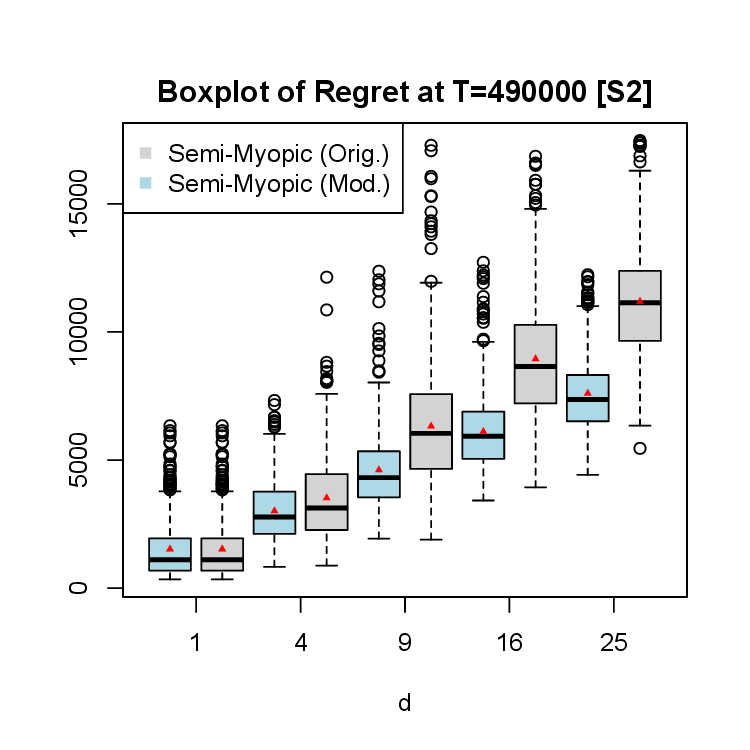}
	\vspace{-0.5cm}
    \end{subfigure}
    ~
    \begin{subfigure}{0.32\textwidth}
	\includegraphics[angle=0, width=1.1\textwidth]{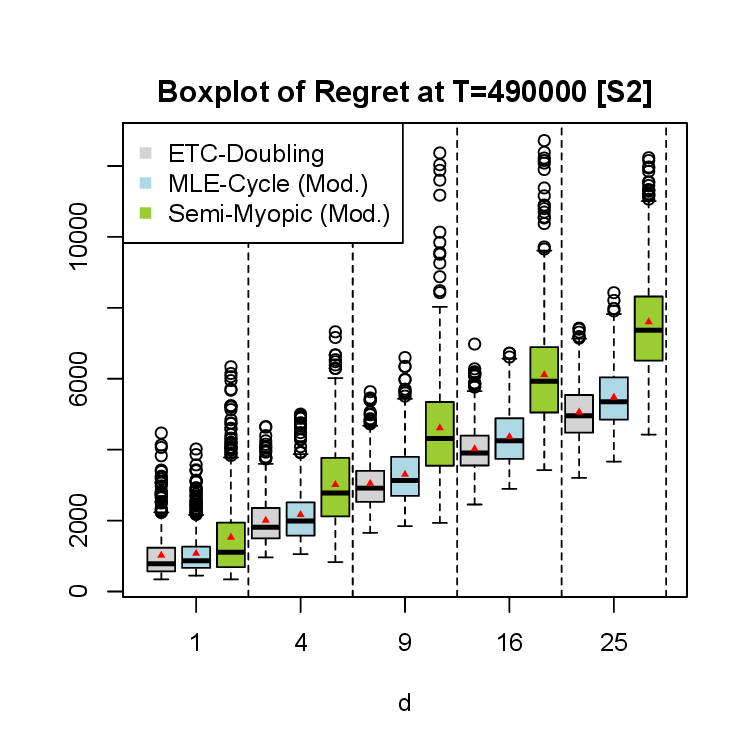}
	\vspace{-0.5cm}
    \end{subfigure}
    \vspace{-0.5cm}
    \caption{\textcolor{black}{(S2) Boxplot of regrets at $T=490000$ for MLE-Cycle (Original v.s.\ Modified), Semi-Myopic (Original v.s.\ Modified), and for \textit{modified} MLE-Cycle v.s.\ \textit{modified} Semi-Myopic v.s.\ ETC-Doubling.}}
    \label{fig:unknownT2_box}
\end{figure}

\begin{figure}[H]
\vspace{-3mm}
\hspace*{-6mm}
    \begin{subfigure}{0.32\textwidth}
	\includegraphics[angle=270, width=1.1\textwidth]{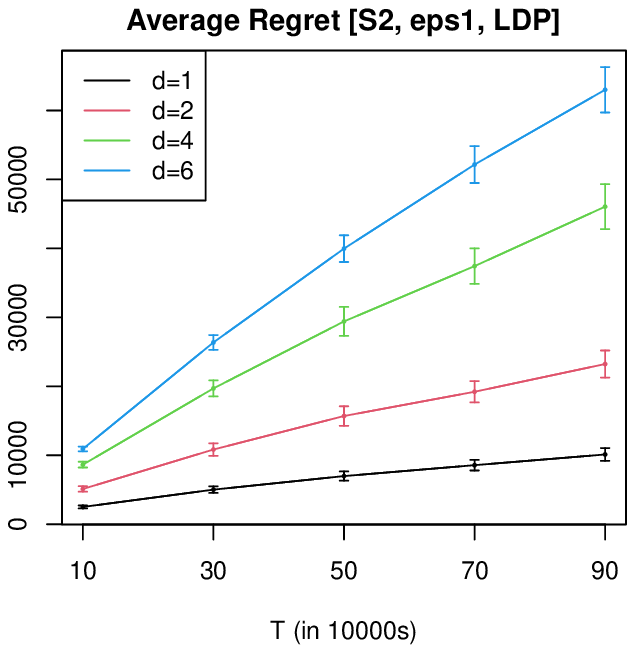}
	\vspace{-0.5cm}
    \end{subfigure}
    ~
    \begin{subfigure}{0.32\textwidth}
	\includegraphics[angle=270, width=1.1\textwidth]{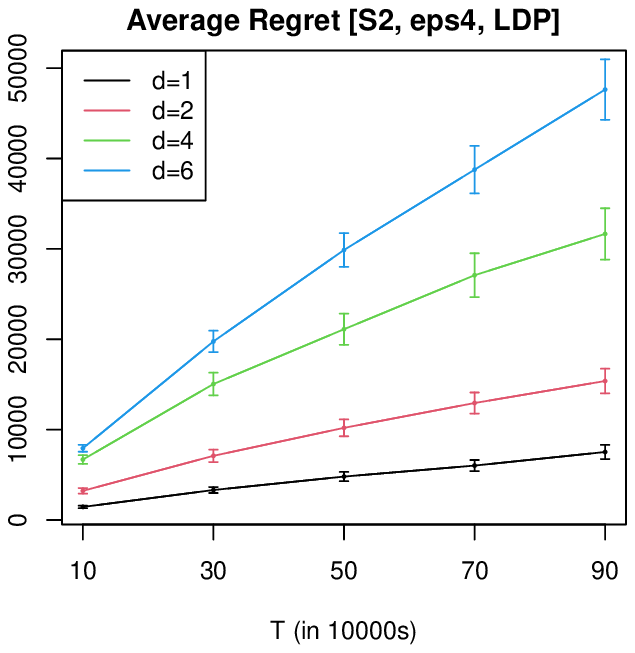}
	\vspace{-0.5cm}
    \end{subfigure}
    ~
    \begin{subfigure}{0.32\textwidth}
	\includegraphics[angle=270, width=1.1\textwidth]{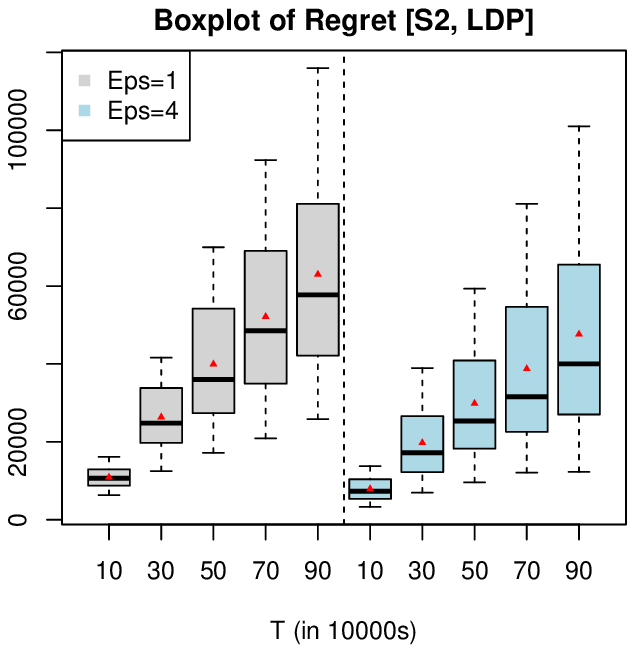}
	\vspace{-0.5cm}
    \end{subfigure}
    \caption{Performance of ETC-LDP under (S2). [Left]:  Mean regret (with C.I.) under different $(d,T)$ and $\epsilon=1$. 
    [Middle]: Mean regret (with C.I.) under different $(d,T)$ and $\epsilon=4$. 
    [Right]: Boxplot of regrets (based on 500 experiments) at different $T$ (with $d=6$) with $\epsilon=(1,4)$.}
    \label{fig:ETCLDP_S2}
\end{figure}

\begin{figure}[H]
\vspace{-5mm}
\hspace*{-6mm}
    \begin{subfigure}{0.32\textwidth}
	\includegraphics[angle=270, width=1.1\textwidth]{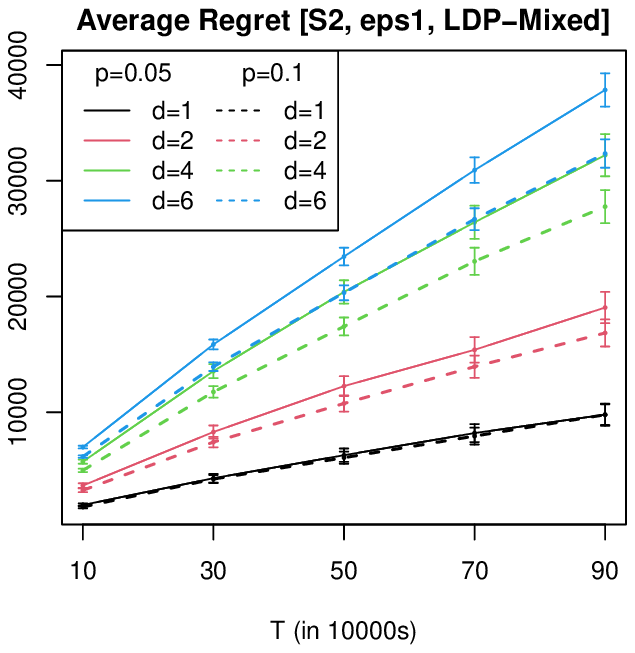}
	\vspace{-0.5cm}
    \end{subfigure}
    ~
    \begin{subfigure}{0.32\textwidth}
	\includegraphics[angle=270, width=1.1\textwidth]{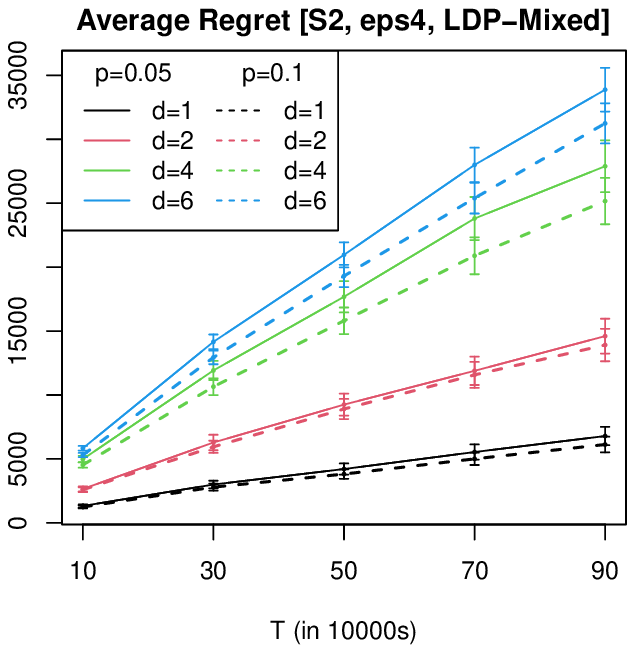}
	\vspace{-0.5cm}
    \end{subfigure}
    ~
    \begin{subfigure}{0.32\textwidth}
	\includegraphics[angle=270, width=1.1\textwidth]{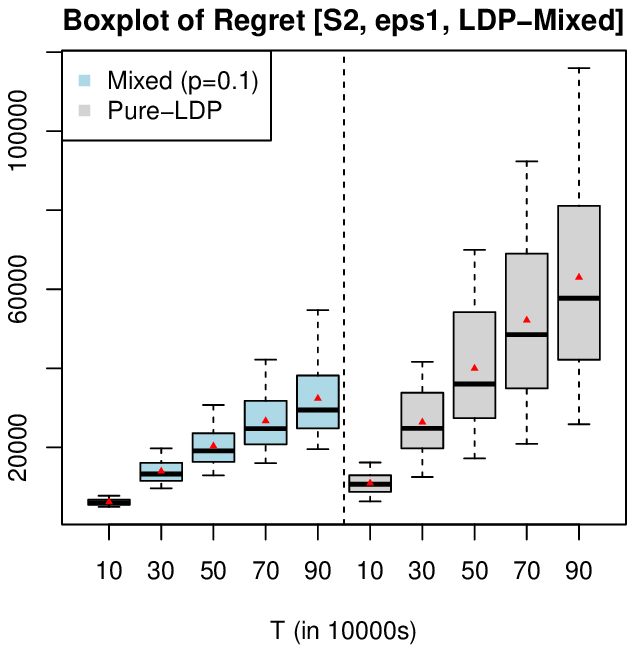}
	\vspace{-0.5cm}
    \end{subfigure}
    \caption{{\color{black} Performance of ETC-LDP-Mixed under (S2). [Left]:  Mean regret (with C.I.) under different $(d,T)$ and $\epsilon=1$ for $p=(0.05,0.1)$. [Middle]: Mean regret (with C.I.) under different $(d,T)$ and $\epsilon=4$ for $p=(0.05,0.1)$. [Right]: Boxplot of regrets (based on 500 experiments) at different $T$ (with $d=6$ and $\epsilon=1$) for pure LDP (i.e. $p=0$) and mixed LDP $(p=0.1)$.}}
    \label{fig:ETCLDP_Mixed_S2}
\end{figure}

\subsection{Additional results for real data analysis}\label{subsec:add_real}

\begin{table}[H]
\caption{Summary of the auto loan dataset used in \Cref{subsec:autoloan}}
\label{tab:loan_covariates}
\begin{tabular}{lll}
\hline\hline
Variable            & Type        & Description                                                                                                               \\\hline
apply               & Binary      & Indicator for eventual contract (dependent variable)                                                                      \\
Price               & Continuous  & Price of the loan                                                                                                         \\
Primary\_FICO       & Continuous  & FICO score                                                                                                                \\
Competition\_rate   & Continuous  & Competitor’s rate                                                                                                         \\
Amount\_Approved    & Continuous  & Loan amount approved                                                                                                      \\
onemonth            & Continuous  & Prime rate                                                                                                                \\
Term                & Continuous  & Approved term in months  \\
\hline\hline
\end{tabular}
\end{table}

\begin{table}[H]
\centering
\caption{Estimated logistic regression on the entire auto loan dataset.}
\label{tab:loan}
\begin{tabular}{rrrrrrr}
  \hline\hline
  & FICO & Competitor Rate & Amount & Prime Rate & Term \\ 
  \hline
  $\alpha$ & $-1.434$ & 0.288 & $-2.348$ & 0.807 & 2.845 \\ 
  $\beta$ & 1.956 & $-1.542$ & $-0.432$ & 0.672 & 0.147 \\ 
  \hline\hline
\end{tabular}
\end{table}

\textbf{Data preprocessing}. \Cref{fig:realdata_description}(left and middle) gives the histograms of covariates norm $\{\|z_t\|_2\}_{t=1}^{208085}$ and price sensitivity $\{z_t^\top \beta^*\}_{t=1}^{208085}$ computed via the fitted ground truth model. The 99\% quantile of covariates norm (which is 3.05) and 1\% quantile of price sensitivity (which is 0.13) are further marked via red vertical lines. To impose an upper bound on the covariates norm and a lower bound on the price sensitivity, we remove covariates $z_t$ with norm larger than 3.05 or price sensitivity smaller than 0.13. We denote the remaining $\{z_t\}$ as $\mathcal{Z}$, which consists of 204782 covariates. \Cref{fig:realdata_description}(right) further gives the histogram of the optimal prices (in thousands) of all covariates in $\mathcal{Z}$ based on the ground truth model, where the maximum is 9.994 and the minimum is 0.632 (in thousands).

\begin{figure}[ht]
\hspace*{-6mm}
    \begin{subfigure}{0.32\textwidth}
	\includegraphics[angle=0, width=1.1\textwidth]{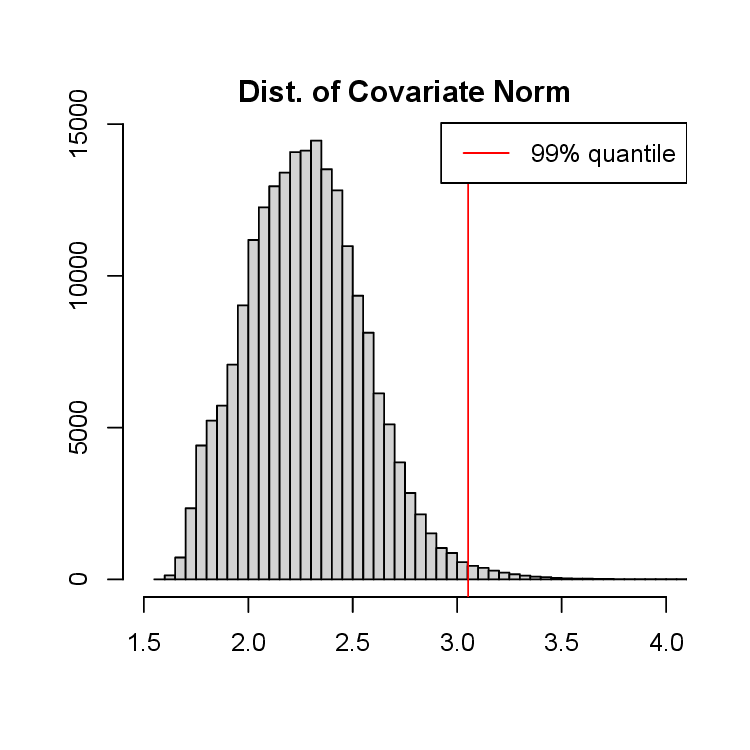}
	\vspace{-0.5cm}
    \end{subfigure}
    ~
    \begin{subfigure}{0.32\textwidth}
	\includegraphics[angle=0, width=1.1\textwidth]{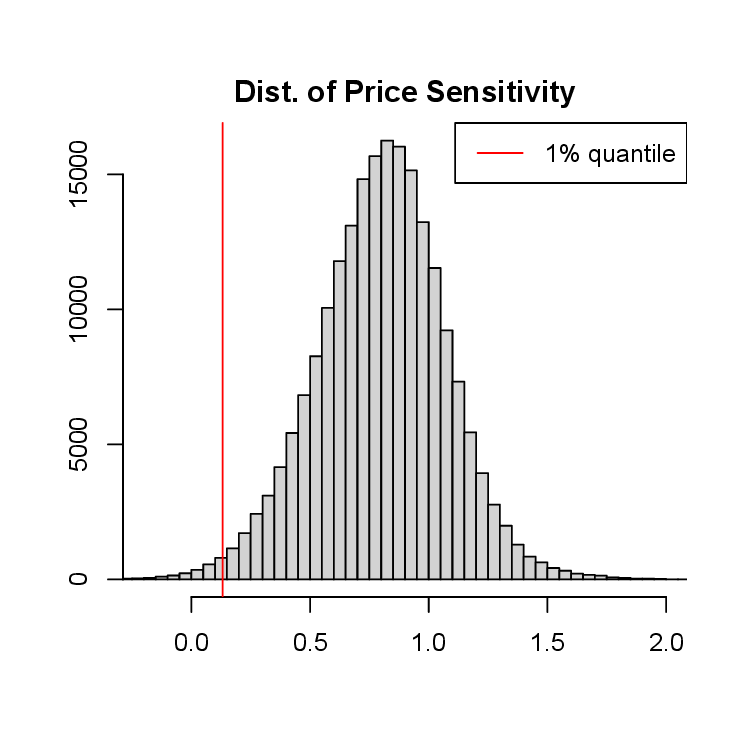}
	\vspace{-0.5cm}
    \end{subfigure}
    ~
    \begin{subfigure}{0.32\textwidth}
	\includegraphics[angle=0, width=1.1\textwidth]{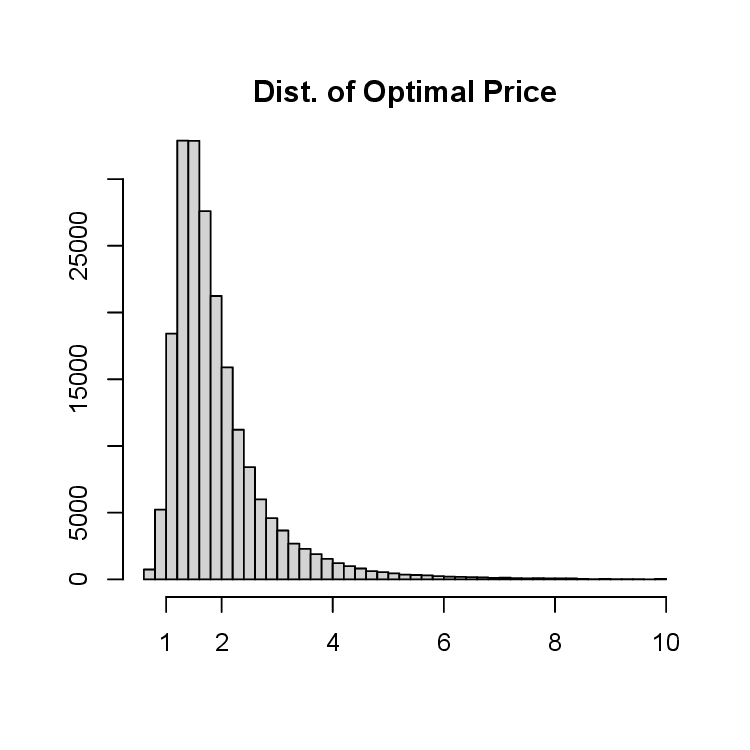}
	\vspace{-0.5cm}
    \end{subfigure}
    \vspace{-4mm}
    \caption{Histograms of covariate norm, price sensitivity and optimal price.}
    \label{fig:realdata_description}
\end{figure}

\noindent\textbf{Implementation details}. For ETC-LDP, the implementation follows that in \Cref{subsubsec:ETC_LDP}, while to handle the unknown horizon $T$, we use doubling trick. In particular, for the $q$th episode, which is of length $T_q=2^q$, we set the number of price experiments as $\tau_q=\min\big\{T_q, d\sqrt{2T_q}\log T_q/\epsilon\big\}$, where recall $\epsilon$ is the privacy parameter. To be more realistic, we set $\Theta=\{\theta: \|\theta-\theta^\circ\|\leq \sqrt{d}\}$, where for each experiment, $\theta^o$ is uniformly sampled from the unit ball centered at $\theta^*.$ In other words, $\theta^o$ is a perturbed version of $\theta^*$. Since ETC-LDP requires more price experiments, we set the price range $[l,u]$ to be $[0,3]$ for price experiments and keep the price range to be $[0,10]$ for price exploitation, which improves its performance.

\section{The supCB algorithm}\label{sec-supCB}
As discussed in the main text, a key element of supCB is to approximate the contextual dynamic pricing problem via contextual MAB based on discretization. In particular, we approximate the price interval $[l,u]$, which is a one-dimensional continuous action space, by a set of $K$ equispaced price points $\mathcal{P}_K=\{p^{(a)}\}_{a \in [K]}$, where $p^{(a)}=l+(a-1) (u-l)/(K-1)$ for $a\in [K]$. We thus have a generalized linear bandit~(GLB) with $K$ arms, where $x_{t,a}=(z_t^\top,-p^{(a)}\cdot z_t^\top)^\top$ is the context associated with arm $a \in [K]$ at time $t.$ We name it \textit{revenue}-GLB.  Unlike the standard GLB~\citep[e.g.][]{filippi2010parametric,li2017provably}, where the reward is defined as the conditional mean $\be (y_t|x_{t,a})=\psi'(x_{t,a}^\top\theta^*)$, our reward here is the more complicated revenue function $r(p^{(a)},z_t,\theta^*)=p^{(a)}\cdot\psi'(x_{t,a}^\top\theta^*).$ The optimal arm at time $t$ is therefore the arm associated with $p_t^o \in \argmax_{p\in\mathcal{P}_K}r(p,z_t,\theta^*)$.

Denote $\pi$ as a policy for the revenue-GLB based on $\mathcal{P}_K$. By simple algebra, its regret for the \textit{original} contextual dynamic pricing problem as defined in \eqref{eq:regret} can be rewritten as
\begin{align*}
        R_T^\pi & = \sum_{t=1}^{T} \{r(p_t^*, z_t, \theta^*)-r(p_t^\pi, z_t, \theta^*)\} =\underbrace{\sum_{t=1}^{T} \{r(p_t^*, z_t, \theta^*)-r(p_t^o, z_t, \theta^*)\}}_{\text{regret due to approximation error}} +\underbrace{\sum_{t=1}^{T} \{r(p_t^o, z_t, \theta^*)-r(p_t^\pi, z_t, \theta^*)\}}_{\text{regret due to revenue-GLB}}.
\end{align*}
Importantly, \Cref{lem:Lips} in the supplement shows that the class of revenue functions $\{r(p, z, \theta^*):z\in\mathcal Z\}$ is uniformly $L_r$-Lipschitz continuous with $L_r:=M_{\psi 1} + uu_b M_{\psi 2}$.  This implies that the approximation error is of order $O(T/K)$ and thus can be easily controlled.

Therefore, the key to an optimal regret $R_T^\pi$ is an algorithm that can efficiently bound the regret of the revenue-GLB. We then propose supCB in \Cref{algorithm:SupCB}, with a CB-GLM subroutine in \Cref{algorithm:CBsubroutine}. Similar to standard UCB algorithms~\citep[e.g.][]{lai1985asymptotically}, the idea of supCB is to balance exploration and exploitation via confidence bounds. However, to achieve the near-optimal $\widetilde O(\sqrt{dT})$ regret, supCB is substantially more involved than the standard UCB in its algorithm design. 

We first discuss the CB-GLM subroutine. In particular, based on an observed sample $\{y_t,x_t\}_{t\in\Psi}$, CB-GLM first estimates $\theta^*$ via the MLE $\widehat\theta$ and constructs the design matrix $V=\sum_{t\in \Psi}x_tx_t^\top$. It then provides an estimator $r_a=p^{(a)}\cdot \psi'( x_a^\top\widehat\theta)$ and a confidence bound $w_a=\alpha \|x_a\|_{V^{-1}}$, for the expected revenue $r(p^{(a)},z,\theta^*)=p^{(a)}\cdot \psi'( x_a^\top\theta^*)$ of arm $x_a=(z^\top,-p^{(a)}z^\top)^\top$ at a new context $z$. Note that given \Cref{thm:theta_err}(iii), by the smoothness of $\psi'(\cdot)$ and boundedness of $p^{(a)}$, we have that $w_a$ can serve as a valid high probability confidence bound for $|r_a- r(p^{(a)},z,\theta^*)|$ with an appropriately chosen $\alpha$.

\vspace{2mm}
\begin{algorithm}[ht]
\begin{algorithmic}
    \State \textbf{Input}: A new context vector $z \in \mathbb{R}^{d}$, price candidate index set $A \subset [K]$, and data index set $\Psi \subset [T]$,\\ confidence level parameter $\alpha > 0$.
    \medskip    
    \State a.\ Let $\widehat\theta$ be the MLE, i.e.~the maximizer of $L(\theta)=\sum_{i\in \Psi}\{y_i x_i^\top \theta - \psi(x_i^\top \theta)\}$.
    \State b.\ Set $V=\sum_{i\in \Psi}x_ix_i^\top$.
    \State c.\ For all $a\in A$, set $x_a=(z^\top,-p^{(a)}z^\top)^\top$, $r_a=p^{(a)}\cdot \psi'( x_a^\top\widehat\theta)$ and $w_a=\alpha \|x_a\|_{V^{-1}}$.
    \State d.\ Return $\{r_a, \, w_a\}_{a \in A}$. 
    \caption{The CB-GLM subroutine}
    \label{algorithm:CBsubroutine}
\end{algorithmic}
\end{algorithm}

However, \Cref{thm:theta_err} requires the crucial assumption that $\{y_t\}_{t\in\Psi}$ are \textit{conditionally independent} given $\{x_t\}_{t\in\Psi}$. Such an assumption in general does not hold when the sample $\{y_t,x_t\}_{t\in\Psi}$ are adaptively generated based on standard bandit algorithms such as UCB, as the outcome of $y_t$ will influence the future arm choice $x_{t'}$ for $t'>t.$ This issue is solved via a \textit{decomposition} technique developed in \cite{auer2002using}.  This technique is previously used in linear contextual bandits \citep[e.g.][]{auer2002using, chu2011contextual} and GLB \citep[e.g.][]{li2017provably} to create conditionally independent samples.

In supCB~(\Cref{algorithm:SupCB}), we adapt this technique to the revenue-GLB. In particular, in each round~$t$, supCB decomposes its decision process into $S=\lfloor \log_2 (T)\rfloor$ stages. At stage $s\in [S]$, in step~I, we first call CB-GLM (based on samples in $\mathcal{F}_s\cup\Psi_s$) to compute the estimated revenue $r_{t,a}^{(s)}$ and confidence bound $w_{t,a}^{(s)}$ for all remaining candidate arms. In step II, if $w_{t,a}^{(s)}>2^{-s}$ for some $a$~(i.e.\ uncertainty is high), supCB will choose arm $a$ for exploration, collect the result in $\Psi_s$ and move to round $t+1$. Otherwise, step III checks whether all $w_{t,a}^{(s)} \leq 1/\sqrt{T}$ and if satisfied, we simply pick the arm with highest $r_{t,a}^{(s)}$ and move to round $t+1$ as no exploration is needed. If not, the arms are filtered in step IV to ensure the price points passed to the next stage $s+1$ are close enough to the optimal one. Since all $w_{t,a}^{(s)}\leq 2^{-s}$ in step IV and $w_{t,a}^{(s)}$ is a high probability bound on $|r_{t,a}^{(s)}- r(p^{(a)},z_t,\theta^*)|$, we know that an arm $a$ cannot be optimal if $r_{t,a}^{(s)} < \max_{j\in A_s} r_{t,j}^{(s)}-2\cdot 2^{-s}$. Note that this decision process is guaranteed to terminate in at most $S$ stages since $1/\sqrt{T}> 2^{-S}$ by design.

Denote $\Psi_s(t)$ as the sample collected in stage $s$ before round $t$ and denote $\Psi_s^o(t)=\Psi_s(t)\cup \mathcal{F}_s.$ The key property of supCB is that due to the decomposition design, whether an arm $x_{t,a}$ is selected in stage~$s$ depends on historical data $\cup_{s'<s}\{y_i,x_i\}_{i\in \Psi_{s'}^o(t)}$ in \textit{previous} stages and the confidence bound $w_{t,a}^{(s)}$, which only depends on historical \textit{covariates} $\{x_i\}_{i \in \Psi_s^o(t)}$ in stage $s$. Importantly, the arm choice does not depend on past \textit{outcomes} $\{y_i\}_{i \in \Psi_s^o(t)}$ in stage $s.$ Intuitively, this provides the desired conditional independence of $\{y_i\}_{i\in \Psi_s^o(t)} $ given $\{x_i\}_{i\in \Psi_s^o(t)}$ and we refer to \Cref{lem:indp} for the formal proof.

Note that supCB starts with a pure exploration phase during which $p_t$ is drawn uniformly from~$\mathcal{P}_K.$ Denote $\mu_p$ and $\sigma_p^2$ as the mean and variance of the uniform distribution on $\mathcal{P}_K$ and denote $\Sigma_p=[1,-\mu_p; -\mu_p, \mu^2_p+\sigma^2_p] \in \mathbb{R}^{2 \times 2}$. For all $K \in \mathbb{N}_+$, it is easy to show that $\mu_p=(u+l)/2$, $\sigma_p^2 \geq (u-l)^2/12$ and that~$\Sigma_p$ is a positive definite matrix with
\begin{equation} \label{eq-Lp-def-supplement}
    \lambda_{\min}(\Sigma_p)\geq L_p:=(u-l)^2/[4(u^2+l^2+ul+3)]. 
\end{equation}

\vspace{2mm}
\begin{algorithm}[ht]
\begin{algorithmic}
    \State \textbf{Input}: Total rounds $T$, price interval $[l,u]$,     discretization rate $K$, exploration length $\tau$, confidence parameter $\alpha$. 
    \vskip 1mm
    
    \State \textbf{Initialization}: Discretize the price interval $[l,u]$ into $K$ equispaced points as in $\mathcal{P}_K=\{p^{(a)}\}_{a\in[K]}$. Set the number of stages $S=\lfloor \log_2 T\rfloor$ and set $\Psi_0=\Psi_1=\cdots=\Psi_S=\varnothing.$
    \vskip 1mm
    
    \For{$s \in [S]$}
    \State Set the exploration set as $\mathcal{F}_s=\{(s-1)\tau+1, (s-1)\tau+2, \cdots, s\tau\}$.
    \For{$t \in \mathcal{F}_s$}
    \State Randomly choose $p_t \in \mathcal{P}_K$, record $y_t$ and $x_t=(z_t^\top,-p_tz_t^\top)^\top$.
    \EndFor
    \EndFor
    \vskip 1mm

    \For{$t= S\tau+1,S\tau+2,\cdots, T$}
    \State Initialize $A_1=[K]$, $s=1$ and $a_t = \mathrm{NULL}$.
    \While{$a_t= \text{NULL}$}
        \State I.\ ~ Run CB-GLM with $\alpha$ and $\Psi_s\cup \mathcal{F}_s$ to obtain $\{r_{t,a}^{(s)}, \, w_{t,a}^{(s)}\}_{a \in A_s}$.
        
        \State II.\ \hspace{0.3mm} If $w_{t,a}^{(s)}>2^{-s}$ for some $a\in A_s$, set $a_t=\arg\max_{a\in A_s}w_{t,a}^{(s)} $, update $\Psi_{s}=\Psi_{s}\cup \{t\}.$
        
        \State III.\ Else if $w_{t,a}^{(s)}\leq 1/\sqrt{T}$ for all $a\in A_s$,
        set $a_t=\arg\max_{a\in A_s} r_{t,a}^{(s)}$ and update $\Psi_{0}=\Psi_{0}\cup \{t\}.$

        \State IV.\ Else if $w_{t,a}^{(s)}\leq 2^{-s}$ for all $a\in A_s$,
update $A_{s+1}=\left\{a\in A_s\big\vert r_{t,a}^{(s)}\geq \max_{j\in A_s} r_{t,j}^{(s)}-2\cdot 2^{-s}\right\}$ and update $s \leftarrow s+1.$
    
    \EndWhile 
    \State Set $p_t=p^{(a_t)}$, record $y_t$ and $x_t=(z_t^\top,-p_tz_t^\top)^\top$.
    
    \EndFor
    \caption{The supCB algorithm for dynamic pricing}
    \label{algorithm:SupCB}
\end{algorithmic}
\end{algorithm}

The pure exploration phase creates an i.i.d.\ sample of size $\tau$~(collected into $\mathcal{F}_s$) for each stage $s\in [S]$, which together with \eqref{lambda_min} ensures that the minimum eigenvalue of the design matrix based on $\Psi_s^o(t)=\Psi_s(t)\cup \mathcal{F}_s$ surpasses the threshold required in \Cref{thm:theta_err}(iii). \Cref{thm:ucb_regret} establishes both high-probability and expectation upper bounds on the regret of supCB when $T$ is known. 

\vspace{-2mm}
\begin{thm}\label{thm:ucb_regret}
Suppose \Cref{assum_feature} holds. For any $\delta \in (0,1),$ set $K=\sqrt{T/d}/\log(T)$, $\tau=\sqrt{dT}$ and $\alpha = {3\sigma u M_{\psi2}}/{\kappa} \cdot \sqrt{\log(3TKS/\delta)}$. Recall $S=\lfloor \log_2(T) \rfloor$. Provided that
    \begin{align}\label{eq:minT_thm_short}
        T\geq \left[B_{S1} (d\lambda_z^4)^{-1} ((\log(d))^2+(\log(T))^2) \right]\vee \left[B_{S2} (d\lambda_z^2)^{-1}(d^4+ (\log(3TKS/\delta))^2) \right],        
    \end{align}
    we have that,  with probability at least $1-\delta-2\log (T)/T$, the regret of the supCB algorithm in \Cref{algorithm:SupCB} is upper bounded by $R_T \leq B_{S3}\cdot \sqrt{dT\log(T) \log(T/\delta)\log (T/d)}$,
    where $B_{S1},B_{S2},B_{S3}>0$ are absolute constants that only depend on quantities in \Cref{assum_feature}. In particular, setting~$\delta=1/\sqrt{T}$, we have that $\mathbb E(R_T)\lesssim\log^{3/2}(T)\sqrt{dT}.$ 
\end{thm}
\vspace{-2mm}

The details of $B_{S1},B_{S2},B_{S3}$ can be found in the proof of \Cref{thm:ucb_regret} in the supplement. \Cref{thm:ucb_regret} states that when the sample size $T$ is sufficiently large, supCB can achieve a regret of order $O(\sqrt{dT}\log^{3/2}(T))$, which is optimal up to logarithmic terms (see the lower bound $\Omega(\sqrt{dT})$ in \Cref{thm_lb} later). To our knowledge, this improves the best available regret in the contextual dynamic pricing literature~\citep[e.g.][]{ban2021personalized,wang2021dynamic} by a factor of $\sqrt{d}$.

As discussed in the introduction, there are two types of contextual MAB problems. For the first type, its arm space $\mathcal{X}$ is a $d$-dimensional (uncountable) compact subset of $\mathbb R^d$ and the optimal regret $\widetilde O(d\sqrt{T})$ can be achieved via standard UCB~\citep{dani2008stochastic,abbasi2011improved}. The second type has a finite $K$ number of arms in each round $t$, where the optimal regret is $\widetilde O(\sqrt{dT\log (K)})$~\citep{auer2002using,chu2011contextual}. The key insight of \Cref{thm:ucb_regret} and supCB is that the contextual dynamic pricing problem can be essentially viewed as a contextual MAB of the second type. The intuition is that, despite being uncountable, the action space $[l,u]$ of dynamic pricing is of dimension \textit{one} regardless of the context dimension $d$. Moreover, thanks to the uniform Lipschitz continuity of the revenue function, we can discretize the action space $[l,u]$ into $K$ arms and achieve a controllable approximation error as long as $K$ is of polynomial order in $T$ (i.e.\ $O(\sqrt{T})$, independent of $d$).

\textbf{Unknown T}. When $T$ is unknown, we run supCB with the standard doubling trick \citep[e.g.][]{auer1995gambling} in the bandit literature. In particular, we partition~$\mathbb{N}_+$ into non-overlapping episodes, with the $k$th episode of length $E_k=2^{k}$. We run \Cref{algorithm:SupCB} with $T=E_k$ for the $k$th episode and then restart at the next episode. To conserve space, we refer to \Cref{thm:ucb_regret_unknownT} for its theoretical guarantee. See also \Cref{algorithm:ETC-Doubling} later for a detailed implementation of ETC with the doubling trick.

{\color{black}
\subsection{A modified supCB for adversarial contexts}\label{sec-supp-supCB-adversarial}
In the following, we propose a (slightly) modified supCB algorithm based on regularized MLE for adversarial contexts. It requires a stronger assumption on the convexity of the link function $\psi(\cdot)$ as is documented in \Cref{assum_feature_adv}(b). In particular, a linear demand model satisfies this assumption as $\psi(x)=1/2x^2$ in such case.

\begin{ass}\label{assum_stronger_psi}
    (a) The feature vector $\{z_t\}_{t \in [T]} \subset \mathcal{Z}\subseteq \{z\in \mathbb R^d: \, \|z\|\leq 1\}$ is an arbitrary sequence of bounded vectors. (b) The function $\psi(\cdot)$ is three times continuously differentiable and its second order derivative $\psi''(a)\geq c_\psi>0$ for all $a\in\mathbb R$, where $c_\psi\leq 1$ is an absolute constant. (c) The true model parameter is bounded with $\|\theta^*\|\leq c_\theta$ for some absolute constant $c_\theta$.
\end{ass}

\textbf{Non-asymptotic bounds for regularized MLE}: For $\tau \in\mathbb{N}_+$, given a sample $\{(y_t,x_t)\}_{t\in[\tau]}$, the log-likelihood of the GLM takes the form $L(\theta)=\sum_{t=1}^\tau \{y_t x_t^\top \theta - \psi(x_t^\top \theta)\}$. We define the regularized (negative) log-likelihood function
\begin{align}\label{eq:ridge_loglik}
    \widetilde{L}(\theta)=-L(\theta)+1/2\|\theta\|_2^2=-\sum_{t=1}^\tau\{y_tx_t^\top\theta -\psi(x_t^\top \theta)\}+1/2\|\theta\|_2^2.
\end{align}
Note that $\widetilde{L}(\theta)$ is strongly convex due to the $\|\theta\|_2^2$ term and thus always has a unique global minimizer~$\widetilde\theta_\tau$ over $\mathbb{R}^{2d}$ for any sample $\{(y_t,x_t)\}_{t\in[\tau]}$. Denote $\widetilde{V}_\tau=\sum_{t=1}^\tau x_tx_t^\top + I_{2d}$ as the regularized design matrix.

\Cref{lem:theta_err_ridge} establishes a non-asymptotic bounds for the regularized MLE $\widetilde{\theta}_\tau$ with \Cref{assum_feature_adv}. It can be viewed as the counterpart of \Cref{thm:theta_err}(iii) under the adversarial contexts.

\vspace{-2mm}
\begin{lem}\label{lem:theta_err_ridge}
Suppose that \Cref{assum_feature}(c)-(d) and \Cref{assum_feature_adv} hold and $\{y_t\}_{t\in[\tau]}$ are {conditionally independent} given $\{x_t\}_{t\in[\tau]}$ such that $f(y_1,\cdots,y_\tau|x_1,\cdots,x_\tau)=\prod_{t=1}^\tau f(y_t|x_t;\theta^*)$, where $f(y_t|x_t;\theta^*)$ is the true GLM model in \eqref{eq:GLM}. We have that for any $\tau \geq 0$, any $\delta \in (0,1)$ and any $x\in \mathbb R^{2d}$, with probability at least $1-\delta$, it holds that
$$\big|x^\top (\widetilde{\theta}_\tau-\theta^*)\big|\leq  3\sigma c_{\psi}^{-1}\sqrt{\log(1/\delta)}\cdot \|x\|_{\widetilde V_\tau^{-1}}.$$
\end{lem}

\textbf{A modified supCB algorithm}: Based on \Cref{lem:theta_err_ridge}, we modify the CB-GLM subroutine in \Cref{algorithm:CBsubroutine} into a regularized version. In particular, instead of using the log-likelihood $L(\theta)$ and MLE $\widehat{\theta}$, we use the regularized log-likelihood $\widetilde{L}(\theta)$ and MLE $\widetilde{\theta}$
in step a of CB-GLM and further modify the design matrix $V$ to its regularized version $\widetilde{V}=\sum_{i\in\Psi} x_ix_i^\top + I_{2d}$ in step b of CB-GLM. Moreover, due to the use of regularized log-likelihood, we no longer require the exploration stage in the original supCB algorithm and thus set $\tau\equiv 0$ in \Cref{algorithm:SupCB}.

\Cref{thm:ucb_regret_adv} gives theoretical guarantees for the modified supCB algorithm under adversarial contexts. In particular, it shows that the modified supCB algorithm can achieve the near-optimal regret of order $\widetilde{O}(\sqrt{dT})$ under the adversarial contexts with the additional \Cref{assum_feature_adv}.

\vspace{-2mm}
\begin{thm}\label{thm:ucb_regret_adv}
Suppose \Cref{assum_feature} holds. For any $\delta \in (0,1),$ set $K=\sqrt{T/d}/\log(T)$, $\tau=0$ and $\alpha = {3\sigma u M_{\psi2}}/{c_\psi} \cdot \sqrt{\log(3TKS/\delta)}$. Recall $S=\lfloor \log_2(T) \rfloor$. For any $T$, we have that,  with probability at least $1-\delta-2\log (T)/T$, the regret of the modified supCB algorithm is upper bounded by $R_T \leq B_{S3}\cdot \sqrt{dT\log(T) \log(T/\delta)\log (T/d)}$,
where $B_{S1},B_{S2},B_{S3}>0$ are absolute constants that only depend on quantities in \Cref{assum_feature} and \Cref{assum_feature_adv}. In particular, setting~$\delta=1/\sqrt{T}$, we have that $\mathbb E(R_T)\lesssim\log^{3/2}(T)\sqrt{dT}.$ 
\end{thm}

The proof of \Cref{thm:ucb_regret_adv} follows the same structure as that of \Cref{thm:ucb_regret}. In particular, the results in Lemmas \ref{lem:Lips}, \ref{lem:indp} and \ref{lem:gap_reward} continue to hold for the modified supCB algorithm. The only modification is to re-establish the high probability confidence bound results in \Cref{lem:HPCB} (in particular, \eqref{eq-def-event-EX}) based on the newly developed \Cref{lem:theta_err_ridge} (previously established based on \Cref{thm:theta_err}(c)), which is straightforward. The detailed proof is thus omitted to conserve space.

\subsection{Technical details}
\begin{proof}[\textbf{Proof of \Cref{lem:theta_err_ridge}}]
In the proof, we drop the subscript $\tau$ in $\theta$ for notational simplicity. By the strong convexity and differentiability of $\widetilde{L}(\theta)$, we have that $\widetilde{L}'(\widetilde{\theta})=0$, where recall $\widetilde{\theta}$ is the unique global minimizer of $\widetilde{L}(\theta)$ over $\mathbb{R}^{2d}.$

In particular, for any $\theta\in \mathbb{R}^{2d}$, by Taylor expansion, we have that
\begin{align*}
    \widetilde L(\theta)-\widetilde L(\theta^*)=\widetilde L'(\theta^*)^\top(\theta-\theta^*)+1/2 (\theta-\theta^*)^\top\widetilde L''(\theta^+)(\theta-\theta^*),
\end{align*}
for some $\theta^+$ between $\widetilde{\theta}$ and $\theta^*.$ Thus, we have $\widetilde L(\theta)>\widetilde L(\theta^*)$ if it holds that
\begin{align*}
    (\theta-\theta^*)^\top \widetilde L''(\theta^+)(\theta-\theta^*) >   2 |\widetilde L'(\theta^*)^\top(\theta-\theta^*)|,
\end{align*}
By Cauchy-Schwartz inequality and the fact that $\widetilde L''(\theta^+)=\sum_{t=1}^\tau\psi''(x_t^\top \theta^+)x_tx_t^\top + I_{2d} \geq I_{2d}$, the above equation holds if $\|\theta-\theta^*\|>2\|\widetilde L'(\theta^*)\|=2\|\sum_{t=1}^\tau (y_t-\psi'(x_t^\top\theta^*)x_t\|=2\|\sum_{t=1}^\tau \varepsilon_t x_t\|.$ In other words, the unique minimizer $\widetilde{\theta}$ of $\widetilde{L}(\theta)$ over $\mathbb{R}^{2d}$ can always be found in the domain $\big\{\theta: \|\theta-\theta^*\|\leq \|\sum_{t=1}^\tau \varepsilon_t x_t\|\big\}$ and we have that $\widetilde{L}'(\widetilde{\theta})=0$.

Furthermore, by a Taylor expansion of $\widetilde{L}'(\widetilde{\theta})$, we have that $0=\widetilde{L}'(\widetilde{\theta})=\widetilde{L}'(\theta^*) + \widetilde{L}''(\theta^+)(\widetilde{\theta}-\theta^*)$, for some $\theta^+$ between $\widetilde{\theta}$ and $\theta^*.$ Rearranging the equation, we have that
\begin{align*}
    \widetilde{L}''(\theta^+)(\widetilde{\theta}-\theta^*)=& \left[\sum_{t=1}^\tau\psi''(x_t^\top \theta^+)x_tx_t^\top + I_{2d} \right] (\widetilde{\theta}-\theta^*)\\
    =&-\widetilde{L}'(\theta^*)=\left\{\sum_{t=1}^\tau (y_t-\psi'(x_t^\top\theta^*)x_t\right\}-\theta^* =\left\{\sum_{t=1}^\tau\varepsilon_t x_t\right\}-\theta^*.
\end{align*}
Denote $\widetilde{\Sigma}_\tau=\sum_{t=1}^\tau \psi''(x_t^\top \theta^+)x_tx_t^\top + I_{2d}$ and $S_\tau=\sum_{t=1}^\tau\varepsilon_t x_t$. We have that $\widetilde{\theta}-\theta^*=\widetilde{\Sigma}_\tau^{-1}(S_\tau -\theta^*)$. In addition, by \Cref{assum_feature_adv}(ii), we have that
\begin{align}\label{eq:bound_of_adv_Sigma}
    \widetilde{\Sigma}_\tau\geq c_\psi \widetilde V_\tau \text{ and } \widetilde{\Sigma}_\tau^{-1}\leq c_\psi^{-1} \widetilde V_\tau^{-1},
\end{align}
 where recall $\widetilde{V}_\tau=\sum_{t=1}^\tau x_tx_t^\top +I_{2d}$.

Therefore, we have that
\begin{align*}
    \big|x^\top (\widetilde{\theta}_\tau-\theta^*)\big| \leq \big|x^\top\widetilde{\Sigma}_\tau^{-1}S_\tau\big| +\big|x^\top\widetilde{\Sigma}_\tau^{-1}\theta^*\big|\leq \big|x^\top\widetilde{\Sigma}_\tau^{-1}S_\tau\big| +c_\theta\big\|x^\top\widetilde{\Sigma}_\tau^{-1}\big\|.
\end{align*}
We now bound the two terms one by one. For the first term, we have $x^\top\widetilde{\Sigma}_\tau^{-1}S_\tau=\sum_{t=1}^\tau x^\top \widetilde{\Sigma}_\tau^{-1}x_t\varepsilon_t.$ Therefore, by \Cref{assum_feature}(d) and Azuma-Hoeffding inequality, we have that
\begin{align*}
    \mathbb{P}\left\{|x^\top\widetilde{\Sigma}_\tau^{-1}S_\tau|> c\right\}\leq 2\exp\left\{-\frac{c^2}{2\sum_{t=1}^\tau (x^\top \widetilde{\Sigma}_\tau^{-1}x_t)^2\sigma^2} \right\}\leq 2\exp\left\{-\frac{c^2}{2\sigma^2 c_\psi^{-2}\|x\|_{\widetilde{V}_\tau^{-1}}^2} \right\},
\end{align*}
where the second inequality follows from that
\begin{align*}
    \sum_{t=1}^\tau (x^\top \widetilde{\Sigma}_\tau^{-1}x_t)^2=x^\top \widetilde{\Sigma}_\tau^{-1} \sum_{t=1}^\tau x_tx_t^\top\widetilde{\Sigma}_\tau^{-1} \leq x^\top \widetilde{\Sigma}_\tau^{-1} \widetilde{V}_\tau\widetilde{\Sigma}_\tau^{-1}x\leq c_\psi^{-2} x^\top\widetilde{V}_\tau^{-1}x = c_\psi^{-2}\|x\|_{\widetilde{V}_\tau^{-1}}^2.
\end{align*}
Setting $c=\sqrt{2}\sigma c_{\psi}^{-1}\|x\|_{\widetilde{V}_\tau^{-1}}\sqrt{\log(1/(4\delta))}$, we have with probability at least $1-\delta$, it holds that
\begin{align*}
    |x^\top\widetilde{\Sigma}_\tau^{-1}S_\tau|\leq \sqrt{2}\sigma c_{\psi}^{-1}\sqrt{\log(1/(2\delta))}\|x\|_{\widetilde{V}_\tau^{-1}}.
\end{align*}

For the second term, we have
\begin{align*}
    c_\theta\big\|x^\top\widetilde{\Sigma}_\tau^{-1}\big\|\leq c_\theta\sqrt{x^\top\widetilde{\Sigma}_\tau^{-1}\widetilde{\Sigma}_\tau^{-1}x}\leq c_\theta\sqrt{x^\top\widetilde{\Sigma}_\tau^{-1}\widetilde{V}_\tau\widetilde{\Sigma}_\tau^{-1}x}\leq c_\theta c_\psi^{-1}\|x\|_{\widetilde{V}_\tau^{-1}}.
\end{align*}
Combining the two results together finishes the proof.
\end{proof}

}

{\color{black}
\section{Dynamic pricing under adversarial contexts}\label{sec:adversarial_supplement} 

In this section, we further study dynamic pricing under adversarial contexts. In particular, we relax the stochastic contexts condition in \Cref{assum_feature}(a)-(b) of the main text to \Cref{assum_feature_adv}, where we only assume $\{z_t\}$ is a sequence of bounded vectors.

\begin{ass}\label{assum_feature_adv}
(a) The feature vector $\{z_t\}_{t \in [T]} \subset \mathcal{Z}\subseteq \{z\in \mathbb R^d: \, \|z\|\leq 1\}$ is an arbitrary sequence of bounded vectors. (b).\ The parameter space $\Theta$ is compact such that $\|\Theta\|\leq c_\theta$.
\end{ass}

Due to the adversarial nature of $\{z_t\}$, the design matrix of the log-likelihood function $L(\theta)$ can be rank deficient. Therefore, for model estimation, we instead rely on the regularized negative log-likelihood function 
$$\widetilde L_t(\theta)=-\left\{\sum_{i=1}^t y_ix_i^\top \theta -\psi(x_i^\top \theta)\right\}+\lambda\|\theta\|_2^2,$$
where $\lambda>0$ is a constant (e.g.\ $\lambda=1$).

\Cref{algorithm:UCB_adv} proposes a novelly designed UCB algorithm for dynamic pricing under the adversarial contexts. We further establish its theoretical guarantees in \Cref{thm:UCB_adv}, where we show that it can achieve a  regret upper bound of order $\widetilde{O}(d\sqrt{T})$ under adversarial contexts, matching prior works in the literature~\citep{wang2021dynamic}.

\vspace{2mm}
\begin{algorithm}[ht]
{\color{black}
\begin{algorithmic}
    \State \textbf{Input}: Tuning parameter $\lambda$, price interval $[l,u]$, confidence parameter $\alpha$, parameter space $\Theta$.
    \vskip 1mm
    
    \State \textbf{Initialization}: Set $V_0=\lambda I_{2d}$.
    \vskip 1mm

    \For{$t= 1, 2, \cdots, T$}
    \State a.\ Solve the regularized regression and obtain $\widehat{\theta}_t=\argmin_{\theta\in\Theta}\widetilde L_{t-1}(\theta)$.
    \State b.\ Solve the one-dimensional optimization
    \begin{align}\label{eq:ucb_price}
        p_t=\argmax_{p\in [l,u]} U(p,z_t):= p\cdot \psi'\left( x_t^{(p)\top}\widehat\theta_t+ \alpha \|x_t^{(p)}\|_{V_{t-1}^{-1}}\right),
    \end{align}
    \vskip -2mm
    \hspace{1.8mm} where we denote $x_t^{(p)}=(z_t^\top,-pz_t^\top)^\top$.
    \vskip 1mm
    \State c.\ Offer price $p_t$, record $y_t$ and $x_t=(z_t^\top,-p_tz_t^\top)^\top$. Set $V_t=V_{t-1}+ x_tx_t^\top$.
    \EndFor
    \caption{\color{black} The UCB algorithm for dynamic pricing}
    \label{algorithm:UCB_adv}
\end{algorithmic}}
\end{algorithm}


\begin{thm}\label{thm:UCB_adv}
Suppose Assumption \ref{assum_feature}(c)-(e) and Assumption \ref{assum_feature_adv} hold. Set $\lambda=1$ and $\alpha=B_{A1} \sqrt{ d\log(T(1+u^2)/\delta)}$. We have for any $T>0$, the regret of UCB is upper bounded by $R_T \leq B_{A2}\cdot d\sqrt{T}\log (T/\delta)$,
where $B_{A1}, B_{A2}>0$ are absolute constants that only depend on quantities in Assumption \ref{assum_feature}(c)-(e) and \Cref{assum_feature_adv}.
\end{thm}

The key novelty of our proposed UCB algorithm lies in its design of the upper confidence bound for the revenue function (i.e.\ \eqref{eq:ucb_price}), which takes a much simpler structure than existing UCB algorithms in prior works (\cite{wang2021dynamic}, see below for more discussions) and thus is computationally more efficient. In particular, simple algebra shows that the upper confidence bound $U(p,z_t)$ in \eqref{eq:ucb_price} takes the form $U(p,z_t)=p\cdot \psi'\left(a_{t0}+a_{t1}p+a_{t2}p^2\right)$, for some $a_{t0},a_{t1},a_{t2}$, which is a {one}-dimensional function of $p$ and can therefore be optimized efficiently. 

\textbf{Intuition of our UCB}: We now give some intuition of our UCB $U(p,z_t)$ in \eqref{eq:ucb_price}. Recall that the true revenue function takes the form $r(p,z_t,\theta^*)=p\cdot\psi'(x_t^{(p)\top}\theta^*)$, where we denote $x_t^{(p)}=(z_t^\top,-pz_t^\top)^\top$. A key observation is that due to the structure of $r(p,z_t,\theta^*)$ and the fact that $p>0$, an upper bound of $\psi'(x_t^{(p)\top}\theta^*)$ directly leads to a upper bound of $r(p,z_t,\theta^*)$. Moreover, due to the intrinsic convexity of GLM (i.e.\ $\psi''(\cdot)>0$, see \Cref{assum_feature}(e)), we have $\psi'(\cdot)$ is a strictly increasing function and thus an upper bound of $x_t^{(p)\top}\theta^*$ directly leas to an upper bound of $\psi'(x_t^{(p)\top}\theta^*)$. This argument leads to our upper confidence bound $U(p,z_t)$, as we show that, with a properly chosen parameter $\alpha$, it holds that $x_t^{(p)\top}\widehat\theta_t+ \alpha \|x_t^{(p)}\|_{V_{t-1}^{-1}}$ is a high probability upper bound for $x_t^{(p)\top}\theta^*$ for all $p\in [l,u]$ and $t\in [T]$.

\textbf{Comparison with existing works}: 
In the literature, \cite{wang2021dynamic} is the only work that considers dynamic pricing under adversarial contexts with GLM demand models. In particular, \cite{wang2021dynamic} proposes a (computationally difficult) UCB algorithm and shows its regret upper bound is $\widetilde{O}(d\sqrt{T}).$ Their UCB takes the form (in our notation)
\begin{align}\label{eq:WTL}
    p_t=\argmax_{p\in[l,u]} \widetilde U(p,z_t):= \max_{\theta=(\alpha,\beta)\in \mathcal{C}_t} p\cdot \psi'(z_t^\top \alpha-z_t^\top\beta p) 
\end{align}
where $C_t=\{\theta\in\Theta: \|\theta-\widehat{\theta}_t\|_{V_{t-1}}\leq \alpha\}$ is the confidence set for $\theta$ at time $t$. As noted by \cite{wang2021dynamic}, the inner optimization of $\widetilde U(p,z_t)$ is a $2d$-dimensional non-convex optimization and is difficult to solve. To bypass the computational difficulty, the authors propose to \textit{approximate} the optimization in \eqref{eq:WTL} via a heuristic Monte Carlo approach, which then incurs a worse regret upper bound than $\widetilde{O}(d\sqrt{T})$ and thus is unsatisfactory.

An an alternative, \cite{wang2021dynamic} further proposes a Thompson Sampling based algorithm that is computationally feasible and achieves the $\widetilde{O}(d\sqrt{T})$ regret upper bound. However, it requires a strong assumption~(see Assumption 4 therein) that there exists an absolute constant $\underline{g}>0$ such that 
\begin{align*}
    \inf_{z\in \mathcal{Z}, p \in [l,u], \theta\in\widetilde{\Theta}}\psi''(z^\top \alpha-z^\top \beta p)>\underline{g}>0,
\end{align*}
where $\widetilde{\Theta}=\left\{\theta \in \mathbb R^{2d}: \|\theta-\theta^*\|_2\leq c\cdot \sqrt{d\log T} \right\}$ for some constant $c>0.$ Note this may not hold even for logistic regression since $\widetilde{\Theta}$ is \textit{not} compact but expands with $d,T$. In addition, our numerical experiments in the next section show that our UCB algorithm is computationally more efficient than the TS algorithm in \cite{wang2021dynamic}.

\subsection{Numerical experiments}
In this section, we conduct numerical experiments to compare the finite sample performance and computational time of our proposed UCB algorithm in \Cref{algorithm:UCB_adv} and the Thompson Sampling based algorithm~(hereafter TS-WTL) in \cite{wang2021dynamic}. Note that we choose not to implement the UCB algorithm in \cite{wang2021dynamic} since it requires a $2d$-dimensional non-convex optimization and thus is computationally highly demanding. Moreover, the numerical result in \cite{wang2021dynamic} indicates that their heuristic Monte Carlo based UCB algorithm has very similar performance in terms of regret as their Thompson Sampling based algorithm.

For the demand model, we follow \Cref{subsec:simu_setting} of the main text and consider the logistic regression setting as in \eqref{eq:logistic_pricing} with simulation scenario {(S1)}, where $\alpha^*=1.6\cdot\mathbf{1}_d/\sqrt{d}$ and $\beta^*=\mathbf{1}_d/\sqrt{d}$. However, the feature vector $\{z_t=(z_{t,1},\cdots,z_{t,d})^{\top}\}_t$ are not i.i.d.\ but adversarial. In particular, we follow the setting in \cite{wang2021dynamic} and consider two scenarios.

\textbf{(A1)}.\ For $t\leq T/2$, we set $z_t=(z_t^*, 1.5\cdot \mathbf{1}_{d/2})/\sqrt{d}$, and for $t>T/2$, we set $z_t=(1.5\cdot \mathbf{1}_{d/2}, z_t^*)/\sqrt{d}$, where $z_t^*$ is a $d/2$-dimensional random vector with $\{z_{t,i}^*\}_{i=1}^{d/2}$ being i.i.d.\ uniform$(0, 3)$ random variables. 

\textbf{(A2)}.\ For all $t\in[T]$, we set $z_t=(z_t^*, 1.5\cdot \mathbf{1}_{d/2})/\sqrt{d}$, where $z_t^*$ is a $d/2$-dimensional random vector with $\{z_{t,i}^*\}_{i=1}^{d/2}$ being i.i.d.\ uniform$(0, 3)$ random variables. 

By design, the feature vector $\{z_t\}$ is adversarial under both (A1) and (A2), where its covariance matrix is always of rank $d/2$ at any time $t\in [T]$, making accurate estimation of the model parameter infeasible. We set the confidence parameter as $\alpha=\sqrt{d\log T}\log(\log T)/10$ for the UCB algorithm in \Cref{algorithm:UCB_adv} and follow \cite{wang2021dynamic} for the implementation of the TS-WTL algorithm. Similar to \cite{wang2021dynamic}, we set the horizon $T\in \{1000, 2500, 5000,7500, 10000\}$ and set the dimension $d\in \{4,8,12,16\}$. 

\Cref{fig:A1} reports the performance of UCB and TS-WTL under (A1). Specifically, \Cref{fig:A1}(left) reports the mean regret $\overline{R}_{T,d}$ (and the confidence interval $I_{T,d}$) of TS-WTL and UCB at different $(T,d)$ under (A1). As can be seen, despite the adversarial nature of the feature vectors, for both UCB and TS-WTL, given a fixed dimension $d$, the regret scales sublinearly with respect to $T$, providing numerical evidence for Theorem \ref{thm:UCB_adv}. In addition, given a fixed horizon $T$, the regret of both algorithms increase with the dimension $d$. Moreover, UCB consistently achieves lower regret than TS-WTL. For illustration, \Cref{fig:A1}(middle) further gives the boxplot of $\{R_{T,d}^{(i)}\}_{i=1}^{500}$ for both UCB and TS-WTL at each $d\in \{4,8,12,16\}$ with a fixed $T=10000.$ \Cref{fig:A1}(right) reports the average computational time (and the confidence interval) of TS-WTL and UCB at different $(T,d)$ under (A1), where UCB is seen to be computationally more efficient than TS-WTL. The performance of UCB and TS-WTL under (A2) is reported in \Cref{fig:A2} of the supplement, where similar phenomenon is observed.

\begin{figure}[H]
\vspace{-1mm}
\hspace*{-6mm}
    \begin{subfigure}{0.32\textwidth}
	\includegraphics[angle=0, width=1.1\textwidth]{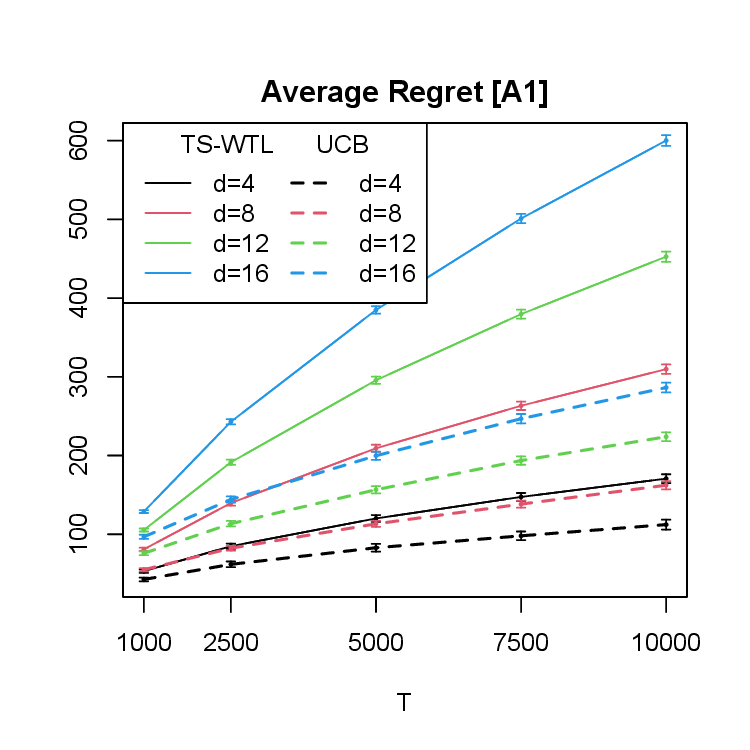}
	\vspace{-0.5cm}
    \end{subfigure}
    ~
    \begin{subfigure}{0.32\textwidth}
	\includegraphics[angle=0, width=1.1\textwidth]{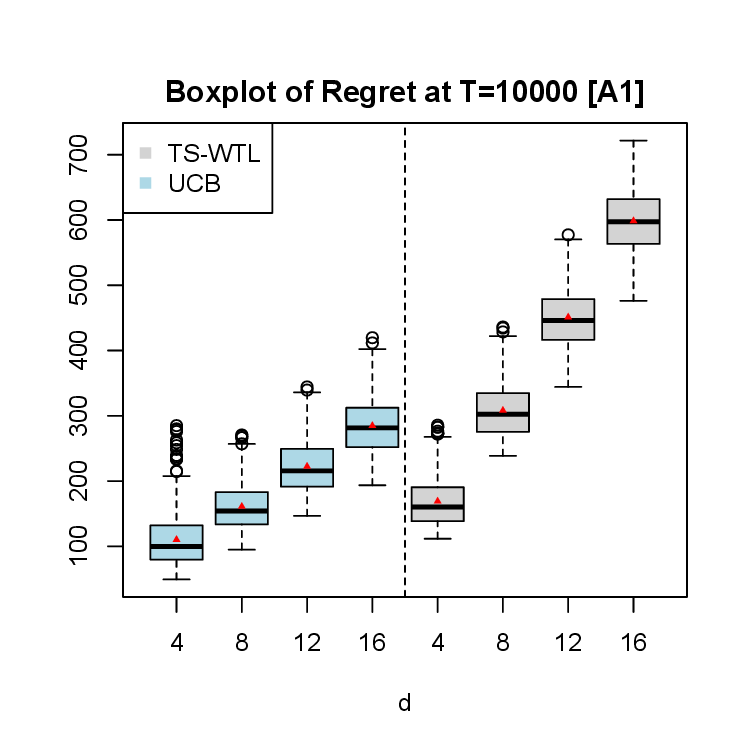}
	\vspace{-0.5cm}
    \end{subfigure}
    ~
    \begin{subfigure}{0.32\textwidth}
	\includegraphics[angle=0, width=1.1\textwidth]{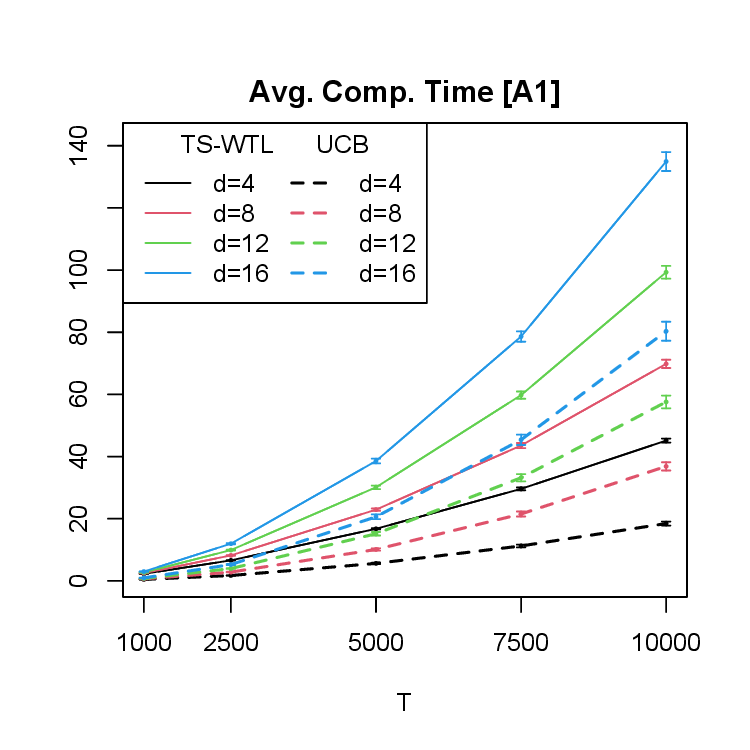}
	\vspace{-0.5cm}
    \end{subfigure}
    \caption{{\color{black} Performance of TS-WTL and UCB under (A1). [Left]:  Mean regret (with C.I.) under different $(d,T)$.  [Middle]: Boxplot of regrets at different $d$ (with $T=10000$). [Right]: Mean comp. time (with C.I.) under different $(d,T)$.}}
    \label{fig:A1}
\end{figure}

\begin{figure}[H]
\hspace*{-6mm}
    \begin{subfigure}{0.32\textwidth}
	\includegraphics[angle=0, width=1.1\textwidth]{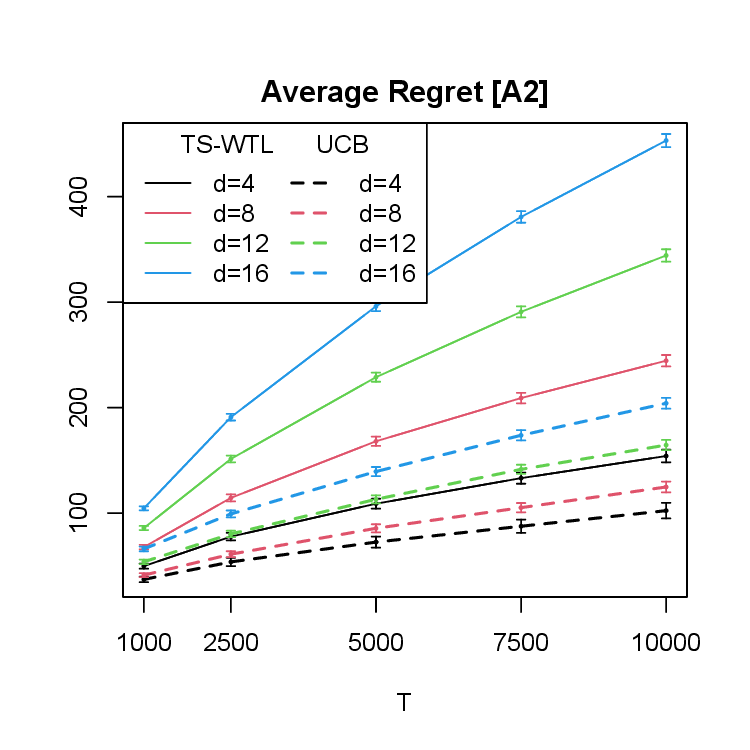}
	\vspace{-0.5cm}
    \end{subfigure}
    ~
    \begin{subfigure}{0.32\textwidth}
	\includegraphics[angle=0, width=1.1\textwidth]{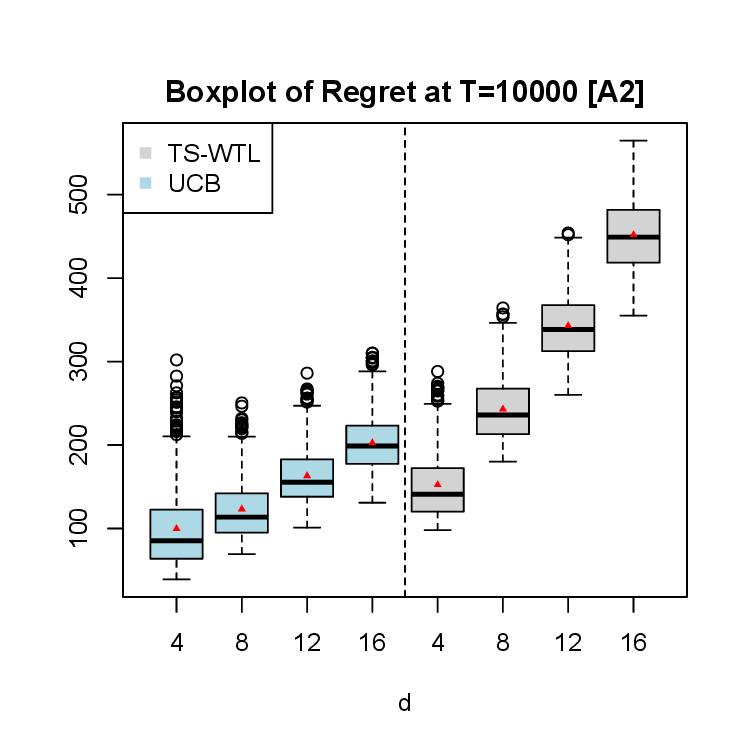}
	\vspace{-0.5cm}
    \end{subfigure}
    ~
    \begin{subfigure}{0.32\textwidth}
	\includegraphics[angle=0, width=1.1\textwidth]{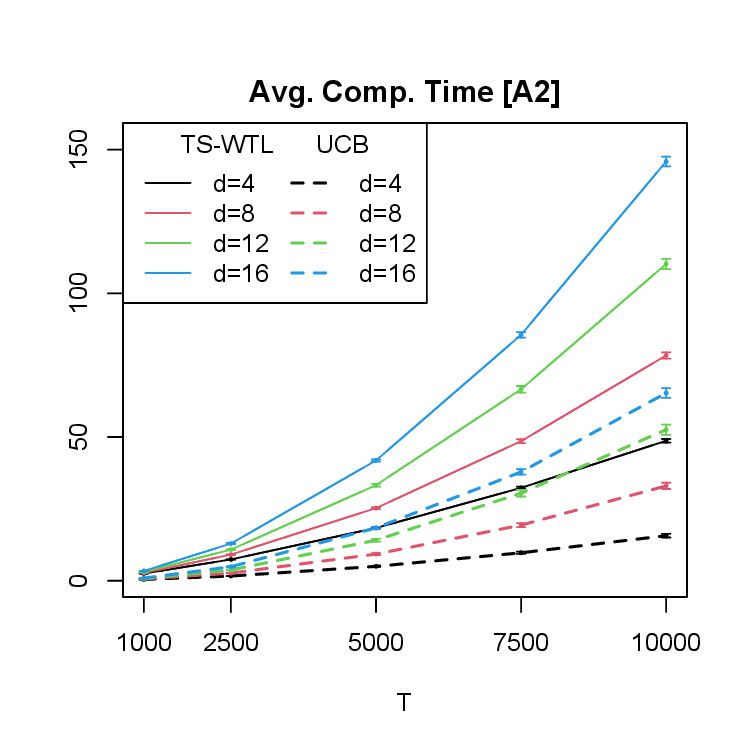}
	\vspace{-0.5cm}
    \end{subfigure}
    \caption{{\color{black} Performance of TS-WTL and UCB under (A2). [Left]:  Mean regret (with C.I.) under different $(d,T)$.  [Middle]: Boxplot of regrets at different $d$ (with $T=10000$). [Right]: Mean comp. time (with C.I.) under different $(d,T)$.}}
    \label{fig:A2}
\end{figure}

\subsection{Proof of \Cref{thm:UCB_adv}}
\begin{proof}[\textbf{Proof of \Cref{thm:UCB_adv}}]
Fix any time point $t,$ denote $p^*_t$ as the optimal price given $z_t$.

First, by \Cref{lem:ridge_error}, with probability at least $1-\delta$, we have that
\begin{align}\label{eq:score_diff}
    |x_t^{(p)\top}(\widehat{\theta}_t-\theta^*)|\leq \|x_t^{(p)}\|_{V_{t-1}^{-1}}\|\widehat{\theta}_t-\theta^*\|_{V_{t-1}}\leq \alpha \|x_t^{(p)}\|_{V_{t-1}^{-1}}
\end{align}
for all $p\in [l,u]$ and $t\in [T].$

Define $M_{\psi 2}^o:=\sup_{x\in\mathcal{X},\theta\in\Theta}|\psi''(x^\top\theta)|\vee 1$, which is an absolute constant. We have that the regret on time $t$ can be bounded via
\begin{align*}
    & r(p_t^*,z_t,\theta^*)-r(p_t,z_t,\theta^*)\\
    =&p_t^*\psi'\left( x_t^{(p^*_t)\top}\theta^*\right)-p_t\psi'\left( x_t^{(p_t)\top}\theta^*\right)\\
    \leq &  p_t^*\psi'\left( x_t^{(p_t^*)\top}\widehat\theta_t+ \alpha \|x_t^{(p_t^*)}\|_{V_{t-1}^{-1}}\right)-p_t\psi'\left( x_t^{(p_t)\top}\theta^*\right)\\
    \leq & p_t \psi'\left( x_t^{(p_t)\top}\widehat\theta_t+ \alpha \|x_t^{(p_t)}\|_{V_{t-1}^{-1}}\right)-p_t\psi'\left( x_t^{(p_t)\top}\theta^*\right)\\
    \leq & p_t M_{\psi 2}^o\cdot\left(x_t^{(p_t)\top}\widehat\theta_t-x_t^{(p_t)\top}\theta^*+\alpha \|x_t^{(p_t)}\|_{V_{t-1}^{-1}}\right)\\
    \leq & 2u M_{\psi 2}^o \alpha \|x_t^{(p_t)}\|_{V_{t-1}^{-1}},
\end{align*}
where the first inequality follows from \eqref{eq:score_diff} and that $\psi'(\cdot)$ is an increasing function, the second inequality follows from the definition of $p_t$, the third inequality follows from a Taylor expansion and the last inequality follows from \eqref{eq:score_diff} and Cauchy-Schwartz. Note that one can clip the upper confidence bound $x_t^{(p)\top}\widehat\theta_t+ \alpha \|x_t^{(p)}\|_{V_{t-1}^{-1}}$ at $\sup_{x\in\mathcal{X},\theta\in \Theta}x^\top \theta$ for all $p$ and $t,$ which ensures $M_{\psi 2}^o$ is sufficient for the third inequality. We omit this for notational simplicity.

Therefore, we have that with probability at least $1-\delta$, the total regret can be bounded via
\begin{align*}
    R_T=&\sum_{t=1}^T r(p^*_t,z_t,\theta^*)-r(p_t,z_t,\theta^*)\leq 2u M_{\psi 2}^o \alpha \sum_{t=1}^T\|x_t^{(p_t)}\|_{V_{t-1}^{-1}}\\
    \leq & 2u M_{\psi 2}^o \alpha\sqrt{T} \sqrt{\sum_{t=1}^T\|x_t^{(p_t)}\|_{V_{t-1}^{-1}}^2}\leq 2u M_{\psi 2}^o \alpha\sqrt{T}\sqrt{2\log(\det (V_T))}\\
    \leq & 2u M_{\psi 2}^o \alpha\sqrt{T}\sqrt{2d\log(T(1+u^2))},
\end{align*}
where the second inequality follows from Cauchy-Schwartz, and the third and fourth inequalities follows from Lemmas 11 and 10 in \cite{abbasi2011improved} respectively.
\end{proof}

\begin{lem}\label{lem:ridge_error}
    Denote $\Delta_t=\widehat{\theta}_t-\theta^*$. For any $\delta>0$, with probability at least $1-\delta$, for all $t>0$, we have that
    \begin{align*}
        \|\Delta_t\|_{V_{t-1}}^2\leq B_{A1} \sqrt{ d\log(t(1+u^2)/\delta)},
    \end{align*}
    where $B_{A1}$ is an absolute constant that only depend on $\lambda$ and quantities in Assumption \ref{assum_feature}(c)-(e) and \Cref{assum_feature_adv}.
\end{lem}

\begin{proof}[Proof]
    By the definition of $\widehat{\theta}_t$, we have $\widetilde{L}(\widehat{\theta}_t)\leq \widetilde{L}(\theta^*).$ Define $L_t(\theta)=\sum_{i=1}^t y_ix_i^\top \theta -\psi(x_i^\top \theta)$ as the log-likelihood. By simple algebra and Taylor expansion, we have that
    \begin{align*}
        \lambda(\|\widehat{\theta}_t\|_2^2-\|\theta^*\|_2^2)\leq &L(\widehat{\theta}_t)-L_t(\theta^*)=L'_t(\theta^*)(\widehat{\theta}_t-\theta^*)-\frac{1}{2}\Delta_t^\top\sum_{i=1}^{t-1} \psi''(x_i^\top\widetilde{\theta})x_ix_i^\top \Delta_t\\
        \leq & L'_t(\theta^*)(\widehat{\theta}_t-\theta^*)-\frac{\kappa^o}{2}\Delta_t^\top(V_{t-1}-\lambda I_{2d})\Delta_t,
    \end{align*}
    where $\kappa^o:=\inf_{x\in \mathcal{X}, \theta\in\Theta}\psi''(x^\top \theta)\wedge 1>0$ is an absolute constant.

    Rearranging the above equation and use the fact that $\|\Theta\|<c_\theta$, we further have that
    \begin{align*}
        \frac{\kappa_o}{2}\|\Delta_t\|_{V_{t-1}}^2&\leq L'_t(\theta^*)(\widehat{\theta}_t-\theta^*) + ({\lambda\kappa^o}/{2}+\lambda)c_\theta^2\\
        &\leq \left\|\sum_{i=1}^{t-1}(y_i-\psi'(x_i^\top\theta^*))x_i\right\|_{V_{t-1}^{-1}}\|\Delta_t\|_{V_{t-1}}+({\lambda\kappa^o}/{2}+\lambda)c_\theta^2\\
        &= \left\|\sum_{i=1}^{t-1}\varepsilon_ix_i\right\|_{V_{t-1}^{-1}}\|\Delta_t\|_{V_{t-1}}+({\lambda\kappa^o}/{2}+\lambda)c_\theta^2.
    \end{align*}
    By the self-normalized bound for vector-valued martingales as in Theorem 1 in \cite{abbasi2011improved}, we have that with probability at least $1-\delta$, for all $t>0$, it holds that
    \begin{align*}
        \left\|\sum_{i=1}^{t-1}\varepsilon_ix_i\right\|_{V_{t-1}^{-1}}^2 \leq 2\sigma^2 \log\left( \det(V_{t-1})/(\delta\lambda^d)\right)\leq 2\sigma^2 \cdot d\log(t(1+u^2)/\delta) ,
    \end{align*}
    where the second inequality follows from Lemma 10 in \cite{abbasi2011improved} and the fact that $\|x_t\|_2\leq \sqrt{1+u^2}$ for all $t$. Therefore, we have with probability at least $1-\delta$, it holds for all $t$ that
    \begin{align*}
        \|\Delta_t\|_{V_{t-1}}\leq \frac{4}{\kappa_o}\left( \sqrt{\sigma^2 d\log(t(1+u^2)/\delta)}+\sqrt{\lambda}c_\theta\right),
    \end{align*}
    which provides the confidence set for $\theta^*.$
\end{proof}

}

\section{Technical details in \Cref{sec-without-LDP} and \Cref{sec-supCB}}\label{sec:supplement_proof_nonprivate}

The results collected in the Appendix rely on notation defined in \eqref{eq-constants-definitions} of the main text.  For convenience, we copy them here:
\begin{align*}
    & M_{\psi 1} = \sup_{x\in\mathcal{X}} |\psi'(x^\top \theta^*)| \vee 1, \quad M_{\psi 2} = \sup_{\{x\in\mathcal X, ~\|\theta-\theta^*\|\leq 1\}} |\psi''(x^\top \theta)|\vee 1, \\
    & M_{\psi 3} = \sup_{\{x\in\mathcal{X}, ~\|\theta-\theta^*\|\leq 1\}} |\psi'''(x^\top \theta)|\vee 1 \quad \text{and} \quad \kappa = \inf_{\{x\in\mathcal{X}, ~\|\theta-\theta^*\|\leq 1\}} \psi''(x^\top\theta)\wedge 1.
\end{align*}

\subsection{Key technical results}

{\color{black}
\begin{pro}[Matrix Bernstein Inequality]\label{prop:matrix_bern}
    For $\tau \in \mathbb{N}_+$, define $V_\tau=\sum_{t=1}^\tau x_tx_t^\top$, where $\{x_t \in \mathbb{R}^d\}_{t \in [\tau]}$ are i.i.d.~from an unknown distribution with $\|x_1\|\leq u_x$, $u_x > 0$.  Denote $\Sigma_x = \mathbb E(x_1x_1^\top)$. We have for all $\Delta\geq 0$, 
    \begin{align*}
        \mathbb{P}\big\{ \|V_\tau/\tau- \Sigma_x\|_{\mathrm{op}}\geq \Delta \big\} \leq 2d\exp\left\{\frac{-\tau \Delta^2}{2u_x^2(u_x^2+2\Delta/3)}\right\}.
    \end{align*}
\end{pro}

\begin{proof}[\textbf{Proof of \Cref{prop:matrix_bern}}] 
We prove by verifying the condition of the matrix Bernstein inequality (Theorem 6.1.1) in \cite{tropp2015introduction}. In particular, define $S_t=(x_tx_t^\top -\Sigma_x)/\tau$ for $t\in[\tau].$ We have $\mathbb{E}(S_t)=0$ and $\|S_t\|_{\mathrm{op}}\leq (\|x_tx_t^\top\|_{\mathrm{op}}+\|\Sigma_x\|_{\mathrm{op}})/\tau\leq 2u_x^2/\tau$ for all $t\in[\tau].$ Moreover, we have
\begin{align*}
    v_\tau &:=\left\|\sum_{t=1}^\tau\mathbb{E}(S_tS_t^\top) \right\|_{\mathrm{op}}=\left\|\tau\mathbb{E}(S_1S_1^\top) \right\|_{\mathrm{op}}=\frac{1}{\tau}\left\|\mathbb{E}(x_1x_1^\top x_1 x_1^\top)-\Sigma_x\Sigma_x \right\|_{\mathrm{op}}\\
     & \leq \frac{1}{\tau}\left\|\mathbb{E}(x_1x_1^\top x_1 x_1^\top) \right\|_{\mathrm{op}}\leq u_x^4/\tau.
\end{align*}
where the first inequality follows from the fact that $\mathbb{E}(S_1S_1^\top)$ is positive semi-definite and Weyl's inequality. Thus, by Theorem 6.1.1 in \cite{tropp2015introduction}, we have 
\begin{align*}
    \mathbb{P}\big\{ \|V_\tau/\tau- \Sigma_x\|_{\mathrm{op}}\geq \Delta \big\} \leq 2d\exp\left\{\frac{-\tau \Delta^2}{2u_x^2(u_x^2+2\Delta/3)}\right\},
\end{align*}
which completes the proof.
\end{proof}

}

\Cref{lem:eigen} establishes a high-probability upper bound on the behavior of the sample design matrix of the GLM.  \Cref{lem:eigen} serves as the basis for later theoretical results.

\begin{lem}\label{lem:eigen}
    For $\tau \in \mathbb{N}_+$, define $V_\tau=\sum_{t=1}^\tau x_tx_t^\top$, where $\{x_t\}_{t \in [\tau]}$ are i.i.d.~from an unknown distribution with $\|x_1\|\leq u_x$, $u_x > 0$.  Denote $\Sigma_x = \mathbb E(x_1x_1^\top)$. There exist absolute constants $c_1, c_2>0$ such that for any $\delta\in(0,1/2)$, with probability at least~$1-\delta$, we have that
    \begin{itemize}
        \item [(i)] {\color{black} $\|V_\tau/\tau-\Sigma_x\|_{\mathrm{op}} \leq 3u_x^2\max(\varsigma,\varsigma^2)$, where $\varsigma = \sqrt{\log(2d)+\log(1/\delta)}/\sqrt\tau$;
        }
        \item [(ii)] {\color{black} $\lambda_{\min} (V_\tau)\geq B$ for any $B>0$ provided that
        \begin{align}\label{eq:samplesize}
            \tau\geq  u_x^4\left\{  \frac{c_1\sqrt{\log{d}}+c_2\sqrt{\log(1/\delta)}}{\lambda_{\min}(\Sigma_x)}\right\}^2+ \frac{2B}{\lambda_{\min}(\Sigma_x)}.
        \end{align}}
    \end{itemize}
\end{lem}


\begin{proof}[\textbf{Proof of \Cref{lem:eigen}}] 
{\color{black} Note that (i) can be directly derived based on the concentration inequality result in \Cref{prop:matrix_bern} via simple algebra. As for (ii), first, it follows from (i) and Weyl's inequality that any $\delta \in (0, 1/2)$, with probability at least $1-\delta$, we have
    \begin{align*}
        \lambda_{\min}(V_\tau) \geq \tau\lambda_{\min}(\Sigma_x) -3\tau u_x^2\max(\varsigma,\varsigma^2),
    \end{align*}
    where $\varsigma = \sqrt{\log(2d)+\log(1/\delta)}/\sqrt\tau$. Note that under \eqref{eq:samplesize}, $\varsigma<1$.  We then have that, with probability at least $1-\delta$,
    \begin{align*}
        \lambda_{\min}(V_\tau) \geq \tau\lambda_{\min}(\Sigma_x)- \sqrt{\tau}u_x^2(c_1\sqrt{\log d}+c_2\sqrt{\log(1/\delta)}).
    \end{align*}    
    By simple algebra, (ii) is proved.
}
\end{proof}

\begin{proof}[\textbf{Proof of \Cref{thm:theta_err}}]
For notational convenience, in the proof, we omit the index $\tau$ when no confusion arises. Recalling that $L(\theta)=\sum_{t=1}^\tau \{y_t x_t^\top \theta - \psi(x_t^\top \theta)\}$, let $L'(\theta)=\sum_{t=1}^\tau \{ y_t - \psi'(x_t^\top \theta)\} x_t$ and $L''(\theta)=-\sum_{t=1}^\tau  \psi''(x_t^\top \theta) x_tx_t^\top$ be the first and second order derivatives of $L(\theta)$.  By the definitions of~$\theta^*$ and $\widehat\theta$, recalling that $\varepsilon_t=y_t-\psi'(x_t^\top \theta^*)$, we have that $L'(\theta^*)=\sum_{t=1}^\tau\varepsilon_t x_t$ and $L'(\widehat\theta)=0$. Define $G(\theta)=L'(\theta^*)-L'(\theta)$. We have that $G(\theta^*)=0$ and $G(\widehat\theta)=\sum_{t=1}^\tau\varepsilon_t x_t.$

\textbf{Consistency}: For any $\theta_1,\theta_2\in \mathbb R^d,$ by the mean value theorem, we have that
\begin{align*}
    G(\theta_1)-G(\theta_2)= G'(\bar\theta) (\theta_1-\theta_2)= \left\{\sum_{t=1}^\tau \psi''(x_t^\top \bar\theta) x_tx_t^\top\right\} (\theta_1-\theta_2),
\end{align*}
for some $\bar\theta$ between $\theta_1$ and $\theta_2$. By the property of GLM detailed in \Cref{assum_feature}(e), we have that $\psi''(a)>0$ for all $a\in \mathbb R$. In addition, since $V_\tau$ is positive definite, we have that $(\theta_1-\theta_2)^\top  (G(\theta_1)-G(\theta_2)) >0$ for any $\theta_1\neq \theta_2.$ In other words, $G(\theta)$ is an injection from $\mathbb R^d$ to $\mathbb R^d$. Define $H(\theta)=V_\tau^{-1/2}G(\theta)$, clearly $H(\theta)$ is also an injection from $\mathbb R^d$ to $\mathbb R^d$. Moreover, we have $H(\theta^*)=0$.

Denote the $\varsigma$-neighborhood of $\theta^*$ as $\mathbb B_{\varsigma}:=\{\theta: \|\theta-\theta^*\|\leq \varsigma\}$, where $\varsigma\in (0,1).$ Define $\kappa_\varsigma:=\inf_{\{x\in\mathcal{X},~ \|\theta-\theta^*\|\leq \varsigma\}} \psi''(x^\top\theta)$. By definition, we have that $\kappa_\varsigma\geq \kappa_1 \geq \kappa$.  Therefore, for any $\theta \in \mathbb B_\varsigma$, we have that
\begin{align*}
    \|H(\theta)\|^2 &=G(\theta)^\top V_\tau^{-1} G(\theta) = \{G(\theta)-G(\theta^*)\}^\top V_\tau^{-1} \{G(\theta)-G(\theta^*)\} \\
    &\geq \kappa_\varsigma^2 (\theta-\theta^*)^\top V_\tau (\theta-\theta^*) \geq  \kappa^2 \lambda_{\min}(V_\tau) \|\theta-\theta^*\|^2.
\end{align*}

In other words, denote $\partial \mathbb B_{\varsigma}:=\{\theta: \|\theta-\theta^*\|= \varsigma\}$, we have that $\inf_{\theta \in \partial \mathbb B_{\varsigma}}\|H(\theta)\|\geq {\kappa}\varsigma\sqrt{\lambda_{\min}(V_\tau)}.$ Importantly, by Lemma A in \cite{chen1999strong}, this implies that
\begin{align*}
    \bigg\{\theta : \|H(\theta)\| \leq \kappa\varsigma\sqrt{\lambda_{\min}(V_\tau)} \bigg\} \subset \mathbb B_\varsigma.
\end{align*}

Therefore, we have that $\widehat\theta\in \mathbb B_\varsigma$ if we can show that $\|H(\widehat\theta)\|< \kappa\varsigma\sqrt{\lambda_{\min}(V_\tau)}$.  Note that $\|H(\widehat\theta)\|^2=\bepsilon^\top X(X^\top X)^{-1} X^\top \bepsilon$, where $\bepsilon=(\varepsilon_1,\cdots,\varepsilon_\tau)^{\top}$ and $X=[x_1,x_2,\cdots,x_\tau]^\top.$ Denote $\Sigma= X(X^\top X)^{-1} X^\top$. We have $\tr(\Sigma)=\tr(\Sigma^2)=d$ and $\|\Sigma\|_{\mathrm{op}}=1$. By Theorem 2.1 in \cite{hsu2012tail} and \Cref{assum_feature}(d), we have that for any $\delta\in (0,1),$
\begin{align}\label{eq:quad_bound}
    \bp\left[ \bepsilon^\top \Sigma \bepsilon>\sigma^2 \big\{d+2\sqrt{d\log( 1/\delta)}+2\log (1/\delta)\big\} \right] \leq \delta.
\end{align}
In addition, by the Cauchy--Schwarz inequality, we have that $d+2\sqrt{d\log( 1/\delta)}+2\log (1/\delta)\leq 2d + 3\log (1/\delta)\leq 3\{d + \log (1/\delta)\}$. Together, it implies that with probability at least $1-\delta$, the event $\mathcal{E}_H:=\big\{ \|H(\widehat\theta)\|\leq \sqrt{3}\sigma\sqrt{d+\log(1/\delta)}  \big\}$ holds. 

Recall that we have $\lambda_{\min}(V_\tau)\geq 3\sigma^2/(\kappa^2\varsigma^2) \{d+\log(1/\delta)\}$. Combined with the above result, it implies that with probability at least $1-\delta$, we have $\|H(\widehat\theta)\|< \kappa\varsigma\sqrt{\lambda_{\min}(V_\tau)}$, and hence $\|\widehat\theta-\theta^*\|_2<\varsigma$.

\textbf{Prediction error}:  By the definition of $\widehat\theta$, we have that $L(\widehat\theta)\geq L(\theta^*)$. By Taylor's expansion,
    \begin{align*}
        L(\widehat\theta)=L(\theta^*) +L'(\theta^*)^\top(\widehat\theta-\theta^*) + \frac{1}{2} (\widehat\theta-\theta^*)^\top L''(\theta^+)(\widehat\theta-\theta^*),
    \end{align*}
    for some $\theta^+$ between $\widehat\theta$ and $\theta^*.$ By the above result, we know that with probability $1-\delta$, we have $\|\widehat\theta-\theta^*\|\leq 1.$ Together, this implies that
    \begin{align*}
        (\widehat\theta-\theta^*)^\top V_\tau (\widehat\theta-\theta^*)\leq 2/\kappa |L'(\theta^*)^\top(\widehat\theta-\theta^*)| \leq 2/\kappa \cdot \|L'(\theta^*)^\top V_\tau^{-1/2} \| \cdot \|(\widehat\theta-\theta^*)^\top V_\tau^{1/2}\|,
    \end{align*}
    where the second inequality follows from the Cauchy--Schwarz inequality. This further gives that, with probability at least $1- \delta$,
    \begin{align*}
        (\widehat\theta-\theta^*)^\top V_\tau (\widehat\theta-\theta^*)&\leq \frac{4}{\kappa^2} \cdot L'(\theta^*)^\top V_\tau^{-1} L'(\theta^*)=\frac{4}{\kappa^2} \cdot \bepsilon^\top X(X^\top X)^{-1}X^\top \bepsilon\\
        &= \frac{4}{\kappa^2} \cdot \bepsilon^\top \Sigma \bepsilon\leq \frac{12\sigma^2}{\kappa^2}\{d+\log(1/\delta)\},
    \end{align*}
        where $\bepsilon=(\varepsilon_1,\cdots, \varepsilon_n)^\top$ and the last inequality follows from \eqref{eq:quad_bound}. This completes the proof.

\textbf{Estimation error}: By the definition of $\widehat\theta$, we have that
\begin{align*}
    0=L'(\widehat\theta)=L'(\theta^*) + L''(\theta^+)(\widehat\theta-\theta^*)
\end{align*}
for some $\theta^+$ between $\widehat\theta$ and $\theta^*.$ Recall that with probability at least $1-\delta$, we have $\|\widehat\theta-\theta^*\|\leq \varsigma$ for sufficiently small $\kappa$.  Combined with the above result, with probability at least $1-\delta$, we have 
\begin{align*}
    \|\widehat\theta-\theta^* \|& \leq \frac{1}{\kappa}\|V_\tau^{-1} L'(\theta^*)\|= \frac{1}{\kappa}\|V_\tau^{-1} X^\top \bepsilon\| \leq \frac{1}{\kappa\lambda_{\min}^{1/2}(V_\tau)}  \| V_\tau^{-1/2}X^\top \bepsilon\| \\
    &= \frac{1}{\kappa\lambda_{\min}^{1/2}(V_\tau)}  \| H(\widehat\theta)\|\leq \frac{1}{\kappa\lambda_{\min}^{1/2}(V_\tau)} \sqrt{3}\sigma\sqrt{d+\log(1/\delta)}
\end{align*}
where the last inequality follows from \eqref{eq:quad_bound}.

The last result in \Cref{thm:theta_err} is a direct adaptation of Theorem 1 in \cite{li2017provably} and we thus omit the proof. 
\end{proof}

\subsection{Technical details in \Cref{sec-supCB}}\label{sec-app-proof-supCB}

\begin{lem}\label{lem:Lips}
    Suppose \Cref{assum_feature} holds.  For all $z \in \mathcal{Z}$, the true revenue function $r(p,z,\theta^*)$ defined in \eqref{eq-revenue} is $L_r$-Lipschitz in $p \in [l,u]$, with $L_r =M_{\psi 1} + uu_b M_{\psi 2}$. For all $p\in [l,u]$ and $z \in \mathcal{Z}$, it holds that $r(p,z,\theta^*)\leq u M_{\psi 1}$ and $|\partial^2 r(p,z,\theta^*) /\partial p^2|\leq C_r=2M_{\psi2}u_b+uu_b^2M_{\psi3}$.
\end{lem}

\begin{proof}[\textbf{Proof of \Cref{lem:Lips}}] 
Denote $p\neq \tilde p$ and $p,\tilde p \in [l,u]$, and denote $x=(z,-p z)$ and $\tilde x=(z, -\tilde p z).$ By definition, we have that
    \begin{align*}
        r(p,z,\theta^*)-r(\tilde p,z,\theta^*)=&p \psi'(x^\top\theta^*)-\tilde p  \psi'(\tilde x^\top\theta^*)\\
        =&(p-\tilde p) \psi'(x^\top\theta^*) - \tilde p \{\psi'(\tilde x^\top\theta^*)-\psi'(x^\top\theta^*)\}\\
        =&(p-\tilde p) \psi'(x^\top\theta^*) - \tilde p \psi''(\dot x^\top\theta^*)(\tilde x^\top\theta^*- x^\top\theta^*)\\
        =& (p-\tilde p) \psi'(x^\top\theta^*) + \tilde p \psi''(\dot x^\top\theta^*)z^\top\beta^*(\tilde p -p),
    \end{align*}
    where the third equality follows from the mean value theorem with $\dot x^\top\theta^*$ between $x^\top\theta^*$ and $\tilde x^\top\theta^*.$

    Therefore, we have that
    \begin{align*}
             \big|r(p,z,\theta^*)-r(\tilde p,z,\theta^*) \big|
        & \leq  \big|(p-\tilde p) \psi'(x^\top\theta^*) \big| + \big| \tilde p \psi''(\dot x^\top\theta^*)z^\top\beta^*(\tilde p -p)\big| \\        
        \leq & (M_{\psi 1} + uu_b M_{\psi 2}) |p-\tilde p|,
    \end{align*}
    where the second inequality follows from \Cref{assum_feature} and the definition of $M_{\psi 1}$ and $M_{\psi 2}$. The other results can be proved similarly.
\end{proof}

\Cref{lem:indp} is adapted from Lemma 4 in \cite{li2017provably} and Lemma 14 in \cite{auer2002using} with a formal proof. Denote $\Psi_s^o(t)=\Psi_s(t)\cup \mathcal{F}_s$, where $\mathcal{F}_s$ is the $s$th pure price experiment set and $\Psi_s(t)$ is the index set that collects all time periods in $[S\tau+1, t-1]$ such that supCB stopped in stage $s$.  

\begin{lem}\label{lem:indp} 
In \Cref{algorithm:SupCB}, for all $s\in [S]$ and $t\in [S\tau+1, T]$, given the features $\{x_i, i\in \Psi^o_s(t)\}$, the demands $\{y_{i}, i \in \Psi^o_s(t)\}$ are conditionally independent GLM random variables with parameter $\theta^*.$
\end{lem}

\begin{proof}[\textbf{Proof of \Cref{lem:indp}}]
    Note that we need to show
    \begin{align}\label{eq:cond_ind}
        \bp(\{y_i\}, i\in \Psi^o_s(t) | \{x_i\}, i\in \Psi^o_s(t)) = \prod_{i\in \Psi^o_s(t)} \bp(y_i|x_i).
    \end{align}
    
    For notational simplicity, denote $X_{\leq t}^s=\{x_i: i \in \Psi^o_s(t)\}$ and $Y^s_{\leq t}=\{y_i: i \in \Psi^o_s(t)\}$. For $s \geq 2$, define $\Phi^o_{<s}(t)=\bigcup_{\sigma <s} \{Y^\sigma_{\leq t}, X^\sigma_{\leq t}\}$ as the results observed at stages $[s-1]$ up to time $t$ (not including $t$). Furthermore, for time points $t_1<t_2$, define $\Phi^o_{<s}(t_1,t_2)=\Phi^o_{<s}(t_2) \setminus\Phi^o_{<s}(t_1)$ as the results observed at stages $[s-1]$ between time $t_1$ to $t_2-1$, inclusive.
    
    In the following, we show that for all $t \in [S\tau+1, T]$ and $s\in [S]$, we have that
    \begin{align}\label{eq:cond_ind1}
        \bp(Y_{\leq t}^s|X_{\leq t}^s, \Phi^o_{<s}(t))= \prod_{i\in \Psi^o_s(t)} \bp(y_i|x_i).
    \end{align}
    Note that by the property of conditional probability, \eqref{eq:cond_ind1} directly implies \eqref{eq:cond_ind}. We proceed by mathematical induction. First, note that for $t=S\tau+1$, \eqref{eq:cond_ind1} clearly holds as $\Psi_s^o(S\tau+1)=\mathcal{F}_s$, which consists of only price experiments.

    Now, suppose \eqref{eq:cond_ind1} holds for $t.$ Denote $t^+$ as the first time point such that $\Psi_s^o(t^+)\setminus \Psi_s^o(t)\neq \varnothing.$ Note that clearly $t^+>t$ and in fact we have $\Psi_s^o(t^+)\setminus \Psi_s^o(t)=t^+-1$. In other words, $(y_{t^*}, x_{t^*})$ is observed in stage $s$, where for notational simplicity, we define $t^*=t^+-1$. To finish the proof, we only need to show that
    \begin{align*}
        \bp(Y_{\leq t^+}^s|X_{\leq t^+}^s, \Phi^o_{<s}(t^+)) = \prod_{i\in \Psi^o_s(t^+)} \bp(y_i|x_i).
    \end{align*}
    In particular, by the property of conditional probability, we have
    \begin{align*}
        &\bp\{Y_{\leq t^+}^s,X_{\leq t^+}^s, \Phi^o_{<s}(t^+)\} = \bp\{y_{t^*},Y_{\leq t}^s, x_{t^*}, X_{\leq t}^s, \Phi^o_{<s}(t), \Phi^o_{<s}(t,t^+)\}\\
       =&\bp\{y_{t^*}|Y_{\leq t}^s,x_{t^*}, X_{\leq t}^s, \Phi^o_{<s}(t), \Phi^o_{<s}(t,t^+)\} \bp\{x_{t^*}| Y_{\leq t}^s, X_{\leq t}^s, \Phi^o_{<s}(t), \Phi^o_{<s}(t,t^+)\} \\
       \times & \bp\{\Phi^o_{<s}(t,t^+)| Y_{\leq t}^s, X_{\leq t}^s, \Phi^o_{<s}(t)\}  \bp\{ Y_{\leq t}^s|X_{\leq t}^s, \Phi^o_{<s}(t)\} \bp\{X_{\leq t}^s, \Phi^o_{<s}(t)\}.
    \end{align*}
    We now analyze the above terms one by one.
    \begin{itemize}
        \item First, by \eqref{eq:GLM}, we have that $\bp(y_{t^*}|Y_{\leq t}^s, x_{t^*}, X_{\leq t}^s, \Phi^o_{<s}(t), \Phi^o_{<s}(t,t^+))=\bp(y_{t^*}|x_{t^*})$.
        
        \item Second, importantly, by design, $t^*$ can only be added to stage $s$ in \textbf{Step II} of \Cref{algorithm:SupCB}. For this to happen, it only depends on the results in the previous stages $\Phi^o_{<s}(t^+)$ and $w_{t^*,a}^{(s)}$, and moreover, $w_{t^*,a}^{(s)}$ only depends on $X_{\leq t}^s$ but not $Y_{\leq t}^s$. Therefore, the possible values of $x_{t^*}$ that leads to $t^*$ being added in stage $s$ is entirely determined by $X_{\leq t}^s$ and $\Phi^o_{<s}(t^+)$, and is not affected by the value of\ $Y_{\leq t}^s$. In other words, we have the conditional independence result
        \[
            \bp\{x_{t^*}| Y_{\leq t}^s, X_{\leq t}^s, \Phi^o_{<s}(t), \Phi^o_{<s}(t,t^+)\} = \bp\{x_{t^*}| X_{\leq t}^s, \Phi^o_{<s}(t), \Phi^o_{<s}(t,t^+)\} = \bp\{x_{t^*}| X_{\leq t}^s, \Phi^o_{<s}(t^+)\}.
        \]
        \item Third, by the design of \Cref{algorithm:SupCB}, the observations $\Phi^o_{<s}(t,t^+)$ at stages $[s-1]$ does not depends on $Y_{\leq t}^s, X_{\leq t}^s$ and therefore, we have $\bp\{\Phi^o_{<s}(t,t^+)| Y_{\leq t}^s, X_{\leq t}^s, \Phi^o_{<s}(t)\} = \bp\{\Phi^o_{<s}(t,t^+)| \Phi^o_{<s}(t)\}$.               
    \end{itemize}
    Moreover, by the induction assumption, we have that $\bp( Y_{\leq t}^s|X_{\leq t}^s, \Phi^o_{<s}(t))=\prod_{i\in \Psi^o_s(t)} \bp(y_i|x_i)$. Therefore, we have
    \begin{align*}
         &\bp\{Y_{\leq t^+}^s,X_{\leq t^+}^s, \Phi^o_{<s}(t^+)\}\\
        =&\left[\bp(y_{t^*}|x_{t^*})  \prod_{i\in \Psi^o_s(t)} \bp(y_i|x_i) \right] \left[\bp\{x_{t^*}| X_{\leq t}^s, \Phi^o_{<s}(t^+)\} \bp\{\Phi^o_{<s}(t,t^+)| \Phi^o_{<s}(t)\} \bp\{X_{\leq t}^s, \Phi^o_{<s}(t)\}\right].
    \end{align*}
    Thus, to finish the proof, we only need to show that
    \begin{align*}
        \bp\{X_{\leq t^+}^s, \Phi^o_{<s}(t^+)\} = \bp\{x_{t^*}| X_{\leq t}^s, \Phi^o_{<s}(t^+)\} \bp\{\Phi^o_{<s}(t,t^+)| \Phi^o_{<s}(t)\}  \bp\{X_{\leq t}^s, \Phi^o_{<s}(t)\},
    \end{align*}
    which can be shown based on the exact same arguments as above and thus is omitted.
\end{proof}

\Cref{lem:HPCB} establishes a high probability confidence bound for the true revenue function.
\begin{lem}\label{lem:HPCB} 
    For any $\delta \in (0,1)$, set the supCB algorithm with $\tau=\sqrt{dT}$ and $\alpha= {3\sigma u M_{\psi2}}/{\kappa}\cdot \sqrt{\log(3TKS/\delta)}$. Suppose $T$ satisfies that
    \begin{align}\label{eq:minT}
            T\geq \left\{\frac{4u_x^8}{L_p^4\lambda_z^4} \left(c_1\sqrt{\log(d)}+c_2\sqrt{\log (T)} \right)^4/d \right\}\vee \left\{ \frac{2048 M_{\psi 3}^4 \sigma^4}{\kappa^8L_p^2\lambda_z^2} \left(d^2+\log (3TKS/\delta) \right)^2/d \right\},
    \end{align}
    The following event $\mathcal{E}_X$ holds with probability at least $1-\delta-S/T$, where we define
    \begin{align}\label{eq-def-event-EX}
        \mathcal{E}_X:= \left\{\Big|r_{t,a}^{(s)}-r(p^{(a)}, z_t, \theta^*)\Big|\leq w_{t,a}^{(s)}, \text{ for all } t\in [S\tau+1,T], s\in[S], a\in[K] \right \}.
    \end{align}
\end{lem}

\begin{proof}[\textbf{Proof of \Cref{lem:HPCB}}]
    Setting $B=\tau\lambda_{\min}(\Sigma_x)/4$ and $\delta=1/T$, by simple algebra, we have that \eqref{eq:samplesize} in \Cref{lem:eigen} with probability at least $1 - T^{-1}$ holds for any $\tau$ such that 
    \begin{align}\label{eq:HPCB_mintau}
        \tau \geq  2u_x^4\left\{\frac{c_1\sqrt{\log(d)}+c_2\sqrt{\log(T)}}{\lambda_{\min}(\Sigma_x)}\right\}^2.
    \end{align}   
    In \Cref{algorithm:SupCB}, we set $\tau=\sqrt{dT}$. Therefore, under condition \eqref{eq:minT}, by simple algebra and the fact that $\lambda_{\min}(\Sigma_x)\geq L_p\lambda_z$, with $L_p$ defined in \eqref{eq-Lp-def}, we have that $\tau$ satisfies \eqref{eq:HPCB_mintau} and furthermore
    \begin{align*}
        B= \tau\lambda_{\min}(\Sigma_x)/4 \geq \frac{512 M_{\psi3}^2\sigma^2}{\kappa^4}\left\{d^2+\log\left(\frac{3TKS}{\delta}\right)\right\}.
    \end{align*}
    Denote $V_{s,\tau}=\sum_{i\in \mathcal{F}_s}x_i x_i^\top.$ Therefore, by \Cref{lem:eigen}, under condition \eqref{eq:minT}, for any $s\in [S]$, we have with probability at least $1-1/T$, it holds that
    \begin{align*}
        \lambda_{\min}(V_{s,\tau})\geq B\geq \frac{512 M_{\psi3}^2\sigma^2}{\kappa^4}\left\{d^2+\log\left(\frac{3TKS}{\delta}\right)\right\}.
    \end{align*}

    Recall that for any $t\in [S\tau+1,T]$, $\Psi^o_s(t)$ is the union of $\mathcal{F}_s$ and all periods before $t$ that are in $\Psi_s$. Denote $V_{s,t}=\sum_{i\in \Psi^o_s(t)} x_ix_i^\top$ and denote $\widehat\theta^s_t$ as the MLE based on $\{y_i, x_i\}_{i\in \Psi^o_s(t)}$. By design, the index set $\Psi^o_s(t)$ has $\mathcal{F}_s$ as a subset and thus $\lambda_{\min}(V_{s,t})\geq \lambda_{\min}(V_{s,\tau})$ for all $t\in [S\tau+1,T]$ and $s\in[S]$. Moreover, by \Cref{lem:indp}, on each $\Psi^o_s(t), s\in [S]$, the reward $\{y_i\}_{i\in \Psi^o_s(t)}$ are conditionally independent given $\{x_i\}_{i\in \Psi^o_s(t)}$. Therefore, by \Cref{thm:theta_err}, for any fixed $t\in [S\tau+1,T]$ and $s\in [S]$ and $a\in[K]$, with probability at least $1-\delta/(TKS)$, we have that $\|\widehat\theta_t^s-\theta^*\|\leq 1$ and 
    \begin{align}\label{eq:cb1}
        \big|x_{t,a}^\top(\widehat\theta_t^s-\theta^*)\big|\leq \frac{3\sigma}{\kappa}\sqrt{\log(3TKS/\delta)}\|x_{t,a}\|_{V_{s,t}^{-1}}.
    \end{align}
    By a union bound argument, we know that \eqref{eq:cb1} holds for all $t\in [S\tau+1,T]$ and $s\in [S]$ and $a\in[K]$ with probability at least $1-\delta.$ In the following, assume this good event holds.

    In addition, by the smoothness of the revenue function, we have that
    \begin{align*}
    \Big|r_{t,a}^{(s)}-r(p^{(a)}, z_t, \theta^*)\Big|&=
        \big|r_{t,a}^{(s)}-p^{(a)}\cdot \psi'(x_{t,a}^\top\theta^*)\big|=\big|p^{(a)}\cdot \psi'(x_{t,a}^\top\widehat\theta_t^s)-p^{(a)}\cdot \psi'(x_{t,a}^\top\theta^*)\big|\\
        &\leq p^{(a)}M_{\psi2}\big|x_{t,a}^\top(\widehat\theta_t^s-\theta^*)\big|\leq uM_{\psi2}\big|x_{t,a}^\top(\widehat\theta_t^s-\theta^*)\big|,
    \end{align*}
    where the first inequality follows from \Cref{assum_feature}, the mean value theorem and the fact that $\|\widehat\theta_t^s-\theta^*\|\leq 1$, and the second inequality follows from that $p^{(a)}\leq u$ for all $a\in [K].$

    Recall by definition, $w_{t,a}^{(s)}=\alpha \|x_{t,a}\|_{V_{s,t}^{-1}}.$ Since $\alpha= {3\sigma u M_{\psi2}}/{\kappa}\cdot \sqrt{\log(3TKS/\delta)}$, we have $\mathcal{E}_X$ holds with probability at least $1-\delta-S/T$.  This completes the proof.
\end{proof}

\Cref{lem:gap_reward} establishes an upper bound on the regret for each round $t \in [T]$ given that the high probability confidence bound \eqref{eq-def-event-EX} in \Cref{lem:HPCB} covers the true revenue function.  
\begin{lem}\label{lem:gap_reward}
    Suppose that event $\mathcal{E}_X$ defined in \eqref{eq-def-event-EX} holds, and that in round $t$, the action $a_t$ is chosen at stage $s_t$. Denote $a_t^*$ as the action of the optimal price given $z_t.$ Then $a_t^*\in A_s$ for all $s\leq s_t$. Furthermore, we have that
    \begin{align*}
        r(p^{(a_t^*)}, z_t,\theta^*)-r(p^{(a_t)}, z_t,\theta^*)\leq
        \begin{cases}
            (4\vee L_r u)2^{-(s_t-1)} ~\text{  if $a_t$ is chosen in step II},\\ 
            2/\sqrt{T} ~\text{  if $a_t$ is chosen in step III}.
        \end{cases}
    \end{align*}
\end{lem}

\begin{proof}[\textbf{Proof of \Cref{lem:gap_reward}}]
    We prove by induction. For $s=1$, we have $A_1=[K]$ and thus $a_t^*\in A_1$. Suppose $a_t^*\in A_s$ for some $1\leq s<s_t$. Since the supCB algorithm does not stop at $s$, by \textbf{Step IV} in \Cref{algorithm:SupCB}, we have that $w_{t,a}^{(s)}\leq 2^{-s}$ for all $a\in A_s.$ Given that $\mathcal{E}_X$ holds, we have
    \begin{align*}
        r_{t,a_t^*}^{(s)} \geq  r(p^{(a_t^*)}, z_t,\theta^*) - w_{t,a_t^*}^{(s)}\geq r(p^{(a)}, z_t,\theta^*) - w_{t,a_t^*}^{(s)} \geq r_{t,a}^{(s)} -  w_{t,a}^{(s)}- w_{t,a_t^*}^{(s)} \geq r_{t,a}^{(s)}- 2^{1-s},
    \end{align*}
    for all $a\in A_s,$ where the first and third inequalities follow from the definition of $\mathcal{E}_X$, and the second inequality holds by the definition of $a_t^*.$ Therefore, by the definition of $A_{s+1},$ we have that $a_t^* \in A_{s+1}.$ This finishes the proof of the first part.

    Now suppose $a_t$ is chosen at \textbf{Step II} of \Cref{algorithm:SupCB}. If $s_t=1$, by \Cref{lem:Lips}, we have that $|r(p^{(a_t)}, z_t,\theta^*)-r(p^{(a_t^*)}, z_t,\theta^*)|\leq L_r u.$ Suppose $s_t>1$, then based on $s=s_t-1$, we have that
    \begin{align*}
        &r(p^{(a_t)}, z_t,\theta^*)\geq r_{t,a_t}^{(s_t-1)}-w_{t,a_t}^{(s_t-1)}\\
    \geq & r_{t,a_t^*}^{(s_t-1)}- 2^{1-(s_t-1)}-w_{t,a_t}^{(s_t-1)}\geq r(p^{(a_t^*)}, z_t,\theta^*)- 2^{1-(s_t-1)}-w_{t,a_t^*}^{(s_t-1)}-w_{t,a_t}^{(s_t-1)}\\
    \geq & r(p^{(a_t^*)}, z_t,\theta^*) - 2^{2-(s_t-1)},
    \end{align*}
    where the first inequality follows from the definition of the event $\mathcal{E}_X$ defined in \eqref{eq-def-event-EX}, the second inequality follows from the fact that both $a_t$ and $a_t^*$ are in $A_{s_t-1}$, and the last from the design of the supCB algorithm and the fact that the algorithm does not stop at the $s_t - 1$. Therefore, we have that $r(p^{(a_t^*)}, z_t,\theta^*)-r(p^{(a_t)}, z_t,\theta^*)\leq (4\vee L_r u)2^{-(s_t-1)}$ given that $a_t$ is chosen at \textbf{Step II}.

    Now suppose $a_t$ is chosen at \textbf{Step III}, we have that
    \begin{align*}
        r(p^{(a_t)}, z_t,\theta^*) \geq r_{t,a_t}^{s_t}-1/\sqrt{T} \geq r_{t,a_t^*}^{s_t}-1/\sqrt{T} \geq r(p^{(a_t^*)}, z_t,\theta^*)-2/\sqrt{T}. 
    \end{align*}
    This completes the proof.
\end{proof}

\begin{proof}[\textbf{Proof of \Cref{thm:ucb_regret}}]
   First, the sample size condition on $T$ in \eqref{eq:minT_thm_short} in \Cref{thm:ucb_regret} can be explicitly written as
   \begin{align}\label{eq:minT_thm}
        T\geq \left\{\frac{4u_x^8}{L_p^4\lambda_z^4} \left(c_1\sqrt{\log(d)}+c_2\sqrt{\log (T)} \right)^4/d \right\}\vee \left[\frac{2048 M_{\psi 3}^4 \sigma^4}{\kappa^8L_p^2\lambda_z^2} \left\{d^2+\log (3TKS/\delta) \right\}^2/d \right],
    \end{align}
    where we denote $u_x=\sqrt{1+u^2}$ and $c_1,c_2>0$ are absolute constants from the concentration inequalities used in \Cref{lem:eigen}.
   
    By \Cref{lem:HPCB}, the good event $\mathcal{E}_X$ defined in \eqref{eq-def-event-EX} holds with probability at least $1-\delta-2\log(T)/T$, with $S = \lfloor \log_2(T)\rfloor$.  Recall we set the number of arms $K=\sqrt{T/d}/\log (T)$.
    
    Recall $\Psi_s(T)$ collects all rounds in $[S\tau+1,\tau+2,\cdots T]$ such that $a_t$ is chosen in \textbf{Step II} at the stage~$s$. Denote $\Psi^o_s(T)=\mathcal{F}_s\cup \Psi_s(T)$. Define $V_{s,t}=\sum_{i\in \Psi^o_s(t)} x_{i,a_i}x_{i,a_i}^\top$. By the proof of \Cref{lem:HPCB}, for all $T$ that satisfies \eqref{eq:minT_thm}, with probability at least $1-1/T$, we have $\lambda_{\min}(V_{s,t}) \geq \lambda_{\min}(V_\tau) \geq B$, where $B=\tau\lambda_{\min}(\Sigma_x)/4\geq 1$. Therefore, by Lemma 2 in \cite{li2017provably}, with probability at least $1- 2\log (T)/T$, we have that   
    \begin{align}\label{eq-proof-thm21-1}
        \sum_{t\in \Psi_s(T)} w_{t,a_t}^{(s)} = \sum_{t\in \Psi_s(T)} \alpha \|x_{t,a_t}\|_{V_{s,t}^{-1}} \leq \alpha\sqrt{2d\log(T/d)|\Psi_s(T)|},
    \end{align}
    for all $s\in[S].$ On the other hand, by \textbf{Step II} of \Cref{algorithm:SupCB}, we have that 
    \begin{align}\label{eq-proof-thm21-2}
        \sum_{t\in \Psi_s(T)} w_{t,a_t}^{(s)} \geq 2^{-s}|\Psi_s(T)|.
    \end{align}
    Combining \eqref{eq-proof-thm21-1} and \eqref{eq-proof-thm21-2}, we have that
    \begin{align}\label{eq:card_Psi}
        |\Psi_s(T)| \leq 2^s \alpha \sqrt{2d\log(T/d)|\Psi_s(T)|}.
    \end{align}
    
    Denote $\Psi_0(T)$ as the collection of rounds where $a_t$ is chosen in \textbf{Step III}.  Since $S= \lfloor\log_2 T\rfloor$, we have that $2^{-S}< 1/\sqrt{T}$ for $T \geq 4$. Therefore, each $t\in[S\tau+1,T]$ must be in one of $\Psi_s(T)$.  It holds that $\{S\tau+1,\tau+2,\cdots,T\}=\Psi_0(T)\cup \big\{\cup_{s=1}^S\Psi_s(T) \big\}$. Recall that $\tau=\sqrt{dT}.$ Together, with probability at least $1-\delta-2\log (T) /T$,  we have that
    \begin{align*}
        R_T=&\sum_{t=1}^T \big\{r(p_t^*, z_t, \theta^*)- r(p^{(a_t)}, z_t,\theta^*)\big\}\\
        =&\sum_{t=1}^T \big\{r(p_t^*, z_t, \theta^*)- r(p^{(a_t^*)}, z_t,\theta^*)\big\} + \sum_{t=1}^T \big\{r(p^{(a_t^*)}, z_t, \theta^*)- r(p^{(a_t)}, z_t,\theta^*)\big\}\\
        \leq & TL_r(u-l)/K  + \sum_{t=1}^{S\tau} \big\{r(p^{(a_t^*)}, z_t, \theta^*)- r(p^{(a_t)}, z_t,\theta^*)\big\} + \sum_{t=S\tau+1}^T \big\{r(p^{(a_t^*)}, z_t, \theta^*)- r(p^{(a_t)}, z_t,\theta^*)\big\}\\
        \leq & uL_r\sqrt{dT}\log (T) + S\tau uL_r + \sum_{t=S\tau+1}^T \big\{r(p^{(a_t^*)}, z_t, \theta^*)- r(p^{(a_t)}, z_t,\theta^*)\big\}\\
        \leq & \{\log (T)+S\}uL_r\sqrt{dT} + 2|\Psi_0(T)|/\sqrt{T} + \sum_{s=1}^S |\Psi_s(T)| (4\vee L_ru)2^{-(s-1)}\\
        \leq &3 uL_r\sqrt{dT}\log T + 2\sqrt{T} + 2\alpha (4\vee L_ru)\sum_{s=1}^S \sqrt{2d\log(T/d)|\Psi_s(T)|}\\
        \leq & 3 uL_r\sqrt{dT}\log T + 2\sqrt{T} + 2\alpha (4\vee L_ru) \sqrt{2STd\log(T/d)}\\
        \leq &  24\sigma u M_{\psi2}/\kappa \cdot (4 \vee L_r u)   \sqrt{dT\log (T)\log(T/\delta)\log (T/d)},
    \end{align*}
    where the first inequality follows from the Lipschitz condition of the revenue function as demonstrated in \Cref{lem:Lips} and that we partition $[l,u]$ into $K=\sqrt{T/d}/\log (T)$ equally spaced price points, the second inequality follows from the Lipschitz condition of the revenue function, the third inequality follows from \Cref{lem:gap_reward}, the fourth inequality follows from \eqref{eq:card_Psi}, and the fifth inequality follows from the Cauchy--Schwarz inequality.

    Therefore, set $\delta=1/\sqrt{T}$, we further have that
    \begin{align*}
        \mathbb E(R_T)& \leq  2/\sqrt{T} \cdot TL_r u + 24\sqrt{1.5}\sigma u M_{\psi2}/\kappa \cdot (4 \vee L_r u)  \cdot \sqrt{dT} \cdot \{\log (T)\}^{3/2}\\
        &\leq 30\sigma u M_{\psi2}/\kappa \cdot (4 \vee L_r u)  \cdot \sqrt{dT} \cdot \{\log (T)\}^{3/2}.
    \end{align*}
    This completes the proof.
\end{proof}

\Cref{thm:ucb_regret_unknownT} provides theoretical guarantees for the regret of supCB combined with the standard doubling trick described at the end of \Cref{sec-supCB}. For a given dimension $d$, denote $T_{d}^*$ as the minimum $T$ such that condition \eqref{eq:minT_thm_short} holds with $\delta=1/T$. 

\begin{thm}\label{thm:ucb_regret_unknownT}
Suppose \Cref{assum_feature} holds.   For any $T$ satisfying $T\geq T_d^{*2}/d$, we have that with probability at least $1-6\log (T)/\sqrt{dT}$, the regret of supCB with the doubling trick described at the end of \Cref{sec-supCB} is upper bounded by
    \begin{align*}
       R_T \leq B_{S4}\cdot \sqrt{dT} \{\log (T)\}^{3/2},       
    \end{align*}
where $B_{S4}>0$ is an absolute constant that only depend on quantities in \Cref{assum_feature}.
\end{thm}

\begin{proof}[Proof of \Cref{thm:ucb_regret_unknownT}]
It is easy to see that the total rounds up to the $k$th episode is $2^{k+1}-2.$ For any $T \geq 2$, there exists a $k\in \mathbb N_+$ such that $2^{k+1}-2\leq T <2^{k+2}-2.$ Denote $\overline T=2^{k+2}-2$. Clearly, $R_T \leq  R_{\overline T}$. Therefore, in the following, we focus on the analysis of $R_{\overline{T}}$, which is the total regret till the end of the $(k+1)$th episode. Define
      \begin{align}\label{eq-def-j}
         j=\min\Big\{i\in \mathbb N_+: 2^{i} \geq \sqrt{dT}\Big\}.
      \end{align}
      
	Without loss of generality, we assume $j \geq 2$ and $k \geq 3$.  To establish a high-probability upper bound, we decompose the regret $R_{\overline{T}}$ into $R^{(1)}$, which accounts for regret from episode $i \in [j-1]$ for $j > 1$, and $R^{(2)}$, which accounts for regret from episode $i \in[k+1]\setminus [j-1]$.  Denote $E_i=2^{i}$ as the length of the $i$th episode. First, by \Cref{lem:Lips},  we have
      \begin{align*}
        R^{(1)}\leq L_ru \sum_{i=1}^{j-1} E_i= L_ru \sum_{i=1}^{j-1} 2^{i} \leq L_ru 2^{j} \leq 2L_ru\sqrt{dT},
      \end{align*}
      where the last inequality follows from the definition of $j$ in \eqref{eq-def-j}.

      Second, by assumption, we know that $E_j\geq \sqrt{dT}\geq T_d^*.$ In other words, \Cref{thm:ucb_regret} applies to all episodes $[k+1]\setminus [j-1]$. Note that $\sum_{i=j}^{k+1}\sqrt{2^{i}}\leq 2.5 \sqrt{2^{k+2}}$. For $R^{(2)}$, by a union bound argument and \Cref{thm:ucb_regret}, we have that
      \begin{align*}
        R^{(2)}&\leq  24\sigma u M_{\psi2}/\kappa \cdot (4 \vee L_r u) \sum_{i=j}^{k+1}  \sqrt{d E_i} \cdot \sqrt{2}\{\log (E_i)\}^{3/2}\\
        &\leq 24\sigma u M_{\psi2}/\kappa \cdot (4 \vee L_r u) \cdot 2.5\sqrt{2} \sqrt{d 2^{k+2}} \{\log (T)\}^{3/2}\\
        &\leq 120\sigma u M_{\psi2}/\kappa \cdot (4 \vee L_r u) \cdot \sqrt{d T} \{\log (T)\}^{3/2},
      \end{align*}
      with probability at least
      \begin{align*}
        1-\sum_{i=j}^{k+1}\left\{1/E_i + 2\log (E_i)/E_i \right\} \geq 1-3\log (T)\sum_{i=j}^{k+1}1/E_i\geq 1- 3\log (T)/2^{j-1} \geq 1-6\log (T)/\sqrt{dT},
      \end{align*}
      where the second inequality holds by definition of $j$ in \eqref{eq-def-j}.
\end{proof}

\subsection{Technical details in \Cref{sec-ETC}}\label{subsec:ETC-appendix}
\begin{pro}\label{lem_punique}
Define $R_{\varsigma_o} = \big\{(\alpha^{\top} z,\beta^\top z): \, z\in\bar{\mathcal{Z}},~ \|(\alpha^{\top}, \beta^{\top})^{\top} - \theta^*\|\leq \varsigma_o\big\}.$ Under Assumptions \ref{assum_feature}(a), (c), (e) and \ref{assum:etc}, there exists a bivariate Lipschitz continuous function ${\varphi}(\cdot, \cdot)$ over $R_{\varsigma_o}$ with a Lipschitz constant $C_\varphi > 0$ such that
\[
    p^* = \argmax_{p\in[l,u]} \left\{p \psi'(\alpha^\top z - \beta^\top z p)\right\} = {\varphi}(\alpha^{\top}z,\beta^\top z),
\]
for all $z\in\bar{\mathcal{Z}}$ and $\theta\in \{(\alpha^{\top}, \beta^{\top}): \, \|\theta-\theta^*\|\leq \varsigma_o\} \subset \mathbb{R}^{2d}$. Moreover, $C_\varphi$ is an absolute constant that only depends on quantities in Assumptions \ref{assum_feature} and \ref{assum:etc}.
\end{pro}

\begin{proof}[\textbf{Proof of \Cref{lem_punique}}]
By definition,
\begin{align*}
   p^*=\argmax_{p\in[l,u]} \{p\cdot\psi'(\alpha^\top z - \beta^\top z p)\}.
\end{align*}
Thus, $p^*$ depends on $\theta$ and $z$ via $\alpha^{\top} z$ and $\beta^\top z$.

By Assumption \ref{assum:etc}, for all $z\in\bar{\mathcal{Z}}$ and $\theta=(\alpha,\beta)$ where $\|\theta-\theta^*\|\leq \varsigma_o$, the optimal price $p^*$ is unique. Therefore, for all $(\alpha^{\top} z, \beta^{\top} z)\in R_{\varsigma_o}$, we have that $p^*$ is in fact a function of $\alpha^{\top} z$ and $\beta^\top z$ and can be written as $p^*={\varphi}(\alpha^{\top}z,\beta^\top z)$. Note it is clear that $R_{\varsigma_o}$ is a compact set in $\mathbb R^2$ and furthermore $R_{\varsigma_o} \subseteq \{(w,v): |w|\leq u_a+\varsigma_o, |v|\leq u_b+\varsigma_o\}$ by \Cref{assum_feature}(a) and (c). 

We next prove that $ {\varphi}(w,v)$ is continuous over $R_{\varsigma_o}$ by Berge's Maximum Theorem \citep{berge1957two}. To apply Berge's Maximum Theorem, we first define $\Gamma(w,v)=[l,u]$ for any $w,v$. Note that $\Gamma(w,v)$ is a set-valued function and is continuous at any $(w,v)$. By the definition of $\varphi(w,v)$, we have that
    \begin{align*}
        \varphi(w,v)=\argmax_{p\in[l,u]} \{p\cdot \psi'(w -v p)\} = \argmax_{p\in\Gamma(w,v)} \{p\cdot \psi'(w -v p)\}.
    \end{align*}

By \Cref{assum_feature}(e), $p\cdot \psi'(w -v p)$ is a continuous function of $(p,w,v)$. Together with the fact that $\Gamma(w,v)$ is continuous at any $(w,v)$, by Berge's Maximum Theorem, we have that $\varphi(w,v)$ (as a set-valued function) is upper hemicontinuous at any $(w,v)$. Moreover, by the above discussion, we know that $\varphi(w,v)$ is indeed a function over $R_{\varsigma_o}$. Thus, by the definition of upper hemicontinuity, we have that $\varphi(w,v)$ is continuous over $R_{\varsigma_o}$.

We now further show that $\varphi(w,v)$ is a continuously differentiable function over $R_{\varsigma_o}$ based on an application of the implicit function theorem. For any fixed $(w_o,v_o) \in R_{\varsigma_o}$, by \Cref{assum:etc}, there is a unique optimal price $p^* \in (l,u).$ Note that by \Cref{assum_feature}(e), the revenue function $r(p; w_o,v_o)=p\cdot \psi'(w_o -v_o p)$ is twice continuously differentiable. Define $f(p;w_o,v_o)=\psi'(w_o-v_op)-p\psi'(w_o-v_op)v_o$, which is the derivative of $r(p;w_o,v_o).$ Since $p^*$ is the unique maximizer of $r(p;w_o,v_o)$, we therefore have that $f(p^*;w_o,v_o)=0$ and $f'(p^*;w_o,v_o)<0$. Therefore, viewing $f(p;w,v)=\psi'(w-vp)-p\psi'(w-vp)v$ as a trivariate function of $(w,v,p)$, by the implicit function theorem, we have that there exists a continuously differentiable function $g(w,v)$ (that may depend on $(w_o,v_o)$) such that for all points $(w,v)$ in a small neighborhood of $(u_o,v_o)$, we have that $f(g(w,v), w,v)=0.$ By definition, this $g(w,v)$ is the optimal price $p^*$ at $(w,v)$. On the other hand, by our above result, we know that $p^*=\varphi(w,v)$. Therefore, we have that $\varphi(w,v)=g(w,v)$ for all $(w,v)$ in  $R_{\varsigma_o}$ and thus $\varphi(w,v)$ is continuously differentiable over $R_{\varsigma_o}$, which completes the proof.
\end{proof}


\begin{proof}[\textbf{Proof of \Cref{thm:ETC}}]
First, the sample size condition on $T$ in \eqref{eq:ETC_strong_cond_short} in \Cref{thm:ETC} can be explicitly written as
\begin{align}\label{eq:ETC_strong_cond}
    T \geq \left\{\frac{4u_x^8}{L_p^4d\lambda_z^4} \left(c_1\sqrt{\log(d)}+c_2\sqrt{\log  (1/\delta)} \right)^4 \right\} \vee \left\{\frac{144\sigma^4}{L_p^2d\lambda_z^2\kappa^4\varsigma_o^4} \left(d+\log  (1/\delta) \right)^2\right\}.
\end{align}
Recall that we denote $u_x=\sqrt{1+u^2}$ and $c_1,c_2>0$ are absolute constants from the concentration inequalities used in \Cref{lem:eigen}. We now give the proof in three steps. 

\textbf{(I).\ Bounds on the estimation error and for the design matrix}: Denote $\Psi$ as the index set consisting of price experiments. By design, $\tau=|\Psi|$. Denote $V=\sum_{t\in \Psi} x_tx_t^\top $ and denote $\widehat\theta$ as the MLE based on $\Psi$. By \Cref{assum_feature}(b) and its follow-up discussion, we have that $\lambda_{\min}(\Sigma_x)\geq L_p\lambda_{\min}(\Sigma_z).$
 
Set $\delta_1=\delta_2=\delta$ and define 
\begin{align*}
    B_0:=\frac{3\sigma^2}{\kappa^2\varsigma_o^2} (d+\log(1/\delta_2)) \text{ and } B_1:=\frac{\tau \lambda_{\min}(\Sigma_x)}{4}.
\end{align*}
By simple algebra and \Cref{lem:eigen}(ii), for any $\tau$ such that 
\begin{align}\label{eq:min_cycle}
    \tau>2u_x^4\left( \frac{c_1\sqrt{\log d}+c_2\sqrt{\log(1/\delta_1)}}{\lambda_{\min}(\Sigma_x)}\right)^2,
\end{align}
with probability at least $1-\delta_1$, we have that $\lambda_{\min}(V)\geq B_1.$ In the ETC algorithm, we set $\tau=\sqrt{dT\log(1/\delta)}$. Therefore, under condition \eqref{eq:ETC_strong_cond}, by simple algebra, we have $\tau=\sqrt{dT\log(1/\delta)}$ satisfies \eqref{eq:min_cycle} and furthermore $B_1 \geq B_0.$

Therefore, by \Cref{thm:theta_err} and a union bound argument, with probability at least $1-\delta_1-\delta_2$, the event $\mathcal{A}$ holds, where
\begin{align*}
   \mathcal{A}=\bigg\{ &\|\widehat\theta-\theta^*\| \leq \varsigma_o \text{ and } (\widehat\theta-\theta^*)^\top V (\widehat\theta-\theta^*) \leq \frac{12\sigma^2}{\kappa^2}\{d+ \log(1/\delta_2)\} \\
   &\text{ and } \|\widehat\theta-\theta^*\|\leq \frac{2\sqrt{3}\sigma}{\kappa \sqrt{\lambda_{\min}(\Sigma_x)}}\sqrt{\frac{d+\log(1/\delta_2)}{\tau}} \bigg\}.
\end{align*}

\textbf{(II).\ High probability events}: Denote $\Phi$ as the index set consisting of price explorations. By design, we have that $|\Phi|=T-\tau.$ Denote $U=\sum_{i\in \Phi}z_iz_i^\top.$ Set $\delta_3=\delta$. By \Cref{lem:eigen}(i), we have that with probability at least $1-\delta_3$ the events $\mathcal B$ and $\mathcal{C}$ hold, where
\begin{align*}
    &\mathcal{B}=\left\{ \left\|V/\tau-\Sigma_x\right\|_{\mathrm{op}}\leq \max(\varsigma_1,\varsigma_1^2)  \right\}, \text{ with } \varsigma_1=3u_x^2\sqrt{\log(2d)+\log(1/\delta_3)}/\sqrt{\tau},\\
    &\mathcal{C} = \left\{ \left\|U/(T-\tau)-\Sigma_z\right\|_{\mathrm{op}}\leq \max(\varsigma_2,\varsigma_2^2)  \right\}, \text{ with } \varsigma_2=3\sqrt{\log(2d)+\log(1/\delta_3)}/\sqrt{T-\tau}.
\end{align*}

In the following, we assume that the good events $\mathcal{A},\mathcal{B}$ and $\mathcal{C}$ hold, which is of probability at least $1-\delta_1-\delta_2-\delta_3.$ For $t\in \Phi$, denote the estimated optimal price based on $\widehat\theta=(\widehat\alpha, \widehat\beta)$ as $\widehat p_t=\argmax_{p\in[l,u]} \{p\cdot\psi'(\widehat\alpha^\top z_{t} - \widehat\beta^\top z_{t} p)\}$. By \Cref{lem_punique}, we have that there exists a continuous function $\varphi(u,v)$ such that
\begin{align*}
    \widehat p_t= \varphi(\widehat\alpha^\top z_t,\widehat\beta^\top z_t)
\end{align*}
for all $z_t\in\mathcal{Z}$ and $\|\widehat\theta-\theta^*\|\leq \varsigma_o$. Moreover,  $\varphi(v_1,v_2)$ is $C_{\varphi}$-Lipschitz continuous in $u, v$ where $C_{\varphi}$ is an absolute constant that only depends on $\psi$ and the price range $[l,u].$ Therefore, we have that
\begin{align}\label{bound_explore}
    &\sum_{t \in \Phi} \{r(p_t^*, z_t, \theta^*)- r(\widehat p_t, z_t, \theta^*) \}\\\notag
   =&\sum_{t \in \Phi} \bigg|\frac{\partial^2}{\partial p^2}r(\widetilde p_t, z_t, \theta^*)\bigg| (p_t^*-\widehat p_t)^2\\\notag
   \leq & C_r C_\varphi^2  \sum_{t \in \Phi}\big[ (\widehat\alpha ^\top z_t-\alpha^{*\top} z_t)^2 + (\widehat\beta ^\top z_t-\beta^{*\top} z_t)^2\big]\\\notag
   =& C_r C_\varphi^2 \big[ (\widehat\alpha-\alpha^*)^{\top}U (\widehat\alpha-\alpha^*) + (\widehat\beta-\beta^*)^{\top}U (\widehat\beta-\beta^*)\big],
\end{align}
where the first equality follows from a Taylor expansion with $\widetilde p_t$ between $p_t^*$ and $\widehat p_t$ and the inequality follows from the fact that the second order derivative of $r(p, z_t, \theta^*)$ is upper bounded by $C_r:=2M_{\psi2}u_b+uu_b^2M_{\psi3}$ for $p\in [l,u]$ due to the boundedness of $z_t^\top \alpha^*$ and $z_t^\top \beta^*$ and the differentiability of $\psi$ in  in \Cref{assum_feature}. We now analyze the last term in more details. In particular, denote $N=T-\tau$, we consider two scenarios.

\textbf{(III).\ Regret bounds}: \textsc{Under condition \eqref{eq:ETC_strong_cond} on $(d,T)$:} We have
\begin{align*}
&(\widehat\alpha-\alpha^*)^{\top}U (\widehat\alpha-\alpha^*) + (\widehat\beta-\beta^*)^{\top}U (\widehat\beta-\beta^*)\\
=&N(\widehat\alpha-\alpha^*)^{\top}\Sigma_z (\widehat\alpha-\alpha^*) +  N(\widehat\beta-\beta^*)^{\top}\Sigma_z (\widehat\beta-\beta^*)\\
   &+ (\widehat\alpha-\alpha^*)^{\top}\left(U -N\Sigma_z\right)(\widehat\alpha-\alpha^*) + (\widehat\beta-\beta^*)^{\top}\left(U -N\Sigma_z\right) (\widehat\beta-\beta^*)\\
   =& N  (\widehat\theta-\theta^*)^{\top}(I_2\otimes\Sigma_z) (\widehat\theta-\theta^*) + (\widehat\theta-\theta^*)^{\top}\left(I_2\otimes (U -N\Sigma_z)\right)(\widehat\theta-\theta^*)\\
   \leq & \frac{N}{L_p}  (\widehat\theta-\theta^*)^{\top}(\Sigma_p\otimes\Sigma_z) (\widehat\theta-\theta^*) +    \|U -N\Sigma_z\|_{\mathrm{op}}\|\widehat\theta-\theta^*\|^2\\
   =&  \frac{N}{L_p\tau} \cdot (\widehat\theta-\theta^*)^{\top}V(\widehat\theta-\theta^*) +  \frac{1}{L_p} \cdot (\widehat\theta-\theta^*)^{\top}(N\Sigma_x-NV/\tau) (\widehat\theta-\theta^*) + \|U -N\Sigma_z\|_{\mathrm{op}}\|\widehat\theta-\theta^*\|^2\\
   \leq & \frac{N}{L_p\tau} \cdot (\widehat\theta-\theta^*)^{\top}V (\widehat\theta-\theta^*) + \frac{N}{L_p}\|V/\tau-\Sigma_x\|_{\mathrm{op}}\|\widehat\theta-\theta^*\|^2+N\|U /N-\Sigma_z\|_{\mathrm{op}}\|\widehat\theta-\theta^*\|^2\\
   \leq & \frac{T}{L_p\tau}  \frac{12\sigma^2}{\kappa^2}(d+\log (1/\delta_2))  + \frac{T}{L_p\tau}\max(\varsigma_1,\varsigma_1^2) \cdot \frac{12\sigma^2(d+\log(1/\delta_2))}{\kappa^2 {\lambda_{\min}(\Sigma_x)}}+\frac{T}{\tau}\cdot \frac{12 \sigma^2(d+\log(1/\delta_2))}{L_p\kappa^2} \\
   =& \frac{T}{\tau}\cdot \frac{12 \sigma^2(d+\log  (1/\delta))}{L_p\kappa^2} \left\{ 2+ \frac{\max(\varsigma_1,\varsigma_1^2)}{\lambda_{\min}(\Sigma_x)} \right\}\\
   \leq & \frac{T}{\tau}\cdot \frac{36 \sigma^2(d+\log  (1/\delta))}{L_p\kappa^2} \leq \frac{36 \sigma^2}{L_p\kappa^2}\sqrt{dT\log  (1/\delta)},
\end{align*}
where the first inequality follows from fact that the minimum eigenvalue of $\Sigma_p$ is lower bounded by $L_p$, the third inequality follows from the fact that the events $\mathcal{A}$, $\mathcal{B}$ and $\mathcal{C}$ hold, the fourth inequality follows from the fact that under condition \eqref{eq:ETC_strong_cond}, we have $\varsigma_1 /\lambda_{\min}(\Sigma_x) <1$.

By \Cref{lem:Lips}, we have that $r(p_t^*,z_t,\theta^*)\leq u M_{\psi 1}.$ Therefore, with probability at least $1-3\delta$, 
under condition \eqref{eq:ETC_strong_cond} we have that
\begin{align*}
    R_T=& \sum_{t=1}^\tau r(p_t^*, z_t, \theta^*)- r(p_t, z_t, \theta^*) + \sum_{t=\tau+1}^T r(p_t^*, z_t, \theta^*)- r(\widehat p_t, z_t, \theta^*) \\
\leq& \sqrt{dT\log  (1/\delta)}uM_{\psi 1} +C_rC_\varphi^2 \frac{36 \sigma^2}{L_p\kappa^2}\sqrt{dT\log  (1/\delta)} \leq \left(uM_{\psi 1} + \frac{36 \sigma^2 C_rC_\varphi^2}{L_p\kappa^2}\right) \sqrt{dT\log  (1/\delta)}.
\end{align*}
The expectation bound follows directly with $\delta=1/T$ and thus completes the proof.
\end{proof}

\begin{proof}[\textbf{Proof of \Cref{thm:etc_regret_unknownT}}]
      The proof is similar to that of Theorem \ref{thm:ucb_regret_unknownT}. 
      
      It is easy to see that the total rounds up to the $k$th episode is $2^{k+1}-2.$ For any $T$, there exists a $k\in \mathbb N_+$ such that $2^{k+1}-2\leq T <2^{k+2}-2.$ Denote $\overline T=2^{k+2}-2$. Clearly, $R_T \leq  R_{\overline T}$. Therefore, in the following, we focus on the analysis of $R_{\overline{T}}$, which is the total regret till the $k+1$th episode. Define
      \begin{align*}
         j=\min_{i}\Big\{i\in \mathbb N_+: 2^{i} \geq \sqrt{dT}\Big\}.
      \end{align*}
      
      To establish the high probability bound, we decompose the regret $R_{\overline{T}}$ into $R^{(1)}$, which accounts for regret from episode $i=1,\cdots,j-1$, and $R^{(2)}$, which accounts for regret from episode $i=j,j+1,\cdots,k+1.$  Denote $E_i=2^{i}$ as the length of the $i$th episode. First, by \Cref{lem:Lips},  we have that $r(p_t^*,z_t,\theta^*)\leq u M_{\psi 1},$  hence 
      \begin{align*}
        R^{(1)}\leq uM_{\psi1} \sum_{i=1}^{j-1} E_i= uM_{\psi 1} \sum_{i=1}^{j-1} 2^{i} \leq uM_{\psi 1} 2^{j} \leq 2uM_{\psi 1}\sqrt{dT}.
      \end{align*}
      where the last inequality follows from the definition of $j.$

      Second, by assumption, we know that $E_j\geq \sqrt{dT}\geq T_{d}^*.$ In other words, \Cref{thm:ETC} applies to all episodes $j,j+1,\cdots,k+1.$ Note that $\sum_{i=j}^{k+1}\sqrt{2^{i}}\leq 2.5 \sqrt{2^{k+2}}$. For $R^{(2)}$, by union bound and \Cref{thm:ETC}, we have that
      \begin{align*}
         R^{(2)}\leq&   \left[
         uM_{\psi 1} + {36 \sigma^2C_rC_\varphi^2}/(L_p\kappa^2)    \right]\sum_{i=j}^{k+1} \sqrt{E_id\log E_i}
            \\\leq &  \left[
         uM_{\psi 1} + {36 \sigma^2C_rC_\varphi^2}/(L_p\kappa^2)    \right] 2.5\sqrt{ 2^{k+2}d\log T}
             \\\leq & \left[
         4uM_{\psi 1} + {144 \sigma^2C_rC_\varphi^2}/(L_p\kappa^2)    \right] \sqrt{ Td\log T}
      \end{align*}
      with probability at least
      \begin{align*}
          1-\sum_{i=j}^{k+1} 3/E_i\geq 1-3/2^{j-1}  \geq 1-6/\sqrt{dT},
      \end{align*}
      where the second inequality holds by definition of $j$.
\end{proof}

\subsection{The MLE-cycle algorithm and a simple modification}\label{sec:mle_cycle}
The MLE-Cycle algorithm is first proposed in \cite{broder2012dynamic} for dynamic pricing under an unknown demand model with no covariates (i.e.\ context-free) and is shown to provide a regret of order $O(\sqrt{T})$. We summarize MLE-Cycle in \Cref{algorithm:mlecycle}. A (slight) variant of MLE-Cycle is further introduced in \cite{ban2021personalized} for contextual dynamic pricing with a $d$-dimensional covariate and is shown to provide a regret of order $\widetilde O(d\sqrt{T})$.

The key idea of the MLE-Cycle algorithm is to divide the horizon $\mathbb{N}_+$ into cycles of increasing lengths. In particular, the $c$th cycle is of length $c+k$, where $k$ is a constant. Within each cycle, the first $k$ rounds are allocated for price experiments and the next $c$ rounds are used for price exploitation. By simple algebra, it is easy to see that for all sufficiently large $T,$ the number of price experiments conducted till the $T$th round is of order $O(k\sqrt{T})$.

In the above mentioned literature, $k$ is set as an absolute constant (typically $2$) and MLE-Cycle gives a regret of order $O(d\sqrt{T})$, which is sub-optimal - supported by \Cref{thm_lb}. The intuitive reason is that since $k$ is a constant that does not increase with the dimension $d$, the number of exploration is of order $O(\sqrt{T})$ instead of the desired $O(\sqrt{dT\log (T)})$. In other words, the (original) MLE-Cycle algorithm under-explores. Therefore, a simple fix is to make $k$ increase with the dimension $d$. In particular, in the modified MLE-Cycle algorithm, we set the length of price experiments for the $c$th cycle as 
$$k=\sqrt{d\log (2c)}.$$

By simple algebra, for all sufficiently large $T,$ the number of price experiments conducted by the modified MLE-Cycle till the $T$th round is of order $O(\sqrt{dT\log (T)})$. Using similar arguments as the ones in our proof of \Cref{thm:ETC}, we can show that, under similar conditions as the ones used in \Cref{thm:etc_regret_unknownT}, the modified MLE-Cycle can achieve the near-optimal regret of order $O(\sqrt{dT\log (T)})$.

\vspace{2mm}
\begin{algorithm}[ht]
\begin{algorithmic}
    \State \textbf{Input}: Price interval $[l,u]$, exploration length $k$ and $\Psi=\varnothing.$
    \medskip

    \For{cycles $c=1,2,\cdots$}
    \State a.\ (Exploration): For the first $k$ rounds, randomly choose $p_t \in [l,u]$, and record $y_t$\\\hspace{10mm} and $x_t=(z_t^\top,-p_tz_t^\top)^\top$ into the price experiment set $\Psi$.
    \medskip

    \State c.\ (Estimation): Obtain MLE $\widehat\theta$ based on observations in $\Psi$.
    \medskip

    \State d.\ (Exploitation): For the next $c$ rounds, offer the greedy price at $p^*(\widehat\theta,z_t)$ based on $\widehat\theta$.
    \EndFor
    
    \caption{The MLE-Cycle algorithm for dynamic pricing}
    \label{algorithm:mlecycle}
\end{algorithmic}
\end{algorithm}

{
\color{black}

\subsection{The Semi-Myopic algorithm and a simple modification}\label{sec:semi-myopic}
The Semi-Myopic type algorithm is a class of pricing policies modified from the greedy pricing policy and is proposed in \cite{keskin2014dynamic} for dynamic pricing under an unknown linear demand model with no covariates~(i.e.\ context-free). Variants of Semi-Myopic algorithms have also been proposed in the literature, such as Controlled Variance in \cite{den2014simultaneously} for context-free dynamic pricing and Random Shock Design in \cite{simchi2024pricing} for contextual dynamic pricing (and its statistical inference problem). 

These works show Semi-Myopic can provide a regret of order $\widetilde{O}(\sqrt{T})$, without investigating the impact of dimension $d$ (i.e.\ treating the dimension $d$ as fixed). In fact, one can easily show that for $d$-dimensional contextual dynamic pricing, these algorithms achieve a regret of order $\widetilde{O}(d\sqrt{T})$. We summarize a variant of Semi-Myopic in \Cref{algorithm:semi_myopic}, which is based on the CILS algorithm in \cite{keskin2014dynamic} and the random shock design in \cite{simchi2024pricing}. 

The key idea of the Semi-Myopic algorithm is to inject a carefully calibrated price deviation into the price given by the greedy algorithm (which is myopic), hence the name Semi-Myopic. In particular, at each round $t$, given context $z_t$, the firm offers a price at $p^*(\widehat\theta,z_t)+\delta_t$, where $p^*(\widehat\theta,z_t)$ is the greedy price based on the estimated parameter $\widehat{\theta}$ and $\delta_t=\kappa t^{-1/4}\cdot B_t$ is the injected price deviation with $B_t$ being a Rademacher random variable. Importantly, one can show that thanks to the injected price deviation, for all sufficiently large $t,$ the minimum eigenvalue of the design matrix based on observations in rounds $\{1,2,\cdots,t\}$ is guaranteed to be of order $\Omega(\sqrt{t}\lambda_z)$, which is the key component for achieving a sublinear regret.

In the above mentioned literature, $\kappa$ is set as an absolute constant (typically $1$) and Semi-Myopic gives a regret of order $\widetilde{O}(d\sqrt{T})$, which is sub-optimal - supported by \Cref{thm_lb}. The intuitive reason is that since $\kappa$ is a constant that does not increase with the dimension $d$, making the minimum eigenvalue of the design matrix take the order $\Omega(\sqrt{t}\lambda_z)$ instead of the desired $\Omega(\sqrt{dt}\lambda_z)$. In other words, the (original) Semi-Myopic algorithm under-explores and injects less price deviation than the optimal level. Therefore, a simple fix is to make $\kappa$ increase with the dimension $d$. In particular, in the modified Semi-Myopic algorithm, we instead set the price deviation for round $t$ as $\kappa=d^{1/4}\cdot t^{-1/4}.$ Using similar arguments as the ones in our proof of \Cref{thm:ETC}, we can show that, under similar conditions as the ones used in \Cref{thm:etc_regret_unknownT}, the modified Semi-Myopic can achieve the near-optimal regret of $\widetilde{O}(\sqrt{dT})$.

\textbf{Computational consideration}: Note that due to its design, Semi-Myopic require the computation of $O(T)$ many MLEs, which is computationally much more intensive than MLE-Cycle and ETC. Therefore, in our implementation, we only run step (b) and (d) of Semi-Myopic every 5 rounds, which reduces its computational cost while still maintaining its theoretical guarantees.

\vspace{2mm}
\begin{algorithm}[ht]
{\color{black}
\begin{algorithmic}
    \State \textbf{Input}: Price interval $[l,u]$, exploration parameter $\kappa$ and $\Psi=\varnothing.$
    \medskip

    \For{$t=1,2,\cdots$}
    \State a.\ (Price deviation): Set $\delta_t=\kappa t^{-1/4}\cdot B_t$, where $B_t$ is a Rademacher random variable. 
    \medskip

    \State b.\ (Estimation): Obtain MLE $\widehat\theta$ based on observations.
    \medskip

    \State c.\ (Semi-Exploitation): Offer price $p_t$ at $p^*(\widehat\theta,z_t)+\delta_t$, where $p^*(\widehat\theta,z_t)$ is the greedy price\\\hspace{10mm} based on the MLE.
    
    \State d.\ (Record data): Record $y_t$ and $x_t=(z_t^\top,-p_tz_t^\top)^\top$ into the set $\Psi$.
    \EndFor
    
    \caption{\textcolor{black}{The Semi-Myopic algorithm for dynamic pricing}}
    \label{algorithm:semi_myopic}
\end{algorithmic}}
\end{algorithm}

}

\section{Technical details in \Cref{sec-ldp}}\label{sec:supplement_proof_private}

\subsection{LDP guarantees}

\Cref{lem_l2ball} is directly from Section I.2 in \cite{duchi2018minimax} and we omit its proof.  It ensures the LDP guarantees of \Cref{algorithm:l2ball}.

\begin{lem}\label{lem_l2ball}
Let $\mathrm{W}(g)$  be the output of $L_2$-ball($g,C_g,\epsilon$). Then, $\mathrm{W}(g)$
is an $\epsilon$-LDP view of $g$, and $\mathbb{E}\left[ \mathrm{W}(g)|g\right]=g$ and $\| \mathrm{W}(g)\|=C_gr_{\epsilon, d}$.
\end{lem}

\begin{pro}\label{prop-eps-LDP}
The LDP-ETC policy is $\epsilon$-LDP. 
\end{pro}

\begin{proof}[\textbf{Proof of \Cref{prop-eps-LDP}}] 

Note that by design of ETC-LDP, for all $\theta\in\Theta$, $
        \|g_t^{[C]}(\theta;s_t)\|\leq C_g.$ 
Define the truncation mechanism of $g_t$ as the mapping of $\mathrm{T}: g_t\mapsto g_t^{[C]}$, and the $L_2$-ball mechanism as $\mathrm{W}:g_t^{[C]}\mapsto w_t$, then our mechanism is a composition of truncation and $L_2$-ball, i.e.~$\mathrm{W}\circ\mathrm{T}: g_t\mapsto w_t$. By definition,  for all $s_t$, $s'_t$, $w_{<t}$,  $g_t$ and $g_t'$, we have for any measurable set $S$, 
\begin{flalign*}
    \mathbb{P}(\mathrm{W}(\mathrm{T}(g_t))\in S|s_t, w_{<t})=&\mathbb{P}(\mathrm{W}(g_t^{[C]})\in S|s_t, w_{<t})\\\leq & \exp(\epsilon)\mathbb{P}(\mathrm{W}(g_t'^{[C]})\in S|s'_t, w_{<t})=\exp(\epsilon)\mathbb{P}(\mathrm{W}(\mathrm{T}(g_t'))\in S|s'_t, w_{<t}),
\end{flalign*}
where the inequality holds by Lemma \ref{lem_l2ball} and the definition of LDP. This implies that $\mathrm{W}\circ\mathrm{T}$ is an $\epsilon$-LDP mechanism.  We note that no privacy issue is involved for $t>\tau$.   
\end{proof}

\subsection{Proof of \Cref{lem:sgd}}
This section collects the proof of \Cref{lem:sgd} and auxiliary lemmas.   Recall that 
\[
    l(\theta;s_t)=y_tx_t^{\top}\theta-\psi(x_t^{\top}\theta) \quad \mbox{and} \quad g(\theta;s_t)=[y_t-\psi'(x_t^{\top}\theta)]x_t, 
\]
are the stochastic likelihood function and the stochastic gradient function, respectively, with their population counterparts being 
\[
    l(\theta)=\mathbb{E}[y_tx_t^{\top}\theta-\psi(x_t^{\top}\theta)] \quad \mbox{and} \quad g(\theta) = \mathbb{E}\{[y_t-\psi'(x_t^{\top}\theta)]x_t\}.
\]
\begin{proof}[\textbf{Proof of \Cref{lem:sgd}}]
By Taylor's expansion,  for any $\theta_1,\theta_2\in\Theta$, and some $\widetilde{\theta}$ in between, we have 
$$
l(\theta_1)-l(\theta_2)=g(\theta_2)^{\top}(\theta_1-\theta_2)-\frac{1}{2}(\theta_1-\theta_2)^{\top}\mathbb{E}[\psi''(x_t^{\top}\widetilde{\theta})x_t^{\top}x_t](\theta_1-\theta_2).
$$

Therefore, under Assumption \ref{assum_ldp}, we have that the negative likelihood function $-\ell(\theta)$ is $\kappa^*\lambda_{\min}(\Sigma_x)$-strongly convex, i.e. 
\begin{flalign}\label{convex}
 l(\theta_2)-l(\theta_1)+g(\theta_2)^{\top}(\theta_1-\theta_2)\geq    \frac{\kappa^*\lambda_{\min}(\Sigma_x)}{2}\|\theta_1-\theta_2\|^2.   
\end{flalign}
Note that $l(\theta^*)\geq l(\theta)$ for any $\theta\in\Theta$, and $g(\theta^*)=0$, by setting $\theta_2=\theta^*$ and $\theta_1=\theta$, we thus have 
\begin{flalign}\label{convex2}
  l(\theta^*)-l(\theta)\geq \frac{\kappa^*\lambda_{\min}(\Sigma_x)}{2}\|\theta^*-\theta\|^2.
\end{flalign}
Recall that $\lambda_{\min}(\Sigma_x)\geq L_p\lambda_z$ by \eqref{lambda_min}, given that $\zeta_l\leq \kappa^*\lambda_{\min}(\Sigma_x)$, for all $1\leq t\leq \tau$, we have  
\begin{flalign}\label{bound_theta}
\begin{split}
    &\|\widehat{\theta}_t-\theta^*\|^2\\=&\|\Pi_{\Theta}[\widehat{\theta}_{t-1}+\eta_{t}w_{t}]-\theta^*\|^2\\ \stackrel{(a)}{\leq}  &\|\widehat{\theta}_{t-1}+\eta_{t}w_{t}-\theta^*\|^2
    \\=&\|\widehat{\theta}_{t-1}-\theta^*\|^2+2\eta_{t}w_{t}^{\top}(\widehat{\theta}_{t-1}-\theta^*)+ \eta_{t}^2\|w_{t}\|^2
    \\ \stackrel{(b)}{=}&\|\widehat{\theta}_{t-1}-\theta^*\|^2+2\eta_{t}g(\widehat{\theta}_{t-1})^{\top}(\widehat{\theta}_{t-1}-\theta^*)+2\eta_{t}[w_{t}-g(\widehat{\theta}_{t-1})]^{\top}(\widehat{\theta}_{t-1}-\theta^*) +\eta_{t}^2C_g^2r_{\epsilon,d}^2\\
    \stackrel{(c)}{\leq} &(1-2\eta_t\zeta_l)\|\widehat{\theta}_{t-1}-\theta^*\|^2 +2\eta_{t}[w_{t}-g(\widehat{\theta}_{t-1})]^{\top}(\widehat{\theta}_{t-1}-\theta^*) +\eta_{t}^2C_g^2r_{\epsilon,d}^2,
\end{split}
\end{flalign}
where (a) holds by noting  $\theta^*\in \Theta$,  (b) holds since  $\|w_t\|=C_gr_{\epsilon,d}$ due to Lemma \ref{lem_l2ball}, and  (c) holds by \eqref{convex}, \eqref{convex2} and the fact that $\zeta_l\leq \kappa^*\lambda_{\min}(\Sigma_x)$, such that
$$
g(\widehat{\theta}_{t-1})^{\top}(\widehat{\theta}_{t-1}-\theta^*) \leq -\frac{\zeta_l}{2}\|\widehat{\theta}_{t-1}-\theta^*\|^2+l(\widehat{\theta}_{t-1})-l(\theta^*)\leq -\zeta_l\|\widehat{\theta}_{t-1}-\theta^*\|^2. 
$$Hence, we have 
$$
\|\widehat{\theta}_t-\theta^*\|^2\leq (1-\frac{2}{t})\|\widehat{\theta}_{t-1}-\theta^*\|^2+\frac{2}{\zeta_lt}[w_{t}-g(\widehat{\theta}_{t-1})]^{\top}(\widehat{\theta}_{t-1}-\theta^*) +\frac{C_g^2r_{\epsilon,d}^2}{\zeta_l^2t^2}.
$$
Therefore, by iteration, with the convention $\prod_{j=t+1}^{t}=1,$ we have that for $t\geq 2$, 
\begin{flalign*}
     \|\widehat{\theta}_t-\theta^*\|^2 \leq \frac{2}{\zeta_l}\sum_{i=2}^{t}\frac{1}{i}\prod_{j=i+1}^{t}(1-\frac{2}{j})[w_{i}-g(\widehat{\theta}_{i-1})]^{\top}(\widehat{\theta}_{i-1}-\theta^*)+\frac{C_g^2r_{\epsilon,d}^2}{\zeta_l^2}\sum_{i=2}^{t}\frac{1}{i^2}\prod_{j=i+1}^{t}(1-\frac{2}{j}).
\end{flalign*}
Using the fact that $\prod_{j=i+1}^{t}(1-2/j)=i(i-1)/[t(t-1)]$, we have that 
\begin{flalign*}
    &\sum_{i=2}^{t}\frac{1}{i}\prod_{j=i+1}^{t}(1-\frac{2}{j})[w_{i}-g(\widehat{\theta}_{i-1})]^{\top}(\widehat{\theta}_{i-1}-\theta^*)=\frac{1}{t(t-1)}\sum_{i=2}^t (i-1)[w_{i}-g(\widehat{\theta}_{i-1})]^{\top}(\widehat{\theta}_{i-1}-\theta^*),\quad 
    \\&\sum_{i=2}^{t}\frac{1}{i^2}\prod_{j=i+1}^{t}(1-\frac{2}{j})\leq \frac{1}{t}.
\end{flalign*}
Therefore, we have for all $2\leq t\leq \tau$,
\begin{flalign}\label{bound_thetatau}
\|\widehat{\theta}_{t}-\theta^*\|^2 \leq  &\frac{2}{\zeta_l(t-1)t}\sum_{i=2}^{t}(i-1)[w_{i}-g(\widehat{\theta}_{i-1})]^{\top}(\widehat{\theta}_{i-1}-\theta^*)+\frac{C_g^2r_{\epsilon,d}^2}{\zeta_l^2t}.
\end{flalign}

Define $\mathcal{F}_{t}=\sigma(w_t,s_t,w_{t-1},s_{t-1},\cdots)$, and $\widetilde{g}_i^{[C]}= \mathbb{E}[g_i^{[C]}|\mathcal{F}_{i-1}]$. To handle the bias due to the truncation, in the following analysis, we decompose $w_{i}-g(\widehat{\theta}_{i-1})$ into a bias term, $\iota_i^b=\widetilde{g}_i^{[C]}-g(\widehat{\theta}_{i-1})$, and  a variance term, $\iota_i^v=w_i-\widetilde{g}_i^{[C]}$. Based on \eqref{bound_thetatau}, we thus have that for $t\leq \tau$,
\begin{flalign*}
   \|\widehat{\theta}_{t}-\theta^*\|^2 \leq& \frac{2}{\zeta_l(t-1)t}\sum_{i=2}^{t}(i-1)[\iota_i^b]^{\top}(\widehat{\theta}_{i-1}-\theta^*)+\frac{2}{\zeta_l(t-1)t}\sum_{i=2}^{t}(i-1)[\iota_i^v]^{\top}(\widehat{\theta}_{i-1}-\theta^*)+\frac{C_g^2r_{\epsilon,d}^2}{\zeta_l^2t}\\:=& R_{t1}+R_{t2}+\frac{C_g^2r_{\epsilon,d}^2}{\zeta_l^2t}. 
\end{flalign*}

\noindent{[\textbf{Bound on }$R_{t1}$]:} Note that $g(\widehat{\theta}_{i-1})=\mathbb{E}[g(\widehat{\theta}_{i-1};s_i)|\mathcal{F}_{i-1}]$, since $\{s_t\}_{t=1}^{\tau}$ are i.i.d.  We have  $\iota_i^b= \mathbb{E}[g_i^{[C]}-g_i|\mathcal{F}_{i-1}]$, by letting $g_i = g(\widehat{\theta}_{i-1};s_i)$.  By Jensen's inequality  and the triangle inequality, 
\begin{flalign*}
   \|\iota_i^b\|= \|\mathbb{E}[g_i^{[C]}-g_i|\mathcal{F}_{i-1}]\|\leq \mathbb{E}[\|g_i^{[C]}-g_i\||\mathcal{F}_{i-1}]\leq  \mathbb{E}[(\|g_i\|-C_g)\mathbb{I}(\|g_i\|>C_g)|\mathcal{F}_{i-1}].
\end{flalign*}
Recall that $g_i=[y_i-\psi'(x_i^{\top}\widehat{\theta}_{i-1})]x_i$. Hence, using $\|x_i\|\leq \sqrt{1+u^2}$,   $|y_i|\leq M_{\psi1}^*+|\varepsilon_i|$ and $|\psi'(x_i^{\top}\widehat{\theta}_i)|\leq M_{\psi1}^* $,  we have 
\begin{flalign}\label{bound_iotab}
\begin{split}
 \|\iota_i^b\|\leq &\sqrt{1+u^2}\mathbb{E}\left[(|y_i-\psi'(x_i^{\top}\widehat{\theta}_i)|-2M_{\psi1}^*-2\sigma\log^{1/2} (\tau))\mathbb{I}(\|g_i\|>C_g)|\mathcal{F}_{i-1}\right]\\
=&\sqrt{1+u^2}\mathbb{E}\left[(|\varepsilon_i+\psi'(x_i^{\top}{\theta}^*)-\psi'(x_i^{\top}\widehat{\theta}_i)|-2M_{\psi1}^*-2\sigma\log^{1/2} (\tau))\mathbb{I}(\|g_i\|>C_g)|\mathcal{F}_{i-1}\right]\\
 \leq &\sqrt{1+u^2}\mathbb{E}\left[(|\varepsilon_i|-2\sigma\log^{1/2} (\tau))\mathbb{I}(\|g_i\|>C_g)|\mathcal{F}_{i-1}\right]
\\\leq&\sqrt{1+u^2}\mathbb{E}\left[(|\varepsilon_i|-2\sigma\log^{1/2} (\tau))\mathbb{I}(|\varepsilon_i|>2\sigma\log^{1/2} (\tau))\right]
\\\leq& \sqrt{1+u^2}\left[\sqrt{(\mathbb{E}\varepsilon_i^2)\mathbb{P}(|\varepsilon_i|>2\sigma\log^{1/2} (\tau))}-2\sigma\log^{1/2} (\tau)\mathbb{P}(|\varepsilon_i|>2\sigma\log^{1/2} (\tau))\right]
\\ \leq&4\sigma\sqrt{1+u^2}[\tau^{-1}-\tau^{-2}\log^{1/2} (\tau)]\leq 2\sqrt{1+u^2}\sigma\tau^{-1},
\end{split}
\end{flalign}
where the third inequality holds since $\{\|g_i\|>C_g\}\subset \{|\varepsilon_i|>2\sigma\log^{1/2} (\tau)\}$ and the independence between $\varepsilon_i$ and $\mathcal{F}_{i-1}$, and the fourth holds by the Cauchy--Schwarz inequality, and the penultimate by Assumption \ref{assum_feature}(d)  and the Chernoff bound with $\tau\geq 4.$

By the Cauchy--Schwarz inequality, Assumption \ref{assum_ldp}(a), and \eqref{bound_iotab}, we have 
\begin{equation}\label{bound_tau1}
   R_{t_1}\leq \frac{2}{\zeta_l(t-1)t}\sum_{i=2}^{t}(i-1)\|\iota_i^b\|\|\widehat{\theta}_{i-1}-\theta^*\|\leq \frac{2c_{\theta}\sigma \sqrt{1+u^2}}{\zeta_l\tau}. 
\end{equation}

\noindent{[\textbf{Bound on }$R_{t2}$]:} Note that   $$\mathbb{E}(w_i|\mathcal{F}_{i-1})=\mathbb{E}[\mathbb{E}(\mathrm{W}(g_i^{[C]})|\mathcal{F}_{i-1}, s_t)|\mathcal{F}_{i-1}]=\mathbb{E}({g}_i^{[C]}|\mathcal{F}_{i-1}), $$
where the first equality holds by the tower property of conditional expectation, and the second by the fact that the $L_2$-ball mechanism ensures $\mathbb{E}(\mathrm{W}(g_i^{[C]})|\mathcal{F}_{i-1},s_t)=g_i^{[C]}$, 
due to Lemma \ref{lem_l2ball}.

This further implies that 
\[
   \mathbb{E}\{[w_{i}-\widetilde{g}_i^{[C]}]^{\top}(\widehat{\theta}_{i-1}-\theta^*)|\mathcal{F}_{i-1}\} = [\mathbb{E}(w_{i}|\mathcal{F}_{i-1})-\widetilde{g}_i^{[C]}]^{\top}(\widehat{\theta}_{i-1}-\theta^*)=0. 
\]
Hence, $\{(\iota_i^v)^{\top}(\widehat{\theta}_{i-1}-\theta^*)\}$ forms a sequence of martingale differences.

By the triangle inequality and Lemma \ref{lem_l2ball}, $\|\iota_i^v\|\leq \|w_i\|+\|\widetilde{g}_{i}^{[C]}\|\leq C_g(r_{\epsilon,d}+1)\leq 2C_gr_{\epsilon,d}$, thus
$$
\sum_{i=2}^{t}\mathrm{Var}\left\{(i-1)(\iota_i^v)^{\top}(\widehat{\theta}_{i-1}-\theta^*)|\mathcal{F}_{i-1}\right\} \leq 4C_g^2r_{\epsilon,d}^2\sum_{i=2}^{t}(i-1)^2\|\widehat{\theta}_{i-1}-\theta^*\|^2.
$$
In addition, by the Cauchy--Schwarz inequality, we have for $i\leq t,$
$$
|(i-1)(\iota_i^v)^{\top}(\widehat{\theta}_{i-1}-\theta^*)|\leq 2(i-1)C_gr_{\epsilon,d}c_{\theta}\leq 2(t-1) C_gc_{\theta}r_{\epsilon,d}.
$$

Therefore, by Lemma \ref{lem_freedman} with  $V_t=4C_g^2r_{\epsilon,d}^2\sum_{i=2}^{t}(i-1)^2\|\widehat{\theta}_{i-1}-\theta^*\|^2$, $b=2(t-1) C_gc_{\theta}r_{\epsilon,d}$, we have that with probability larger than $1-\delta\log (\tau)$, for all $2\leq t\leq\tau,$ and $\tau\geq 4,$
\begin{flalign}\label{bound_tau2}
\begin{split}
   & \sum_{i=2}^{t}(i-1)(\iota_i^v)^{\top}(\widehat{\theta}_{i-1}-\theta^*) \\\leq& 4C_gr_{\epsilon,d}\max\left\{
2\sqrt{\sum_{i=2}^{t}(i-1)^2\|\widehat{\theta}_{i-1}-\theta^*\|^2},  c_{\theta}(t-1)\sqrt{\log (1/\delta)}
\right\}\sqrt{\log(1/\delta)}. 
\end{split}
\end{flalign}

\noindent{[\textbf{Finish the proof}]:} Combining \eqref{bound_thetatau}, \eqref{bound_tau1} and \eqref{bound_tau2},  thus with probability at least $1-\delta\log(\tau)$, for all $2\leq t\leq \tau,$ we  have
\begin{flalign*}
     \|\widehat{\theta}_{t}-\theta^*\|^2 \leq &
     \frac{8C_gr_{\epsilon,d}}{\zeta_l(t-1)t}\max\left\{
2\sqrt{\sum_{i=2}^{t}(i-1)^2\|\widehat{\theta}_{i-1}-\theta^*\|^2},  c_{\theta}(t-1)\sqrt{\log (1/\delta)}
\right\}\sqrt{\log (1/\delta)}
\\&+\frac{2c_{\theta}\sigma \sqrt{1+u^2}}{\zeta_l\tau}+\frac{C_g^2r_{\epsilon,d}^2}{\zeta_l^2t}
\\\leq &  \frac{16C_gr_{\epsilon,d}\sqrt{\log (1/\delta)}}{\zeta_l(t-1)t}\sqrt{\sum_{i=2}^{t}(i-1)^2\|\widehat{\theta}_{i-1}-\theta^*\|^2} \\&+\frac{8C_gr_{\epsilon,d}c_{\theta}{\log (1/\delta)}\zeta_l+2c_{\theta} \sigma\sqrt{1+u^2}\zeta_l+C_g^2r_{\epsilon,d}^2}{\zeta_l^2t}\\:=& \frac{\varrho_2}{(t-1)t}\sqrt{\sum_{i=2}^{t}(i-1)^2\|\widehat{\theta}_{i-1}-\theta^*\|^2} +\frac{\varrho_3}{t},
\end{flalign*}
where $\varrho_2= \zeta_l^{-1}16C_gr_{\epsilon,d}\sqrt{\log (1/\delta)}$, and $\varrho_3=\zeta_l^{-2}8C_gr_{\epsilon,d}c_{\theta}{\log (1/\delta)}\zeta_l+2c_{\theta} \sigma\sqrt{1+u^2}\zeta_l+C_g^2r_{\epsilon,d}^2$.

We now show that for some $\varrho_1$,
$$\|\widehat{\theta}_{t}-\theta^*\|^2\leq \varrho_1/(t+1).$$
By induction, it suffices to find $a$  such that
$$
\frac{\varrho_1}{t+1}\geq \frac{\varrho_2}{(t-1)t}\sqrt{\sum_{i=2}^{t}(i-1)^2\frac{\varrho_1}{i}} +\frac{\varrho_3}{t}.
$$
By elementary but tedious algebra, it is sufficient to let $\varrho_1\geq 9\varrho_2^2/4+3\varrho_3$, which gives 
$$
\varrho_1=\frac{576C_g^2r_{\epsilon,d}^2\log (1/\delta)+3\left\{8C_gr_{\epsilon,d}c_{\theta}{\log (1/\delta)}\zeta_l+2c_{\theta} \sigma\sqrt{1+u^2}\zeta_l+C_g^2r_{\epsilon,d}^2\right\}}{\zeta_l^2}.
$$
Note that  $r_{\epsilon,d}\asymp \sqrt{d}\epsilon^{-1}$, we thus have 
$$
\varrho_1\lesssim \frac{d C_g^2 \log(1/\delta)}{\zeta_l^2\epsilon^{2}}\vee \frac{\sqrt{d} C_g \log(1/\delta)}{\zeta_l\epsilon}\lesssim \frac{d C_g^2 \log(1/\delta)}{\zeta_l^2\epsilon^{2}},
$$
where we hide the constant dependent on $c_{\theta}$, $u$ and $\sigma$. 
\end{proof}

{\color{black} The following lemma gives a uniform concentration inequality for martingale difference sequences that are bounded or satisfy certain moment conditions (e.g.\ sub-Gaussian), which is useful for bounding the variance component of the SGD estimator in \Cref{lem:sgd} ($\epsilon$-LDP), \Cref{lem:sgdmix} (mixed privacy constraint) and \Cref{lem:approx_ldp_est_theta} (approximate LDP). It is a generalization of Lemma 3 in \citealt{rakhlin2011making}, which only covers the case of bounded random variables.


\begin{lem}\label{lem_freedman}
      Suppose $X_1,\cdots,X_\tau$ is a martingale difference sequence with natural filtration $\mathcal{F}_t=\sigma(X_1,\cdots,X_t)$, $t \in [\tau]$. Let $\sigma_t^2=\mathbb{E}(X_t^2|\mathcal{F}_{t-1})$,  $V_s=\sum_{t=1}^{s}\sigma_t^2$, for $1\leq s\leq\tau.$ Suppose that for some $b\in(0,\infty)$, we have either $\mathbb{E}(|X_t|^k|\mathcal{F}_{t-1})\leq (k!/2)  b^{k}$ for all $k\geq 2$, or $\mathbb{P}(|X_t|\leq b|\mathcal{F}_{t-1})=1$. Then for any $0<\delta<e^{-1}$, and $\tau\geq 4$, we have $$\mathbb{P}\left(\sum_{t=1}^{s}X_t>2\max\left\{2\sqrt{V_s},b\sqrt{\log(1/\delta)}\right\}\sqrt{\log(1/\delta)},\quad \text{for some $s\leq \tau$}\right)\leq 2\log(\tau)\delta.$$  
\end{lem}
\begin{proof}
Let $0=a_{-1}<a_0<\ldots<a_l$ for some $l\geq 0$ such that $a_{i+1}=2a_i$  and $a_l\geq b\sqrt{\tau}$ and $a_0=b\sqrt{\log(1/\delta)}$. Note that $V_s\leq sb^2\leq \tau b^2$ for all $1\leq s\leq\tau.$  We have 
\begin{flalign*}
&\mathbb{P}\left(\sum_{t=1}^{s}X_t>2\max\left\{2\sqrt{V_s},b\sqrt{\log(1/\delta)}\right\}\sqrt{\log(1/\delta)},\quad \text{for some $s\leq \tau$}\right)\\
=&\sum_{j=0}^l \mathbb{P}\left(\sum_{t=1}^{s}X_t>2\max\left\{2\sqrt{V_s},a_0\right\}\sqrt{\log(1/\delta)},\quad a_{j-1}<\sqrt{V_s}\leq a_j, \quad \text{for some $s\leq \tau$}\right)\\
\leq &\sum_{j=0}^l \mathbb{P}\left(\sum_{t=1}^{s}X_t>2a_j\sqrt{\log(1/\delta)},\quad a_{j-1}^2<V_s\leq a_j^2, \quad \text{for some $s\leq \tau$}\right)\\
\leq &\sum_{j=0}^l \mathbb{P}\left(\sum_{t=1}^{s}X_t>2a_j\sqrt{\log(1/\delta)},\quad V_s\leq a_j^2, \quad \text{for some $s\leq \tau$}\right)\\
\leq &\sum_{j=0}^l \exp\left\{ -\frac{4a_j^2\log(1/\delta)}{2(a_j^2+b\times 2a_j\sqrt{\log(1/\delta)})} 
\right\}=\sum_{j=0}^l \exp\left\{ -\frac{2a_j\log(1/\delta)}{a_j+2b\sqrt{\log(1/\delta)}} 
\right\},
\end{flalign*}
where the last inequality holds by the Freedman's inequality, see Theorem 1.2A in \cite{de1999general}. 

Note that $a_j\geq a_0=b\sqrt{\log(1/\delta)}$ for all $j\geq 0$,  we thus obtain that 
\begin{flalign*}
&\mathbb{P}\left(\sum_{t=1}^{s}X_t>2\max\left\{2\sqrt{V_s},b\sqrt{\log(1/\delta)}\right\}\sqrt{\log(1/\delta)},\quad \text{for some $s\leq \tau$}\right)\\
\leq &\sum_{j=0}^l \exp\left\{ -\frac{2}{3}\log(1/\delta)
\right\}\leq (l+1)\delta.
\end{flalign*}
The result follows by choosing $l=\log\tau $ such that $a_l=2^{l} a_0\geq b\sqrt{\tau}$, and $l+1\leq 2\log\tau. $
\end{proof}

\begin{lem}\label{lem_subGaussian}
    For a sub-Gaussian random variable $X$ with mean zero and variance proxy $\sigma^2$, then for any $k\geq 2$,\begin{flalign*}
        \mathbb{E}|X|^k  \leq (2\sigma^2)^{k/2}k\Gamma(k/2)\leq (2\sigma)^{k} (k!/2).
        \end{flalign*}
\end{lem}
\begin{proof}
    Note by the property of sub-Gaussianity, see Lemma 1.4 in \cite{rigollet2023high},  we have for $k\geq 2$, \begin{flalign*}
        \mathbb{E}|X|^k \leq& (2\sigma^2)^{k/2}k\Gamma(k/2)\leq  (\sqrt{2}\sigma)^{k} k! \leq (2\sigma)^{k} (k!/2),
        \end{flalign*}
        where the second inequality holds by noting $k\Gamma(k/2)\leq k!$ for $k\geq 2.$ 
\end{proof}
}

\subsection{Proof of \Cref{thm:LDP-ETC}}
This subsection collects the proof of \Cref{thm:LDP-ETC}. 

\begin{proof}[\textbf{Proof of \Cref{thm:LDP-ETC}}]

First, the sample size condition on $T$ in \eqref{eq:minTSGD} in \Cref{thm:LDP-ETC} can be explicitly written as 
\begin{align}\label{eq:minTSGD_add}
       T\geq 2\lambda_z^{-4} [ B_{L1}^2\varsigma_o^{-4}\epsilon^{-2}\log (T)\log(1/\delta)],
\end{align}
where $B_{L1}>0$ is the absolute constant in \Cref{lem:sgd}.

We now give the proof. We partition the horizon $[T]$ into two parts, denoted by  $\Psi$  with $|\Psi|=\tau$, and $\Phi$ with $|\Phi|=T-\tau$, corresponding to the sets of exploration and exploitation, respectively.  

Let $U=\sum_{t\in\Phi}z_tz_t^{\top}$, then by Lemma \ref{lem:eigen}(i), we have with probability at least $1-\delta$ such that event $\mathcal{C}$ holds, where 
$$
\mathcal{C} = \left\{ \left\|U/(T-\tau)-\Sigma_z\right\|_{\mathrm{op}}\leq \max(\varsigma_2,\varsigma_2^2)  \right\}, \text{ with } \varsigma_2=3(c_1\sqrt{\log(d)}+c_2\sqrt{\log (1/\delta)})/\sqrt{T-\tau}.
$$
  
By similar arguments used in the proof of \Cref{thm:ETC}, we have that 
    \begin{flalign*}
        &(\widehat\alpha_{\tau}-\alpha^*)^{\top}U (\widehat\alpha_{\tau}-\alpha^*) + (\widehat\beta_{\tau}-\beta^*)^{\top}U (\widehat\beta_{\tau}-\beta^*) \\\leq & (T-\tau)\bar{\lambda}_z\|\widehat\theta_{\tau}-\theta^*\|^2 + (T-\tau)\|U /(T-\tau)-\Sigma_z\|_{\mathrm{op}}\|\widehat\theta_{\tau}-\theta^*\|^2\\\leq& 
        \left( (T-\tau)\bar{\lambda}_z + \sqrt{(T-\tau)}(c_1\sqrt{\log(d)}+c_2\sqrt{\log (1/\delta)})   \right) \|\widehat\theta_{\tau}-\theta^*\|^2.
         \\\leq & \frac{T}{\tau}\frac{B_{L1}d\log(\tau)\log(1/\delta)}{\epsilon^2 \lambda_z}\left[\frac{\bar{\lambda}_z}{{\lambda}_z}+\frac{c_1\sqrt{\log(d)}+c_2\sqrt{\log (1/\delta)}}{{\lambda}_z\sqrt{T}}\right]
        \\\leq &  2\frac{T}{\tau}\frac{B_{L1} d\log (T)\log (1/\delta) \bar{\lambda}_z}{\epsilon^2 \lambda_z^2},
    \end{flalign*}
    where the third inequality holds by noting \Cref{lem:sgd} with $c_l=1$, and the last inequality holds by the fact that $(c_1\sqrt{\log(d)}+c_2\sqrt{\log (1/\delta)})/({\lambda}_z\sqrt{T}) \leq 1\leq \bar{\lambda}_z/{\lambda}_z$ under \eqref{eq:minTSGD_add}.
 
Note that under \eqref{eq:minTSGD_add} and  the setting of $\tau$, we have that  $\tau \geq  B_{L1}d\log (T)\log (1/\delta)/(\lambda_z\epsilon\varsigma_o)^2 $. Hence by  \Cref{lem:sgd},   $\|\widehat\theta_{\tau}-\theta^*\|\leq \varsigma_o$. 
 Then, by   \Cref{lem:Lips}, and equation \eqref{bound_explore} in the proof of \Cref{thm:ETC}, we obtain that 
    \begin{align}\label{RT_sgd}
    \begin{split}
    R_T=& \sum_{t=1}^\tau \{r(p_t^*, z_t, \theta^*)- r(p_t, z_t, \theta^*)\} + \sum_{t=\tau+1}^T \{r(p_t^*, z_t, \theta^*)- r(\widehat p_t, z_t, \theta^*)\} \\
\leq& \tau uM_{\psi 1} +  2C_rC_\varphi^2B_{L1} \frac{T}{\tau}\frac{ d\log (T)\log (1/\delta) \bar{\lambda}_z}{\epsilon^2 \lambda_z^2}. 
    \end{split}
\end{align}
The high probability bound thus follows by letting $\tau= d\sqrt{\log (T) \log (1/\delta)T}/\epsilon$ and noting $c_{\lambda1}/d\leq \lambda_z\leq \bar{\lambda}_z\leq c_{\lambda2}/d$ under \Cref{assum_ldp}.

If $\delta=T^{-1}$, $R_T$ is upper bounded by \eqref{RT_sgd} with probability at least $1-\log (T)/T$. Otherwise, $R_T$ is always upper bounded by $uM_{\psi 1}T$. Hence the expectation bound follows.
\end{proof}

{
\color{black}
\subsection{Proof of \Cref{thm:LDP-ETC-Mixed}}

\Cref{lem:sgdmix} gives a high-probability upper bound of the estimation error for the two-stage SGD estimator in the exploration phase of ETC-LDP-Mixed, serving as the basis for bounding the regret of ETC-LDP-Mixed. For convenience of presentation, we first recap the two-stage SGD estimator in ETC-LDP-Mixed with slightly different notations.

Suppose we have $\tau_1$ private sample points under the $\epsilon$-LDP constraint and $\tau_2$ non-private sample points. The two-stage SGD combines a private SGD and a non-private SGD to estimate the model parameter $\theta^*$. In particular, it first runs a private SGD on the $\tau_1$ private data as in \Cref{lem:sgd}, giving the output $\widehat{\theta}_{\tau_1}$. It then runs a non-private SGD with $\widehat{\theta}_{\tau_1}$ as the initial point and step size
\begin{align}\label{eq:step_size_nonprivateSGD_supp}
  \eta_{\tau_1+t}^*=\zeta_l^{-1}(\tau_1\epsilon^2/d+t)^{-1}, \text{ for } t=1,2,\cdots,\tau_2,
\end{align}
on the $\tau_2$ non-private data, which gives the final estimator $\widehat{\theta}_{\tau_1+\tau_2}$.

We remark that for the non-private SGD, we use the true gradient instead of truncating the gradient to a bounded vector. In other words, the gradient can be potentially unbounded. Therefore, the technical argument for \Cref{lem:sgdmix} needs to be substantially modified compared to that in \Cref{lem:sgd}. In particular, we rely on the newly developed \Cref{lem_freedman} and \Cref{lem_subGaussian} to control the unbounded gradient. Technically, we extend the uniform concentration inequality for martingale difference sequences in \colorlinks{black}{\cite{rakhlin2011making}} from bounded random variables to random variables with certain moment conditions (e.g.\ sub-Gaussian), which is of independent interest.

\begin{lem}\label{lem:sgdmix}
    Suppose \Cref{assum_feature} and \ref{assum_ldp}(a) hold. For any $\tau_1,\tau_2\geq 4$, set $C_g=2\sqrt{1+u^2}[M_{\psi 1}^*+\sigma\sqrt{\log(\tau_1+\tau_2)}]$ and $\zeta_l=c_l\cdot \phi_z$ with some $c_l\in(0,1]$, where recall $\phi_z=L_p\kappa^*\lambda_z$. We have that, for any $\epsilon>0$ and $0<\delta<(e\vee \log (\tau_1 \tau_2))^{-1}$, with probability at least $1-\delta\log(\tau_1\tau_2)$, it holds that   
    $$
\|\widehat{\theta}_{\tau_1+\tau_2}-\theta^*\|^2 \leq B_{M1}\frac{\log(\tau_1+\tau_2)\log(1/\delta)}{c_l^2\lambda_z^2(\tau_1\epsilon^2/d+\tau_2)},$$
where $B_{M1}>0$ is an absolute constant that only depends on quantities in  Assumptions \ref{assum_feature} and \ref{assum_ldp}(a).
\end{lem}


\begin{proof}[\textbf{Proof of \Cref{lem:sgdmix}}]

Note that the analysis of the first-stage private SGD on the $\tau_1$ private data follows exactly from \Cref{lem:sgd}. Therefore, we can directly analyze the performance of the non-private SGD on the $\tau_2$ non-private data.

First, using the same argument as that in \eqref{bound_theta}, we have 
\begin{equation*}
\|\widehat{\theta}_{\tau_1+t}-\theta^*\|^2 \leq  (1-2\eta_{\tau_1+t}^*\zeta_l)\|\widehat{\theta}_{\tau_1+t-1}-\theta^*\|^2 +2\eta_{\tau_1+t}^*[g_{\tau_1+t}-g(\widehat{\theta}_{\tau_1+t-1})]^{\top}(\widehat{\theta}_{\tau_1+t-1}-\theta^*) +\eta_{\tau_1+t}^{*2}\|g_{\tau_1+t}\|^2,
\end{equation*}
where recall the step size $\eta_{\tau_1+t}^*$ is defined in \eqref{eq:step_size_nonprivateSGD_supp}, $g_{\tau_1+t}=g_{\tau_1+t}(\widehat{\theta}_{\tau_1+t-1};s_t)$ is the true gradient at round $\tau_1+t$, and $g(\widehat{\theta}_{\tau_1+t-1})=\mathbb{E}[g_{\tau_1+t}(\widehat{\theta}_{\tau_1+t-1};s_t)|\mathcal{F}_{t-1}].$

By iteration, we have that 
\begin{flalign*}
\|\widehat{\theta}_{\tau_1+t}-\theta^*\|^2 
   \leq& \prod_{i=1}^t (1-2\eta_{\tau_1+i}^*\zeta_l)  \|\widehat{\theta}_{\tau_1}-\theta^*\|^2 +
    \sum_{i=1}^t \eta_{\tau_1+i}^{*2}\|g_{\tau_1+i}\|^2\prod_{j=i+1}^t  (1-2\eta_{\tau_1+j}^*\zeta_l)\\&+2   \sum_{i=1}^t \eta_{\tau_1+i}^{*}[g_{\tau_1+i}-g(\widehat{\theta}_{\tau_1+i-1})]^{\top}(\widehat{\theta}_{\tau_1+i-1}-\theta^*) \prod_{j=i+1}^t  (1-2\eta_{\tau_1+j}^*\zeta_l).
\end{flalign*}
Using the fact that
$$
\prod_{j=i+1}^t (1-2\eta_{\tau_1+j}^*\zeta_l)=\prod_{j=i+1}^t (1-\frac{2}{\tau_1\epsilon^2/d+j})= \frac{(\tau_1\epsilon^2/d+i-1)(\tau_1\epsilon^2/d+i)}{(\tau_1\epsilon^2/d+t-1)(\tau_1\epsilon^2/d+t)},
$$
and for $\tau_1\epsilon^2/d\geq 1,$
\begin{flalign*}
    \sum_{i=1}^t \eta_{\tau_1+i}^{*2}\|g_{\tau_1+i}\|^2\prod_{j=i+1}^t  (1-2\eta_{\tau_1+j}^*\zeta_l)=&\zeta_l^{-2}\sum_{i=1}^t \frac{(\tau_1\epsilon^2/d+i-1)\|g_{\tau_1+i}\|^2}{(\tau_1\epsilon^2/d+t-1)(\tau_1\epsilon^2/d+t)(\tau_1\epsilon^2/d+i)}\\\leq& \frac{\sum_{i=1}^t \|g_{\tau_1+i}\|^2}{\zeta_l^2(\tau_1\epsilon^2/d+t)(\tau_1\epsilon^2/d+t-1)}\leq\frac{\sum_{i=1}^t \|g_{\tau_1+i}\|^2}{\zeta_l^2(\tau_1\epsilon^2/d+t)t},
\end{flalign*}
one can show that
\begin{flalign}\label{bound_taunew}
\begin{split}
    \|\widehat{\theta}_{\tau_1+t}-\theta^*\|^2 \leq& \frac{\tau_1\epsilon^2/d}{(\tau_1\epsilon^2/d+t)}\|\widehat{\theta}_{\tau_1}-\theta^*\|^2+\frac{\sum_{i=1}^t \|g_{\tau_1+i}\|^2}{\zeta_l^2t(\tau_1\epsilon^2/d+t)}\\&+2\frac{\sum_{i=1}^t (\tau_1\epsilon^2/d+i-1)[g_{\tau_1+i}-g(\widehat{\theta}_{\tau_1+i-1})]^{\top}(\widehat{\theta}_{\tau_1+i-1}-\theta^*) }{\zeta_l(\tau_1\epsilon^2/d+t-1)(\tau_1\epsilon^2/d+t)}
\\:=&\frac{\tau_1\epsilon^2/d}{(\tau_1\epsilon^2/d+t)}\|\widehat{\theta}_{\tau_1}-\theta^*\|^2+R_{t1}^*+R_{t2}^*.
\end{split}
\end{flalign}
For $R_{t1}^*$, recall that $$g_{t}=[y_t-\psi'(x_t^{\top}\widehat{\theta}_{t-1})]x_t=[\varepsilon_t+\psi'(x_t^{\top}\theta^*)-\psi'(x_t^{\top}\widehat{\theta}_{t-1})]x_t.$$
Since for any $\theta\in\Theta$, $|\psi'(x_t^{\top}\theta)|\leq M_{\psi1}^*$, and that $\|x_t\|\leq \sqrt{1+u^2}$, we have 
\begin{equation}\label{ineq_R1}
    R_{t1}^*\leq \frac{\sum_{i=1}^t (1+u^2)(|\varepsilon_{\tau_1+i}|+2M_{\psi 1}^*)^2}{\zeta_l^2t(\tau_1\epsilon^2/d+t)}\leq \frac{\sum_{i=1}^t 2(1+u^2)(\varepsilon_{\tau_1+i}^2+M_{\psi 1}^{*2})}{\zeta_l^2t(\tau_1\epsilon^2/d+t)}.
\end{equation}
Since $\varepsilon_t$ is i.i.d. mean-zero sub-Gaussian with variance proxy $\sigma^2$, by \Cref{lem_subGaussian}, we have for any $k\geq 2$, 
$$
\mathbb{E}|\varepsilon_i^2-\sigma^2|^k\leq \mathbb{E}[\varepsilon_i^2]^k=\mathbb{E}|\varepsilon_i|^{2k}\leq 2(2\sigma^2)^{k} k\Gamma(k)\leq (4\sigma^2)^k (k!/2).
$$
Therefore,  letting $X_i=\varepsilon_{t+i}^2-\sigma^2$, $V_t=t16\sigma^4$, and $b=4\sigma^2$ in \Cref{lem_freedman}, we have  with probability larger than $1-2\delta\log\tau_2$, for all $1\leq t\leq \tau_2$,
$$
\sum_{i=1}^t(\varepsilon_{\tau_1+i}^2-\sigma^2)\leq  8\sigma^2\max\{2\sqrt{t}, \sqrt{\log(1/\delta)}\}\sqrt{\log(1/\delta)},
$$
which further implies  that 
$
t^{-1}{\sum_{i=1}^t\epsilon_{\tau_1+i}^2}\leq 17\sigma^2\sqrt{\log(1/\delta)}
$ since $\sqrt{t}/t\leq 1.$
Hence,  by \eqref{ineq_R1}, with probability larger than $1-2\delta\log\tau_2$, we can show that 
\begin{equation}\label{bound_Rt1}
     R_{t1}^*\leq \frac{ 2(1+u^2)(17\sigma^2\sqrt{\log(1/\delta)}+M_{\psi 1}^{*2})}{\zeta_l^2(\tau_1\epsilon^2/d+t)}.
\end{equation}

As for $R_{t2}^*$, since for any fixed vector $a\in\mathbb{R}^{2d}$, we have $|[g_{\tau_1+i}-g(\widehat{\theta}_{\tau_1+i-1})]^{\top} a|\leq \|a\|\sqrt{1+u^2}(|\varepsilon_{\tau_1+i}|+4M_{\psi1}^{*})$. Since  $\varepsilon_{\tau_1+i}$ is sub-Gaussian with variance proxy $\sigma^2$,  we thus have $|[g_{\tau_1+i}-g(\widehat{\theta}_{\tau_1+i-1})]^{\top} a|$  is sub-Gaussian with variance proxy $\|a\|^2\sigma_g^2,$ with $\sigma_g^2\leq 2(1+u^2)(\sigma^2+16M_{\psi1}^{*2}).$  Therefore, similar to \eqref{bound_tau2}, by \Cref{lem_freedman} and \Cref{lem_subGaussian}, we can show with probability larger than $1-2\delta\log(\tau_2)$, for all $1\leq t\leq \tau_2$,
\begin{flalign}\label{bound_Rt2}
\begin{split}
&R_{t2}^*\leq\frac{8\sigma_g\sqrt{\log(1/\delta)}}{\zeta_l(\tau_1\epsilon^2/d+t-1)(\tau_1\epsilon^2/d+t)}\\&\times \max\left\{
\sqrt{\sum_{i=1}^{t}(\tau_1\epsilon^2/d+i-1)^2\|\widehat{\theta}_{\tau_1+i-1}-\theta^*\|^2}, 2c_{\theta}(\tau_1\epsilon^2/d+t-1)\sqrt{\log (1/\delta)}
\right\}. 
\end{split}
\end{flalign}
Note that by \Cref{lem:sgd}, we have with probability larger than $1-2\delta\log(\tau_1)$, $$\|\widehat{\theta}_{\tau_1}-\theta^*\|^2\lesssim\frac{dC_g^2\log(1/\delta)}{\zeta_l^2\epsilon^{2}\tau_1}.$$
Therefore, combining with \eqref{bound_taunew}, \eqref{bound_Rt1} and \eqref{bound_Rt2}, we have with probability larger than $1- 2\delta\log(\tau_1\tau_2)$, 
\begin{flalign*}
    \|\widehat{\theta}_{\tau_1+t}-\theta^*\|^2 \lesssim & \frac{C_g^2\log(1/\delta)+\sqrt{\log(1/\delta)}}{\zeta_l^2(\tau_1\epsilon^2/d+t)}\\&+ \frac{\sqrt{\log(1/\delta)}}{\zeta_l(\tau_1\epsilon^2/d+t-1)(\tau_1\epsilon^2/d+t)}\sqrt{\sum_{i=1}^{t}(\tau_1\epsilon^2/d+i-1)^2\|\widehat{\theta}_{\tau_1+i-1}-\theta^*\|^2}. 
\end{flalign*}
By similar  induction argument  in the proof of \Cref{lem:sgd}, we obtain that 
\begin{align}\label{eq:mix_sgd_withCg}
    \|\widehat{\theta}_{\tau_1+t}-\theta^*\|^2 \lesssim  \frac{C_g^2\log(1/\delta)}{\zeta_l^2(\tau_1\epsilon^2/d+t)},
\end{align}
which completes the proof.
\end{proof}

\begin{proof}[\textbf{Proof of \Cref{thm:LDP-ETC-Mixed}}]

First, the sample size condition on $T$ in \Cref{thm:LDP-ETC-Mixed} can be explicitly written as 
\begin{align}\label{eq:minTSGD_mixed}
       T\geq 2\lambda_z^{-4} [ B_{M1}^2\varsigma_o^{-4}\epsilon^{-2}\log (T)\log(1/\delta)],
\end{align}
where $B_{M1}>0$ is the absolute constant in \Cref{lem:sgdmix}.

In the following, without loss of generality, we assume $\epsilon^2/d<1/2$. Note that for $\epsilon^2/d\geq 1/2$, the regret under pure $\epsilon$-LDP is of the same order as that of the non-private data. In such case, one can treat all data as if they were private since up to a constant factor, the regret order remains unchanged. 


In the first-stage exploration phase (Stage I), we obtain a sequence of i.i.d. Bernoulli random variables, namely, $\{m_t\}_{t=1}^{\lfloor\sqrt{Td}\rfloor}$ while simultaneously executing the private SGD algorithm while storing any non-private data. Denote the estimator of the proportion of non-private data as 
$\widehat{p}^*=\lfloor\sqrt{Td}\rfloor^{-1}\sum_{t=1}^{\lfloor\sqrt{Td}\rfloor}m_t.$ The total exploration length is then set at
$$
\tau =  \tau(\widehat{p}^*)=\frac{\log T\sqrt{Td}}{\sqrt{\widehat{p}^*(1-\epsilon^2/d)+\epsilon^2/d}}
$$
as in \eqref{eq:tau_p_mixed}. It is clear that $\log T\sqrt{Td}\leq  \tau(\widehat{p}^*) \leq \log T\sqrt{T}d/\epsilon.$


We proceed by considering the effective sample size given by $\tau(\widehat{p}^*)$ under two different cases. Our result relies on the classical Chernoff bound where for any fixed $n$, and any $\delta>0$,
\begin{equation}\label{ineq_chernoff}
    \mathbb{P}\left(\left|n^{-1}\sum_{t=1}^{n}m_t-p^*\right|\geq \delta p^* \right)\leq 2\exp(-\delta^2np^*/3).
\end{equation}

\noindent\textbf{Case 1:  $p^*\leq \epsilon^2/d$.}  In this case, the effective sample size from the private data in the exploration stage will dominate. We  set $n=\lfloor \sqrt{Td}\rfloor$, and  $\delta=\epsilon^2/(2dp^*)\geq 1/2$ in \eqref{ineq_chernoff}, then 
\begin{equation}\label{bound_p1}
    \mathbb{P}\left(\left|n^{-1}\sum_{t=1}^{n}m_t-p^*\right|\geq  \epsilon^2/(2d)\right)\leq 2\exp(-\lfloor \sqrt{Td}\rfloor\delta\epsilon^2/(6d))\leq 2\exp(-\lfloor \sqrt{Td}\rfloor\epsilon^2/(12d))\leq 2/T,
\end{equation}
where the second inequality holds by noting $\delta\geq 1/2$ and the last inequality holds by noting $T\geq 144d\epsilon^{-4}\log^2T$.

Therefore, by \eqref{bound_p1}, with probability larger than $1-2/T$, we have $0\leq\widehat{p}^*\leq 3\epsilon^2/(2d)$, and hence 
\begin{equation*}
    \sqrt{2/5}\log T\sqrt{T}d/\epsilon \leq \tau(3\epsilon^2/(2d))\leq \tau(\widehat{p}^*)\leq \tau(0)=\log T\sqrt{T}d/\epsilon.
\end{equation*}
This implies that with high probability, we have
\begin{equation}\label{bound_explore1}
    \tau(\widehat{p}^*)\asymp\tau(p^*)\asymp\tau(0).
\end{equation}


Denote the number of private data at the exploration phase as $\tau_1(\widehat{p}^*)$, and hence the number of non-private data is $\tau(\widehat{p}^*)-\tau_1(\widehat{p}^*).$ By the Chernoff bound in \eqref{ineq_chernoff} with $n=\tau(\widehat{p}^*)>\sqrt{Td}$, and $\delta=\epsilon^2/(2dp^*)$, it is easy to show that with probability larger than $1-2/T$, we have 
$
\tau(\widehat{p}^*)-\tau_1(\widehat{p}^*)\leq  3\tau(\widehat{p}^*)\epsilon^2/(2d).
$
Since $\epsilon^2/d<1/2$, we have 
$\tau(\widehat{p}^*)-\tau_1(\widehat{p}^*)\leq 3\tau(\widehat{p}^*)/4,$ and $\tau_1(\widehat{p}^*)\geq \tau(\widehat{p}^*)/4$.
This further implies that 
$$
\tau(\widehat{p}^*)\epsilon^2/(4d)\leq \tau_1(\widehat{p}^*)\epsilon^2/d\leq \underbrace{\tau(\widehat{p}^*)-\tau_1(\widehat{p}^*)+\tau_1(\widehat{p}^*)\epsilon^2/d}_{\text{effective sample size}}\leq 5\tau(\widehat{p}^*)\epsilon^2/(2d).
$$
In other words, we have that the effective sample size takes the order
$$
\tau(\widehat{p}^*)-\tau_1(\widehat{p}^*)+\tau_1(\widehat{p}^*)\epsilon^2/d \asymp \tau(\widehat{p}^*)\epsilon^2/d\asymp\tau(0)\epsilon^2/d\asymp [p^*+(1-p^*)\epsilon^2/d]\tau(p^*).
$$

\noindent\textbf{Case 2: $p^*>\epsilon^2/d$.}
Setting $n=\lfloor \sqrt{Td}\rfloor$ and  $\delta=1/2$ in \eqref{ineq_chernoff}, we have that 
$$
\mathbb{P}\left(\left|n^{-1}\sum_{t=1}^{n}m_t-p^*\right|\geq  p^*/2\right)\leq 2\exp(-\lfloor \sqrt{Td}\rfloor p^*/12)\leq 2\exp(-\lfloor \sqrt{Td}\rfloor\epsilon^2/(12d))\leq 2/T,
$$
where the second inequality holds by noting $p^*>\epsilon^2/d$ and $T\geq 144d\epsilon^{-4}\log^2T$.

Therefore, with probability larger than $1-2/T$, we have 
$p^*/2\leq \widehat{p}^*\leq 3p^*/2,$  and this implies that  
$$
\sqrt{2/3}\tau(p^*)\leq\tau(3p^*/2)\leq\tau(\widehat{p}^*)\leq\tau(p^*/2)\leq\sqrt{2}\tau(p^*).
$$
In other words, we have with high probability
\begin{flalign}\label{bound_explore2}
\tau(\widehat{p}^*)\asymp\tau(p^*).
\end{flalign}
Recall we denote the number of private data at the exploration phase as $\tau_1(\widehat{p}^*)$ and the number of non-private data as $\tau(\widehat{p}^*)-\tau_1(\widehat{p}^*).$ By the Chernoff bound in \eqref{ineq_chernoff} with $n=\tau(\widehat{p}^*)>\sqrt{Td}$, and $\delta=1/2$, we know that with probability larger than $1-2/T$,  
\begin{align}\label{eq:chernoff_largep}
    \tau(\widehat{p}^*)p^*/2\leq \tau(\widehat{p}^*)-\tau_1(\widehat{p}^*)\leq 3\tau(\widehat{p}^*)p^*/2.
\end{align}
For $\epsilon^2/d<p^*\leq 1/2$, this implies that $$(1-p^*)\tau(\widehat{p})/2\leq (1-3p^*/2)\tau(\widehat{p})\leq \tau_1(\widehat{p})\leq (1-p^*/2)\tau(\widehat{p})\leq 3(1-p^*)\tau(\widehat{p})/2.$$
Therefore, we have that 
$$
\tau(\widehat{p}^*)[p^*+(1-p^*)\epsilon^2/d]/2\leq \underbrace{\tau(\widehat{p}^*)-\tau_1(\widehat{p}^*)+\tau_1(\widehat{p}^*)\epsilon^2/d}_{\text{effective sample size}}\leq 3\tau(\widehat{p}^*)[p^*+(1-p^*)\epsilon^2/d]/2.
$$
Similarly, for $p^*>1/2$, it is easy to see that \eqref{eq:chernoff_largep} implies that 
$$\tau(\widehat{p}^*)[p^*+(1-p^*)\epsilon^2/d]/4\leq\tau(\widehat{p}^*)p^*/2\leq \underbrace{\tau(\widehat{p}^*)-\tau_1(\widehat{p}^*)+\tau_1(\widehat{p}^*)\epsilon^2/d}_{\text{effective sample size}}\leq \tau(\widehat{p}^*) \leq  2\tau(\widehat{p}^*)[p^*+(1-p^*)\epsilon^2/d].$$

Combined with \eqref{bound_explore2}, we have that for $p^*>\epsilon^2/d,$ the effective sample size takes the order
$$
\tau(\widehat{p}^*)-\tau_1(\widehat{p}^*)+\tau_1(\widehat{p}^*)\epsilon^2/d\asymp [p^*+(1-p^*)\epsilon^2/d]\tau(\widehat{p}^*)\asymp [p^*+(1-p^*)\epsilon^2/d]\tau(p^*).
$$

Combining the two cases, for any $p^*\in [0,1]$, we have that with probability at least $1-2/T$, it holds that the effective sample size is
\begin{align*}
    \tau(\widehat{p}^*)-\tau_1(\widehat{p}^*)+\tau_1(\widehat{p}^*)\epsilon^2/d\asymp [p^*+(1-p^*)\epsilon^2/d]\tau(p^*).
\end{align*}

Therefore, by \Cref{lem:sgdmix}, in particular \eqref{eq:mix_sgd_withCg}, setting $C_g=2\sqrt{1+u^2}[M_{\psi 1}^*+\sigma \sqrt{\log \tau(0)}]$ and $\delta=1/T$, with $\tau_1=\tau_1(\widehat{p}^*)$ and $\tau_2=\tau(\widehat{p}^*)-\tau_1(\widehat{p}^*)$, we have with probability larger than $1-2\log T/T,$
$$
\|\widehat{\theta}_{\tau(\widehat{p}^*)}-\theta^*\|^2\lesssim  \frac{C_g^2\log(T)}{\zeta_l^2[\tau(\widehat{p}^*)-\tau_1(\widehat{p}^*)+\tau_1(\widehat{p}^*)\epsilon^2/d]} \asymp\frac{C_g^2\log(T)}{\zeta_l^2[p^*+(1-p^*)\epsilon^2/d]\tau(p^*)}.
$$
In terms of regret analysis, recall that $C_g^2\asymp \log T$, by similar arguments used in the proof of \Cref{thm:LDP-ETC},  we have that with probability larger than $1-(4+2\log T)/T$, 
\begin{flalign*}
R_T=&\sum_{t=1}^\tau r\left(p_t^*, z_t, \theta^*\right)-r\left(p_t, z_t, \theta^*\right)+\sum_{t=\tau+1}^T r\left(p_t^*, z_t, \theta^*\right)-r\left(\widehat{p}_t, z_t, \theta^*\right)
\\\leq& \sqrt{2}\tau(p^*)u M_{\psi 1}+2C_rC_{\varphi}^2 B_{Lm}\frac{T\log^2T\bar{\lambda}_z}{[p^*+(1-p^*)\epsilon^2/d]\tau(p^*)\lambda_z^2},
\end{flalign*}
provided that $T$ satisfies \eqref{eq:minTSGD_mixed}.

Taking the value of $\tau(p^*)=\log T\sqrt{Td}/\sqrt{p^*+(1-p^*)\epsilon^2/d}$, and noting 
$\lambda_z\asymp\bar{\lambda}_z\asymp d^{-1}$, we thus obtain that with high probability
$$
R_T\lesssim \frac{\log T\sqrt{Td}}{\sqrt{p^*+(1-p^*)\epsilon^2/d}},
$$
which completes the proof.
\end{proof}

}

{
\color{black}
\subsection{Dynamic pricing under approximate LDP}\label{subsec:approx_ldp_supplement}

\Cref{def:LDP_approx} defines the notion of approximate LDP for dynamic pricing, which is modified from the notion of (pure) LDP for dynamic pricing in \Cref{def:LDP} of the main text.
\begin{defn}[$(\epsilon,\Delta)$-LDP for dynamic pricing]\label{def:LDP_approx}
For any $\epsilon> 0$ and $\Delta\in(0,1]$, a pricing policy $\pi=\{Q_t,A_t\}_{t=1}^T$ satisfies $(\epsilon,\Delta)$-LDP if for any $t \in [T]$, $w_{<t}$, and $s_t,s_t'$, it holds that for any measurable set $W$, $Q_t(W|s_t,w_{<t})\leq e^{\epsilon}Q_t(W|s_t',w_{<t})+\Delta$.
\end{defn}

It is clear from the definition that $(\epsilon,\Delta)$-LDP is a relaxed privacy constraint than $\epsilon$-LDP, as it allows a $\Delta\in(0,1]$ probability of privacy leakage. In the following, we propose an ETC-LDP-Approx algorithm in \Cref{algorithm:dp_approx} and further establish its theoretical guarantees in \Cref{lem:approx_ldp_est_theta} and \Cref{thm:approx_ldp}. ETC-LDP-Approx in \Cref{algorithm:dp_approx} follows the same structure as ETC-LDP in \Cref{algorithm:dp}. The only difference lies in their privacy mechanism. In particular, to achieve $(\epsilon,\Delta)$-LDP, we employ the widely used Gaussian mechanism in ETC-LDP-Approx (see step a2) to privatize the truncated gradient $g_t^{[C]}$ into $w_t$, instead of the $L_2$-ball mechanism in ETC-LDP.

\vspace{2mm}
\begin{algorithm}[ht]
{\color{black}
\begin{algorithmic}
    \Indent \textbf{Input}: {Total rounds $T$, price interval $[l,u]$, parameter space $\Theta,$ privacy parameter $(\epsilon, \Delta)$,
    \\ exploration length $\tau$, learning rate of SGD $\zeta_l$, gradient truncation parameter $C_g$}
    \EndIndent
    
    \medskip
    \Indent a.~(Exploration) 
    \State Randomly sample an initial estimator $\widehat{\theta}_0$ from $\Theta$.
    \For {$t\in[\tau]$}
        \State a1.~(Experiment): uniformly choose $p_t$ from $[l,u]$, observe raw data $s_t=(z_t,y_t,p_t).$
        
        \State a2.~(Privatization): compute the truncated gradient
        \[
        g_t^{[C]} = g_t^{[C]}(\widehat{\theta}_{t-1};s_t) = \Pi_{\mathbb B_{C_g}}[g(\widehat\theta_{t-1};s_t)], \text{ with } g(\widehat\theta_{t-1};s_t) \mbox{ defined in \eqref{eq-truncated-gradient};}
        \]
        \State  let $w_t=g_t^{[C]}+\xi_t$, where $\xi_t\sim\mathcal{N}(0,\sigma^2_{\xi}\cdot I_{2d})$ with $\sigma_{\xi}^2={2C_g^2\log(1.25/\Delta)}/{\epsilon^2}$.
        
        \State a3.~(SGD): update estimation via $\widehat{\theta}_{t}=\Pi_{\Theta}[\widehat{\theta}_{t-1}+\eta_{t}w_t]$ with step size $\eta_{t}= (\zeta_l t)^{-1}$.
    \EndFor
    \EndIndent

    \medskip
    \Indent b.~(Exploitation)  
    \For {$t=\tau+1,\cdots,T$}
        \State  offer the greedy price at  $p_t=p^*(\widehat{\theta}_{\tau}, z_t)=\argmax\limits_{p\in[l,u]} r(p,z_t,\widehat\theta_\tau)$ based on $\widehat\theta_\tau.$
    \EndFor
    \EndIndent
    \caption{\textcolor{black}{The ETC-LDP-Approx algorithm for dynamic pricing under $(\epsilon,\Delta)$-LDP}}
    \label{algorithm:dp_approx}
\end{algorithmic}}
\end{algorithm} 

\Cref{lem:approx_ldp_est_theta} gives a high-probability upper bound of the estimation error for the $(\epsilon,\Delta$)-private SGD estimator in the exploration phase of ETC-LDP-Approx, serving as the basis for bounding the regret of ETC-LDP-Approx. In particular, it shows that the estimation error of $\theta^*$ under $(\epsilon,\Delta$)-LDP is of the same order as that under $\epsilon$-LDP in \Cref{lem:sgd} (up to a $O(\log(1.25/\Delta))$ term).

Due to the Gaussian privacy mechanism, the gradient of the SGD is unbounded. To handle this, similar to the proof of \Cref{lem:sgdmix}, we rely on the newly developed \Cref{lem_freedman} and \Cref{lem_subGaussian}, which extends the uniform concentration inequality for martingale difference sequences in \colorlinks{black}{\cite{rakhlin2011making}} from bounded random variables to random variables with certain moment conditions (e.g.\ sub-Gaussian), and thus is of independent interest.

\begin{lem}\label{lem:approx_ldp_est_theta}
    Under the conditions in \Cref{lem:sgd} holds. For any $\tau\geq 4$, set the learning rate $\zeta_l$ and truncation parameter $C_g$ same as in \Cref{lem:sgd}. We have that, for any $\epsilon>0, \Delta\in (0,1]$ and $0<\delta<(e\vee \log\tau)^{-1}$, with probability at least $1-8\delta\log\tau$, it holds that 
    $$
    \|\widehat{\theta}_{\tau}-\theta^*\|^2 \leq  B_{A1}\cdot\frac{d\log(1/\delta)\log(\tau)\log(1.25/\Delta)}{c_l^2\lambda_z^2\epsilon^2 \tau},
    $$
    where $B_{A1}>0$ is an absolute constant that only depends on quantities in Assumptions \ref{assum_feature} and \ref{assum_ldp}(a).
\end{lem}

\Cref{thm:approx_ldp} gives the regret upper bound of the ETC-LDP-Approx algorithm in \Cref{algorithm:dp_approx}, which shows that the regret upper bound for dynamic pricing under $(\epsilon,\Delta$)-LDP is of the same order as that under $\epsilon$-LDP in \Cref{thm:LDP-ETC} (up to a $O(\sqrt{\log(1.25/\Delta)})$ term). Its proof follows the same structure as that of \Cref{thm:LDP-ETC} and is thus omitted.

\begin{thm}\label{thm:approx_ldp}
Suppose conditions in \Cref{thm:LDP-ETC} hold. For any $\epsilon>0$ and $0<\delta<\{e\vee \log(T)\}^{-1}$, set the learning rate $\zeta_l$ and truncation parameter $C_g$ same as in \Cref{thm:LDP-ETC}. For any $T$ satisfying $T\geq B_{A2} \cdot \{\lambda_z^{-4} \epsilon^{-2}\log (T)\log(1/\delta) \},$ we have that, with probability at least $1-\delta\log(T)$, the regret of ETC-LDP-Approx is upper bounded by $ R_T\leq B_{A3} d\sqrt{\log (T)\log (1/\delta)\log(1.25/\Delta)T}/\epsilon$, where $B_{A2}, B_{A3}>0$ are absolute constants only depending on quantities in  Assumptions \ref{assum_feature}, \ref{assum:etc} and \ref{assum_ldp}.
\end{thm}

\subsubsection{Numerical experiments}
In this section, we further conduct numerical experiments to examine the finite sample performance of ETC-LDP-Approx under $(\epsilon,\Delta)$-LDP and compare it with ETC-LDP.

In particular, we use the exact same simulation setting in \Cref{subsubsec:ETC_LDP} of the main text. For ETC-LDP-Approx, we fix $\epsilon=1$ and vary the privacy leakage probability $\Delta\in \{10^{-2},10^{-3}\}$. The result under $\epsilon\in \{2,4\}$ is similar and thus omitted to conserve space.

\Cref{fig:ETCLDP_Approx_S1}(left) reports the mean regret $\overline{R}_{T,d,\epsilon}$ of ETC-LDP-Approx with $\Delta\in \{10^{-2},10^{-3}\}$ at different $(T,d)$ and $\epsilon=1$ under (S1). As can be seen, given a fixed $d$, the regret scales sublinearly with respect to $T$, and given a fixed horizon $T$, the regret grows with the dimension $d.$ Moreover, a smaller privacy leakage probability $\Delta$ leads to a higher regret, which is intuitive and confirms the theoretical results in \Cref{thm:approx_ldp}. 

For illustration, \Cref{fig:ETCLDP_Approx_S1}(middle) further gives the boxplot of $\{R_{T,d,\epsilon}^{(i)}\}_{i=1}^{500}$ achieved by ETC-LDP (i.e.\ pure $\epsilon$-LDP) and ETC-LDP-Approx (at $\Delta=10^{-2}$) for each $T\in \{(1,3,5,7,9)\times 10^5\}$ and $\epsilon=1$ with a fixed $d=6.$ As can be seen, the regret given by the two algorithms are similar with ETC-LDP giving slightly lower regret (and smaller variance). This confirms our theoretical results that the two algorithms share the same order of regret upper bound up to a $O(\sqrt{\log(1.25/\Delta)}$ term. \Cref{fig:ETCLDP_Approx_S1}(right) gives the boxplot of $\{R_{T,d,\epsilon}^{(i)}\}_{i=1}^{500}$ achieved by ETC-LDP (i.e.\ pure $\epsilon$-LDP) and ETC-LDP-Approx (at $\Delta=10^{-2}$) for each $d\in\{1,2,4,6\}$ and $\epsilon=1$ with a fixed $T=9\times 10^5$, where similar phenomenon is observed. In addition, \Cref{fig:ETCLDP_Approx_S2} gives the performance of ETC-LDP-Approx under (S2), where similar phenomenon as the one in \Cref{fig:ETCLDP_Approx_S1} is observed.

\begin{figure}[H]
\vspace{-5mm}
\hspace*{-6mm}
    \begin{subfigure}{0.32\textwidth}
	\includegraphics[angle=270, width=1.1\textwidth]{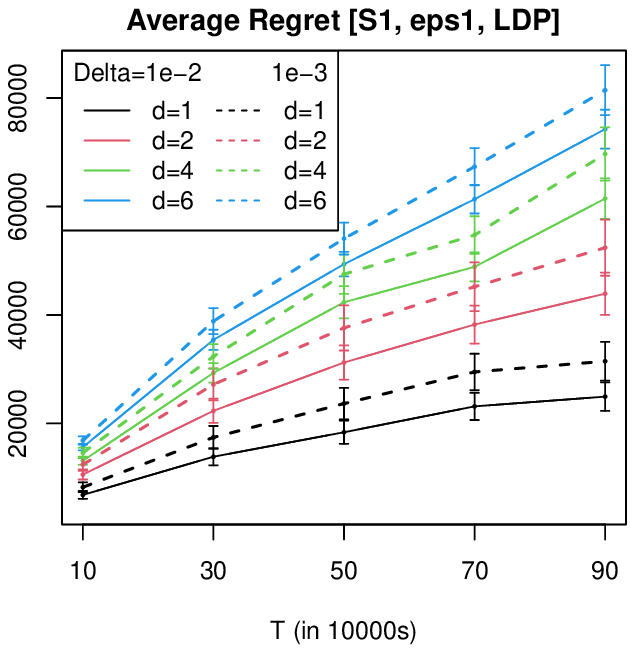}
	\vspace{-0.5cm}
    \end{subfigure}
    ~
    \begin{subfigure}{0.32\textwidth}
	\includegraphics[angle=270, width=1.1\textwidth]{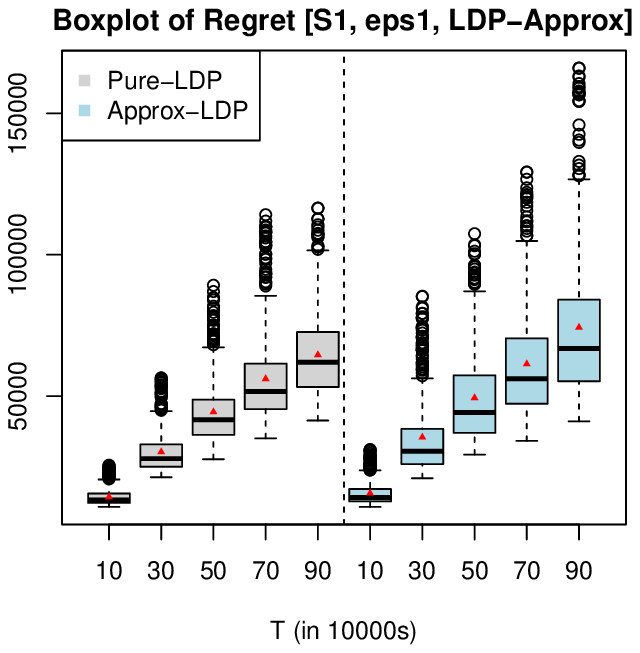}
	\vspace{-0.5cm}
    \end{subfigure}
    ~
    \begin{subfigure}{0.32\textwidth}
	\includegraphics[angle=270, width=1.1\textwidth]{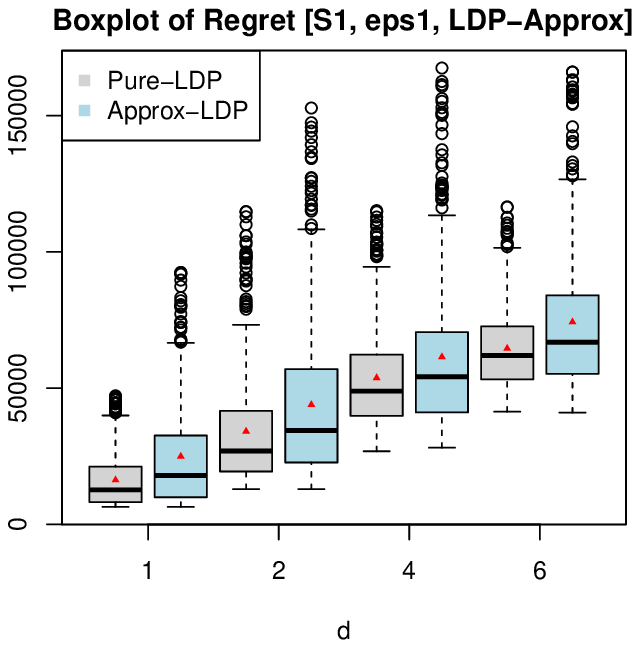}
	\vspace{-0.5cm}
    \end{subfigure}
    \caption{{\color{black} Performance of ETC-LDP-Approx under (S1). [Left]:  Mean regret (with C.I.) under different $(d,T)$ and $\epsilon=1$ for $\Delta=(10^{-2},10^{-3})$. [Middle]: Boxplot of regrets at different $T$ (with $d=6$ and $\epsilon=1$) for pure-LDP and approx-LDP $(\Delta=10^{-2})$. [Right]: Boxplot of regrets at different $d$ (with $T=9\times 10^{5}$ and $\epsilon=1$) for pure-LDP and approx-LDP $(\Delta=10^{-2})$.}}
    \label{fig:ETCLDP_Approx_S1}
\end{figure}

\begin{figure}[H]
\vspace{-5mm}
\hspace*{-6mm}
    \begin{subfigure}{0.32\textwidth}
	\includegraphics[angle=270, width=1.1\textwidth]{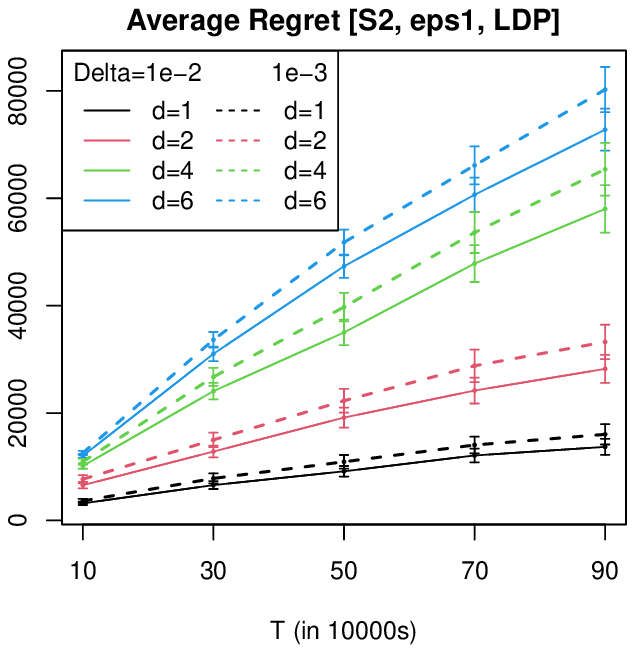}
	\vspace{-0.5cm}
    \end{subfigure}
    ~
    \begin{subfigure}{0.32\textwidth}
	\includegraphics[angle=270, width=1.1\textwidth]{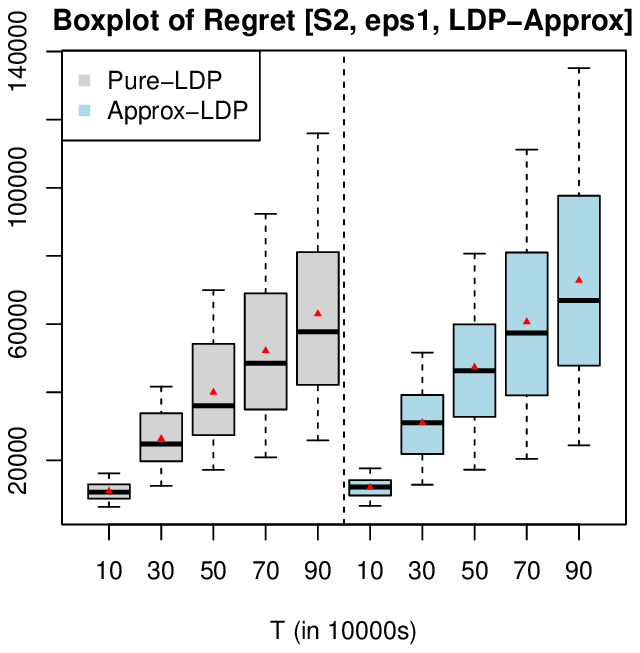}
	\vspace{-0.5cm}
    \end{subfigure}
    ~
    \begin{subfigure}{0.32\textwidth}
	\includegraphics[angle=270, width=1.1\textwidth]{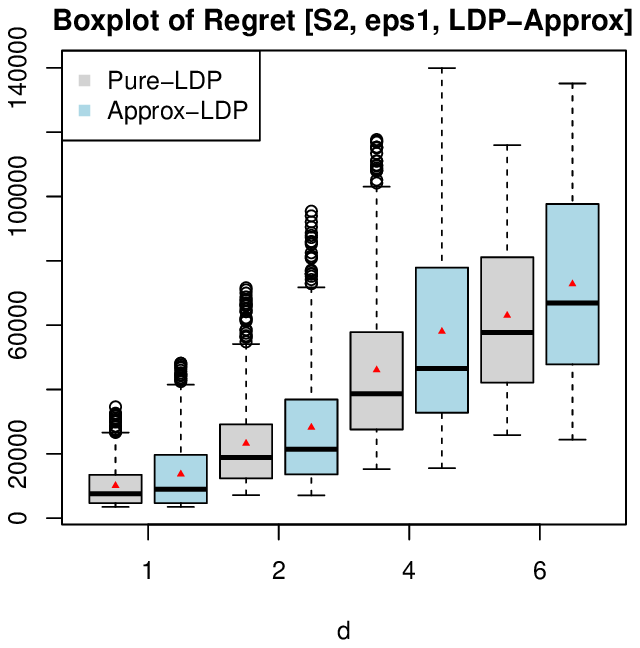}
	\vspace{-0.5cm}
    \end{subfigure}
    \caption{{\color{black} Performance of ETC-LDP-Approx under (S2). [Left]:  Mean regret (with C.I.) under different $(d,T)$ and $\epsilon=1$ for $\Delta=(10^{-2},10^{-3})$. [Middle]: Boxplot of regrets at different $T$ (with $d=6$ and $\epsilon=1$) for pure-LDP and approx-LDP $(\Delta=10^{-2})$. [Right]: Boxplot of regrets at different $d$ (with $T=9\times 10^{5}$ and $\epsilon=1$) for pure-LDP and approx-LDP $(\Delta=10^{-2})$.}}
    \label{fig:ETCLDP_Approx_S2}
\end{figure}

\subsubsection{Proof of \Cref{lem:approx_ldp_est_theta}}
\begin{proof}[\textbf{Proof of \Cref{lem:approx_ldp_est_theta}}]

Note that 
\begin{flalign*}
    \|\widehat{\theta}_t-\theta^*\|^2
    \leq& \|\widehat{\theta}_{t-1}+\eta_tg_t^{[C]}-\theta^*\|^2+2\eta_t\xi_t^{\top}[\widehat{\theta}_{t-1}-\theta^*+\eta_tg_t^{[C]}]+\eta_t^2\|\xi_t\|^2\\
\leq &    
    (1-2\eta_t\zeta_l)\|\widehat{\theta}_{t-1}-\theta^*\|^2+2\eta_t[g_t^{[C]}-g(\widehat{\theta}_{t-1})]^{\top}(\widehat{\theta}_{t-1}-\theta^*)+\eta_t^2C_g^2\\&+2\eta_t\xi_t^{\top}[\widehat{\theta}_{t-1}-\theta^*+\eta_tg_t^{[C]}]+\eta_t^2\|\xi_t\|^2
    \\= &  (1-2\eta_t\zeta_l)\|\widehat{\theta}_{t-1}-\theta^*\|^2+2\eta_t[g_t^{[C]}+\xi_t-g(\widehat{\theta}_{t-1})]^{\top}(\widehat{\theta}_{t-1}-\theta^*)+\eta_t^2C_g^2\\&+2\eta_t^2(\xi_t^{\top}g_t^{[C]})+\eta_t^2\|\xi_t\|^2
    \\\leq &  (1-2\eta_t\zeta_l)\|\widehat{\theta}_{t-1}-\theta^*\|^2+2\eta_t[g_t^{[C]}+\xi_t-g(\widehat{\theta}_{t-1})]^{\top}(\widehat{\theta}_{t-1}-\theta^*)+2\eta_t^2C_g^2+2\eta_t^2\|\xi_t\|^2,
\end{flalign*}
where the second inequality holds by similar arguments as \eqref{bound_theta},  and the last by Cauchy-Schwartz inequality and  that $\|g_t^{[C]}\|\leq C_g.$  

Then, following the same iteration arguments after equation \eqref{bound_thetatau}, we obtain 
\begin{flalign}\label{bound_theta_aLDP}
\begin{split}
     \|\widehat{\theta}_t-\theta^*\|^2\leq &\frac{2}{\zeta_l(t-1)t}\sum_{i=2}^t (i-1)[g_i^{[C]}-g(\widehat{\theta}_{i-1})]^{\top}(\widehat{\theta}_{i-1}-\theta^*)\\&+\frac{2}{\zeta_l(t-1)t}\sum_{i=2}^t (i-1)\xi_i^{\top}(\widehat{\theta}_{i-1}-\theta^*)+ \frac{2\sum_{i=2}^t\|\xi_i\|^2}{t(t-1)\zeta_l^2}+\frac{2C_g^2}{\zeta_l^2t}
     \\:=&\sum_{j=1}^{4}B_{t,j}.        
\end{split}
\end{flalign}
[\textbf{Bound for }$B_{t,1}$] Similar to the proof in \eqref{bound_tau1} and \eqref{bound_tau2}, we have with probability larger than $1-2\delta\log\tau$, for all $2\leq t\leq \tau$,
\begin{flalign}\label{bound_B1}
\begin{split}
    B_{t,1} \leq &
     \frac{16C_g}{\zeta_l(t-1)t}\max\left\{
\sqrt{\sum_{i=2}^{t}(i-1)^2\|\widehat{\theta}_{i-1}-\theta^*\|^2},  c_{\theta}(t-1)\sqrt{\log (1/\delta)}
\right\}\sqrt{\log (1/\delta)}
\\&+\frac{2c_{\theta}\sigma \sqrt{1+u^2}}{\zeta_l\tau}.
\end{split}
\end{flalign}
[\textbf{Bound for }$B_{t,2}$] The proof of $B_{t,2}$ is similar to the proof of $B_{t,1}$. Since $\xi_i$ is i.i.d. Gaussian with independent entries, hence for any given $a\in \mathbb{R}^{2d}$, $a^{\top}\xi_i$ is still Gaussian with zero mean and variance $\|a\|^2\sigma_\xi^2.$  Hence, we have  $\mathbb{E}[(i-1)\xi_i^{\top}(\widehat{\theta}_{i-1}-\theta^*)|\mathcal{F}_{i-1}]=0,$ and  by \Cref{lem_subGaussian}, we have 
$$\mathbb{E}\left[\left|(i-1)\xi_i^{\top}(\widehat{\theta}_{i-1}-\theta^*)\right|^2|\mathcal{F}_{i-1}\right]=[(i-1)\|\widehat{\theta}_{i-1}-\theta^*\|]^2 \sigma_{\xi}^2,
$$
and 
$$\mathbb{E}\left[\left|(i-1)\xi_i^{\top}(\widehat{\theta}_{i-1}-\theta^*)\right|^k|\mathcal{F}_{i-1}\right]\leq [2\sigma_{\xi}(i-1)\|\widehat{\theta}_{i-1}-\theta^*\|]^k (k!/2).
$$
Therefore, applying \Cref{lem_freedman} with  $V_t= 4\sigma_{\xi}^2\sum_{i=2}^t(i-1)^2\|\widehat{\theta}_{i-1}-\theta^*\|^2/[\zeta_l(t-1)t]^2$ and $b=8(t-1) c_{\theta}\sigma_{\xi}/[\zeta_l(t-1)t]$, we can show that with probability larger than $1-2\delta\log\tau$, for any $2\leq t\leq\tau$ and $\tau\geq 4$,
\begin{flalign}\label{bound_B2}
    B_{2,t}\leq \frac{16\sigma_{\xi}}{\zeta_l(t-1)t}\max\left\{\sqrt{\sum_{i=2}^t(i-1)^2\|\widehat{\theta}_{i-1}-\theta^*\|^2}, (t-1)c_{\theta}\sqrt{\log(1/\delta)}\right\}\sqrt{\log(1/\delta)}.
\end{flalign}

\noindent[\textbf{Bound for }$B_{t,3}$] It is easy to see that $\xi_t^2/\sigma_\xi^2$ follows a $\chi^2_{2d}$ distribution. Hence,  $
\mathbb{E}\left|\|\xi_t\|^2-\mathbb{E}\|\xi_t\|^2\right|^2= 4d \sigma_{\xi}^4,
$
and it can be shown that for $k\geq 2,$
$$
\mathbb{E}\left|\|\xi_t\|^2-\mathbb{E}\|\xi_t\|^2\right|^k\leq \mathbb{E}(\|\xi_t\|^{2})^k=(2\sigma_{\xi}^2)^k\Gamma(d+k)/\Gamma(d)\leq (2\sigma_{\xi}^2)^k (k!) [e(d+k)/k]^k\leq (2ed\sigma_{\xi}^2)^k (k!/2),
$$
where the first equality uses the moments of $\chi^2_{2d}$ random variables, and the second inequality holds by the upper bound of binomial coefficient $(d+k-1)!/[(d-1)!k!]\leq [e(d+k-1)/k]^k$.
Therefore, letting $X_i=\|\xi_i\|^2-\mathbb{E}\|\xi_i\|^2$, $V_t=4d(t-1)\sigma_\xi^4$, and $b=2e d \sigma_\xi^2$ in \Cref{lem_freedman}, we have  with probability larger than $1-2\delta\log\tau$, for all $2\leq t\leq \tau$,
$$
\sum_{i=2}^t(\|\xi_t\|^2-\mathbb{E}\|\xi_t\|^2)\leq 4\sigma_\xi^2\max\left\{2\sqrt{d(t-1)}, e d\sqrt{\log(1/\delta)}\right\}\sqrt{\log(1/\delta)}.
$$
Therefore, since $\sqrt{1/(t-1)}\leq 1$, we have with probability larger than $1-2\delta\log\tau$, for all $2\leq t\leq \tau$,
\begin{equation}\label{bound_B3}
\sum_{i=2}^t\|\xi_t\|^2/(t-1)\leq 8d\sigma_\xi^2 \log(1/\delta) + 2d\sigma_{\xi}^2\leq 10d\sigma_\xi^2 \log(1/\delta).
\end{equation}

Therefore, summarizing  \eqref{bound_theta_aLDP}, \eqref{bound_B1}, \eqref{bound_B2} and \eqref{bound_B3}, we obtain that with probability larger than $1-6\delta\log\tau$,  for any $2\leq t\leq\tau$ and $\tau\geq 4$,   
\begin{flalign*}
    \|\widehat{\theta}_t-\theta^*\|^2\leq & \frac{16C_g+16\sigma_{\xi}}{\zeta_l(t-1)t} \max\left\{\sqrt{\sum_{i=2}^t(i-1)^2\|\widehat{\theta}_{i-1}-\theta^*\|^2}, (t-1)c_{\theta}\sqrt{\log(1/\delta)}\right\}\sqrt{\log(1/\delta)}\\
    &+\frac{20d\sigma_{\xi}^2\log(1/\delta)+2C_g^2}{t  \zeta_l^2}+\frac{2c_{\theta}\sigma \sqrt{1+u^2}}{\zeta_l\tau}.
\end{flalign*}
By similar  induction argument as that in the proof of \Cref{lem:sgd}, and note the definition of $\sigma_{\xi}^2$, we obtain that 
$$
 \|\widehat{\theta}_t-\theta^*\|^2\lesssim \frac{dC_{g}^2\log(1/\delta)\log(1.25/\Delta)}{t \zeta_l^2\epsilon^2},
$$
which completes the proof.
\end{proof}
}

\section{Technical details in \Cref{sec-optimality}}\label{sec:supplement_proof_lb}
{\color{black}
\subsection{Roadmap of the proof}\label{subsec:roadmap_lb}
In this subsection, we first provide a roadmap of the technical proof for the lower bound construction under both non-private and private dynamic pricing, where we point out key steps and key lemmas used in the proof and discuss their intuition. There are in total 4 steps (Step 0 - Step 3).

\noindent {\bf  Step 0.  Lower bound construction. } [Section \ref{sec-lb-construction}]

We first construct a set of $2^d$ revenue functions, indexed by the hyper parameter $\bm{v}=(v_1,\cdots,v_d)^{\top}\in\{0,1\}^d$ such that distinguishing  $v_i=0$ from  $v_i=1$ for each coordinate can only depend on the information from the $i$th covariate. Specifically, let $z_t=(\mathbb{I}(M_t=1),\cdots,\mathbb{I}(M_t=d))^{\top}\in\{0,1\}^d$, where $M_t$ is a multinomial distribution on $\{1,\cdots,d\}$ with equal marginal probability.  Consider the following demand model indexed by the hyper parameter $\bm{v}=(v_1,\cdots,v_d)^{\top}\in\{0,1\}^d$, such that
\begin{flalign*}
   y_t=\mathbb{E}_{\bm{v}}[y_t|z_t,p_t]+\varepsilon_t \,\, \mbox{and} \,\,   \mathbb{E}_{\bm{v}}[y_t|z_t,p_t]:=\lambda_{\bm{v}}(z_t,p_t)= 2-p_t+\Delta\sum_{i=1}^d v_i(1-p_t) \mathbb{I}(M_t=i),
\end{flalign*}
where $\{\varepsilon_t\}_{t=1}^T$ is a sequence of independent $\mathcal{N}(0,1)$ random variables, and  $\Delta\in[0,1/2]$ quantifies the difficulty of the problem, and  will be chosen differently for the non-private and private cases. Note by our construction, the effective information for $v_i$ is contained through the $i$th element of $z_t$, i.e. $\mathbb{I}(M_t=i)$ only.

For $M_t=i$, an important property of our construction is the following connection of the instant regret and the policy induced price:\begin{flalign}\label{regret_case1}
    \begin{split}
        \mathrm{regret}_t^\pi = \begin{cases}
  (p_t^\pi-1)^2, & \text{if } v_i=0; \\
  (1+\Delta)(p_t^\pi-\frac{2+\Delta}{2+2\Delta})^2& \text{if } v_i=1.
\end{cases}
    \end{split}
\end{flalign}

\noindent {\bf  Step 1. Reduction to classification problem.} [\Cref{lem_class1} for the non-private case, and \Cref{lem_class}  for the private case.]

\noindent[\textsc{The non-private case, \Cref{lem_class1}}] We first construct a classifier using data $s_t=\{z_t,p_t,y_t\}$ such that  $\psi:(s_1,\cdots,s_T)\to \{\hat{v}_i\}_{i=1}^d \in \{0,1\}^d$. In \Cref{lem_class1}, we show, with high probability, for any non-anticipating policy $\pi$, 
\begin{equation}\label{regret_class}
        \sum_{t=1}^T \mathrm{regret}^{\pi}_t\geq \frac{1}{144}T \Delta^2 d^{-1}   \sum_{i=1}^d \mathbb{I}(\hat{v}_i\neq v_i).
\end{equation}
This implies that  if there is a low-regret  pricing policy, then there
is a classifier $\psi$ that achieves small classification errors. Therefore, this motivates us to use information-theoretical  tools (Assouad's Lemma in particular) to show a lower bound on the classification error. 

\noindent[\textsc{The private case, \Cref{lem_class}}] Under the LDP constraint,  we cannot directly use the classifier defined in the non-private setting as it utilizes non-privatized contextual information. Instead, we adopt the construction in Lemma 3 of \cite{chen2022differential} and augment the $\epsilon$-LDP policy $\pi$ to construct a $2\epsilon$-LDP policy $\pi'$.  In particular, we have to  additionally privatize the information used for classification, which necessitates elevating the regret analysis from an $\epsilon$-LDP policy to a $2\epsilon$-LDP policy. To address this, we privatize the information used in the classification process by introducing Laplacian noise into the response vector of the classification. \Cref{lem_class} then shows that,  with high probability,  for any non-anticipating $\epsilon$-LDP policy $\pi$, there exists a 2$\epsilon$-LDP classifier $\psi:(w_1,\cdots,w_T)\to \hat{\bm{v}}\in \{0,1\}^d$, such that  \eqref{regret_class} still holds.

\noindent {\bf Step 2. Assouad's Lemma to lower bound the classification error.}  [\Cref{lem_assouad1} for the non-private case, and \Cref{lem_assouad2}  for the private case.]

\noindent[\textsc{The non-private case, \Cref{lem_assouad1}}]
To introduce the method, we define some additional notations. For the non-private case, define $$
P_{+i}^\pi:=\frac{1}{2^{d-1}} \sum_{\boldsymbol{v}: {v}_i=1} P_{\boldsymbol{v}}^\pi, \quad P_{-i}^\pi:=\frac{1}{2^{d-1}} \sum_{\boldsymbol{v}: v_i=0} P_{\boldsymbol{v}}^\pi,
$$
with $P_{\bm{v}}^\pi$ being the joint distribution of $\{s_1,\cdots,s_T\}$ under model $\boldsymbol{v}$ and policy $\pi$. 
Then, we have 
\begin{flalign*}
    \begin{split}  
  \sup _{\boldsymbol{v} } \mathbb{E}_{\boldsymbol{v}}^{\pi, \psi}\left[\sum_{i=1}^d \mathbb{I}(\hat{v}_i\neq v_i)\right]\geq\frac{1}{2} \sum_{i=1}^d (1- \mathrm{TV}(P_{+i}^{\pi}, P_{-i}^{\pi}))\geq&\frac{1}{2} \sum_{i=1}^d \left(1-\sqrt{\frac{1}{4}D_{\mathrm{KL}}^{\mathrm{sy}}(P_{+i}^{\pi}, P_{-i}^{\pi})}\right)\\\geq&
  \frac{d}{2}  \left(1-\sqrt{\frac{1}{4d}\sum_{i=1}^d   D_{\mathrm{KL}}^{\mathrm{sy}}(P_{+i}^{\pi}, P_{-i}^{\pi})}\right), 
    \end{split}
\end{flalign*}
where the first inequality holds by Assouad's Lemma, the second by Pinsker's inequality, and the last by the Cauchy--Schwarz inequality. 

Therefore, it suffices to calculate the KL divergence between $P_{+i}^{\pi}$ and $P_{-i}^{\pi}$. Thanks to the Gaussian design of the demand model, and the chain-rule of KL divergence using the conditioning argument, we manage to show that 
$$
\sum_{i=1}^d D_{\mathrm{KL}}^{\mathrm{sy}}(P_{+i, t}^\pi\left(\cdot \mid s_{<t}\right)| P_{-i, t}^\pi\left(\cdot \mid s_{<t}\right))=\frac{\Delta^2}{2^d}\sum_{\boldsymbol{v}}\sum_{t=1}^T\mathbb E_{\boldsymbol{v}}^\pi \left\{ (p_t-1)^2\right\}.
$$
This further connects the classification error with the regret via \eqref{regret_case1}.

\noindent[\textsc{The private case, \Cref{lem_assouad2}}] This is the key component of our lower bound construction to derive the $\Omega(d\sqrt{T}/\epsilon)$ sharp bound under LDP. First, same as the non-private case, we begin by applying the standard Assouad's Lemma to connect the classification error with TV. Next, however, instead of using the KL divergence as in the non-private case, we resort to the sharper Hellinger distance to further bound TV. In fact, for two arbitrary distributions $P_1$ and $P_2$, we have 
$$
\{\mathrm{TV}(P_1, P_2)\}^2\leq 2 \mathrm{H}^2(P_1,P_2)\leq D_{\mathrm{KL}}(P_1,P_2),
$$
where $\mathrm{H}^2(P_1,P_2)=\frac{1}{2}\int_{\mathcal{X}}[\sqrt{P_1(\mathrm{d}x)-P_2(\mathrm{d}x)}]^2$ is the Hellinger distance between $P_1$ and $P_2$ on the support $\mathcal{X}.$  The above inequality implies that using the Hellinger distance allows for more delicate analysis than the KL divergence. 


Specifically, in the proof of \Cref{lem_assouad2}, we first show that, 
 $$
   \sup _{\boldsymbol{v}} \mathbb{E}_{\boldsymbol{v}}^{\pi, \psi}\left[\sum_{i=1}^d \mathbb{I}(\hat{v}_i\neq v_i)\right]\geq \frac{d}{2}\left(1-
 \sqrt{ \frac{2}{d  }\sum_{i=1}^d \mathrm{H}^2(M_{+i}^{\pi}, M_{-i}^{\pi})}
   \right),
    $$ 
here $M_+^{\pi}$ and $M_-^{\pi}$  can be viewed as the private version of $P_+^\pi$ and $P_-^\pi$ respectively. We then apply a chain-rule type inequality (Lemma 4 in \cite{jayram2009}) to the Hellinger distance of the joint distribution above, decomposing it into the instant Hellinger distance associated with each time point.

We then adapt the newly-developed information contraction bound technique in \cite{acharya2024unified}, which is designed for private statistical estimation problem under i.i.d.\ data, to further bound the sum of Hellinger distance under our setting, where the data can be adaptively generated by \textit{any} non-anticipating pricing policy with $\epsilon$-LDP. In particular, we manage to show that 
\begin{align}\label{eq:cost_of_privacy_equation}
    2\sum_{i=1}^d \mathrm{H}^2(M_{+i}^{\pi}, M_{-i}^{\pi})\leq \frac{21\Delta^2[\exp(2\epsilon)-1]^2}{d}\frac{1}{2^d}\sum_{\boldsymbol{v}}\sum_{t=1}^T\mathbb E_{\boldsymbol{v}}^\pi \left\{ (p_t-1)^2\right\},
\end{align}
which is our key result. Note that, compared with the non-private version, the bound above inflate by a factor of  $[\exp(2\epsilon)-1]^2/d\asymp \epsilon^2/d$, which is the price to pay for protecting privacy. 

The proof of \Cref{lem_assouad2} is a key component and novelty of our lower bound construction and we give a more detailed discussion of this step in \textbf{Step 2A} below, where we discuss the technical challenge and important modifications we made to adapt the information contraction bound in \cite{acharya2024unified} (designed for i.i.d. data) to the setting of private dynamic pricing (where data is adaptive).

We remark that our analysis in \Cref{lem_assouad2} is new and first time seen in the dynamic pricing literature. In fact, existing lower bound construction in contextual bandit~\citep{han2021generalized} and contextual (non-parametric) dynamic pricing~\citep{chen2022differential} under LDP constraints mainly rely on the well-known private Le Cam Lemma (Theorem 1 therein) and private Assouad's Lemma (Proposition 3 therein) developed in \cite{duchi2018minimax}, combined with KL divergence bounds. However, we find that this strategy (i.e.\ technical tools in \cite{duchi2018minimax} + KL divergence) can only derive a $\Omega(\sqrt{dT}/\epsilon)$ lower bound, which is not sharp in dimension $d.$ Instead, our proof strategy is to leverage the Hellinger distance and modify the information contraction bound in \cite{acharya2024unified} to derive the sharp $\Omega(d\sqrt{T}/\epsilon)$ bound.

\noindent\textbf{Step 3. Finish the proof}. Therefore, to summarize, for both the non-private and private cases, we effectively create a dilemma for the regret of any non-anticipating policy. In particular, for small values of $\Delta$, we can always identify a setting where the associated classification error is large due to the Assouad's Lemma. Conversely, for large values of 
$\Delta$, as is implied by \eqref{regret_case1}, the cumulative regret also becomes large. The result then follows by carefully balancing two scenarios and choosing an appropriate value of $\Delta$.

\noindent {\bf Step 2A. Key modification of \cite{acharya2024unified} for LDP dynamic pricing}  [Step II in the proof \Cref{lem_assouad2} and \Cref{lem_varphi}.]

We first introduce the information contraction bound result in \cite{acharya2024unified}, which applies to statistical estimation problems under $\epsilon$-LDP but for i.i.d.\ data only. We then discuss the key non-trivial modifications we made to tailor for dynamic pricing under $\epsilon$-LDP, where the data is no longer i.i.d.\ but instead can be adaptively collected by \textit{any} non-anticipating pricing policy $\pi.$

\textbf{Results in \cite{acharya2024unified}}. Let $\{{P}_{\bm{v}}: \bm{v}\in \{0,1\}^d \}$ be a family of distributions indexed by $\bm{v}$. Let $P_{+i}^T=\frac{1}{2^{d-1}}\sum_{v_i=1} P_{\bm{v}}$,  $P_{-i}^T=\frac{1}{2^{d-1}}\sum_{v_i=0} P_{\bm{v}}^{T}$, where $P_{\bm{v}}^T$ denotes the joint distribution of $T$ i.i.d. samples $\{x_t\in \mathcal{X}\}_{t=1}^T$ whose marginal distribution is $P_{\bm{v}}$.  Furthermore, define the   sequentially interactive $\epsilon$-LDP mechanism $Q$ that output privatized data $w_t$ given input $x_t$ and past information $w_{<t}$, i.e. 
$$
M_{t,\bm v}(S)=\int Q(S|w_1,\cdots, w_{t-1};x_t)\mathrm{d}M_{\bm{v}}(w_1,\cdots, w_{t-1}) \mathrm{d}P_{\bm{v}}x_t,
$$
where $M_{t,\bm v}(\cdot)$ denotes the  marginal distribution of $w_t$, and $M_{\bm{v}}(w_1,\cdots,w_{t})$  is the joint distribution  of $\{w_{s}\}_{s=1}^t$ indexed by $\bm{v}$. Define $M_{+i}^T=\frac{1}{2^{d-1}}\sum_{v_i=1} M_{\bm{v}}(w_{1:T})$,   $M_{-i}^T=\frac{1}{2^{d-1}}\sum_{v_i=0} M_{\bm{v}}(w_{1:T})$. 

Further define $\bm{v}^{\oplus i}$ as the vector with coordinates, 
$$\left(\bm{v}^{\oplus i}\right)_j= \begin{cases}v_j, & j \neq i, \\ 1-v_j, & \text { otherwise, }\end{cases}$$
which flips the $i$th coordinate of $\bm v$ from $v_i$ to $1-v_i.$ Then, the following two assumptions are imposed. 

\textsc{Acharya-Assumption 1 (Densities Exist).} For every \( \bm v \in \{0,1\}^d \) and \( i \in [d] \), it holds that \( p_{\bm{v}^{\oplus i}} \ll p_{\bm v} \), and there exist measurable functions \( \varphi_{\bm{v},i}: \mathcal{X} \to \mathbb{R} \) such that
\[
\frac{d {P}_{{\bm v}^{\oplus i}}}{d {P}_{\bm v}} = 1 + \varphi_{{\bm v},i}.
\]
The functions \( \varphi_{\bm v,i} \) are referred to as ``score function" and capture the change in density when the coordinate \( i \) is flipped. 

\textsc{Acharya-Assumption 2 (Orthogonality and Variance).} There exists some \( \alpha^2 \geq 0 \) such that, for all \( \bm{v} \in \{0,1\}^d \) and distinct \( i,j \in [d] \),
\[
\mathbb{E}_{{P}_{\bm v}} [\varphi_{\bm v,i} \varphi_{\bm v,j}] = 0 \quad \text{and} \quad \mathbb{E}_{{P}_{\bm v}} [\varphi_{\bm v,i}^2] \leq \alpha^2.
\]

The following result is taken from Corollary 1 in \cite{acharya2024unified}.

\textsc{[Corollary 1 in \cite{acharya2024unified}]} Suppose the family of distributions  $\{{P}_{\bm{v}}: \bm{v}\in \{0,1\}^d \}$ satisfies the Acharya-Assumptions 1 and 2 above. Then, for any sequentially $\epsilon$-LDP mechanism, we have 
$$
\frac{1}{d}\left\{\sum_{i=1}^d \mathrm{TV}(M_{+i}^T,M_{-i}^T)\right\}^2\leq \frac{21}{d} T\alpha^2[\exp(\epsilon)-1]^2.
$$
The proof of Corollary 1 uses similar idea in our Step 2 above by bounding the total variation distance with the Hellinger distance as an intermediate step. The multiplicative factor $[\exp(\epsilon)-1]^2/d\asymp \epsilon^2/d$ characterizes the cost of privacy.

\textbf{Key modifications made to tailor for dynamic pricing under $\epsilon$-LDP}. The result \cite{acharya2024unified} only applies to i.i.d.\ data and thus is not directly applicable to dynamic pricing under $\epsilon$-LDP, where the data $\{s_t=(y_t,z_t,p_t)\}$ can be adaptively generated by any non-anticipating pricing policy $\pi$. In particular, note that the price $p_t$ given by $\pi$ at any time point $t$ is a function of the current covariate $z_t$ and all historical private observations $w_{<t}=\{w_1,\cdots, w_{t-1}\}$ (see \eqref{eq:pricing} and \eqref{eq:privacy} in the main text), which clearly makes $\{s_t\}$ non-i.i.d.\ with complex inter-dependence. Therefore, non-trivial modification is required to adapt the technical tools in \cite{acharya2024unified} to our setting of dynamic pricing under $\epsilon$-LDP. Our main technical contributions are as follows.

First, to address the complex inter-dependence of the raw data $\{s_t\}$, we develop a novel decomposition argument and show that instead of analyzing the (intractable) marginal distribution of $s_t$ due to adaptivity of $\pi$, we can analyze the conditional distribution of $s_t$ given past private observations $w_{<t}$, denoted by $dP_{\bm{v}, t}^\pi (s_t|w_{<t})/ds_t$. Importantly, leveraging the two-part structure of the LDP pricing policy $\pi$, as described in \eqref{eq:pricing} and \eqref{eq:privacy} of the main text, we show that $dP_{\bm{v}, t}^\pi (s_t|w_{<t})/ds_t$ is in fact tractable. Denote $\varphi_{\bm{v},i}(s_t|w_{<t})$ as the score function defined as in Acharya-Assumption 1 above but for $dP_{\bm{v}, t}^\pi (s_t|w_{<t})/ds_t$, we show in \Cref{lem_varphi} that $\varphi_{\bm{v},i}(s_t|w_{<t})$ still exhibits nice structures. In particular, $\varphi_{\bm{v},i}(s_t|w_{<t})$ is a martingale difference and the orthogonality condition in Acharya-Assumption 2 above still holds for $\varphi_{\bm{v},i}(s_t|w_{<t})$. This is indeed intuitive as by our lower bound construction in Step 0, the instant information gain at time $t$ stems solely from the current active dimension $i$, with no contribution from other inactive dimensions.

Second, Acharya-Assumption 2 imposes a constant bound $\alpha$ on the variation of the score function $\varphi_{\bm{v},i}$, which is sufficient for lower bound construction for statistical estimation. However, this is not sufficient for our purposes as we need to establish a bound on the TV that is a function of the regret of $\pi$ as is in \eqref{eq:cost_of_privacy_equation}. We address this by deriving a new upper bound on the variance of $\varphi_{\bm{v},i}$. In particular, in the proof of \Cref{lem_varphi}, we show its variance can be further upper bounded by a quadratic term in both the policy-induced price and the optimal price, which subsequently leads to the derivation of the instant regret due to \eqref{regret_case1}.

}

\subsection{Lower bound construction}\label{sec-lb-construction}

At a high level, we construct a set of $2^d$ revenue functions, indexed by the hyper parameter $\bm{v}=(v_1,\cdots,v_d)^{\top}\in\{0,1\}^d$ such that distinguishing  $v_i=0$ from  $v_i=1$ for each coordinate can only depend on the information from the $i$th covariate. This
implies that learning the parameters associated with the $i$th covariate can only make use of the data sample of effective sample size of order $O(T/d)$. In addition, by carefully crafting the model parameters, the revenue functions in the set are difficult to distinguish. Therefore, it is very costly to differentiate these functions,  which will inevitably incur large regret.

{\bf Lower bound construction}
Let $z_t=(\mathbb{I}(M_t=1),\ldots,\mathbb{I}(M_t=d))^{\top}\in\{0,1\}^d$, where $M_t$ is a multinomial distribution on $\{1,\ldots,d\}$ with equal marginal probability.  Consider the following demand model indexed by the hyper parameter $\bm{v}=(v_1,\ldots,v_d)^{\top}\in\{0,1\}^d$, such that
\begin{flalign}\label{lb_setting}
   y_t=\mathbb{E}_{\bm{v}}[y_t|z_t,p_t]+\varepsilon_t \,\, \mbox{and} \,\,   \mathbb{E}_{\bm{v}}[y_t|z_t,p_t]:=\lambda_{\bm{v}}(z_t,p_t)= 2-p_t+\Delta\sum_{i=1}^d v_i(1-p_t) \mathbb{I}(M_t=i),
\end{flalign}
where $\{\varepsilon_t\}_{t=1}^T$ is a sequence of independent $\mathcal{N}(0,1)$ random variables, and  $\Delta\in[0,1/2]$ is a small quantity to be defined later. 

Therefore, the instant revenue function (conditional on $z_t,p_t$) is 
$$
r(p_t,z_t,\theta^*)= p_t\lambda_{\bm{v}}(z_t,p_t).
$$
This further implies that  conditional on  $\{M_t=i\}$ where $i\in\{1,\ldots,d\}$, we have that the optimal price takes the form
\begin{equation}\label{opt_price}
    p^*(z_t) = \begin{cases}
        1, & \text{if } v_i=0; \\
  \frac{2+\Delta}{2+2\Delta}, & \text{if } v_i=1.
\end{cases}
\end{equation}
For $M_t=i$, the instant regret associated with the policy $\pi$ is given by \begin{flalign}\label{regret_case}
    \begin{split}
        \mathrm{regret}_t^\pi = \begin{cases}
  (p_t^\pi-1)^2, & \text{if } v_i=0; \\
  (1+\Delta)(p_t^\pi-\frac{2+\Delta}{2+2\Delta})^2& \text{if } v_i=1.
\end{cases}
    \end{split}
\end{flalign}

Define $S_p=\{p_t: |p_t-1|\leq \Delta/6\}.$ For $\Delta\leq 1/2$, it is clear that combined with \eqref{regret_case}, we have for $M_t=i$, it holds that
\begin{equation}\label{bound_regret_t}
    \mathrm{regret}_t^\pi \geq \frac{\Delta^2}{36}, \quad \text{if } \begin{cases}
  &  p_t^\pi\in S_p^c \text{ and } v_i=0; \\
&  p_t^\pi\in S_p \text{ and } v_i=1.
\end{cases}
\end{equation}
The property in \eqref{bound_regret_t} plays an important role in the subsequent analysis.

\Cref{prop:ldp_in_assumption} states that our lower bound instances constructed later in the proof of Theorems~\ref{thm_lb} and \ref{thm:LDP_lb} satisfy the corresponding assumptions in the upper bound results.
\begin{pro}\label{prop:ldp_in_assumption}
    The lower bound construction in Theorems~\ref{thm_lb} and \ref{thm:LDP_lb} satisfies  \Cref{assum_feature} and \ref{assum:etc}. If {$T\geq c \epsilon^{-2}d^4$ }for some constant $c>0$, then  the parameter space in \Cref{thm:LDP_lb}  satisfies  \Cref{assum_ldp}.
\end{pro} 
\begin{proof}
{For Assumption \ref{assum_feature}:}  Clearly, we have $\|z_t\|=1$ and $\Sigma_z=d^{-1}I_d$,   hence (a) and (b) are satisfied. 
Since $\alpha^*=2+\Delta \bm{v}$,  $\beta^*=1+\Delta \bm{v}$, and $\|\bm{v}\|_{\infty}\leq 1,$ we have  $|z_t^{\top}\beta^*|\leq |z_t^{\top}\alpha^*|\leq 2+\Delta\leq 3.$  This implies that (c) is satisfied.  The verification for (d) and (e) is trivial for linear model with Gaussian error.

{For Assumption \ref{assum:etc}:} 
We next show that it holds for $\varsigma_o=1/16$ and $[l,u]= [1/3,3/2].$
Due to the linearity of the demand function, we have that for any $\theta$, $r(p,z_t,\theta)$ is uniquely maximized at  $p_t^*(z_t,\theta)=z_t^{\top}\alpha/z_t^{\top}\beta$.  It suffices to show that for $\|\theta-\theta^*\|\leq  \varsigma_o$, we have $p_t^*(z_t,\theta)\in [1/3,3/2]$.  Without loss of generality, we can assume $M_t=i$, and hence $p_t^*(z_t,\theta)=\alpha_i/\beta_i$.  We thus have $$|p_t^*(z_t,\theta)-p_t^*(z_t,\theta^*)|\leq \beta_i^{-1}|\alpha_i^*-\alpha_i|+[(\beta_i^*)^{-1}-\beta_i^{-1}]\alpha_i^*\leq (\beta_i^*-\varsigma_o)^{-1}\varsigma_o+\alpha_i^*(\beta_i^*-\varsigma_o)^{-2}\varsigma_o,$$
where the last inequality holds by $|\alpha_i-\alpha_i^*|,|\beta_i-\beta_i^*|\leq \varsigma_o.$ Using $\alpha_i^*\leq 3,$ and $ \beta_i^*\geq 1$, we have $|p_t^*(z_t,\theta)-p_t^*(z_t,\theta^*)|\leq \varsigma_o/(1-\varsigma_o)+3\varsigma_o/(1-\varsigma_o)^2=(4\varsigma_o-3\varsigma_o^2)/(1-\varsigma_o)^2\leq 61/225<1/2.$   Clearly, in view of \eqref{opt_price},  for $\Delta\leq 1/2$, we have $p_t^*(z_t,\theta^*)\in [5/6,1]\subset [1/3,3/2]$.  This further implies that $p_t^*(z_t,\theta)\in[1/3,3/2]$.

{For Assumption \ref{assum_ldp}:} Note that in the proof of \Cref{thm:LDP_lb}, we let $\Delta^2=\sqrt{14}d/[672\epsilon\sqrt{T}]$. Hence, for any pairs of $\bm{v}_1,\bm{v}_2\in\{0,1\}^d$, if $T\geq c\epsilon^{-2}d^4$,  we have $\|\theta_{\bm{v}_1}-\theta_{\bm{v}_2}\|^2\leq 2d\Delta^2 = \sqrt{14}/336 \epsilon^{-1}d^{2}T^{-1/2}\leq \sqrt{14/c}/336$. Hence,  letting $c_{\theta}^2=\sqrt{14/c}/336$ is sufficient for (a).  Part (b) is trivial.
\end{proof}

\subsection{Proof of \Cref{thm_lb}}

For each $i \in [d]$, we first provide a high probability concentration bounds on $\sum_{t=1}^T\mathbb{I}(M_t=i)$, which implies the effective sample size for each coordinate of $z_t$ is larger than $T/(2d)$.  In particular, by Hoeffding's inequality,  we have 
\begin{equation}\label{large_pro2}
    \mathbb{P}\left(\sum_{t=1}^T \mathbb{I}\left\{M_t=i\right\}\geq T/d-\sqrt{T\log (T)} \right)\geq 1-T^{-2}.
\end{equation}
We define $\mathcal{A}_i$ as the above high probability event, i.e.
\[
  \mathcal{A}_i = \left\{\sum_{t=1}^T \mathbb{I}\left\{M_t=i\right\}\geq T/d-\sqrt{T\log T} \right\},  
\]
such that $\mathbb{P}(\mathcal{A}_i)\geq 1-T^{-2}$.
The subsequent analysis is conducted on the event $\mathcal{A}=\cap_{i=1}^d \mathcal{A}_i$, which,  by a union bound argument, holds with probability larger than $1-dT^{-2}$.

Next, we lower bound the regret by the classification error, which is typically known as the Assouad's method.  We construct a classifier using data $s_t=\{z_t,p_t,y_t\}$ such that  $\psi:(s_1,\cdots,s_T)\to \{\hat{v}_i\}_{i=1}^d \in \{0,1\}^d$. In particular, given a policy $\pi$,  let 
\begin{equation}\label{defn_gamma_zeta}
   {\gamma}_i=\sum_{t=1}^T\mathbb{I}\left\{M_t=i,  p_t^\pi \in S_p\right\} \quad \mbox{and} \quad {\zeta}_i=\sum_{t=1}^T\mathbb{I}\left\{M_t=i,p_t^\pi \in S_p^c\right\}.
\end{equation}
The classifier is given as 
$$
\hat{v}_i:=\left\{\begin{array}{l}
1, \text { if } \gamma_i \leq {\zeta}_i  ; \\
0, \text { if } {\gamma}_i >{\zeta}_i.  
\end{array}\right.
$$
We then show that large classification error will lead to large regret. 
\begin{lem}\label{lem_class1}
For any non-anticipating policy $\pi$, with probability at least $1-dT^{-2}$, we have that
    $$
    \sum_{t=1}^T \mathrm{regret}^{\pi}_t\geq \frac{1}{144}T \Delta^2 d^{-1}   \sum_{i=1}^d \mathbb{I}(\hat{v}_i\neq v_i).
    $$
\end{lem}
\begin{proof}

Suppose $v_i=0$ and $\hat{v}_i=1$. This means $\gamma_i \leq {\zeta}_i$, and hence 
$$
2{\zeta}_i\geq \gamma_i+\zeta_i =\sum_{t=1}^T \mathbb{I}\left\{M_t=i\right\}. 
$$
Therefore, given that $\sqrt{T/\log (T)}\geq 2d$, on $\mathcal{A}_i$ we have  \begin{equation}\label{bound_zeta}
    {\zeta}_i\geq T/(4d).
\end{equation}
Similarly, if $v_i=1$ and $\hat{v}_i=0$, we have \begin{equation}\label{bound_gamma}
    {\gamma}_i\geq T/(4d).
\end{equation}
Hence, if $\hat{v}_i\neq v_i,$ we have either $\gamma_i$ or $\zeta_i$ is at least $T/(4d)$.

Recall we assume $\mathcal{A}=\cap_{i=1}^d\mathcal{A}_i$ holds, which holds with probability larger than $1-dT^{-2}$. Thus, we have that,
\begin{flalign*}
    \sum_{t=1}^T \mathrm{regret}^{\pi}_t\notag\stackrel{(a)}{=}&\sum_{t=1}^T\left[\mathrm{regret}^{\pi}_t \sum_{i=1}^d \mathbb{I}(M_t=i)\right]\\\geq&\sum_{t=1}^T\left[\mathrm{regret}^{\pi}_t \sum_{i=1}^d \mathbb{I}(M_t=i,\hat{v}_i\neq v_i)\right] \\=& \sum_{t=1}^T \sum_{i=1}^d \mathrm{regret}^{\pi}_t\mathbb{I}(M_t=i,v_i=0,\hat{v}_i=1)+\sum_{t=1}^T \sum_{i=1}^d \mathrm{regret}^{\pi}_t\mathbb{I}(M_t=i,v_i=1,\hat{v}_i=0)
   \\
   \geq  &\sum_{t=1}^T \sum_{i=1}^d \mathrm{regret}^{\pi}_t\mathbb{I}(M_t=i, p_t^\pi\in S_p^c, v_i=0,\hat{v}_i=1)\\&+\sum_{t=1}^T \sum_{i=1}^d \mathrm{regret}^{\pi}_t\mathbb{I}(M_t=i, p_t^\pi\in S_p, v_i=1,\hat{v}_i=0)\\
   \notag\stackrel{(b)}{\geq} & \sum_{t=1}^T \sum_{i=1}^d \frac{\Delta^2}{36}[\mathbb{I}(M_t=i, p_t^\pi\in S_p^c, v_i=0,\hat{v}_i=1)+\mathbb{I}(M_t=i, p_t^\pi\in S_p, v_i=1,\hat{v}_i=0)]\\
  \notag\stackrel{(c)}{\geq}& \frac{\Delta^2}{36} \frac{T}{4d}\sum_{i=1}^d\mathbb{I}(\hat{v_i}\neq v_i),
\end{flalign*}
where (a) holds by noting $1=\sum_{i=1}^d \mathbb{I}(M_t=i)$, and (b) by \eqref{bound_regret_t}, and (c) by \eqref{bound_zeta} and  \eqref{bound_gamma}.   
\end{proof}

\begin{lem}\label{lem_assouad1}
For any  non-anticipating policy $\pi$ and classifier $\psi$,  we have 
    $$
   \sup _{\boldsymbol{v}} \mathbb{E}_{\boldsymbol{v}}^{\pi, \psi}\left[\sum_{i=1}^d \mathbb{I}(\hat{v}_i\neq v_i)\right]\geq \frac{d}{2} \left(1-\sqrt{\frac{\Delta^2}{4d}\sum_{t=1}^T \mathbb{E}_{\bar{\boldsymbol{v}}}^{\pi}[p_t-1]^2}\right),
    $$ 
    where $\bar{\boldsymbol{v}}$ denotes the mixture distribution of all possible $\boldsymbol{v}$, i.e.\ $P_{\bar{\boldsymbol{v}}}^\pi:= \sum_{\boldsymbol{v}}P_{\boldsymbol{v}}^\pi/2^d.$    
\end{lem}

\begin{proof}
Define the symmetric Kullback--Leibler
(KL) divergence between probability measure $P$ and $Q$ as $D_{\mathrm{KL}}^{\mathrm{sy}}(P,Q)=D_{\mathrm{KL}}(P|Q)+D_{\mathrm{KL}}(Q|P)$. Define $$
P_{+i}^\pi:=\frac{1}{2^{d-1}} \sum_{\boldsymbol{v}: {v}_i=1} P_{\boldsymbol{v}}^\pi, \quad P_{-i}^\pi:=\frac{1}{2^{d-1}} \sum_{\boldsymbol{v}: v_i=0} P_{\boldsymbol{v}}^\pi,
$$
with $P_{\bm{v}}^\pi$ being the joint distribution of $\{s_1,\cdots,s_T\}$ under model $\boldsymbol{v}$ and policy $\pi$. 
Then, we have 
\begin{flalign*}
    \begin{split}  
  \sup _{\boldsymbol{v} } \mathbb{E}_{\boldsymbol{v}}^{\pi, \psi}\left[\sum_{i=1}^d \mathbb{I}(\hat{v}_i\neq v_i)\right]\geq\frac{1}{2} \sum_{i=1}^d (1- \mathrm{TV}(P_{+i}^{\pi}, P_{-i}^{\pi}))\geq&\frac{1}{2} \sum_{i=1}^d \left(1-\sqrt{\frac{1}{4}D_{\mathrm{KL}}^{\mathrm{sy}}(P_{+i}^{\pi}, P_{-i}^{\pi})}\right)\\\geq&
  \frac{d}{2}  \left(1-\sqrt{\frac{1}{4d}\sum_{i=1}^d   D_{\mathrm{KL}}^{\mathrm{sy}}(P_{+i}^{\pi}, P_{-i}^{\pi})}\right), 
    \end{split}
\end{flalign*}
where the first inequality holds by Assouad's Lemma, the second by Pinsker's inequality, and the last by the Cauchy--Schwarz inequality.

We now explicitly compute $\sum_{i=1}^d D_{\mathrm{KL}}^{\mathrm{sy}}(P_{+i}^{\pi}, P_{-i}^{\pi})=\sum_{i=1}^d \{D_{\mathrm{KL}}(P_{+i}^{\pi}| P_{-i}^{\pi}) + D_{\mathrm{KL}}(P_{-i}^{\pi}| P_{+i}^{\pi})\}.$ We first analyze the term $D_{\mathrm{KL}}(P_{+i}^{\pi}| P_{-i}^{\pi})$ in detail. In particular, using the chain-rule of the KL divergence, we have that
\begin{flalign*}
  & D_{\mathrm{KL}}(P_{+i}^{\pi}| P_{-i}^{\pi})=\sum_{t=1}^T \mathbb{E}_{+i}^\pi  \left\{D_{\mathrm{KL}}(P_{+i, t}^\pi\left(\cdot \mid s_{<t}\right)| P_{-i, t}^\pi\left(\cdot \mid s_{<t}\right))\right\}.
\end{flalign*}
Furthermore, we have that  
\begin{align}
    & D_{\mathrm{KL}}(P_{+i, t}^\pi\left(\cdot \mid s_{<t}\right)| P_{-i, t}^\pi\left(\cdot \mid s_{<t}\right)) \nonumber \\
    = & \mathbb E_{+i}^\pi  \left\{\log\left(\frac{f_{+i}^\pi(y_t,p_t,z_t|s_{<t})}{f_{-i}^\pi(y_t,p_t,z_t|s_{<t})} \right)\Bigg \vert s_{<t} \right\} =\mathbb E_{+i}^\pi  \left\{ \log\left(\frac{f_{+i}^\pi(y_t|p_t,z_t)}{f_{-i}^\pi(y_t|p_t,z_t)}\right) \Bigg \vert s_{<t} \right\}. \label{eq-KL-convex-1}
\end{align}
Note that if $M_t \neq i$, then $f_{+i}^\pi(y_t|p_t,z_t) = f_{-i}^\pi(y_t|p_t,z_t)$ and consequently 
\begin{equation}\label{eq-KL-convex-2}
    D_{\mathrm{KL}}(P_{+i, t}^\pi(\cdot \mid s_{<t})| P_{-i, t}^\pi\left(\cdot \mid s_{<t}\right)) = 0.
\end{equation}
If $M_t = i$, then for any $v$ with $v_i = 1$, we have that
\begin{equation}\label{eq-KL-convex-3}
    \frac{f_{+i}^\pi(y_t|p_t,z_t)}{f_{-i}^\pi(y_t|p_t,z_t)} = \exp\left\{-\frac{\Delta^2}{2}(1 - p_t)^2 + (y_t - 2 + p_t)\Delta (1 - p_t)\right\}.
\end{equation}
Combining \eqref{eq-KL-convex-1}, \eqref{eq-KL-convex-2} and \eqref{eq-KL-convex-3}, we have that
\[
D_{\mathrm{KL}}(P_{+i, t}^\pi\left(\cdot \mid s_{<t}\right)| P_{-i, t}^\pi\left(\cdot \mid s_{<t}\right))
=\frac{\Delta^2}{2} \mathbb E_{+i}^\pi \left\{ (p_t-1)^2 \cdot \mathbb I(M_t=i)\big \vert s_{<t}\right\},
\]
where the last equality follows by noting that $\lambda_{+i}(z_t,p_t)\neq \lambda_{-i}(z_t,p_t)$ if and only if $M_t=i$. Therefore, we have that
\begin{align*}
    D_{\mathrm{KL}}(P_{+i}^{\pi}| P_{-i}^{\pi})=\frac{\Delta^2}{2} \sum_{t=1}^T \mathbb{E}_{+i}^\pi\left\{ (p_t-1)^2 \cdot \mathbb I(M_t=i)\right\}.
\end{align*}
Similarly, we have that 
\begin{flalign*}
    D_{\mathrm{KL}}(P_{-i}^{\pi}|P_{+i}^{\pi})=\frac{\Delta^2}{2} \sum_{t=1}^T \mathbb{E}_{-i}^\pi\left\{ (p_t-1)^2 \cdot \mathbb I(M_t=i)\right\}.
\end{flalign*}
Putting everything together, we have that
\begin{align*}
&\sum_{i=1}^d D_{\mathrm{KL}}^{\mathrm{sy}}(P_{+i}^{\pi}, P_{-i}^{\pi})=\sum_{i=1}^d \{D_{\mathrm{KL}}(P_{+i}^{\pi}| P_{-i}^{\pi}) + D_{\mathrm{KL}}(P_{-i}^{\pi}| P_{+i}^{\pi})\}\\
=&\frac{\Delta^2}{2}\sum_{t=1}^T \sum_{i=1}^d \left[ \mathbb E_{+i}^\pi \left\{ (p_t-1)^2 \cdot \mathbb I(M_t=i)\right\} + \mathbb E_{-i}^\pi \left\{ (p_t-1)^2 \cdot \mathbb I(M_t=i)\right\}\right]\\
=&\frac{\Delta^2}{2^d}\sum_{t=1}^T \sum_{i=1}^d \left[\sum_{\boldsymbol{v}:v_i=1} \mathbb E_{\boldsymbol{v}}^\pi \left\{ (p_t-1)^2 \cdot \mathbb I(M_t=i)\right\} + \sum_{\boldsymbol{v}:v_i=0}\mathbb E_{\boldsymbol{v}}^\pi \left\{ (p_t-1)^2 \cdot \mathbb I(M_t=i)\right\}\right]\\
=&\frac{\Delta^2}{2^d}\sum_{\boldsymbol{v}}\sum_{t=1}^T\mathbb E_{\boldsymbol{v}}^\pi \left\{ (p_t-1)^2\right\} = \Delta^2\sum_{t=1}^T\mathbb E_{\bar{\boldsymbol{v}}}^\pi \left\{ (p_t-1)^2\right\},
\end{align*}
which finishes the proof.
\end{proof}

\begin{lem}\label{lem:hardthres}
   For any time $t,$ denote $p_t^*$ as the optimal price and $p_t^\pi$ as the policy price. For any $t$ and $p_t^\pi$, we have that
   \begin{align*}
       (p_t^\pi-p_t^*)^2 \geq \frac{1}{2}(p_t^\pi-1)^2 \cdot \mathbb I(|p_t^\pi-1|\geq \sqrt{3}\Delta).
   \end{align*}
\end{lem}

\begin{proof}[Proof of \Cref{lem:hardthres}]
    By design, we know that $p_t^*$ can be 1 or $(2+\Delta)/(2+2\Delta).$ The lemma holds when $p_t^*=1$. Therefore, we only need to consider the case where $p_t^*=(2+\Delta)/(2+2\Delta)$. In particular, we have
    \begin{align*}
        & (p_t^\pi-p_t^*)^2=\left(p_t^\pi-\frac{2+\Delta}{2+2\Delta} \right)^2=\left(p_t^\pi-1+\frac{\Delta}{2+2\Delta} \right)^2\\
    \geq & \frac{3}{4}(p_t^\pi-1)^2 - \frac{3\Delta^2}{(2+2\Delta)^2}\geq  \frac{3}{4}(p_t^\pi-1)^2 -\frac{3}{4} \Delta^2 \geq \frac{1}{2}(p_t^\pi-1)^2 \cdot \mathbb I(|p_t^\pi-1|\geq \sqrt{3}\Delta),
    \end{align*}
    which finishes the proof.
\end{proof}

\begin{proof}[Proof of Theorem \ref{thm_lb}]
Recall that $\mathcal{A}=\cap_{i=1}^d \mathcal{A}_i$ such that $\mathbb{P}(\mathcal{A})\geq 1-dT^{-2}$.  By Lemma \ref{lem_class1}, we have that,
 \begin{flalign*}
    \sup _{\boldsymbol{v}} \mathbb{E}_{\bm{v}}^{\pi}\sum_{t=1}^T \mathrm{regret}_t\geq&  \sup _{\boldsymbol{v}} \mathbb{E}_{\boldsymbol{v}}^{\pi}\sum_{t=1}^T \mathrm{regret}_t \mathbb{I}(\mathcal{A}) \\ \geq& \frac{1}{144}T \Delta^2 d^{-1}  \sup _{\boldsymbol{v}} \mathbb{E}_{\boldsymbol{v}}^{\pi, \psi}[\sum_{i=1}^d \mathbb{I}(\hat{v}_i\neq v_i,\mathcal{A})] \\\geq& \frac{1}{144}T \Delta^2 d^{-1}\left\{   \sup _{\boldsymbol{v}} \mathbb{E}_{\boldsymbol{v}}^{\pi, \psi}[\sum_{i=1}^d \mathbb{I}(\hat{v}_i\neq v_i)]-d^2T^{-2}\right\},
 \end{flalign*}
where the last inequality holds by noting that $\sum_{i=1}^d \mathbb{I}(\hat{v}_i\neq v_i,\mathcal{A})=\sum_{i=1}^d \mathbb{I}(\hat{v}_i\neq v_i)-\sum_{i=1}^d \mathbb{I}(\hat{v}_i\neq v_i,\mathcal{A}^c)$, and that $\sum_{i=1}^d \mathbb{I}(\hat{v}_i\neq v_i,\mathcal{A}^c)\leq d \mathbb{I}(\mathcal{A}^c)$.

By Lemma \ref{lem_assouad1}, we further have that
\begin{align}\label{eq:regret_case1}
      \sup _{\boldsymbol{v}} \mathbb{E}_{\bm{v}}^{\pi}\sum_{t=1}^T \mathrm{regret}_t\geq& \frac{T\Delta^2}{144 } \left\{\frac{1}{2}  \left(1-\sqrt{\frac{\Delta^2}{4d}\sum_{t=1}^T \mathbb{E}_{\bar{\boldsymbol{v}}}^{\pi}[p_t-1]^2}\right)-dT^{-2}\right\}.
\end{align}

On the other hand, we have that
\begin{align}\label{eq:regret_case2}
    &\sup _{\boldsymbol{v}} \mathbb{E}_{\bm{v}}^{\pi}\sum_{t=1}^T \mathrm{regret}_t
   \geq \sup _{\boldsymbol{v}} \mathbb{E}_{\bm{v}}^{\pi}\sum_{t=1}^T (p_t-p_t^*)^2\nonumber\\
   \geq & \frac{1}{2} \sup_{\boldsymbol{v}} \mathbb{E}_{\bm{v}}^{\pi}\sum_{t=1}^T (p_t-1)^2 \cdot \mathbb I(|p_t-1|\geq \sqrt{3}\Delta)\nonumber\\
   \geq & \frac{1}{2} \mathbb{E}_{\bar{\bm{v}}}^{\pi}\sum_{t=1}^T (p_t-1)^2 \cdot \mathbb I(|p_t-1|\geq \sqrt{3}\Delta)
   \geq   \frac{1}{2} \sum_{t=1}^T\mathbb{E}_{\bar{\bm{v}}}^{\pi} [p_t-1]^2 - \frac{3}{2}\Delta^2 T, 
\end{align}
where the first inequality follows from \eqref{regret_case} and the second inequality follows from \Cref{lem:hardthres}.

Now, set $\Delta^2=c\sqrt{d/T}$, where $c=1/3$. Denote $\eta= \sum_{t=1}^T\mathbb{E}_{\bar{\bm{v}}}^{\pi} [p_t-1]^2.$ Suppose $\eta\geq 6\Delta^2 T$, by \eqref{eq:regret_case2}, we have that
$$\sup _{\boldsymbol{v}} \mathbb{E}_{\bm{v}}^{\pi}\sum_{t=1}^T \mathrm{regret}_t\geq 3/2 \Delta^2 T=3c/2\sqrt{dT}.$$ Suppose $\eta< 6\Delta^2 T$, by \eqref{eq:regret_case1}, we have that
\begin{align*}
    \sup _{\boldsymbol{v}} \mathbb{E}_{\bm{v}}^{\pi}\sum_{t=1}^T \mathrm{regret}_t&\geq \frac{T\Delta^2}{144 } \left\{\frac{1}{2}  \left(1-\sqrt{\frac{\Delta^2}{4d}6\Delta^2 T}\right)-dT^{-2}\right\}\\
    &= \frac{c\sqrt{dT}}{144 } \left\{\frac{1}{2}  \left(1-\sqrt{3c^2/2}\right)-dT^{-2}\right\}\geq c\sqrt{dT}/576,
\end{align*}
where the last inequality follows from that $c=1/3$. This finishes the proof.
\end{proof}

\subsection{Proof of \Cref{thm:LDP_lb}}

The following lemma is a privatized version of Lemma \ref{lem_class1}, which lower bounds the regret by a classification error. In particular, it shows that the regret of any $\epsilon$-LDP policy $\pi$ can be lower bounded by the classification error of an associated $2\epsilon$-LDP classifier. 

\begin{lem}\label{lem_class}
Suppose  $\sqrt{T}\geq 34d\epsilon^{-1}\sqrt{\log (T)}$.     For any non-anticipating $\epsilon$-LDP policy $\pi$, there exists a 2$\epsilon$-LDP classifier $\psi:(w_1,\cdots,w_T)\to \hat{\bm{v}}\in \{0,1\}^d$, such that with probability at least $1-5dT^{-2}$, we have 
    $$
    \sum_{t=1}^T \mathrm{regret}^{\pi}_t\geq \frac{1}{144}T \Delta^2 d^{-1 }\sum_{i=1}^d \mathbb{I}(\hat{v}_i\neq v_i).
    $$
\end{lem}
\begin{proof}

For clarity, we structure the proof into three components.

\textsc{[A High probability event $\mathcal{A}^\epsilon$]}: We construct i.i.d.\ Laplace random variables $\{\eta_{ti}\}_{t=1, i = 1}^{T, d}$ such that $\eta_{t i}\sim \operatorname{Lap}(2 / \epsilon)$. Let $\{\eta_{ti}'\}_{t=1, i = 1}^{T, d}$ be an independent copy of $\{\eta_{ti}\}_{t=1, i = 1}^{T, d}$.  

For each $i \in [d]$, by a concentration inequality for Laplace random variables \citep[e.g.~Corollary 2.9 in][]{chan2011private}, for $\sqrt{T}>\sqrt{2\log (T)}$,  we have
\[
    \mathbb{P}\left(\max \left\{\left|\sum_{t=1}^T\eta_{ti}\right|,\left|\sum_{t=1}^T\eta_{ti}'\right|\right\} \leq 8 \epsilon^{-1} \sqrt{T \log (T)} \right)\geq 1-4T^{-2}.
\]
Furthermore, by Hoeffding's inequality in \eqref{large_pro2},  we have 
\begin{equation*}
    \mathbb{P}\left(\sum_{t=1}^T \mathbb{I}\left\{M_t=i\right\}\geq T/d-\sqrt{T\log (T)} \right)\geq 1-T^{-2}.
\end{equation*}
We define $\mathcal{A}_i^{\epsilon}$ as the intersection of the above high probability events, i.e. 
\[
  \mathcal{A}_i^{\epsilon} =\left\{\max \left\{\left|\sum_{t=1}^T\eta_{ti}\right|,\left|\sum_{t=1}^T\eta_{ti}'\right|\right\} \leq 8 \epsilon^{-1} \sqrt{T \log (T)} \right\}\bigcap \left\{\sum_{t=1}^T \mathbb{I}\left\{M_t=i\right\}\geq T/d-\sqrt{T\log (T)} \right\}.
\]
Therefore, we have that $\mathbb{P}(\mathcal{A}_i^{\epsilon})\geq 1-5T^{-2}$. The following analysis is conducted on the event $\mathcal{A}^{\epsilon}=\bigcap_{i=1}^d \mathcal{A}_i^{\epsilon}$, which,  by a union bound argument, holds with probability larger than $1-5dT^{-2}$.

\textsc{[Design a $2\epsilon$-LDP classifier]}: Note that we cannot directly use the classifier defined in the non-private setting as it utilizes non-privatized personal information in \eqref{defn_gamma_zeta}. Instead, we adopt the construction in Lemma 3 of \cite{chen2022differential} and augment the $\epsilon$-LDP policy $\pi$ to construct a $2\epsilon$-LDP policy $\pi'$.  In particular, we have to  additionally privatize the information used for classification, which necessitates elevating the regret analysis from an $\epsilon$-LDP policy to a $2\epsilon$-LDP policy. The resulting lower bound is thus inflated by a multiplicative constant. 

Recall that $w_t$ is the intermediate quantities produced by policy $\pi$, such that the distribution
of $w_t$ is measurable conditional on $s_t$ and $w_1,\cdots,w_{t-1}$. We construct augmented intermediate quantity $w_t'=(w_t,\{\gamma_{ti}\}_{i=1}^d,\{\zeta_{ti}\}_{i=1}^d)$
such that for $i\in\{1,\cdots,d\}$,
 $$
\begin{array}{rlrl}
\gamma_{t i} & =\mathbb{I}\left\{M_t=i,  p_t^{\pi}\in S_p\right\}+\eta_{t i}, &  \eta_{t i} \stackrel{i . i . d .}{\sim} \operatorname{Lap}(2 / \epsilon), \\
\zeta_{t i} & =\mathbb{I}\left\{M_t=i, p_t^\pi \in S_p^c \right\}+\eta_{t i}^{\prime}, & \eta_{t i}^{\prime} \stackrel{i . i . d .}{\sim} \operatorname{Lap}(2 / \epsilon).
\end{array}
$$
Note that $w_t'$ satisfies $2 \epsilon$-LDP thanks to the Laplace mechanism and simple composition of two $\epsilon$-LDP procedures, see, e.g.~\cite{dwork2014algorithmic}. This indicates that $\pi^{\prime}$ is a $2 \epsilon$-LDP policy. Furthermore, note that by construction, the policy $\pi^{\prime}$ has the exact same regret as $\pi$.

Define $\hat{\gamma}_i=\sum_{t=1}^T\gamma_{ti}$ and $\hat{\zeta}_i=\sum_{t=1}^T\zeta_{ti}$. It is clear that, $\hat{\gamma}_i-\gamma_i=\sum_{t=1}^T\eta_{ti}$ and $\hat{\zeta}_i-\zeta_i=\sum_{t=1}^T\eta_{ti}',$ where recall $\gamma_i, \zeta_i$ are defined in \eqref{defn_gamma_zeta}. Therefore, under $\mathcal{A}_i^\epsilon$, we have 
  \begin{equation}\label{large_pro1'}
         \max \left\{\left| \hat{\gamma}_i-\gamma_i\right|,\left|\hat{\zeta}_i-\zeta_i\right|\right\} \leq 8 \epsilon^{-1} \sqrt{T \log (T)}.  
  \end{equation}
We then construct the classifier $\psi:(w_1',\cdots,w_T')\to \{\hat{v}_i\}_{i=1}^d \in \{0,1\}^d$ such that
$$
\hat{v}_i:=\left\{\begin{array}{l}
1, \text { if }\hat{\gamma}_i \leq \hat{\zeta}_i  ; \\
0, \text { if } \hat{\gamma}_i >\hat{\zeta}_i.  
\end{array}\right.
$$
Note that by construction, $\psi$ is a $2\epsilon$-LDP classifier.

\textsc{[Regret lower bound by classification error]}: We now analyze the consequence of mis-classification. Suppose $v_i=0$ and $\hat{v}_i=1$. This means $\hat{\gamma}_i \leq \hat{\zeta}_i$, and hence 
$$
2\zeta_i+(\hat{\zeta}_i-\zeta_i)+(\gamma_i-\hat{\gamma}_i)=2\hat{\zeta}_i +\gamma_i-\hat{\gamma}_i+\zeta_i-\hat{\zeta}_i \geq \gamma_i+\zeta_i =\sum_{t=1}^T \mathbb{I}\left\{M_t=i\right\}. 
$$
This and \eqref{large_pro1'} further imply that 
\begin{equation}\label{large_pro3}
\zeta_i=\sum_{t=1}^T\mathbb{I}\left\{M_t=i, p_t \in S_p^c\right\} \geq \frac{1}{2}\left(\sum_{t=1}^T \mathbb{I}\left\{M_t=i\right\}\right)-8\epsilon^{-1} \sqrt{T \log (T)}.
\end{equation}
Recall on $\mathcal{A}_i^\epsilon$, we have $\sum_{t=1}^T \mathbb{I}\left\{M_t=i\right\}\geq T/d-\sqrt{T\log (T)}$. Hence, for $\epsilon\in (0,1),$ and  $\sqrt{T}\geq 34d\epsilon^{-1}\sqrt{\log (T)}$,  by \eqref{large_pro3}, we have 
\[
      \zeta_i=\sum_{t=1}^T\mathbb{I}\left\{M_t=i, p_t \in S_p^c\right\}\geq T/(4d).
\]
Similarly, if $v_i=1$ and $\hat{v}_i=0$, we have 
\[
      \gamma_i=\sum_{t=1}^T\mathbb{I}\left\{M_t=i, p_t \in S_p\right\}\geq T/(4d).
\]
Hence, if $\hat{v}_i\neq v_i$, we have either $\gamma_i$ or $\zeta_i$ is at least $ T/(4d)$.

By similar arguments as the proof of Lemma \ref{lem_class1} (after equation \eqref{bound_gamma}), we have that $$
 \sum_{t=1}^T \mathrm{regret}^{\pi}_t\geq \frac{\Delta^2}{36} \frac{T}{4d}\sum_{i=1}^d\mathbb{I}(\hat{v_i}\neq v_i).
$$
\end{proof}

Similar to \Cref{lem_assouad1}, we need to lower bound the  classification error under LDP by  the total variation distance between distributions due to Assouad's Lemma. However, it turns out that the proof technique  in \Cref{lem_assouad1} using KL divergence to bound total variation distance will yield a less sharp result in terms of the dimension $d$. Instead, we establish a sharp bound using the Hellinger distance. In fact, for two arbitrary distributions $P_1$ and $P_2$, we have 
$$
\{\mathrm{TV}(P_1, P_2)\}^2\leq 2 \mathrm{H}^2(P_1,P_2)\leq D_{\mathrm{KL}}(P_1,P_2),
$$
where $\mathrm{H}^2(P_1,P_2)=\frac{1}{2}\int_{\mathcal{X}}[\sqrt{P_1(\mathrm{d}x)-P_2(\mathrm{d}x)}]^2$ is the Hellinger distance between $P_1$ and $P_2$ on the support $\mathcal{X}.$
The above inequality implies that using the  Hellinger distance allows for more delicate analysis than the KL divergence.

We first list necessary notation.  Define 
$$
M_{+i}^\pi:=\frac{1}{2^{d-1}} \sum_{\boldsymbol{v}: {v}_i=1} M_{\boldsymbol{v}}^\pi, \quad M_{-i}^\pi:=\frac{1}{2^{d-1}} \sum_{\boldsymbol{v}: v_i=0} M_{\boldsymbol{v}}^\pi,
$$
with $M_{\bm{v}}^\pi$ being the distribution of $\{w_1,\cdots,w_T\}$ under model $\boldsymbol{v}$ and policy $\pi$ (in fact, the $2\epsilon$-LDP policy $\pi'$ to be precise).  More specifically, we have
 $$
M_{\bm{v}}^{\pi}(w_1,\cdots,w_T)= \prod_{t=1}^T M_{\bm{v}}^{\pi}(w_t|w_{<t}),
$$
where the marginal (conditional) distribution is $$
M_{\bm{v}}^{\pi}(\cdot|w_{<t})= \int_{s_t\in\mathcal{S}} Q_t(\cdot|s_t,w_{<t}) \mathrm{d} P^{\pi}_{\bm{v},t}(s_t|w_{<t}).
$$
where $\mathcal{S}$ is the support of $s_t$, and  $P^{\pi}_{\bm{v},t}(s_t|w_{<t})$ is the joint distribution of $s_t$ given past privatized information $w_{<t}$ under policy $\pi$ and model parameter $\bm{v}$.  Let the density function of  $M_{\bm{v}}^{\pi}(\cdot|w_{<t})$ be 
\begin{equation}\label{m_density}
    m_{\bm{v}}^{\pi}(\cdot|w_{<t})=\int_{s_t\in\mathcal{S}} q_t(\cdot|s_t,w_{<t}) \mathrm{d} P^{\pi}_{\bm{v},t}(s_t|w_{<t}). 
\end{equation}

For any $\bm v=(v_1,\cdots,v_d)^{\top}\in\{0,1\}^d$, we define $\bm v^{\oplus i}=(v_1, \cdots, v_{i-1},1-v_i,v_{i+1},\cdots,v_d)^{\top}$ by changing the $i$th element in $\bm v$. For  a set $A\subset [T]$, we let 
$$
M_{\bm{v}^{\oplus i},A}^{\pi}(w_1,\cdots,w_T)=\prod_{t=1}^T [M_{\bm{v}}^{\pi}(w_t|w_{<t})\mathbb{I}(t\not\in A)+M_{\bm{v}^{\oplus i}}^{\pi}(w_t|w_{<t})\mathbb{I}(t\in A)],
$$
as the joint distribution function by substituting the $t$th conditional
distribution of $M_{\bm{v}}^{\pi}(w_t|w_{<t})$ by $M_{\bm{v}^{\oplus i}}^{\pi}(w_t|w_{<t})$  for $t\in A.$ Define its density function as $m_{\bm{v}^{\oplus i},A}^{\pi}$,
and write $M_{\bm{v}^{\oplus i},\{t\}}^{\pi}=M_{\bm{v}^{\oplus i},t}^{\pi}$. 
It is also clear that $M_{\bm{v}^{\oplus i},\emptyset}^{\pi}=M_{\bm{v}}^{\pi}$. 

Now, we present some fundamental properties of the Hellinger distance.

\begin{lem}\label{lem_convex}
    Let two mixture distributions  $F$ and $G$ be $$F=aF_0+(1-a)F_1 \quad \mbox{and} \quad G=aG_0+(1-a)G_1,$$ where $a\in(0,1)$, and $F_0,F_1,G_0,G_1$ are distributions with common support $\mathcal{X}$. Then, 
    $$\mathrm{H}^2(F,G)\leq a \mathrm{H}^2(F_0,G_0)+(1-a)\mathrm{H}^2(F_1,G_1).$$
\end{lem}
\begin{proof}
    Let $f_0,f_1,g_0,g_1$ be associated probability density/mass functions of $F_0,F_1,G_0,G_1$.   Similarly we define $f$ and $g$ for $F$ and $G$, respectively.
    Note 
    $
    \mathrm{H}^2(F,G)=1-\int_{\mathcal{X}} \sqrt{f(x)g(x)}dx,
    $
and that 
    $$
    \frac{\partial^2}{\partial a^2}  H^2(F,G) = \int_{\mathcal{X}}\frac{(f_1(x)g_0(x)-f_0(x)g_1(x))^2}{4(f(x)g(x))^{3/2}}dx\geq 0.
    $$
    This implies that $ H^2(F,G)$ is convex in $a.$ The result follows.
\end{proof}

\begin{lem}\label{lem_acharya}
For any fixed $\bm{v}$, and $i\in[d]$, we have that 
$$
\mathrm{H}^2(M_{\bm{v}}^{\pi},M_{\bm v^{\oplus i}}^\pi)\leq 7\sum_{t=1}^T \mathrm{H}^2(M_{\bm{v}}^{\pi},M_{\bm{v}^{\oplus i},t}^{\pi}).
$$
\end{lem}
\begin{proof}
    This is Lemma 4 in \cite{jayram2009}. 
\end{proof}

\begin{lem}\label{lem_ineq_Hellinger}
For any fixed $\bm{v}$ and $t \in [T]$,
    \begin{equation*}
       \sum_{i=1}^d \mathrm{H}^2(M_{\bm{v}}^{\pi},M_{\bm{v}^{\oplus i},t}^{\pi})\leq \mathbb{E}_{\bm v,<t}^{\pi}\left\{  \sum_{i=1}^d \mathrm{H}^2(M_{\bm{v}}^{\pi}(\cdot|w_{<t}),M_{\bm{v}^{\oplus i}}^{\pi}(\cdot|w_{<t}))\right\}.
    \end{equation*}
\end{lem}
\begin{proof}
    \begin{flalign*}
   \notag &2 \sum_{i=1}^d \mathrm{H}^2(M_{\bm{v}}^{\pi},M_{\bm{v}^{\oplus i},t}^{\pi})\\=&\sum_{i=1}^d  \int_{\mathcal{W}^{\otimes T}}  \left[\sqrt{m_{\bm{v}}^{\pi}(w_1,\cdots,w_T)}-\sqrt{m_{\bm{v}^{\oplus i},t}^{\pi}(w_1,\cdots,w_T)}\right]^2 \prod_{j=1}^T\mathrm{d}w_j
    \\\notag =& \sum_{i=1}^d  \int_{\mathcal{W}^{\otimes T}}  \left[\prod_{j\neq t} m_{\bm{v}}^{\pi}(w_j|w_{<j})\right] \left[\sqrt{m_{\bm{v}}^{\pi}(w_t|w_{<t})}-\sqrt{m_{\bm{v}^{\oplus i}}^{\pi}(w_t|w_{<t})}\right]^2
    \prod_{j=1}^T\mathrm{d}w_j
    \\\notag =& \sum_{i=1}^d  \int_{\mathcal{W}^{\otimes T}}  \left[\prod_{j<t} m_{\bm{v}}^{\pi}(w_j|w_{<j})\right] \left[\sqrt{m_{\bm{v}}^{\pi}(w_t|w_{<t})}-\sqrt{m_{\bm{v}^{\oplus i}}^{\pi}(w_t|w_{<t})}\right]^2 \left[\prod_{j>t } m_{\bm{v}}^{\pi}(w_j|w_{<j})\right]
   \prod_{j=1}^T\mathrm{d}w_j
    \\\notag = &   \int_{\mathcal{W}^{\otimes t}}  \left[\prod_{j<t} m_{\bm{v}}^{\pi}(w_j|w_{<j})\right]  \sum_{i=1}^d \left[\sqrt{m_{\bm{v}}^{\pi}(w_t|w_{<t})}-\sqrt{m_{\bm{v}^{\oplus i}}^{\pi}(w_t|w_{<t})}\right]^2
    \prod_{j=1}^t\mathrm{d}w_j
    \\\notag =& \mathbb{E}_{\bm v,<t}^{\pi}\left\{\sum_{i=1}^d  \int_{\mathcal{W}}  \left[\sqrt{m_{\bm{v}}^{\pi}(w_t|w_{<t})}-\sqrt{m_{\bm{v}^{\oplus i}}^{\pi}(w_t|w_{<t})}\right]^2 \mathrm{d}w_t\right\}
    \\=&  2\mathbb{E}_{\bm v,<t}^{\pi}\left\{ \int_{\mathcal{W}}  \sum_{i=1}^d \mathrm{H}^2(M_{\bm{v}}^{\pi}(\cdot|w_{<t}),M_{\bm{v}^{\oplus i}}^{\pi}(\cdot|w_{<t}))\right\},
\end{flalign*}
where the third-to-last inequality holds since for any fixed $w_{\leq t}$,  $$\int_{\mathcal{W}^{\otimes T-t}} \left[\prod_{j>t} m_{\bm{v}}^{\pi}(w_j|w_{<j})\right]\prod_{j=t+1}^T\mathrm{d}w_j=1.$$
\end{proof}

\Cref{lem_assouad2} is the privatized version of \Cref{lem_assouad1}. Note that due to the LDP constraint, compared to \Cref{lem_assouad1}, we have an additional $[\exp(2\epsilon)-1]^2/d$ factor in the right-hand side of the inequality in \Cref{lem_assouad2}, which is the key for getting the additional $\sqrt{d}/\epsilon$ factor in the regret lower bound under LDP in \Cref{thm:LDP_lb}.

\begin{lem}\label{lem_assouad2}
For any $\epsilon$-LDP policy $\pi$ and the associated $2\epsilon$-LDP classifier $\psi$ as in \Cref{lem_class},  suppose $[(1-u)^2\vee (1-l)^2]\Delta^2\leq 1$.  We have that
    $$
   \sup _{\boldsymbol{v}} \mathbb{E}_{\boldsymbol{v}}^{\pi, \psi}\left[\sum_{i=1}^d \mathbb{I}(\hat{v}_i\neq v_i)\right]\geq \frac{d}{2}\left(1-\sqrt{\frac{21\Delta^2[\exp(2\epsilon)-1]^2}{d^2} \sum_{t=1}^T\mathbb{E}_{\bar{\bm v}}^{\pi}[(1-p_t)^2]}\right),
    $$   where $\bar{\boldsymbol{v}}$ denotes the mixture distribution of all possible $\boldsymbol{v}$, i.e.\ $P_{\bar{\boldsymbol{v}}}^\pi:= \sum_{\boldsymbol{v}}P_{\boldsymbol{v}}^\pi/2^d.$
\end{lem}

\begin{proof}[Proof of \Cref{lem_assouad2}]

By standard Assouad's Lemma,  we have 
\begin{equation}\label{ineq_assouad2}
  \sup _{\boldsymbol{v} } \mathbb{E}_{\boldsymbol{v}}^{\pi, \psi}\left[\sum_{i=1}^d \mathbb{I}(\hat{v}_i\neq v_i)\right]\geq  \frac{1}{2}\sum_{i=1}^d \left(1- \mathrm{TV}(M_{+i}^{\pi}, M_{-i}^{\pi})\right).
\end{equation}

[{\bf Step I}] Bound TV distance by sequential Hellinger distance.

By the Cauchy--Schwarz inequality and the property of TV distance and  Hellinger distance, we have 
\begin{align} \label{ineq_Hellinger}
    & \frac{1}{d}\left(\sum_{i=1}^d \mathrm{TV}(M_{+i}^{\pi}, M_{-i}^{\pi})\right)^2 \leq  \sum_{i=1}^d \{\mathrm{TV}(M_{+i}^{\pi}, M_{-i}^{\pi})\}^2 \leq 2\sum_{i=1}^d \mathrm{H}^2(M_{+i}^{\pi}, M_{-i}^{\pi})\nonumber \\ \stackrel{(a)}{\leq}& 2\sum_{i=1}^d \frac{1}{2^{d-1}}\sum_{\bm{v}:v_i=0} \mathrm{H}^2(M_{\bm{v}}^{\pi},M_{\bm v^{\oplus i}}^\pi) = \frac{1}{2^{d-1}}\sum_{\bm{v}\in\{0,1\}^d}\sum_{i=1}^d \mathrm{H}^2(M_{\bm{v}}^{\pi},M_{\bm v^{\oplus i}}^\pi) \nonumber \\
    \stackrel{(b)}{\leq} & \frac{14}{2^d}\sum_{\bm{v}\in\{0,1\}^d}\sum_{t=1}^T \left[ \sum_{i=1}^d \mathrm{H}^2(M_{\bm{v}}^{\pi},M_{\bm{v}^{\oplus i},t}^{\pi})\right] \nonumber \\
    \stackrel{(c)}{\leq} &\frac{14}{2^d}\sum_{\bm{v}\in\{0,1\}^d}\sum_{t=1}^T\mathbb{E}_{\bm v,<t}^{\pi}\left[ \sum_{i=1}^d \mathrm{H}^2(M_{\bm{v}}^{\pi}(\cdot|w_{<t}),M_{\bm{v}^{\oplus i}}^{\pi}(\cdot|w_{<t}))\right],
\end{align}
where the (a) holds by sequentially applying \Cref{lem_convex}, (b) by \Cref{lem_acharya}, and (c) by \Cref{lem_ineq_Hellinger}.

[{\bf Step II}] Bound the summands on the right-hand side of \eqref{ineq_Hellinger} by instant regret.

In Step II, we work under any given $t \in [T]$, $w_{<t}\in \mathcal{W}^{\otimes t-1}$, and any fixed model $\bm {v}$.

Recall \eqref{m_density}, we have that 
$$
m_{\bm{v}}^{\pi}(w|w_{<t}) =\int_{s_t\in\mathcal{S}} q_t(w|s_t,w_{<t}) \mathrm{d} P^{\pi}_{\bm v,t}(s_t|w_{<t}):= \mathbb{E}_{\bm v,t}^{\pi}[q(w|s_t,w_{<t})|w_{<t}],
$$
where the expectation is taken over $s_t$ with respect to $P^{\pi}_{\bm v,t}(s_t|w_{<t})$ given $w_{<t}$.

Hence,  for any $i\in[d]$, we have 
\begin{flalign*}
    &\mathrm{H}^2(M_{\bm v}^{\pi}(\cdot|w_{<t}), M_{\bm{v}^{\oplus i}}^{\pi}(\cdot|w_{<t}))
    \\=& \frac{1}{2} \int_{w\in\mathcal{W}} \left(\sqrt{\mathbb{E}_{\bm{v},t}^{\pi}[q(w|s_t,w_{<t})|w_{<t}]}-\sqrt{\mathbb{E}_{\bm{v}^{\oplus i},t}^{\pi}[q(w|s_t,w_{<t})|w_{<t}]}\right)^2 \mathrm{d}w    
    \\=& \frac{1}{2} \int_{w\in\mathcal{W}} \left(\frac{{\mathbb{E}_{\bm v,t}^{\pi}[q(w|s_t,w_{<t})|w_{<t}]}-{\mathbb{E}_{\bm{v}^{\oplus i},t}^{\pi}[q(w|s_t,w_{<t})|w_{<t}]}}{\sqrt{\mathbb{E}_{\bm{v},t}^{\pi}[q(w|s_t,w_{<t})|w_{<t}]}+\sqrt{\mathbb{E}_{\bm{v}^{\oplus i},t}^{\pi}[q(w|s_t,w_{<t})|w_{<t}]}}\right)^2 \mathrm{d}w
    \\\leq& \frac{1}{2} \int_{w\in\mathcal{W}} \frac{\left({\mathbb{E}_{\bm{v},t}^{\pi}[q(w|s_t,w_{<t})|w_{<t}]}-{\mathbb{E}_{\bm{v}^{\oplus i},t}^{\pi}[q(w|s_t,w_{<t})|w_{<t}]}\right)^2}{\mathbb{E}_{\bm{v},t}^{\pi}[q(w|s_t,w_{<t})|w_{<t}]} \mathrm{d}w.
\end{flalign*}
Define $$\varphi_{\bm v,i}(s_t|w_{<t})=\frac{\mathrm{d} P^{\pi}_{\bm{v}^{\oplus i},t}(s_t|w_{<t})}{\mathrm{d} P^{\pi}_{\bm{v},t}(s_t|w_{<t})}-1.$$

Now we continue the proof in Step II. By the Radon--Nikodym theorem, we have 
\begin{flalign*}
    \mathbb{E}_{\bm{v}^{\oplus i},t}^{\pi}[q(w|s_t,w_{<t}) |w_{<t}]=&\int_{s_t\in\mathcal{S}} q_t(w|s_t,w_{<t}) [\varphi_{\bm{v},i}(s_t|w_{<t})+1]\mathrm{d} P^{\pi}_{\bm v,t}(s_t|w_{<t}).
\end{flalign*}
Hence, given $w_{<t}$, we have 
\begin{flalign}\label{bound_information}
\begin{split}
    &\sum_{i=1}^d \mathrm{H}^2(M_{\bm v}^{\pi}(\cdot|w_{<t}), M_{\bm{v}^{\oplus i}}^{\pi}(\cdot|w_{<t}))\\
    \leq& \frac{1}{2} \sum_{i=1}^d\int_{w\in\mathcal{W}} \frac{\left({\mathbb{E}_{\bm{v},t}^{\pi}[q(w|s_t,w_{<t})\varphi_{\bm{v},i}(s_t|w_{<t})|w_{<t}]}\right)^2}{\mathbb{E}_{\bm{v},t}^{\pi}[q(w|s_t,w_{<t})|w_{<t}]} \mathrm{d}w
    \\
     =& \frac{1}{2}   \int_{w\in\mathcal{W}} \sum_{i=1}^d\frac{\left(\mathbb{E}_{\bm{v},t}^{\pi}\left\{\left(q(w|s_t,w_{<t})-\mathbb{E}_{\bm{v},t}^{\pi}[q(w|s_t,w_{<t})|w_{<t}]\right)\varphi_{\bm{v},i}(s_t|w_{<t})|w_{<t}\right\}\right)^2}{\mathbb{E}_{\bm{v},t}^{\pi}[q(w|s_t,w_{<t})|w_{<t}]} \mathrm{d}w,
     \end{split}
\end{flalign}
where the equality holds by \Cref{lem_varphi}(1).

Note that $z_t$ is a deterministic function of $M_t$.  Without loss of generality, we can write $s_t=(M_t,p_t,y_t)$.
The following analysis works under a fixed $\bm v$, $w$ and $w_{<t}$. To simplify the notation,  we let $s_{t,i}=(\mathbb{I}(M_t=i),p_t,y_t)$, $o(s)=q(w|s,w_{<t})-\mathbb{E}_{\bm{v},t}^{\pi}[q(w|s_t,w_{<t})|w_{<t}]$, and $\varphi_{\bm v,i}(s|w_{<t})=\varphi_{i}(s)$.  The expectation $\mathbb{E}^{\pi}$ is taken with respect to $P_{\bm v,t}^{\pi}$ for $s_t$ conditional on $w_{<t}$.

Therefore, for the numerator in \eqref{bound_information}, we have 
\begin{flalign*}
&\sum_{i=1}^d\left(\mathbb{E}_{\bm{v},t}^{\pi}\left\{\left(q(w|s_t,w_{<t})-\mathbb{E}_{\bm{v},t}^{\pi}[q(w|s_t,w_{<t})|w_{<t}]\right)\varphi_{\bm{v},i}(s_t|w_{<t})|w_{<t}\right\}\right)^2
\\=&\sum_{i=1}^d\left(\mathbb{E}^{\pi}\left\{o(s_t)\varphi_{i}(s_t)\right\}\right)^2
\\\stackrel{(a)}{=}&\sum_{i=1}^d\left(\mathbb{E}^{\pi}\left[\mathbb{E}^{\pi}\left\{o(s_t)\varphi_{i}(s_t)|M_t\right\}\right]\right)^2\\
\stackrel{(b)}{=}&\sum_{i=1}^d\left(\frac{1}{d}\sum_{j=1}^d\mathbb{E}^{\pi}\left\{o(s_t)\varphi_{i}(s_t)|M_t=j\right\}\right)^2\\
=&\sum_{i=1}^d\left(\frac{1}{d}\sum_{j=1}^d\mathbb{E}^{\pi}\left\{o(s_{t,j})\varphi_{i}(s_{t,j})|M_t=j\right\}\right)^2\\
\stackrel{(c)}{=}&\sum_{i=1}^d\frac{1}{d^2}\left(\mathbb{E}^{\pi}\left\{o(s_{t,i})\varphi_{i}(s_{t,i})|M_t=i\right\}\right)^2\\
\stackrel{(d)}{=}&\frac{1}{d^2}\sum_{i=1}^d\left(\mathbb{E}^{\pi}\left\{o(s_{t,i})\left[\sum_{j=1}^d\varphi_{j}(s_{t,i})\right]\Big|M_t=i\right\}\right)^2
\\\stackrel{(e)}{=}&\frac{1}{d}\mathbb{E}^{\pi}\left(\mathbb{E}^{\pi}\left\{o(s_{t})\left[\sum_{j=1}^d\varphi_{j}(s_{t})\right]\Big|M_t\right\}\right)^2
\\\stackrel{(f)}{\leq} & \frac{1}{d}\mathbb{E}^{\pi}\left[\left\{o(s_{t})\left[\sum_{j=1}^d\varphi_{j}(s_{t})\right]\right\}^2\right]
\\\stackrel{(g)}{\leq}& \frac{1}{d}\mathrm{Var}^{\pi}[o(s_{t})] \mathbb{E}^{\pi}\left\{\left[\sum_{j=1}^d\varphi_{j}(s_{t})\right]^2\right\}
\\\stackrel{(h)}{=}&\frac{1}{d}\mathrm{Var}^{\pi}[o(s_{t})] \sum_{j=1}^d\mathbb{E}^{\pi}\{[\varphi_{j}(s_{t})]^2\}
\\\stackrel{(i)}{\leq}& \frac{3}{d}\mathrm{Var}^{\pi}[o(s_{t})]\sum_{j=1}^d\mathbb{E}^{\pi}[(1-p_t)^2\Delta^2\mathbb{I}(M_t=j)]
\\= & \frac{3}{d}\mathrm{Var}^{\pi}[o(s_{t})] \mathbb{E}^{\pi}[(1-p_t)^2\Delta^2].
\end{flalign*}
In above analysis, (a) holds by the law of iterated expectation, (b) holds by noting $\mathbb{P}(M_t=j)=1/d$ for $j\in[d]$, (c) and (d) both hold by noting that $\varphi_{i}(s_{t,j})=0$ if $i\neq j,$ (e) holds by noting $\mathbb{P}(M_t=i)=1/d$, (f) holds by Jensen's inequality for conditional expectation and law of iterated expectation, i.e.~$\mathbb{E}(\mathbb{E}^2\{X|\mathcal{F}\})\leq \mathbb{E}(\mathbb{E}[X^2|\mathcal{F}])\leq \mathbb{E}[X^2]$, (g) holds by Cauchy-Schwarz inequality,  (h) holds by the orthogonality property in  \Cref{lem_varphi}(3),  (i) holds by \Cref{lem_varphi}(2), and the last equality holds by noting $\sum_{i=1}^d \mathbb{I}(M_t=j)=1.$

Plugging the above inequality into \eqref{bound_information}, we thus obtain that 
\begin{flalign}\label{ineq_step2}
\begin{split}
     &\sum_{i=1}^d \mathrm{H}^2(M_{\bm v}^{\pi}(\cdot|w_{<t}), M_{\bm{v}^{\oplus i}}^{\pi}(\cdot|w_{<t}))\\
   \leq&  \frac{3}{2d}   \int_{w\in\mathcal{W}} \mathbb{E}_{\bm v,t}^{\pi}[(1-p_t)^2\Delta^2|w_{<t}]\frac{\mathrm{Var}_{\bm{v},t}^{\pi}[q(w|s_t,w_{<t})|w_{<t}]}{\mathbb{E}_{\bm{v},t}^{\pi}[q(w|s_t,w_{<t})|w_{<t}]} \mathrm{d}w.
\end{split}
\end{flalign}

[{\bf Step III}] Bound the variance in  \eqref{ineq_step2} using LDP constraints. 

Since $q(\cdot|s_t,w_{<t})$ is the density function of the privacy channel given $s_t$ and $w_{<t}$, we must have that for any $w\in\mathcal{W}$, and $s_{t1},s_{t2}\in\mathcal{S}$,
$$
q(w|s_{t1},w_{<t})-q(w|s_{t2},w_{<t})\leq [\exp(2\epsilon)-1]q(w|s_{t2},w_{<t}).
$$
Therefore, for any policy $\pi'$ and model $\bm{v}$,  taking expectation with respect to $\mathbb{P}_{\bm{v},t}^{\pi}$ for $s_{t2}$, we have 
$$
q(w|s_{t1},w_{<t})-\mathbb{E}_{\bm{v},t}^{\pi}[q(w|s_{t},w_{<t})|w_{<t}]\leq [\exp(2\epsilon)-1]\mathbb{E}_{\bm{v},t}^{\pi}[q(w|s_t,w_{<t})|w_{<t}].
$$
Taking square and then taking expectation w.r.t. $\mathbb{P}_{\bm{v},t}^{\pi}$ for  $s_{t1}$, we have that 
$$
\mathrm{Var}_{\bm{v},t}^{\pi}[q(w|s_t,w_{<t})|w_{<t}]\leq [\exp(2\epsilon)-1]^2\{\mathbb{E}_{\bm v,t}^{\pi}[q(w|s_t,w_{<t})|w_{<t}]\}^2.
$$
Therefore, plugging the above inequality into \eqref{ineq_step2}, we have that 
\begin{flalign}\label{ineq_step3}
\begin{split}
    &\sum_{i=1}^d \mathrm{H}^2(M_{\bm v}^{\pi}(\cdot|w_{<t}), M_{\bm{v}^{\oplus i}}^{\pi}(\cdot|w_{<t}))\\\leq &\frac{3[\exp(2\epsilon)-1]^2\Delta^2}{2d} \mathbb{E}_{\bm v,t}^{\pi}[(1-p_t)^2|w_{<t}]  \int_{w\in\mathcal{W}} \mathbb{E}_{\bm{v},t}^{\pi}[q(w|s_t,w_{<t})|w_{<t}]\mathrm{d}w
    \\= & \frac{3[\exp(2\epsilon)-1]^2\Delta^2}{2d} \mathbb{E}_{\bm v,t}^{\pi}[(1-p_t)^2|w_{<t}],
\end{split}
\end{flalign}
since $ \int_{w\in\mathcal{W}} \mathbb{E}_{\bm{v},t}^{\pi}[q(w|s_t,w_{<t})|w_{<t}]\mathrm{d}w=1$.

[{\bf Step IV}] Finish the proof.

Putting \eqref{ineq_Hellinger} and \eqref{ineq_step3} together, we thus obtain that 
\begin{align*}
    \frac{1}{d}\left\{\sum_{i=1}^d \mathrm{TV}(M_{+i}^{\pi}, M_{-i}^{\pi})\right\}^2\leq  &\frac{14}{2^d}\sum_{\bm{v}\in\{0,1\}^d}\sum_{t=1}^T\mathbb{E}_{\bm v,<t}^{\pi}\left[  \sum_{i=1}^d \mathrm{H}^2(M_{\bm{v}}^{\pi}(\cdot|w_{<t}),M_{\bm{v}^{\oplus i}}^{\pi}(\cdot|w_{<t}))\right]\\
    \leq & \frac{21\Delta^2[\exp(2\epsilon)-1]^2}{d}\frac{1}{2^d}\sum_{\bm{v}\in\{0,1\}^d}\sum_{t=1}^T\mathbb{E}_{\bm v,<t}^{\pi}\left[  \mathbb{E}_{\bm v,t}^{\pi}[(1-p_t)^2|w_{<t}]\right].    
\end{align*}
The result follows by \eqref{ineq_assouad2}.
\end{proof}

The following lemma lists the basic properties of $\varphi_{\bm v,i}(s_t|w_{<t})$. 
\begin{lem}\label{lem_varphi}
    Given $w_{<t}$, and  a fixed $\bm v,$ we have 
    \begin{enumerate}
        \item[(1)]$\mathbb{E}^{\pi}_{\bm{v},t}[\varphi_{\bm v,i}(s_t|w_{<t})|w_{<t}] = 0.$
        \item[(2)] $\mathbb{E}^{\pi}_{\bm{v},t}[\varphi^2_{\bm v,i}(s_t|w_{<t})|w_{<t}]\leq 3 \mathbb{E}_{\bm{v},t}^{\pi}[\mathbb{I}(M_t=i)\Delta^2(1-p_t)^2|w_{<t}].$
        \item[(3)] For any $i, j \in [d]$ and $i \neq j$, $\varphi_{\bm v,i}(s_t|w_{<t})\varphi_{\bm v,j}(s_t|w_{<t})=0.$
    \end{enumerate}
\end{lem}
\begin{proof} [\textbf{Proof of \Cref{lem_varphi}}]
    Recall that $p_t\sim A_t^{\pi}(z_t,w_{<t})$, we have that 
\begin{flalign*}
    \frac{\mathrm{d} P^{\pi}_{\bm{v},t}(s_t|w_{<t})}{\mathrm{d}s_t}= f_Z(z_t) a_t^{\pi}(p_t|w_{<t},z_t) f_{\bm{v},Y}(y_t|p_t,z_t,w_{<t}),
\end{flalign*}
where we use the independence of $z_t$ and $w_{<t}$.
Recall \eqref{lb_setting} that, for any $\bm{v}$, $y_t=\lambda_{\bm{v}}(z_t,p_t)+\varepsilon_t$, where $\varepsilon_t\sim\mathcal{N}(0,1)$ is independence of $z_t$ and $p_t$. 
Under $P^{\pi}_{\bm{v},t}(s_t|w_{<t})$, we thus have 
\begin{flalign*}
    \varphi_{\bm{v},i}(s_t|w_{<t})=&\exp\{-[y_t-\lambda_{\bm{v}^{\oplus i}}(z_t,p_t)]^2/2+[y_t-\lambda_{\bm{v}}(z_t,p_t)]^2/2\}-1
    \\=& \exp\left\{
[\lambda_{\bm{v}^{\oplus i}}(p_t,z_t)-\lambda_{\bm{v}}(p_t,z_t)][2\varepsilon_t+\lambda_{\bm v}(p_t,z_t)-\lambda_{\bm{v}^{\oplus i}}(p_t,z_t)]/2\right\}-1\\
=& \exp(-\mathbb{I}(M_t=i)\Delta^2(1-p_t)^2/2)\exp((-1)^{v_i}\mathbb{I}(M_t=i)\varepsilon_t\Delta(1-p_t)))-1.
\end{flalign*}
Recall that $z_t$ is deterministic on $M_t$.  By design of \eqref{lb_setting}, it holds that $\varphi_{\bm{v},i}(s_t|w_{<t})=0$ if $M_t\neq i$.  Therefore, claim (3) follows easily. 

As for the claims (1) and (2), by the law of iterated expectation, and the moment generating function of $\varepsilon_t$, we have that,
\[
    \mathbb{E}_{\bm{v},t}^{\pi}[ \varphi_{\bm{v},i}(s_t|w_{<t})|w_{<t}]= \mathbb{E}_{\bm{v},t}^{\pi}\{\mathbb{E}_{\bm{v},t}^{\pi}[ \varphi_{\bm{v},i}(s_t|w_{<t})|p_t,z_t,w_{<t}]|w_{<t}\}=0.
\]
In addition,  we have 
\begin{flalign*}
\begin{split}
   \mathbb{E}_{\bm{v},t}^{\pi}[ \varphi_{\bm{v},i}^2(s_t|w_{<t})|w_{<t}]=&\mathbb{E}_{\bm{v},t}^{\pi}[\exp(\mathbb{I}(M_t=i)\Delta^2(1-p_t)^2)|w_{<t}]-1\\\leq& 3 \mathbb{E}_{\bm{v},t}^{\pi}[\mathbb{I}(M_t=i)\Delta^2(1-p_t)^2|w_{<t}],
\end{split}
\end{flalign*}
where the  inequality holds since $\exp(x)-1\leq 3x$ for $0<x\leq 1.$ 
\end{proof}

\begin{proof}[\textbf{Proof of \Cref{thm:LDP_lb}}]
The proof is built on the results in Lemmas~\ref{lem_class} and \ref{lem_assouad2}. Recall that $\mathcal{A}^\epsilon=\bigcap_{i=1}^d \mathcal{A}_i^\epsilon$ with $\mathbb{P}(\mathcal{A}^\epsilon)\geq 1-5dT^{-2}$. By Lemma \ref{lem_class}, we have that,
 \begin{flalign*}
    \sup _{\boldsymbol{v}} \mathbb{E}_{\boldsymbol{v}}^{\pi, \psi}\sum_{t=1}^T \mathrm{regret}_t\geq&  \sup _{\boldsymbol{v}} \mathbb{E}_{\boldsymbol{v}}^{\pi, \psi}\sum_{t=1}^T \mathrm{regret}_t \mathbb{I}(\mathcal{A}^\epsilon) \\ \geq& \frac{1}{144}T \Delta^2 d^{-1}  \sup _{\boldsymbol{v}} \mathbb{E}_{\boldsymbol{v}}^{\pi, \psi}\left[\sum_{i=1}^d \mathbb{I}(\hat{v}_i\neq v_i,\mathcal{A}^\epsilon)\right] \\\geq& \frac{1}{144}T \Delta^2 d^{-1} \left\{  \sup _{\boldsymbol{v}}\mathbb{E}_{\boldsymbol{v}}^{\pi, \psi}\left[\sum_{i=1}^d \mathbb{I}(\hat{v}_i\neq v_i)\right]-5d^2T^{-2}\right\},
 \end{flalign*}
where the last inequality holds by noting that $\sum_{i=1}^d \mathbb{I}(\hat{v}_i\neq v_i,\mathcal{A}^\epsilon)=\sum_{i=1}^d \mathbb{I}(\hat{v}_i\neq v_i)-\sum_{i=1}^d \mathbb{I}(\hat{v}_i\neq v_i,(\mathcal{A}^{\epsilon})^c)$, and that $\sum_{i=1}^d \mathbb{I}(\hat{v}_i\neq v_i,(\mathcal{A}^{\epsilon})^c)\leq d \mathbb{I}((\mathcal{A}^{\epsilon})^c)$.

Using Lemma \ref{lem_assouad2},  we further have that
\begin{flalign}\label{eq:regret_case1privacy}
      \sup _{\boldsymbol{v}} \mathbb{E}_{\boldsymbol{v}}^{\pi}\sum_{t=1}^T \mathrm{regret}_t\geq& \frac{T\Delta^2}{144} \left\{\frac{1}{2} \left(1-\sqrt{\frac{21\Delta^2[\exp(2\epsilon)-1]^2}{d^2} \sum_{t=1}^T\mathbb{E}_{\bar{\bm v}}^{\pi}[(1-p_t)^2]}\right)-5dT^{-2}\right\}. 
\end{flalign}
On the other hand, by \eqref{eq:regret_case2}, we have 
\begin{flalign}\label{eq:regret_case2privacy}
 \sup _{\boldsymbol{v}} \mathbb{E}_{\bm{v}}^{\pi}\sum_{t=1}^T \mathrm{regret}_t \geq   \frac{1}{2} \sum_{t=1}^T\mathbb{E}_{\bar{\bm{v}}}^{\pi} [p_t-1]^2 - \frac{3}{2}\Delta^2 T.    
\end{flalign}

Set $\Delta^2=c\epsilon^{-1}\sqrt{d^2/T}$ with  $c=\sqrt{14}/672$ and denote $\eta=\sum_{t=1}^T\mathbb{E}_{\bar{\bm{v}}}^{\pi} [p_t-1]^2$. Suppose $\eta>6\Delta^2T$, we have by \eqref{eq:regret_case2privacy} that 
$$
\sup _{\boldsymbol{v}} \mathbb{E}_{\bm{v}}^{\pi}\sum_{t=1}^T \mathrm{regret}_t\geq (3c)/2\epsilon^{-1}d\sqrt{T}= \frac{\sqrt{14}}{448} \epsilon^{-1}d\sqrt{T}.
$$

Suppose otherwise, i.e.~$\eta\leq 6\Delta^2T$, then by \eqref{eq:regret_case1privacy}, we have 
\begin{flalign*}
   \sup _{\boldsymbol{v}} \mathbb{E}_{\boldsymbol{v}}^{\pi}\sum_{t=1}^T \mathrm{regret}_t\geq&\frac{T\Delta^2}{144}\left\{\frac{1}{2} \left(1-\sqrt{\frac{126(e^{2\epsilon}-1)^2\Delta^4T}{d^2}}\right)-5dT^{-2}\right\}\\=& \frac{c\epsilon^{-1}d\sqrt{T}}{144}\left\{\frac{1}{2} \left(1-\sqrt{126}c\epsilon^{-1}(e^{2\epsilon}-1)\right)-5dT^{-2}\right\}\\
   \geq& \frac{c\epsilon^{-1}d\sqrt{T}}{144}\left\{\frac{1}{2}(1-\frac{1}{2})-\frac{1}{8}\right\} = \frac{\sqrt{14}\epsilon^{-1}d\sqrt{T}}{774144},
\end{flalign*}
where the last inequality holds by noting $\epsilon^{-1}(e^{2\epsilon}-1)\leq 8$ for $\epsilon\in(0,1)$ and $dT^{-2}<1/40$. 

Here note we require  $T\geq [(1-u)^2\vee (1-l)^2]d^2\epsilon^{-2}/(48\times 672)$ to ensure that the condition in \Cref{lem_assouad2}.  
\end{proof}

\bibliographystyle{apalike}
\bibliography{reference}